\documentclass[11pt, a4paper]{article}

\usepackage{booktabs}       
 \usepackage{multirow}
 \usepackage{amsfonts}       

\usepackage{amsmath}
\usepackage{amssymb}
\usepackage{mathtools}
\usepackage{graphicx}
\usepackage{hyperref}
\usepackage{geometry}
\usepackage{algorithm}
\usepackage{algpseudocode}
\usepackage{booktabs}
\usepackage{caption}
\usepackage{subcaption}
\usepackage{array}
\usepackage{physics}
\usepackage{enumitem}
\usepackage[svgnames]{xcolor}

\usepackage{newtxtext}
\usepackage{newtxmath}

\usepackage{amsthm}
\usepackage{thmtools}
\usepackage{thm-restate}

\theoremstyle{plain}

\numberwithin{theorem}{section}
\numberwithin{lemma}{section}

\usepackage{etoolbox} 
\newbool{includeapp}
\setbool{includeapp}{true} 

\usepackage{siunitx}

\hypersetup{
  colorlinks = true,
  allcolors = DarkRed,
}

\definecolor{color1}{HTML}{0077BB}
\definecolor{color2}{HTML}{EE7733}
\definecolor{color3}{HTML}{33BBEE}
\definecolor{color4}{HTML}{EE3377}
\definecolor{color5}{HTML}{CC3311}
\definecolor{color6}{HTML}{009988}
\definecolor{color7}{HTML}{BBBBBB}

\input{feynman_setup}

\usepackage[
    textwidth=3cm, 
    ]{todonotes}

\newcounter{todocounter}

\colorlet{jgcolor}{green!40!white}

\newcommand{\jginline}[2][]{\tikzexternaldisable
  \ifthenelse { \equal {#1} {} }
    { \def\temp {#2} }  
    { \def\temp {#1} }   
  \refstepcounter{todocounter}\todo[color=jgcolor,inline,caption={\textbf{\thetodocounter. JG} \temp}]{\textbf{\thetodocounter. JG:} #2}{}\tikzexternalenable}
\newcommand{\jgblock}[2][{}]{\tikzexternaldisable
  \ifthenelse { \equal {#1} {} }
  { \def\templist {\emph{block comment}}
    \def\tempheader {}}  
  { \def\templist {#1}
    \def\tempheader {#1}}   
  \refstepcounter{todocounter}\todo[color=jgcolor,inline,caption={\textbf{\thetodocounter. JG} \templist}]{\textbf{\thetodocounter. JG: \tempheader}\\\begin{minipage}{\textwidth}#2\end{minipage}}{}\tikzexternalenable}

\colorlet{mgcolor}{blue!20!white}

\newcommand{\mginline}[2][]{\tikzexternaldisable
  \ifthenelse { \equal {#1} {} }
    { \def\temp {#2} }  
    { \def\temp {#1} }   
  \refstepcounter{todocounter}\todo[color=mgcolor,inline,caption={\textbf{\thetodocounter. MG} \temp}]{\textbf{\thetodocounter. MG:} #2}{}\tikzexternalenable}
\newcommand{\mgblock}[2][{}]{\tikzexternaldisable
  \ifthenelse { \equal {#1} {} }
  { \def\templist {\emph{block comment}}
    \def\tempheader {}}  
  { \def\templist {#1}
    \def\tempheader {#1}}   
  \refstepcounter{todocounter}\todo[color=mgcolor,inline,caption={\textbf{\thetodocounter. MG} \templist}]{\textbf{\thetodocounter. MG: \tempheader}\\\begin{minipage}{\textwidth}#2\end{minipage}}{}\tikzexternalenable}

\colorlet{pmcolor}{yellow!40!white}

\newcommand{\pminline}[2][]{\tikzexternaldisable
  \ifthenelse { \equal {#1} {} }
    { \def\temp {#2} }  
    { \def\temp {#1} }   
  \refstepcounter{todocounter}\todo[color=pmcolor,inline,caption={\textbf{\thetodocounter. PM} \temp}]{\textbf{\thetodocounter. PM:} #2}{}\tikzexternalenable}
\newcommand{\pmblock}[2][{}]{\tikzexternaldisable
  \ifthenelse { \equal {#1} {} }
  { \def\templist {\emph{block comment}}
    \def\tempheader {}}  
  { \def\templist {#1}
    \def\tempheader {#1}}   
  \refstepcounter{todocounter}\todo[color=pmcolor,inline,caption={\textbf{\thetodocounter. PM} \templist}]{\textbf{\thetodocounter. PM: \tempheader}\\\begin{minipage}{\textwidth}#2\end{minipage}}{}\tikzexternalenable}

\usepackage[backend=biber,style=numeric,sorting=none]{biblatex}
\DeclareFieldFormat{urldate}{}

\DeclareSourcemap{
  \maps[datatype=bibtex]{
    \map{
      \step[fieldset=primaryclass, null]
    }
  }
}

\newcommand{\EE}{\mathbb{E}}

\makeatletter
\def\@fnsymbol#1{\ensuremath{\ifcase#1\or *\or a\or b\or c\or d\or e\or f\or g \or h \else\@ctrerr\fi}}
\makeatother



\title{Criticality and Saturation in\\ 
Orthogonal Neural Networks}

\author{%
  Max Guillen\\
  Department of Mathematical Sciences\\
  Chalmers University of Technology\\
  University of Gothenburg\\
  SE-412 96 Gothenburg, Sweden \\
  \texttt{maxgui@chalmers.se} \\
  \and
  Jan E.\ Gerken\\
  Department of Mathematical Sciences\\
  Chalmers University of Technology\\
  University of Gothenburg\\
  SE-412 96 Gothenburg, Sweden \\
  \texttt{gerken@chalmers.se}
}

\date{}

\begin{document}

\maketitle

\begin{abstract}
 It has been known for a long time that initializing weight matrices to be orthogonal instead of having i.i.d. Gaussian components can improve training performance. This phenomenon can be analyzed using finite-width corrections, where the infinite-width statistics are supplemented by a power series in $1/\mathrm{width}$. In particular, recent empirical results by Day et al. show that the tensors appearing in this treatment stabilize for large depth, as opposed to the tensors of i.i.d.-initialized networks. In this article, we derive explicit layer-wise recursion relations for the tensors appearing in the finite-width expansion of the network statistics in the case of orthogonal initializations. We also provide an extension of recently-introduced Feynman diagrams for the corresponding recursions in the i.i.d.-case which are valid to all orders in $1/\mathrm{width}$. Finally, we show explicitly that the recursions we derive reproduce the stability of the finite-width tensors which was observed for activation functions with vanishing fixed point. This work therefore provides a theoretical explanation for the stability of nonlinear networks of finite width initialized with orthogonal weights, closing a long-standing gap in the literature. We validate our theoretical results experimentally by showing that numerical solutions of our recursion relations and their analytical large-depth expansions agree excellently with Monte-Carlo estimates from network ensembles.
\end{abstract}



\section{Introduction}
\label{sec:intro}

Neural network training usually starts from randomly initialized parameters  and a good choice of initialization distribution is critical to ensure stable training dynamics to avoid exploding or vanishing gradients. Typically, the initial weight components are sampled i.i.d.\ from Gaussian distributions whose variance scales inversely with the width of the network~\cite{glorot2010, he2015a}. However, it was found that initializing the weight matrices to be orthogonal (i.e. to satisfy $WW^{\top}=\mathbb{I}$, yielding eigenvalues all exactly one) improves performance~\cite{saxe2014, mishkin2016}.

This phenomenon has been studied theoretically for many years. However, these studies are largely restricted to linear networks~\cite{saxe2014, hu2019} or mean-field theory approaches in the large-width limit~\cite{pennington2017, pennington2018a, xiao2018}. Also the neural tangent kernel (NTK) at infinite width~\cite{jacot2018} has been used to analyze orthogonal initializations. However, at infinite width, the NTKs of networks initialized with orthogonal and Gaussian weights agree~\cite{huang2021}, rendering the infinite-width NTK an unsuitable tool to understand the observed differences between the two initializations.

In this article, we instead use $1/\text{width}$ corrections to the neural network statistics at infinite-width in order to investigate the effect of orthogonal initializations for nonlinear networks at finite width. In this framework~\cite{roberts2022}, the network statistics are studied in terms of the cumulants of the preactivations and derivatives. The cumulants are then decomposed into tensors which are expanded in a power series in $1/\text{width}$, with the zero-order contribution corresponding to the infinite-width limit. For orthogonal initializations, the first corrections to the statistics of preactivations were worked out analytically and corrections to the remaining tensors appearing in the first correction to the infinite-width training dynamics were sampled numerically~\cite{day2023}. These results show empirically that the network statistics stabilize at large depth for orthogonal initializations, in contrast to the case of Gaussian initializations. In summary, no theoretical analysis was so far able to show that orthogonal initializations of nonlinear, finite-width neural networks lead to improved performance over Gaussian initializations. In this article, we fill this gap.

In particular, the finite-width corrections to the tensors governing the network statistics obey layer-wise recursion relations. These have been derived only for the cumulant of four preactivations~\cite{day2023} (the $V_{4}$ tensor). Here, we derive the recursion relations for the additional ten fundamental tensors ($D$, $F$, $A$, $B$, $P$, $Q$, $R$, $S$, $T$ and $U$) which govern the first correction to the training dynamics. The orthogonal statistics are captured by Weingarten functions~\cite{weingarten1978} which add additional terms to the recursion relations compared to the Gaussian case. We then show by explicit iteration of our recursions and by expanding their solutions around the large-depth limit that our theoretical results reproduce the empirically observed stability of these tensors. Since these recursion relations are laborious to derive algebraically, we also extend a recently-introduced framework to facilitate these computations using Feynman diagrams~\cite{guillen2025} to the case of orthogonal weights. This extension, which introduces a novel propagator and charge, simplifies the computation of recursion relations dramatically. We prove the completeness of our Feynman diagrams to all orders in $1\text{width}$ and demonstrate their power by computing the correction to order $1/\text{width}^{2}$ of the recursion relation of the cumulant of six preactivations (the $V_{6}$ tensor).

\paragraph{Limitations} Although our analysis is limited to MLPs, there are no conceptual obstacles for extensions to other architectures. Our experiments and the large-depth expansion are restricted to the $\tanh$ activation function, although our theoretical results are general. Finally, our explicit calculations are mostly restricted to the order $1/\text{width}$, although our Feynman diagrams are valid to all orders.

Our main contributions are:
\begin{itemize}
\item We derive for the case of orthogonal initializations the recursion relations for the tensors $D$, $F$, $A$, $B$, $P$, $Q$, $R$, $S$, $T$, $U$ and the first correction to the NTK mean, $\Theta^{\{1\}}$, at order $1/n$, where $n$ is the width of the network. We furthermore derive the recursion relation for the tensor $V_{6}$ at order $1/n^{2}$.
\item We provide Feynman rules that simplify the computation of these recursion relations dramatically and prove that they reproduce the correct algebraic expressions to all orders.
\item We solve our recursion relations by iteration in the single-input case and show that the solution matches empirical results for tanh-MLPs. We analytically compute a large-depth expansion of the solution and show that it supports the saturation observed empirically. We furthermore extend criticality results for the neural network Gaussian process kernel (NNGP) and the NTK from Gaussian to orthogonal initializations.
\end{itemize}


\section{Related work}
Orthogonal initializations of neural network weights have been considered for a long time~\cite{ngiam2010, saxe2014}. An ablation showing performance boosts by orthogonal initializations can be found in Ref.~\cite{mishkin2016}.

Theoretical studies of orthogonal initializations mainly consider linear neural networks~\cite{saxe2014} or use a mean-field theory approach which requires the infinite-width limit. In the latter case, a considerable amount of work has been done on dynamical isometry, which requires orthogonal initializations and supposes that the singular values of the Jacobian of the network are all close to one. With the help of free probability theory, it could be shown that ReLU nonlinearities destroy the dynamical isometry present in deep linear networks with orthogonal initialization, but the sigmoid nonlinearity retains it~\cite{pennington2017}. These results have been extended to further nonlinearities~\cite{pennington2018a} and to convolutional networks~\cite{xiao2018}. All these studies focus on initialization but convergence rates for orthogonally initialized networks were also derived in the linear setting~\cite{hu2019}, confirming the faster convergence in this case.

The NTK for orthogonal initializations at infinite width was found to agree with that of Gaussian initializations~\cite{huang2021}, motivating the study of finite-width corrections. The approach used in this article was introduced in Ref.~\cite{yaida2020}, with further details available in a textbook~\cite{roberts2022}, yielding results about symmetries~\cite{maiti2021}, initialization stability~\cite{banta2024} and scaling laws~\cite{maloney2022, zhang2025} among others. Our Feynman-diagrammatic treatment of the finite-width corrections is based on the rules introduced in~\cite{guillen2025} which include all tensors appearing in the first correction to the training dynamics. Earlier papers using Feynman diagrams to capture preactivation statistics include~\cite{banta2024, halverson2021, grosvenor2022, demirtas2024, maloney2022, zhang2025}.

This article is most closely related to Ref.~\cite{day2023}, where finite-width corrections were first considered in the orthogonal case. We substantially extend the results in this reference, as outlined above.


\section{Orthogonal neural networks}\label{sec:orthogonal_matrices}
We consider an \(L\)-layer multilayer perceptron (MLP)
\(\mathcal{N} : \mathbb{R}^{n_{\mathrm{in}}} \to \mathbb{R}^{n_{\mathrm{out}}}\),
defined recursively by the feed-forward equations
\begin{align}
z^{(\ell)}_{i}(x)
&=
\sum_{j=1}^{n}
W_{ij}^{(\ell)} \,
\sigma\!\left(z^{(\ell-1)}_{j}(x)\right),
\qquad
i = 1, \ldots, n,\label{eq:7}
\end{align}
where \(x \in \mathbb{R}^{n_{\mathrm{in}}}\) denotes the input, indices \(i,j,\ldots\) label neurons, \(\sigma\) is a pointwise activation function, and the weight matrices \(W^{(\ell)} \in \mathbb{R}^{n \times n}\) are sampled independently from the orthogonal group with respect to the Haar measure. For simplicity, we assume that all hidden layers have the same width \(n_\ell = n\), and that biases are omitted.

For a matrix W drawn from the orthogonal Haar measure \(\mathrm{Haar}(\mathrm{O}(n))\) and scaled by $\sqrt{C_{W}}$, the joint moments of its entries are characterized by the orthogonal Weingarten calculus:
\begin{align}\label{weingartenformula}
\mathbb{E}\!\left[
W^{(\ell)}_{i_1 j_1} \cdots W^{(\ell)}_{i_{2k} j_{2k}}
\right]
&=
(C_{W})^{k}\sum_{\pi,\sigma \in \mathcal{P}_2(2k)}
\delta_{i}^{\pi}\,
\delta_{j}^{\sigma}\,
\mathcal{W}\!\left(\pi^{-1}\sigma\right),
\end{align}
where \(C_{W}\) is a layer-independent hyperparameter, \(\mathcal{P}_2(2k)\) denotes the set of pair partitions of \(\{1,\dots,2k\}\), and \(\mathcal{W}\) is the orthogonal Weingarten function~\cite{Collins2006jgn, Collins2009}. As illustrative examples, the second- and fourth-order moments are given by 
\begin{align}
    \mathbb{E}[W^{(\ell)}_{i_{1}j_{1}}W^{(\ell)}_{i_{2}j_{2}}] &= C_{W}(12)_{i}(12)_{j}\mathcal{W}[1]& \label{twopoints}\\
        \mathbb{E}[W^{(\ell)}_{i_{1}j_{1}}W^{(\ell)}_{i_{2}j_{2}}W^{(\ell)}_{i_{3}j_{3}}W^{(\ell)}_{i_{4}j_{4}}] &= 
        (C_{W})^{2} \begin{pmatrix}
        (12)(34)_{i} \\ (13)(24)_{i} \\ (14)(23)_{i}
        \end{pmatrix}^{T}
        \begin{pmatrix}
        \mathcal{W}[1,1] & \mathcal{W}[2] & \mathcal{W}[2]\\
        \mathcal{W}[2] & \mathcal{W}[1,1] & \mathcal{W}[2]\\
        \mathcal{W}[2] & \mathcal{W}[2] & \mathcal{W}[1,1] 
        \end{pmatrix}
        \begin{pmatrix}
        (12)(34)_{j} \\
        (13)(24)_{j} \\
        (14)(23)_{j}
        \end{pmatrix} & \label{fourpoints}
      \end{align}
where we employ the compact notation \((12)(34)\cdots(2k-1 \, 2k)_{i} \equiv 
\delta_{i_1 i_2}\,\delta_{i_3 i_4}\cdots
\delta_{i_{2k-1} i_{2k}}\), and the corresponding orthogonal Weingarten functions are given by
\begin{align}
\mathcal{W}[1] &= \frac{1}{n}, \qquad
\mathcal{W}[1,1] = \frac{n+1}{(n-1)n(n+2)}, \qquad
\mathcal{W}[2] = -\frac{1}{(n-1)n(n+2)} .
\label{firstWfunctions}
\end{align}
Consequently, in contrast to the standard Gaussian setting where Wick\'s theorem implies vanishing higher-order cumulants, the cumulants of the orthogonal weights \(W\), i.e. the connected parts of their joint moments, are nonzero and play a central role in determining the finite-width statistics of the neural network ensemble \(\mathcal{N}\).

As discussed in~\cite{roberts2022}, effective field theory (EFT) techniques enable the
systematic computation of preactivation statistics to all orders in
\(1/n\). Within this framework, the neural network Gaussian process (NNGP) kernel $\widehat{K}^{(\ell+1)}_{i_{1}i_{2}}(x_{1}, x_{2}) = z_{i_{1}}^{(\ell+1)}(x_{1})z_{i_{2}}^{(\ell+1)}(x_{2})$, defines the expectation $K^{(\ell+1)}(x_{1},x_{2}) \equiv \mathbb{E}[\widehat{K}_{ii}^{(\ell+1)}(x_{1},x_{2})]$. At leading order in \(1/n\), this kernel satisfies the recursion
\begin{align}\label{nngprecursion}
K^{(\ell+1)}(x_{1},x_{2}) &= C_{W}\langle\sigma^{(\ell)}_{1}\sigma^{(\ell)}_{2}\rangle_{K^{(\ell)}} + \mathcal{O}\left(\frac{1}{n}\right)
\end{align}
where we use the shorthand \(\sigma^{(\ell)}_{a} \equiv \sigma\!\left(z^{(\ell)}(x_{a})\right)\), and \(\langle \cdot \rangle_{K}\) denotes Gaussian expectation with covariance \(K\).

At next-to-leading order, the fourth-order cumulant of the preactivations takes the form
\begin{align}
\mathbb{E}[z_{i_{1}}^{(\ell+1)}(x_{1})&z_{i_{2}}^{(\ell+1)}(x_{2})z_{i_{3}}^{(\ell+1)}(x_{3})z_{i_{4}}^{(\ell+1)}(x_{4})]\nonumber\\ 
 &= \frac{1}{n}\bigg[(12)(34)_{i}V^{(\ell+1)}_{1234} + (13)(24)_{i}V^{(\ell+1)}_{1324} + (14)(23)_{i}V^{(\ell+1)}_{1423}\bigg]
\end{align}
where we write $V_{1234}^{(\ell+1)} = V^{(\ell+1)}(x_{1},x_{2},x_{3},x_{4})$. The function \(V^{(\ell+1)}_{1234}\) obeys the recursion
\begin{equation}\label{fourpointrecursion}
\begin{split}
\hspace{-10pt}V^{(\ell+1)}_{1234}
\!\!&=\!\!
(C_W)^{2}
\Big[
\langle\sigma^{(\ell)}_{1}\!\sigma^{(\ell)}_{2}\!\sigma^{(\ell)}_{3}\!\sigma^{(\ell)}_{4}\rangle_{\!K^{(\ell)}}{-}\langle \sigma^{(\ell)}_{1}\!\sigma^{(\ell)}_{2}\rangle_{\!K^{(\ell)}} \!\langle \sigma^{(\ell)}_{3}\!\sigma^{(\ell)}_{4} \rangle_{\!K^{(\ell)}} {-} \langle \sigma^{(\ell)}_{1}\!\sigma^{(\ell)}_{3} \rangle_{\!K^{(\ell)}} \!\langle \sigma^{(\ell)}_{2}\!\sigma^{(\ell)}_{4}\rangle_{\!K^{(\ell)}}\\
&\hspace{-22pt} 
{-}
\langle \sigma^{(\ell)}_{1}\sigma^{(\ell)}_{4} \rangle_{\!K^{(\ell)}} \langle \sigma^{(\ell)}_{2}\sigma^{(\ell)}_{3} \rangle_{\!K^{(\ell)}}
\Big]{+}\frac{(C_W)^{2}}{4}\!\!\!\!\!\!\!
\sum_{\substack{\beta_i \in \{1,2,3,4\} \\ i=1,\ldots,4}}\!\!\!\!\!\!\!
V^{(\ell)}_{\beta_1\beta_2\beta_3\beta_4}
\!
\Biggl\langle \!\frac{d^{2}(
\sigma^{(\ell)}_{1}\sigma^{(\ell)}_{2})}{dz^{(\ell)}_{\beta_{1}}d z^{(\ell)}_{\beta_{2}}} \!\Biggr\rangle_{\!\!\!K^{(\ell)}}\!\!\!
\Biggl\langle \!\frac{d^{2}(
\sigma^{(\ell)}_{3}\sigma^{(\ell)}_{4})}{d z^{(\ell)}_{\beta_{3}}d z^{(\ell)}_{\beta_{4}}} \!\Biggr\rangle_{\!\!\!K^{(\ell)}} 
\end{split}
\end{equation}
As noted in~\cite{day2025}, the recursion~\eqref{fourpointrecursion} differs from its Gaussian counterpart by the last two terms in the square brackets. As we show below, this small change has significant consequences for information processing in orthogonal networks.

\subsection{The NTK at finite-width}
\label{subsec:ntk-finite-width}

The layer-\(\ell\) empirical (NTK) associated with the network \(\mathcal{N}\) is defined as
\begin{align}
\widehat{\Theta}^{(\ell)}_{ij}(x_1,x_2)
=
\sum_{\mu,\ell'}
\lambda_{\theta}^{(\ell')}\frac{\partial z^{(\ell)}_{i}(x_1)}{\partial \theta^{(\ell')}_{\mu}}
\,
\frac{\partial z^{(\ell)}_{j}(x_2)}{\partial \theta^{(\ell')}_{\mu}},
\end{align}
where \(\lambda_{\theta}^{(\ell')}\) is a layer-dependent training hyperparameter, \(\theta^{(\ell')}_{\mu}\) denotes the weight parameters at layer \(\ell'\), and the sum runs over all parameters in layers
\(\ell' \le \ell\). The NTK of the full network is given by
\(\widehat{\Theta} = \widehat{\Theta}^{(L)}\). As is well known, this kernel is a nonlinear, initialization-dependent object that evolves during training. In the small learning-rate regime, the NTK provides a leading-order description of the training dynamics. Consequently, understanding its behavior is important for analyzing neural network learning dynamics~\cite{jacot2018}.

To compute the statistics of the NTK, we employ the chain rule to derive the following forward recursion:
\begin{align}
\widehat{\Theta}^{(\ell+1)}_{i_1 i_2}(x_1,x_2)
&=
(12)_{i}
\left(
\lambda_{b}^{(\ell+1)}
+
\frac{\lambda_{W}^{(\ell+1)}}{n}
\sum_{j=1}^{n}
\sigma^{(\ell)}_{j,1}\,
\sigma^{(\ell)}_{j,2}
\right)
\nonumber
\\
&\quad
+
\sum_{j_1,j_2=1}^{n}
W^{(\ell+1)}_{i_1 j_1}\,
W^{(\ell+1)}_{i_2 j_2}\,
\sigma'^{(\ell)}_{j_1,1}\,
\sigma'^{(\ell)}_{j_2,2}\,
\widehat{\Theta}^{(\ell)}_{j_1 j_2}(x_1,x_2),\label{ntkrecursion}
\end{align}
where the standard rescaling \(\lambda_{W}^{(\ell+1)} \mapsto \lambda_{W}^{(\ell+1)}/n\) has been applied to ensure that all network parameters contribute at the same order. As discussed in~\cite{roberts2022}, an appropriate choice of the depth-dependence of the training hyperparameters depends on the activation function under consideration.

The fluctuations of the NTK are defined as
\(
\widehat{\Delta \Theta}^{(\ell)}
\equiv
\widehat{\Theta}^{(\ell)}
-
\mathbb{E}[\widehat{\Theta}^{(\ell)}].
\)
Applying the EFT techniques developed in~\cite{roberts2022} to the recursionc\eqref{ntkrecursion} yields the following leading-order equation for the mean NTK
\begin{align}\label{ntkrecursionleadingorder}
\Theta^{(\ell+1)}_{12}
&=
\lambda_{b}^{(\ell+1)}
+
\lambda_{W}^{(\ell+1)}
\big\langle
\sigma^{(\ell)}_{1}\sigma^{(\ell)}_{2}
\big\rangle_{K^{(\ell)}}
+
C_{W}
\big\langle
\sigma'^{(\ell)}_{1}\sigma'^{(\ell)}_{2}
\big\rangle_{K^{(\ell)}}
\,
\Theta^{(\ell)}_{12}
+
\mathcal{O}\!\left(\frac{1}{n}\right).
\end{align}
As in the Gaussian case, the leading-order statistics of the NTK in \(1/n\) are determined by the cross-cumulant between the preactivations and the NTK fluctuation, together with the NTK variance. The former admits the following compact expression
\begin{align}
  &\EE^{c}_{\theta}[z^{(\ell+1)}_{i_{1}}(x_{1})z^{(\ell+1)}_{i_{2}}(x_{2})\widehat{\Delta\Theta}^{(\ell+1)}_{i_{3}i_{4}}(\textcolor{color1}{x_{3}},\textcolor{color1}{x_{4}})]\nonumber\\
  &\qquad=\frac{1}{n}\!\!\left(\! D^{(\ell+1)}_{12\textcolor{color1}{34}}(12)(34)_{i}{+}F^{(\ell+1)}_{1\textcolor{color1}{3}2\textcolor{color1}{4}}(13)(24)_{i}{+}F^{(\ell+1)}_{1\textcolor{color1}{4}2\textcolor{color1}{3}}(14)(23)_{i} \!\right)\,,\label{eq:3}
\end{align}
where the tensors $D$ and $F$ can be computed directly from the relation \eqref{ntkrecursion}, see Appendix~\ref{app:orthogonal_NTK_tensors} for details. As shown in the Appendix, a careful algebraic analysis yields, for instance, the following leading-order \(1/n\) recursion relations for the tensor $F$
\begin{align}
  F^{(\ell+1)}_{1\textcolor{color1}{3}2\textcolor{color1}{4}} &=(C_{W})^{2}\left[ \langle \sigma^{(\ell)}_{1}\sigma^{(\ell)}_{2}\sigma'^{(\ell)}_{\textcolor{color1}{3}}\sigma'^{(\ell)}_{\textcolor{color1}{4}} \rangle_{K^{(\ell)}} - \langle \sigma^{(\ell)}_{1}\sigma^{(\ell)}_{2}\rangle_{K^{(\ell)}}\langle \sigma'^{(\ell)}_{\textcolor{color1}{3}}\sigma'^{(\ell)}_{\textcolor{color1}{4}} \rangle_{K^{(\ell)}} \right]\Theta^{(\ell)}_{\textcolor{color1}{3}\textcolor{color1}{4}}\nonumber\\
  &\hspace{-0.5cm}+(C_{W})^{2}\!\!\!\sum_{\alpha,\beta,\gamma,\delta=1}^{4}\langle \sigma^{(\ell)}_{1}\sigma'^{(\ell)}_{\textcolor{color1}{3}}z^{(\ell)}_{\alpha} \rangle_{K^{(\ell)}}\langle \sigma_{2}^{(\ell)}\sigma'^{(\ell)}_{\textcolor{color1}{4}}z^{(\ell)}_{\beta} \rangle_{K^{(\ell)}} K^{\alpha\gamma}_{(\ell)}K^{\beta\delta}_{(\ell)}F^{(\ell)}_{\gamma\textcolor{color1}{3}\delta\textcolor{color1}{4}}+\mathcal{O}\left( \frac{1}{n} \right) \label{eq:F}
\end{align}
As we will show, this recursion can be derived efficiently using a diagrammatic framework based on a consistent set of Feynman rules and graph-theoretic principles, yielding a substantial simplification over conventional algebraic approaches. An analogous construction applies to the tensor $D$, as well as to the NTK variance tensors $A$ and $B$ \cite{roberts2021a}, see Appendix~\ref{app:orthogonal_NTK_tensors} for details.


\section{Feynman diagrams}
\label{sec:feynman-diagrams}

A novel diagrammatic framework, inspired by the Feynman-diagram description of interactions among fundamental particles in high-energy physics, was recently introduced for the study of deep neural networks with Gaussian initializations~\cite{guillen2025,banta2024}. In this section, we extend this framework to incorporate the orthogonality properties of the network parameters.

\subsection{Feynman rules}\label{subsec:feynman-rules}

Without loss of generality, consider the channel $(12)_{i}(34)_{i}\ldots (2k\,2k-1)_{i}$, corresponding to the pairing $(12)(34)\ldots (2k\,2k-1)$. The Feynman rules reproducing tensors with this channel structure can be organized into two groups. The first group implements the orthogonality constraints as follows:
\allowdisplaybreaks
\begin{enumerate}
    \item Preactivations and NTKs are denoted by external lines, as illustrated below.
\begin{align}
    z_{\alpha} \equiv
    \begin{tikzpicture}[baseline=-0.1cm]
      \begin{feynman}
        \vertex (l) {};
        \vertex[right = 0pt of l, dot, minimum size=3pt, label = {left: {\footnotesize $\alpha^{}$}}] (x1) {};
        \vertex[right = 30pt of x1, dot, minimum size=0pt] (b) {};
        \diagram*{
          (x1),
          (x1) -- [inner sep = 4pt] (b) 
        };
      \end{feynman}
    \end{tikzpicture}
    &\qquad \widehat{\Delta \Theta}_{\textcolor{color1}{\alpha\beta}} \equiv
    \begin{tikzpicture}[baseline=0.05cm]
      \begin{feynman}
        \vertex (l) {};
        \vertex[right = 0pt of l, color1, dot, minimum size=3pt, label = {left: {\footnotesize $\textcolor{color1}{\alpha^{}}$}}] (x1) {};
        \vertex[above = 8pt of l, color1, dot, minimum size=3pt, label = {left: {\footnotesize $\textcolor{color1}{\beta^{}}$}}] (x2) {};
        \vertex[right = 30pt of x1, dot, minimum size=0pt] (b) {};
        \vertex[right = 30pt of x2, dot, minimum size=0pt] (bb) {};
        \diagram*{
          (x1), (x2),
          (x1) -- [color1, ghost, inner sep = 4pt] (b),
          (x2) -- [color1, ghost, inner sep = 4pt] (bb)
        };
      \end{feynman}
    \end{tikzpicture}\label{eq:6main}
  \end{align}
In~\eqref{eq:6main}, a colored line represents a single NTK label. Distinct colors are used for external dotted lines associated with different NTKs. The colors of the dotted lines correspond to the $\theta$-indices appearing in these definitions. Since the $\theta$-indices are contracted in pairs, distinct colors encode the corresponding contraction pattern. 

\item Define the cubic vertices as
\begingroup
\allowdisplaybreaks
\begin{align}\label{orthogonalityvertex}
 \begin{tikzpicture}[baseline=(b), scale=0.8, transform shape]
      \begin{feynman}
        \vertex (l) {};
        \vertex[below = 15pt of l, dot, minimum size=3pt, label = {left: {\footnotesize $\alpha^{c}$}}] (x1) {};
        \vertex[above = 15pt of l, dot, minimum size=3pt, label = {left: {\footnotesize $\beta^{c}$}}] (x2) {};
        \vertex[right = 10pt of l, dot, minimum size=0pt] (v12) {};
        \vertex[right = 33pt of v12, dot, minimum size=0pt] (b) {};
        \diagram*{
          (x1) --  (v12) --  (x2),
          (v12) -- [photon, edge label = {\scriptsize \;$\sigma_{i,\alpha}^{(\ell)}\sigma_{i,\beta}^{(\ell)}$}, inner sep = 4pt] (b) 
        };
      \end{feynman}
    \end{tikzpicture} &\sim C_{W} & 
    \begin{tikzpicture}[baseline=(b), scale=0.8, transform shape]
      \begin{feynman}
        \vertex (l) {};
        \vertex[below = 15pt of l, color1, dot, minimum size=3pt, label = {left: {\footnotesize $\textcolor{color1}{\alpha^{}}$}}] (x1) {};
        \vertex[above = 15pt of l, color1, dot, minimum size=3pt, label = {left: {\footnotesize $\textcolor{color1}{\beta^{}}$}}] (x2) {};
        \vertex[right = 10pt of l, dot, minimum size=0pt] (v12) {};
        \vertex[right = 30pt of v12, dot, minimum size=0pt] (b) {};
        \diagram*{
          (x1) --  [color1, ghost] (v12) --  [color1, ghost] (x2),
          (v12) -- [black, photon, edge label = {\scriptsize \;$\sigma_{i,\alpha}^{(\ell)}\sigma_{i,\beta}^{(\ell)}$}, inner sep = 4pt] (b) 
        };
      \end{feynman}
    \end{tikzpicture} &\sim 1  &
    \begin{tikzpicture}[baseline=(b), scale=0.8, transform shape]
      \begin{feynman}
        \vertex (l) {};
        \vertex[below = 15pt of l, color1, dot, minimum size=3pt, label = {left: {\footnotesize $\textcolor{color1}{\alpha}^{c}$}}] (x1) {};
        \vertex[above = 15pt of l, color1, dot, minimum size=3pt, label = {left: {\footnotesize $\textcolor{color1}{\beta}^{c}$}}] (x2) {};
        \vertex[right = 10pt of l, dot, minimum size=0pt] (v12) {};
        \vertex[right = 30pt of v12, dot, minimum size=0pt] (b) {};
        \diagram*{
          (x1) --  [color1, ghost] (v12) --  [color1, ghost] (x2),
          (v12) -- [black, photon, edge label = {\scriptsize \;$\Theta^{(\ell)}_{\alpha\beta}\sigma'^{(\ell)}_{i,\alpha}\sigma'^{(\ell)}_{i,\beta}$}, inner sep = 4pt] (b) 
        };
      \end{feynman}
    \end{tikzpicture} &\sim C_{W}  \nonumber\\
    \begin{tikzpicture}[baseline=(b), scale=0.8, transform shape]
      \begin{feynman}
        \vertex (l) {};
        \vertex[below = 15pt of l, color1, dot, minimum size=3pt, label = {left: {\footnotesize $\textcolor{color1}{\alpha}^{c}$}}] (x1) {};
        \vertex[above = 15pt of l, color1, dot, minimum size=3pt, label = {left: {\footnotesize $\textcolor{color1}{\beta}^{c}$}}] (x2) {};
        \vertex[right = 10pt of l, dot, minimum size=0pt] (v12) {};
        \vertex[right = 30pt of v12, dot, minimum size=0pt] (b) {};
        \diagram*{
          (x1) --  [color1, ghost] (v12) --  [color1, ghost] (x2),
          (v12) -- [color1doubghost, edge label = {\scriptsize \;$\sigma'^{(\ell)}_{i,\textcolor{color1}{\alpha}}\sigma'^{(\ell)}_{i,\textcolor{color1}{\beta}}$}, inner sep = 4pt] (b) 
        };
      \end{feynman}
    \end{tikzpicture}&\sim C_{W}  &
    \begin{tikzpicture}[baseline=(b), scale=0.8, transform shape]
      \begin{feynman}
        \vertex (l) {};
        \vertex[below = 15pt of l, dot, minimum size=3pt, label = {left: {\footnotesize $\alpha^{c}$}}] (x1) {};
        \vertex[above = 15pt of l, color1, dot, minimum size=3pt, label = {left: {\footnotesize $\textcolor{color1}{\beta}^{c}$}}] (x2) {};
        \vertex[right = 10pt of l, dot, minimum size=0pt] (v12) {};
        \vertex[right = 33pt of v12, dot, minimum size=0pt] (b) {};
        \diagram*{
          (x1) --  (v12) --  [color1, ghost] (x2),
          (v12) -- [color1blackghost, edge label = {\scriptsize \;$\sigma^{(\ell)}_{i,\alpha}\sigma'^{(\ell)}_{i,\textcolor{color1}{\beta}}$}, inner sep = 4pt] (b) 
        };
      \end{feynman}
    \end{tikzpicture} &\sim C_{W}  
    &\begin{tikzpicture}[baseline=(b), scale=0.8, transform shape]
      \begin{feynman}
        \vertex (l) {};
        \vertex[below = 15pt of l, color1, dot, minimum size=3pt, label = {left: {\footnotesize $\textcolor{color1}{\alpha}^{c}$}}] (x1) {};
        \vertex[above = 15pt of l, color2, dot, minimum size=3pt, label = {left: {\footnotesize $\textcolor{color2}{\beta}^{c}$}}] (x2) {};
        \vertex[right = 10pt of l, dot, minimum size=0pt] (v12) {};
        \vertex[right = 30pt of v12, dot, minimum size=0pt] (b) {};
        \diagram*{
          (x1) --  [color1, ghost] (v12) --  [color2, ghost] (x2),
          (v12) -- [color2color1ghost, edge label = {\scriptsize \;$\sigma'^{(\ell)}_{i,\textcolor{color1}{\alpha}}\sigma'^{(\ell)}_{i,\textcolor{color2}{\beta}}$}, inner sep = 4pt] (b) 
        };
      \end{feynman}
    \end{tikzpicture} &\sim C_{W} 
  \end{align}
  \endgroup
Here the superscript $c$, referred to as the orthogonality charge, keeps track of the label's orthogonal character, and lines that do not end in a dot represent internal lines.

\item Draw a square propagator connecting internal lines in all possible ways, consistent with the chosen pairing. The square represents the full expectation value
\begin{align}
\begin{tikzpicture}[baseline=(b)]
      \begin{feynman}
        \vertex (l) {};
        \vertex[right = 0pt of l, label = {above: {\textcolor{color1}{$ $}}}] (v12) {};
        \vertex[right = 32pt of v12, squareblob] (b) {};
        \vertex[right = 32pt of b, label = {above: {\textcolor{color1}{$ $}}}] (v34) {};
        \vertex[right = 8pt of v34] (r) {};
        \diagram*{
          (v12) -- [photon] (b) -- [photon] (v34)
        };
        \vertex[below = 15pt of b] {{\scriptsize $\EE[\,\cdot\,]$}};
      \end{feynman}
\end{tikzpicture}
\end{align}
This procedure generates distinct diagram types, both connected and disconnected. The connected diagrams are further classified as $s$-class diagrams, defined by the number $s$ of square propagators apppearing in the diagram. The square propagator obeys an additional selection rule: all k-class diagrams formed from pairs of external lines corresponding to different object types vanish.



\item For each $s$-class diagram, generate all inequivalent permutations of its $2m$ external labels carrying orthogonality charge. Multiply each resulting diagram by $\frac{1}{n}$ for every uncharged pairing, and by the appropriate $m$-class Weingarten function $\mathcal{W}$, determined by the relative ordering $\tau$ of the diagram's labels with respect to the original pairing $\pi = (12)(34)\ldots (2k\,2k-1)$: $\mathcal{W}[\tau,\pi]=\mathcal{W}[\ell(\tau\circ \pi)]$, where $\circ$ denotes ordinary permutation multiplication and $\ell(e)$ denotes the cycle length of $e$.

\item Multiply each $s$-class contribution by the M\"obius coefficient $(-1)^{s-1}(s-1)!$, and sum over all classes.

\end{enumerate}
The second group implements the effective field theory techniques developed in~\cite{roberts2022}, applied to the square propagator in the diagrammatic construction of the previous step, through the following set of Feynman rules analogous to those introduced in~\cite{guillen2025,banta2024}:
\begin{enumerate}
    \setcounter{enumi}{5} 
\item We define the bare propagator as
\begin{align}
    \langle\hspace{3mm}\rangle_{K^{(\ell)}} \equiv
    \begin{tikzpicture}[baseline=-0.1cm]
      \tikzfeynmanset{every blob = {/tikz/fill=white!50, /tikz/minimum size=15pt}}
      \begin{feynman}
        \vertex (l) {};
        \vertex[above = 0pt of l, dot, minimum size=0pt] (v12) {};
        \vertex[above = 0pt of v12, blob] (b) {};
        \vertex[above = 0pt of b, dot, minimum size=0pt] (v34) {};
        \diagram*{
          (v12) -- (b) -- (v34), 
        };
      \end{feynman}
    \end{tikzpicture}   
\end{align}
  where $\langle \hspace{3mm}\rangle_{K^{(\ell)}}$ denotes a zero-mean Gaussian expectation with covariance specified by $K^{(\ell)}$. The expectation value is taken over the decorations of the internal lines attached to the propagator, which obeys the same selection rules described in~\cite{guillen2025}. These rules are summarized as follows:
\begin{enumerate}[label=(\alph*)]
\item Propagators may only connect to internal lines emanating from cubic vertices or from the internal quartic vertices introduced below. In particular, propagators cannot be directly connected to other propagators.

\item Dotted lines attached to a propagator do not enter the Gaussian expectation value, as they carry no decorations.

\item Each preactivation line decorated with $z_{i}$ acts as a derivative with respect to $z_{i}$ acting on the argument of the Gaussian expectation value.

\item The neural indices of all internal lines connected to a propagator must be identical.

\item If both dotted and dashed lines of the same color are attached to the propagator, they must appear in pairs carrying the same sample index. The two lines in each pair attach to different vertices. Moreover, if both vertices are drawn in the orientation specified in the Feynman rules, the top-to-bottom ordering of the sample indices (and therefore the colors) of the lines connected to the two vertices must coincide.

\item A pair of dashed lines of the same color connected to the propagator contributes a factor $\Theta_{\alpha\beta}$ when the two lines attach to different vertices, where $\alpha$ and $\beta$ denote the sample indices of the pair.
\end{enumerate}
  \item Quartic vertices are defined analogously, following~\cite{guillen2025}. Explicitly,
  \begin{align}
     \begin{tikzpicture}[baseline=(b), scale=0.8, transform shape]
      \begin{feynman}
        \vertex (l) {};
        \vertex[below = 15pt of l, dot, minimum size=3pt, label = {left: {\footnotesize $\alpha_{1}$}}] (x1) {};
        \vertex[above = 15pt of l, dot, minimum size=3pt, label = {left: {\footnotesize $\alpha_{2}$}}] (x2) {};
        \vertex[right = 20pt of l, quarticblob] (b) {};
        \vertex[below = 25pt of b] {\scriptsize $\frac{1}{n}V_{\alpha_{1}\alpha_{2}\alpha_{3}\alpha_{4}}^{(\ell+1)}$}; 
        \vertex[right = 20pt of b] (r) {};
        \vertex[above = 15pt of r, dot, minimum size=3pt, label = {right: {\footnotesize $\alpha_{3}$}}] (x3) {};
        \vertex[below = 15pt of r, dot, minimum size=3pt, label = {right: {\footnotesize $\alpha_{4}$}}] (x4) {};
        \diagram*{
          (x1) -- (b) -- (x2), 
          (x3) -- (b) -- (x4)
        };
      \end{feynman}
    \end{tikzpicture}
     & &
    \begin{tikzpicture}[baseline=(b), scale=0.8, transform shape]
      \begin{feynman}
        \vertex (l) {};
        \vertex[below = 15pt of l, dot, minimum size=3pt, label = {left: {\footnotesize $\alpha_{1}$}}] (x1) {};
        \vertex[above = 15pt of l, dot, minimum size=3pt, label = {left: {\footnotesize $\alpha_{2}$}}] (x2) {};
        \vertex[right = 20pt of l, quarticblob] (b) {};
        \vertex[below = 25pt of b] {\scriptsize $\frac{1}{n}D_{\alpha_{1}\alpha_{2}\textcolor{color1}{\alpha_{3}\alpha_{4}}}^{(\ell+1)}$}; 
        \vertex[right = 20pt of b] (r) {};
        \vertex[above = 15pt of r, dot, color1, minimum size=3pt, label = {right: {\footnotesize $\textcolor{color1}{\alpha_{3}}$}}] (x3) {};
        \vertex[below = 15pt of r, dot, color1, minimum size=3pt, label = {right: {\footnotesize $\textcolor{color1}{\alpha_{4}}$}}] (x4) {};
        \diagram*{
          (x1) -- (b) -- (x2), 
          (x3) -- [color1, ghost] (b) -- [color1, ghost] (x4)
        };
      \end{feynman}
    \end{tikzpicture}
     & &
    \begin{tikzpicture}[baseline=(b), scale=0.8, transform shape]
      \begin{feynman}
        \vertex (l) {};
        \vertex[below = 15pt of l, dot, minimum size=3pt, label = {left: {\footnotesize $\alpha_{1}$}}] (x1) {};
        \vertex[above = 15pt of l, dot, color1, minimum size=3pt, label = {left: {\footnotesize $\textcolor{color1}{\alpha_{3}}$}}] (x2) {};
        \vertex[right = 20pt of l, quarticblob] (b) {};
        \vertex[below = 25pt of b] {\scriptsize $\frac{1}{n}F_{\alpha_{1}\textcolor{color1}{\alpha_{3}}\alpha_{2}\textcolor{color1}{\alpha_{4}}}^{(\ell+1)}$}; 
        \vertex[right = 20pt of b] (r) {};
        \vertex[above = 15pt of r, dot, minimum size=3pt, label = {right: {\footnotesize $\alpha_{2}$}}] (x3) {};
        \vertex[below = 15pt of r, dot, color1, minimum size=3pt, label = {right: {\footnotesize $\textcolor{color1}{\alpha_{4}}$}}] (x4) {};
        \diagram*{
          (x1) -- (b) -- [color1, ghost] (x2), 
          (x3) -- (b) -- [color1, ghost] (x4)
        };
      \end{feynman}
    \end{tikzpicture}
    & & 
    \begin{tikzpicture}[baseline=(b), scale=0.8, transform shape]
      \begin{feynman}
        \vertex (l) {};
        \vertex[below = 15pt of l, dot, color2, minimum size=3pt, label = {left: {\footnotesize $\textcolor{color2}{\alpha_{1}}$}}] (x1) {};
        \vertex[above = 15pt of l, dot, color2, minimum size=3pt, label = {left: {\footnotesize $\textcolor{color2}{\alpha_{2}}$}}] (x2) {};
        \vertex[right = 20pt of l, quarticblob] (b) {};
        \vertex[below = 25pt of b] {\scriptsize $\frac{1}{n}A_{\textcolor{color2}{\alpha_{1}}\textcolor{color2}{\alpha_{2}}\textcolor{color1}{\alpha_{3}}\textcolor{color1}{\alpha_{4}}}^{(\ell+1)}$}; 
        \vertex[right = 20pt of b] (r) {};
        \vertex[above = 15pt of r, dot, color1, minimum size=3pt, label = {right: {\footnotesize $\textcolor{color1}{\alpha_{3}}$}}] (x3) {};
        \vertex[below = 15pt of r, dot, color1, minimum size=3pt, label = {right: {\footnotesize $\textcolor{color1}{\alpha_{4}}$}}] (x4) {};
        \diagram*{
          (x1) -- [color2, ghost] (b) -- [color2, ghost] (x2), 
          (x3) -- [color1, ghost] (b) -- [color1, ghost] (x4)
        };
      \end{feynman}
    \end{tikzpicture}
     & & 
    \begin{tikzpicture}[baseline=(b), scale=0.8, transform shape]
      \begin{feynman}
        \vertex (l) {};
        \vertex[below = 15pt of l, dot, color2, minimum size=3pt, label = {left: {\footnotesize $\textcolor{color2}{\alpha_{1}}$}}] (x1) {};
        \vertex[above = 15pt of l, dot, color1, minimum size=3pt, label = {left: {\footnotesize $\textcolor{color1}{\alpha_{3}}$}}] (x2) {};
        \vertex[right = 20pt of l, quarticblob] (b) {};
        \vertex[below = 25pt of b] {\scriptsize $\frac{1}{n}B_{\textcolor{color2}{\alpha_{1}}\textcolor{color1}{\alpha_{3}}\textcolor{color2}{\alpha_{2}}\textcolor{color1}{\alpha_{4}}}^{(\ell+1)}$}; 
        \vertex[right = 20pt of b] (r) {};
        \vertex[above = 15pt of r, dot, color2, minimum size=3pt, label = {right: {\footnotesize $\textcolor{color2}{\alpha_{2}}$}}] (x3) {};
        \vertex[below = 15pt of r, dot, color1, minimum size=3pt, label = {right: {\footnotesize $\textcolor{color1}{\alpha_{4}}$}}] (x4) {};
        \diagram*{
          (x1) -- [color2, ghost] (b) -- [color1, ghost] (x2), 
          (x3) -- [color2, ghost] (b) -- [color1, ghost] (x4)
        };
      \end{feynman}
    \end{tikzpicture} \label{feynmanrulesquarticntk}
  \end{align}
\item Higher-order NTK and preactivation tensors are introduced via a natural generalization of the vertices in~\eqref{feynmanrulesquarticntk}.
\item The square propagator decomposes into all connected and disconnected diagrams built from the bare propagator, quartic vertices, and higher-order vertices, with internal lines remaining undotted. This decomposition respects the selection rules (a)-(f).
\end{enumerate}
In Appendix~\ref{app:proofs_three}, we prove that the Feynman rules (1)-(9) are complete to all orders in the 1/n expansion. We further derive a simplified formulation valid at leading order in \(1/n\) in Appendix~\ref{app:simplified_feynman_rules}. As a consistency check, we recover the results of~\cite{banta2023} and~\cite{guillen2025} for preactivations and the NTK with Gaussian weights by restricting to the leading diagonal Weingarten contributions.

These Feynman rules satisfy the following theorem:

\begin{restatable}[]{theorem}{firsttheorem}
\label{theoremone}
The Feynman rules stated in items~1-9, uniquely determine the recursion relations governing the layerwise evolution of the orthogonal NTK tensors $D$, $F$, $A$, $B$ at order \(1/n\).
\end{restatable}
\begin{proof}

We illustrate the proof using the tensor $F$, deriving the recursion relation~\eqref{eq:F} directly from the Feynman rules above. The analysis for the tensors $D$, $A$ and $B$ is analogous and deferred to Appendix~\ref{app:proofs_one}, completing the proof.

The Feynman rules (1)-(5) produce the following diagrams for the tensor $F$:
\begin{align}
  \begin{tikzpicture}[baseline=(b)]
      \begin{feynman}
        \vertex (l) {};
        \vertex[below = 15pt of l, dot, minimum size=3pt, label = {left: {\footnotesize $1$}}] (x1) {};
        \vertex[above = 15pt of l, color1, dot, minimum size=3pt, label = {left: {\footnotesize $\textcolor{color1}{3}$}}] (x2) {};
        \vertex[right = 20pt of l, quarticblob] (b) {};
        \vertex[right = 20pt of b] (r) {};
        \vertex[above = 15pt of r, dot, minimum size=3pt, label = {right: {\footnotesize $2$}}] (x3) {};
        \vertex[below = 15pt of r, color1, dot, minimum size=3pt, label = {right: {\footnotesize $\textcolor{color1}{4}$}}] (x4) {};
        \diagram*{
          (x1) -- (b) -- [color1, ghost] (x2), 
          (x3) -- (b) -- [color1, ghost](x4)
        };
      \end{feynman}
    \end{tikzpicture}
\!\!& = 
\mathcal{W}[1,1]\,\!\!\sum_{j,k}\!\!
\begin{tikzpicture}[baseline=(b)]
\begin{feynman}
\vertex (l) {};
\vertex[below = 15pt of l, dot, minimum size=3pt, label = {left: {\footnotesize $1^{c}$}}] (x1) {};
\vertex[above = 15pt of l, color1, dot, minimum size=3pt, label = {left: {\footnotesize $\textcolor{color1}{3}^{c}$}}] (x2) {};
\vertex[right = 8pt of l, dot, minimum size=0pt] (v12) {};
\vertex[right = 30pt of v12, squareblob] (b) {};
\vertex[right = 30pt of b, dot, minimum size=0pt] (v34) {};
\vertex[right = 8pt of v34] (r) {};
\vertex[above = 15pt of r, dot, minimum size=3pt, label = {right: {\footnotesize $2^{c}$}}] (x3) {};
\vertex[below = 15pt of r, color1, dot, minimum size=3pt, label = {right: {\footnotesize $\textcolor{color1}{4}^{c}$}}] (x4) {};
\diagram*{
	(x1) -- (v12) -- [color1, ghost](x2),
	(v12) -- [color1blackghost, edge label = {\scriptsize \;$\sigma_{j}\sigma'_{j}$}, inner sep = 4pt] (b) -- [blackcolor1ghost, edge label = {\scriptsize \,$\sigma_{k}\sigma'_{k}$}, inner sep = 4pt] (v34), 
	 (x3) -- (v34) -- [color1, ghost](x4)
};
\end{feynman}
\end{tikzpicture}\!\! +
\mathcal{W}[2]\,\!\!\sum_{j,k}\!\!
\begin{tikzpicture}[baseline=(b)]
\begin{feynman}
\vertex (l) {};
\vertex[below = 15pt of l, dot, minimum size=3pt, label = {left: {\footnotesize $1^{c}$}}] (x1) {};
\vertex[above = 15pt of l, dot, minimum size=3pt, label = {left: {\footnotesize $2^{c}$}}] (x2) {};
\vertex[right = 8pt of l, dot, minimum size=0pt] (v12) {};
\vertex[right = 30pt of v12, squareblob] (b) {};
\vertex[right = 35pt of b, dot, minimum size=0pt] (v34) {};
\vertex[right = 8pt of v34] (r) {};
\vertex[above = 15pt of r, color1, dot, minimum size=3pt, label = {right: {\footnotesize $\textcolor{color1}{3}^{c}$}}] (x3) {};
\vertex[below = 15pt of r, color1, dot, minimum size=3pt, label = {right: {\footnotesize $\textcolor{color1}{4}^{c}$}}] (x4) {};
\diagram*{
	(x1) -- (v12) -- (x2),
	(v12) -- [photon, edge label = {\scriptsize \;$\sigma_{j}\sigma_{j}$}, inner sep = 4pt] (b) -- [photon, edge label = {\scriptsize \,$\Theta\, \sigma'_{k}\sigma'_{k}$}, inner sep = 4pt] (v34), 
	 (x3) -- [color1, ghost] (v34) -- [color1, ghost] (x4)
};
\end{feynman}
\end{tikzpicture}\nonumber\\
&\hspace{-35pt} + \mathcal{W}[2]\,\!\!\sum_{j,k}\!\!
\begin{tikzpicture}[baseline=(b)]
\begin{feynman}
\vertex (l) {};
\vertex[below = 15pt of l, dot, minimum size=3pt, label = {left: {\footnotesize $1^{c}$}}] (x1) {};
\vertex[above = 15pt of l, dot, minimum size=3pt, label = {left: {\footnotesize $2^{c}$}}] (x2) {};
\vertex[right = 8pt of l, dot, minimum size=0pt] (v12) {};
\vertex[right = 30pt of v12, squareblob] (b) {};
\vertex[right = 30pt of b, dot, minimum size=0pt] (v34) {};
\vertex[right = 8pt of v34] (r) {};
\vertex[above = 15pt of r, color1, dot, minimum size=3pt, label = {right: {\footnotesize $\textcolor{color1}{3}^{c}$}}] (x3) {};
\vertex[below = 15pt of r, color1, dot, minimum size=3pt, label = {right: {\footnotesize $\textcolor{color1}{4}^{c}$}}] (x4) {};
\diagram*{
	(x1) -- (v12) -- (x2),
	(v12) -- [photon, edge label = {\scriptsize \;$\sigma_{j}\sigma_{j}$}, inner sep = 4pt] (b) -- [color1doubghost, edge label = {\scriptsize \,$\sigma'_{k}\sigma'_{k}$}, inner sep = 4pt] (v34), 
	 (x3) -- [color1, ghost] (v34) -- [color1, ghost] (x4)
};
\end{feynman}
\end{tikzpicture}
\!\!+ \mathcal{W}[2]\,\!\!\sum_{j,k}\!\!
\begin{tikzpicture}[baseline=(b)]
\begin{feynman}
\vertex (l) {};
\vertex[below = 15pt of l, dot, minimum size=3pt, label = {left: {\footnotesize $1^{c}$}}] (x1) {};
\vertex[above = 15pt of l, color1, dot, minimum size=3pt, label = {left: {\footnotesize $\textcolor{color1}{4}^{c}$}}] (x2) {};
\vertex[right = 8pt of l, dot, minimum size=0pt] (v12) {};
\vertex[right = 30pt of v12, squareblob] (b) {};
\vertex[right = 30pt of b, dot, minimum size=0pt] (v34) {};
\vertex[right = 8pt of v34] (r) {};
\vertex[above = 15pt of r, dot, minimum size=3pt, label = {right: {\footnotesize $2^{c}$}}] (x3) {};
\vertex[below = 15pt of r, color1, dot, minimum size=3pt, label = {right: {\footnotesize $\textcolor{color1}{3}^{c}$}}] (x4) {};
\diagram*{
	(x1) -- (v12) -- [color1, ghost] (x2),
	(v12) -- [color1blackghost, edge label = {\scriptsize \;$\sigma_{j}\sigma'_{j}$}, inner sep = 4pt] (b) -- [blackcolor1ghost, edge label = {\scriptsize \,$\sigma_{k}\sigma'_{k}$}, inner sep = 4pt] (v34), 
	 (x3) -- (v34) -- [color1, ghost] (x4)
};
\end{feynman}
\end{tikzpicture}
\label{ftensorsquareprop}
\end{align}
The absence of 2-class diagrams in~\eqref{ftensorsquareprop} follows from the fact that the square propagator vanishes when acting on a pair of lines of different type. The 1-class diagrams are generated from the inequivalent permutations of the reference pairing $(13)(24)$, namely $\{(13)(24),(12)(34),(14)(23)\}$. The corresponding $k=2$ Weingarten functions multiplying the subdiagrams are determined by the relative ordering of the reference pairing $\tau$ and the pairing $\pi$ appearing in the subdiagram: $\mathcal{W}[\pi,\tau]=\mathcal{W}[1,1]$ for $\pi=\tau$, $\mathcal{W}[\pi,\tau]=\mathcal{W}[2]$ for $\pi\neq\tau$. The overall factor of 1 follows from the M\"obius relation for a single-block partition.

At first order in \(1/n\), the Weingarten functions can be approximated by their leading terms in the expansions \eqref{w11expanded} and \eqref{w2expanded}. In this manner, \eqref{ftensorsquareprop} simplifies to
\begin{align}
  \begin{tikzpicture}[baseline=(b)]
      \begin{feynman}
        \vertex (l) {};
        \vertex[below = 15pt of l, dot, minimum size=3pt, label = {left: {\footnotesize $1$}}] (x1) {};
        \vertex[above = 15pt of l, color1, dot, minimum size=3pt, label = {left: {\footnotesize $\textcolor{color1}{3}$}}] (x2) {};
        \vertex[right = 20pt of l, quarticblob] (b) {};
        \vertex[right = 20pt of b] (r) {};
        \vertex[above = 15pt of r, dot, minimum size=3pt, label = {right: {\footnotesize $2$}}] (x3) {};
        \vertex[below = 15pt of r, color1, dot, minimum size=3pt, label = {right: {\footnotesize $\textcolor{color1}{4}$}}] (x4) {};
        \diagram*{
          (x1) -- (b) -- [color1, ghost] (x2), 
          (x3) -- (b) -- [color1, ghost](x4)
        };
      \end{feynman}
    \end{tikzpicture}
\!\!& = 
\frac{1}{n^{2}}\,\!\!\sum_{j,k}\!\!
\begin{tikzpicture}[baseline=(b)]
\begin{feynman}
\vertex (l) {};
\vertex[below = 15pt of l, dot, minimum size=3pt, label = {left: {\footnotesize $1^{c}$}}] (x1) {};
\vertex[above = 15pt of l, color1, dot, minimum size=3pt, label = {left: {\footnotesize $\textcolor{color1}{3}^{c}$}}] (x2) {};
\vertex[right = 8pt of l, dot, minimum size=0pt] (v12) {};
\vertex[right = 30pt of v12, squareblob] (b) {};
\vertex[right = 30pt of b, dot, minimum size=0pt] (v34) {};
\vertex[right = 8pt of v34] (r) {};
\vertex[above = 15pt of r, dot, minimum size=3pt, label = {right: {\footnotesize $2^{c}$}}] (x3) {};
\vertex[below = 15pt of r, color1, dot, minimum size=3pt, label = {right: {\footnotesize $\textcolor{color1}{4}^{c}$}}] (x4) {};
\diagram*{
	(x1) -- (v12) -- [color1, ghost](x2),
	(v12) -- [color1blackghost, edge label = {\scriptsize \;$\sigma_{j}\sigma'_{j}$}, inner sep = 4pt] (b) -- [blackcolor1ghost, edge label = {\scriptsize \,$\sigma_{k}\sigma'_{k}$}, inner sep = 4pt] (v34), 
	 (x3) -- (v34) -- [color1, ghost](x4)
};
\end{feynman}
\end{tikzpicture} -\frac{1}{n^{3}}\sum_{j,k}\!\!
\begin{tikzpicture}[baseline=(b)]
\begin{feynman}
\vertex (l) {};
\vertex[below = 15pt of l, dot, minimum size=3pt, label = {left: {\footnotesize $1^{c}$}}] (x1) {};
\vertex[above = 15pt of l, dot, minimum size=3pt, label = {left: {\footnotesize $2^{c}$}}] (x2) {};
\vertex[right = 8pt of l, dot, minimum size=0pt] (v12) {};
\vertex[right = 30pt of v12, squareblob] (b) {};
\vertex[right = 35pt of b, dot, minimum size=0pt] (v34) {};
\vertex[right = 8pt of v34] (r) {};
\vertex[above = 15pt of r, color1, dot, minimum size=3pt, label = {right: {\footnotesize $\textcolor{color1}{3}^{c}$}}] (x3) {};
\vertex[below = 15pt of r, color1, dot, minimum size=3pt, label = {right: {\footnotesize $\textcolor{color1}{4}^{c}$}}] (x4) {};
\diagram*{
	(x1) -- (v12) -- (x2),
	(v12) -- [photon, edge label = {\scriptsize \;$\sigma_{j}\sigma_{j}$}, inner sep = 4pt] (b) -- [photon, edge label = {\scriptsize \,$\Theta\, \sigma'_{k}\sigma'_{k}$}, inner sep = 4pt] (v34), 
	 (x3) -- [color1, ghost] (v34) -- [color1, ghost] (x4)
};
\end{feynman}
\end{tikzpicture}\nonumber\\
&\hspace{-55pt} -\frac{1}{n^3}\,\!\!\sum_{j,k}\!\!
\begin{tikzpicture}[baseline=(b)]
\begin{feynman}
\vertex (l) {};
\vertex[below = 15pt of l, dot, minimum size=3pt, label = {left: {\footnotesize $1^{c}$}}] (x1) {};
\vertex[above = 15pt of l, dot, minimum size=3pt, label = {left: {\footnotesize $2^{c}$}}] (x2) {};
\vertex[right = 8pt of l, dot, minimum size=0pt] (v12) {};
\vertex[right = 30pt of v12, squareblob] (b) {};
\vertex[right = 30pt of b, dot, minimum size=0pt] (v34) {};
\vertex[right = 8pt of v34] (r) {};
\vertex[above = 15pt of r, color1, dot, minimum size=3pt, label = {right: {\footnotesize $\textcolor{color1}{3}^{c}$}}] (x3) {};
\vertex[below = 15pt of r, color1, dot, minimum size=3pt, label = {right: {\footnotesize $\textcolor{color1}{4}^{c}$}}] (x4) {};
\diagram*{
	(x1) -- (v12) -- (x2),
	(v12) -- [photon, edge label = {\scriptsize \;$\sigma_{j}\sigma_{j}$}, inner sep = 4pt] (b) -- [color1doubghost, edge label = {\scriptsize \,$\sigma'_{k}\sigma'_{k}$}, inner sep = 4pt] (v34), 
	 (x3) -- [color1, ghost] (v34) -- [color1, ghost] (x4)
};
\end{feynman}
\end{tikzpicture}\!\! -\frac{1}{n^{3}}\,\!\!\sum_{j,k}\!\!
\begin{tikzpicture}[baseline=(b)]
\begin{feynman}
\vertex (l) {};
\vertex[below = 15pt of l, dot, minimum size=3pt, label = {left: {\footnotesize $1^{c}$}}] (x1) {};
\vertex[above = 15pt of l, color1, dot, minimum size=3pt, label = {left: {\footnotesize $\textcolor{color1}{4}^{c}$}}] (x2) {};
\vertex[right = 8pt of l, dot, minimum size=0pt] (v12) {};
\vertex[right = 30pt of v12, squareblob] (b) {};
\vertex[right = 30pt of b, dot, minimum size=0pt] (v34) {};
\vertex[right = 8pt of v34] (r) {};
\vertex[above = 15pt of r, dot, minimum size=3pt, label = {right: {\footnotesize $2^{c}$}}] (x3) {};
\vertex[below = 15pt of r, color1, dot, minimum size=3pt, label = {right: {\footnotesize $\textcolor{color1}{3}^{c}$}}] (x4) {};
\diagram*{
	(x1) -- (v12) -- [color1, ghost] (x2),
	(v12) -- [color1blackghost, edge label = {\scriptsize \;$\sigma_{j}\sigma'_{j}$}, inner sep = 4pt] (b) -- [blackcolor1ghost, edge label = {\scriptsize \,$\sigma_{k}\sigma'_{k}$}, inner sep = 4pt] (v34), 
	 (x3) -- (v34) -- [color1, ghost] (x4)
};
\end{feynman}
\end{tikzpicture}\!\! + \mathcal{O}\left(\frac{1}{n^{2}}\right)
\label{ftensorsquarepropexpanded}
\end{align}
We now apply the second set of Feynman rules (6)–(9). As an immediate consequence, the last terms in~\eqref{ftensorsquarepropexpanded} vanish at order \(1/n\). A nonzero contribution would require the corresponding diagrams to produce a factor of $n^{2}$, which occurs when $j\neq k$ and a square propagator decomposes into two bare propagators forming disconnected subdiagrams. However, such configurations are forbidden by the selection rules (a)–(f), as each subdiagram violates color conservation. Consequently, we are left with
\begin{align}
\hspace{-20pt} \begin{tikzpicture}[baseline=(b)]
      \begin{feynman}
        \vertex (l) {};
        \vertex[below = 15pt of l, dot, minimum size=3pt, label = {left: {\footnotesize $1$}}] (x1) {};
        \vertex[above = 15pt of l, color1, dot, minimum size=3pt, label = {left: {\footnotesize $\textcolor{color1}{3}$}}] (x2) {};
        \vertex[right = 20pt of l, quarticblob] (b) {};
        \vertex[right = 20pt of b] (r) {};
        \vertex[above = 15pt of r, dot, minimum size=3pt, label = {right: {\footnotesize $2$}}] (x3) {};
        \vertex[below = 15pt of r, color1, dot, minimum size=3pt, label = {right: {\footnotesize $\textcolor{color1}{4}$}}] (x4) {};
        \diagram*{
          (x1) -- (b) -- [color1, ghost] (x2), 
          (x3) -- (b) -- [color1, ghost](x4)
        };
      \end{feynman}
    \end{tikzpicture}
\!\!\!\!& = \!
\frac{1}{n^{2}}\,\!\!\sum_{j,k}\!\!\!\!
\begin{tikzpicture}[baseline=(b)]
\begin{feynman}
\vertex (l) {};
\vertex[below = 15pt of l, dot, minimum size=3pt, label = {left: {\footnotesize $1^{c}$}}] (x1) {};
\vertex[above = 15pt of l, color1, dot, minimum size=3pt, label = {left: {\footnotesize $\textcolor{color1}{3}^{c}$}}] (x2) {};
\vertex[right = 8pt of l, dot, minimum size=0pt] (v12) {};
\vertex[right = 30pt of v12, squareblob] (b) {};
\vertex[right = 30pt of b, dot, minimum size=0pt] (v34) {};
\vertex[right = 8pt of v34] (r) {};
\vertex[above = 15pt of r, dot, minimum size=3pt, label = {right: {\footnotesize $2^{c}$}}] (x3) {};
\vertex[below = 15pt of r, color1, dot, minimum size=3pt, label = {right: {\footnotesize $\textcolor{color1}{4}^{c}$}}] (x4) {};
\diagram*{
	(x1) -- (v12) -- [color1, ghost](x2),
	(v12) -- [color1blackghost, edge label = {\scriptsize \;$\sigma_{j}\sigma'_{j}$}, inner sep = 4pt] (b) -- [blackcolor1ghost, edge label = {\scriptsize \,$\sigma_{k}\sigma'_{k}$}, inner sep = 4pt] (v34), 
	 (x3) -- (v34) -- [color1, ghost](x4)
};
\end{feynman}
\end{tikzpicture}\!\!\!\!{-}\frac{1}{n^3}\,\!\!\sum_{j,k}\!\!\!\!
\begin{tikzpicture}[baseline=(b)]
\begin{feynman}
\vertex (l) {};
\vertex[below = 15pt of l, dot, minimum size=3pt, label = {left: {\footnotesize $1^{c}$}}] (x1) {};
\vertex[above = 15pt of l, dot, minimum size=3pt, label = {left: {\footnotesize $2^{c}$}}] (x2) {};
\vertex[right = 8pt of l, dot, minimum size=0pt] (v12) {};
\vertex[right = 30pt of v12, squareblob] (b) {};
\vertex[right = 35pt of b, dot, minimum size=0pt] (v34) {};
\vertex[right = 8pt of v34] (r) {};
\vertex[above = 15pt of r, color1, dot, minimum size=3pt, label = {right: {\footnotesize $\textcolor{color1}{3}^{c}$}}] (x3) {};
\vertex[below = 15pt of r, color1, dot, minimum size=3pt, label = {right: {\footnotesize $\textcolor{color1}{4}^{c}$}}] (x4) {};
\diagram*{
	(x1) -- (v12) -- (x2),
	(v12) -- [photon, edge label = {\scriptsize \;$\sigma_{j}\sigma_{j}$}, inner sep = 4pt] (b) -- [photon, edge label = {\scriptsize \,$\Theta\,\sigma'_{k}\sigma'_{k}$}, inner sep = 4pt] (v34), 
	 (x3) -- [color1, ghost] (v34) -- [color1, ghost] (x4)
};
\end{feynman}
\end{tikzpicture}\!\!\!\!\!{+} \mathcal{O}\left(\frac{1}{n^{2}}\right)\label{ftensorsquarepropexpandedremoved}
\end{align}
We now expand the diagrams in~\eqref{ftensorsquarepropexpandedremoved} in terms of the bare propagator and quartic vertices by analyzing the diagonal and off-diagonal neural components of each subdiagram, together with the selection rules (a)-(f). When $i=j$, the first diagram in~\eqref{ftensorsquarepropexpandedremoved} contributes as
\begin{align}
\frac{1}{n^{2}}\sum_{j}
\begin{tikzpicture}[baseline=(b)]
\tikzfeynmanset{every blob = {/tikz/fill=white!50, /tikz/minimum size=15pt}}
\begin{feynman}
\vertex (l) {};
\vertex[below = 15pt of l, dot, minimum size=3pt, label = {left: {\footnotesize $1$}}] (x1) {};
\vertex[above = 15pt of l, color1, dot, minimum size=3pt, label = {left: {\footnotesize $\textcolor{color1}{3}$}}] (x2) {};
\vertex[right = 8pt of l, dot, minimum size=0pt] (v12) {};
\vertex[right = 30pt of v12, blob] (b) {};
\vertex[right = 30pt of b, dot, minimum size=0pt] (v34) {};
\vertex[right = 8pt of v34] (r) {};
\vertex[above = 15pt of r, dot, minimum size=3pt, label = {right: {\footnotesize $2$}}] (x3) {};
\vertex[below = 15pt of r, color1, dot, minimum size=3pt, label = {right: {\footnotesize $\textcolor{color1}{4}$}}] (x4) {};
\diagram*{
	(x1) -- (v12) -- [color1, ghost] (x2),
	(v12) -- [color1blackghost, edge label = {\scriptsize \;$\sigma_{j}\sigma'_{j}$}, inner sep = 4pt] (b) -- [blackcolor1ghost, edge label = {\scriptsize \,$\sigma_{j}\sigma'_{j}$}, inner sep = 4pt] (v34), 
	 (x3) -- (v34) -- [color1, ghost] (x4)
};
\end{feynman}
\end{tikzpicture}\label{ftensorbarepropijsame}
\end{align}
When $i\neq j$, the first and second diagrams reduce respectively to
\begin{align}
\hspace{-10pt}\frac{1}{n^{2}}\sum_{j_1, j_2}\!\!\!
\begin{tikzpicture}[baseline=(b)]
  \tikzfeynmanset{every blob = {/tikz/fill=white!50, /tikz/minimum size=15pt}}
  \begin{feynman}
    \vertex (l) {};
    \vertex[below = 15pt of l, dot, minimum size=3pt, label = {left: {\footnotesize $1$}}] (x1) {};
    \vertex[above = 15pt of l, color1, dot, minimum size=3pt, label = {left: {\footnotesize $\textcolor{color1}{3}$}}] (x2) {};
    \vertex[right = 8pt of l, dot, minimum size=0pt] (v12) {};
    \vertex[right = 35pt of v12, blob] (b12) {};
    \vertex[right = 16pt of b12, dot, minimum size=0pt] (w12) {};
    \vertex[above = 1pt of w12, label = {above: {\scriptsize \hspace{-10pt} $z_{j_1}$}}] (w12u) {};
    \vertex[below = 8pt of w12] (w12d) {};
    \vertex[left = 10pt of w12d] (w12dl) {};
    \tikzfeynmanset{every blob = {/tikz/fill=gray!50, /tikz/minimum size=15pt}}
    \vertex[right = 3pt of w12, blob , minimum size = 6pt] (b) {};
    \vertex[right = 3pt of b, dot, minimum size=0pt] (w34) {};
    \tikzfeynmanset{every blob = {/tikz/fill=white!50, /tikz/minimum size=15pt}}
    \vertex[right = 16pt of w34, blob] (b34) {};
    \vertex[above = 1pt of w34, label = {above: {\scriptsize \hspace{10pt} $z_{j_2}$}}] (w34u) {};
    \vertex[below = 8pt of w34] (w34d) {};
    \vertex[right = 10pt of w34d] (w34dr) {};
    \vertex[right = 35pt of b34, dot, minimum size=0pt] (v34) {};
    \vertex[right = 8pt of v34] (r) {};
    \vertex[above = 15pt of r, dot, minimum size=3pt, label = {right: {\footnotesize $2$}}] (x3) {};
    \vertex[below = 15pt of r, color1, dot, minimum size=3pt, label = {right: {\footnotesize $\textcolor{color1}{4}$}}] (x4) {};
    \diagram*{
      (x1) -- (v12) -- [color1, ghost] (x2),
      (v12) -- [color1blackghost, edge label = {\scriptsize \;$\sigma_{j_{1}}\sigma'_{j_{1}}$}, inner sep = 4pt] (b12) -- [color1, ghost, quarter left] (w12) -- [quarter left] (b12),
      (b34) -- [color1, ghost, quarter left] (w34) -- [quarter left] (b34) -- [blackcolor1ghost, edge label = {\scriptsize \,$\sigma_{j_{2}}\sigma'_{j_{2}}$}, inner sep = 4pt] (v34),
      (x3) --  (v34) -- [color1, ghost] (x4)
    };
    \draw [decoration={brace}, decorate] (w34dr) -- (w12dl) node [pos=0.5, below = 1pt] {\scriptsize $\frac{1}{n}F_4^{(\ell)}$};
  \end{feynman}
\end{tikzpicture}\hspace{2mm}\!,\!\hspace{2mm}
-\frac{1}{n^3}
\sum_{j_1, j_2}\!\!\!
\begin{tikzpicture}[baseline=(b)]
  \tikzfeynmanset{every blob = {/tikz/fill=white!50, /tikz/minimum size=15pt}}
  \begin{feynman}
    \vertex (l) {};
    \vertex[below = 15pt of l, dot, minimum size=3pt, label = {left: {\footnotesize $1$}}] (x1) {};
    \vertex[above = 15pt of l, , dot, minimum size=3pt, label = {left: {\footnotesize $2$}}] (x2) {};
    \vertex[right = 8pt of l, dot, minimum size=0pt] (v12) {};
    \vertex[right = 35pt of v12, blob] (b12) {};
    \vertex[right = 5pt of b12, dot, minimum size=0pt] (w12) {};
    \vertex[above = 1pt of w12, label = {above: {\scriptsize \hspace{-10pt} $ $}}] (w12u) {};
    \vertex[below = 8pt of w12] (w12d) {};
    \vertex[left = 10pt of w12d] (w12dl) {};
    \tikzfeynmanset{every blob = {/tikz/fill=gray!50, /tikz/minimum size=0pt}}
    \vertex[right = 3pt of w12, blob , minimum size = 0pt] (b) {};
    \vertex[right = 3pt of b, dot, minimum size=0pt] (w34) {};
    \tikzfeynmanset{every blob = {/tikz/fill=white!50, /tikz/minimum size=15pt}}
    \vertex[right = 16pt of w34, blob] (b34) {};
    \vertex[above = 1pt of w34, label = {above: {\scriptsize \hspace{10pt} $ $}}] (w34u) {};
    \vertex[below = 8pt of w34] (w34d) {};
    \vertex[right = 10pt of w34d] (w34dr) {};
    \vertex[right = 35pt of b34, dot, minimum size=0pt] (v34) {};
    \vertex[right = 8pt of v34] (r) {};
    \vertex[above = 15pt of r, color1, dot, minimum size=3pt, label = {right: {\footnotesize $\textcolor{color1}{3}$}}] (x3) {};
    \vertex[below = 15pt of r, color1, dot, minimum size=3pt, label = {right: {\footnotesize $\textcolor{color1}{4}$}}] (x4) {};
    \diagram*{
      (x1) -- (v12) -- (x2),
      (v12) -- [photon, edge label = {\scriptsize \;$\sigma_{j_{1}}\sigma_{j_{1}}$}, inner sep = 4pt] (b12),
      (b34) -- [photon, edge label = {\scriptsize \,$\Theta\,\sigma'_{j_{2}}\sigma'_{j_{2}}$}, inner sep = 4pt] (v34),
      (x3) --  [color1, ghost] (v34) -- [color1, ghost] (x4)
    };
  \end{feynman}
\end{tikzpicture}\label{ftensorbarepropijdifftwo}
\end{align}
Substituting~\eqref{ftensorbarepropijsame} and~\eqref{ftensorbarepropijdifftwo} into~\eqref{ftensorsquarepropexpandedremoved}, we obtain
\begin{align}
  \begin{tikzpicture}[baseline=(b)]
      \begin{feynman}
        \vertex (l) {};
        \vertex[below = 15pt of l, dot, minimum size=3pt, label = {left: {\footnotesize $1$}}] (x1) {};
        \vertex[above = 15pt of l, dot, color1, minimum size=3pt, label = {left: {\footnotesize $\textcolor{color1}{3}$}}] (x2) {};
        \vertex[right = 20pt of l, quarticblob] (b) {};
        \vertex[right = 20pt of b] (r) {};
        \vertex[above = 15pt of r, dot, minimum size=3pt, label = {right: {\footnotesize $2$}}] (x3) {};
        \vertex[below = 15pt of r, dot, color1, minimum size=3pt, label = {right: {\footnotesize $\textcolor{color1}{4}$}}] (x4) {};
        \diagram*{
          (x1) -- (b) -- [color1, ghost] (x2), 
          (x3) -- (b) -- [color1, ghost] (x4)
        };
      \end{feynman}
    \end{tikzpicture}
\!\!&= 
\frac{1}{n^{2}}\sum_{j}
\begin{tikzpicture}[baseline=(b)]
\tikzfeynmanset{every blob = {/tikz/fill=white!50, /tikz/minimum size=15pt}}
\begin{feynman}
\vertex (l) {};
\vertex[below = 15pt of l, dot, minimum size=3pt, label = {left: {\footnotesize $1$}}] (x1) {};
\vertex[above = 15pt of l, color1, dot, minimum size=3pt, label = {left: {\footnotesize $\textcolor{color1}{3}$}}] (x2) {};
\vertex[right = 8pt of l, dot, minimum size=0pt] (v12) {};
\vertex[right = 30pt of v12, blob] (b) {};
\vertex[right = 30pt of b, dot, minimum size=0pt] (v34) {};
\vertex[right = 8pt of v34] (r) {};
\vertex[above = 15pt of r, dot, minimum size=3pt, label = {right: {\footnotesize $2$}}] (x3) {};
\vertex[below = 15pt of r, color1, dot, minimum size=3pt, label = {right: {\footnotesize $\textcolor{color1}{4}$}}] (x4) {};
\diagram*{
	(x1) -- (v12) -- [color1, ghost] (x2),
	(v12) -- [color1blackghost, edge label = {\scriptsize \;$\sigma_{j}\sigma'_{j}$}, inner sep = 4pt] (b) -- [blackcolor1ghost, edge label = {\scriptsize \,$\sigma_{j}\sigma'_{j}$}, inner sep = 4pt] (v34), 
	 (x3) -- (v34) -- [color1, ghost] (x4)
};
\end{feynman}
\end{tikzpicture}
+ \frac{1}{n^{2}}\sum_{j_1, j_2}\!
\begin{tikzpicture}[baseline=(b)]
  \tikzfeynmanset{every blob = {/tikz/fill=white!50, /tikz/minimum size=15pt}}
  \begin{feynman}
    \vertex (l) {};
    \vertex[below = 15pt of l, dot, minimum size=3pt, label = {left: {\footnotesize $1$}}] (x1) {};
    \vertex[above = 15pt of l, color1, dot, minimum size=3pt, label = {left: {\footnotesize $\textcolor{color1}{3}$}}] (x2) {};
    \vertex[right = 8pt of l, dot, minimum size=0pt] (v12) {};
    \vertex[right = 35pt of v12, blob] (b12) {};
    \vertex[right = 16pt of b12, dot, minimum size=0pt] (w12) {};
    \vertex[above = 1pt of w12, label = {above: {\scriptsize \hspace{-10pt} $z_{j_1}$}}] (w12u) {};
    \vertex[below = 8pt of w12] (w12d) {};
    \vertex[left = 10pt of w12d] (w12dl) {};
    \tikzfeynmanset{every blob = {/tikz/fill=gray!50, /tikz/minimum size=15pt}}
    \vertex[right = 3pt of w12, blob , minimum size = 6pt] (b) {};
    \vertex[right = 3pt of b, dot, minimum size=0pt] (w34) {};
    \tikzfeynmanset{every blob = {/tikz/fill=white!50, /tikz/minimum size=15pt}}
    \vertex[right = 16pt of w34, blob] (b34) {};
    \vertex[above = 1pt of w34, label = {above: {\scriptsize \hspace{10pt} $z_{j_2}$}}] (w34u) {};
    \vertex[below = 8pt of w34] (w34d) {};
    \vertex[right = 10pt of w34d] (w34dr) {};
    \vertex[right = 35pt of b34, dot, minimum size=0pt] (v34) {};
    \vertex[right = 8pt of v34] (r) {};
    \vertex[above = 15pt of r, dot, minimum size=3pt, label = {right: {\footnotesize $2$}}] (x3) {};
    \vertex[below = 15pt of r, color1, dot, minimum size=3pt, label = {right: {\footnotesize $\textcolor{color1}{4}$}}] (x4) {};
    \diagram*{
      (x1) -- (v12) -- [color1, ghost] (x2),
      (v12) -- [color1blackghost, edge label = {\scriptsize \;$\sigma_{j_{1}}\sigma'_{j_{1}}$}, inner sep = 4pt] (b12) -- [color1, ghost, quarter left] (w12) -- [quarter left] (b12),
      (b34) -- [color1, ghost, quarter left] (w34) -- [quarter left] (b34) -- [blackcolor1ghost, edge label = {\scriptsize \,$\sigma_{j_{2}}\sigma'_{j_{2}}$}, inner sep = 4pt] (v34),
      (x3) --  (v34) -- [color1, ghost] (x4)
    };
    \draw [decoration={brace}, decorate] (w34dr) -- (w12dl) node [pos=0.5, below = 1pt] {\scriptsize $\frac{1}{n}F_4^{(\ell)}$};
  \end{feynman}
\end{tikzpicture}\nonumber\\
&- 
\frac{1}{n^{3}}
\sum_{j_1, j_2}\!
\begin{tikzpicture}[baseline=(b)]
  \tikzfeynmanset{every blob = {/tikz/fill=white!50, /tikz/minimum size=15pt}}
  \begin{feynman}
    \vertex (l) {};
    \vertex[below = 15pt of l, dot, minimum size=3pt, label = {left: {\footnotesize $1$}}] (x1) {};
    \vertex[above = 15pt of l, , dot, minimum size=3pt, label = {left: {\footnotesize $2$}}] (x2) {};
    \vertex[right = 8pt of l, dot, minimum size=0pt] (v12) {};
    \vertex[right = 35pt of v12, blob] (b12) {};
    \vertex[right = 5pt of b12, dot, minimum size=0pt] (w12) {};
    \vertex[above = 1pt of w12, label = {above: {\scriptsize \hspace{-10pt} $ $}}] (w12u) {};
    \vertex[below = 8pt of w12] (w12d) {};
    \vertex[left = 10pt of w12d] (w12dl) {};
    \tikzfeynmanset{every blob = {/tikz/fill=gray!50, /tikz/minimum size=0pt}}
    \vertex[right = 3pt of w12, blob , minimum size = 0pt] (b) {};
    \vertex[right = 3pt of b, dot, minimum size=0pt] (w34) {};
    \tikzfeynmanset{every blob = {/tikz/fill=white!50, /tikz/minimum size=15pt}}
    \vertex[right = 16pt of w34, blob] (b34) {};
    \vertex[above = 1pt of w34, label = {above: {\scriptsize \hspace{10pt} $ $}}] (w34u) {};
    \vertex[below = 8pt of w34] (w34d) {};
    \vertex[right = 10pt of w34d] (w34dr) {};
    \vertex[right = 35pt of b34, dot, minimum size=0pt] (v34) {};
    \vertex[right = 8pt of v34] (r) {};
    \vertex[above = 15pt of r, color1, dot, minimum size=3pt, label = {right: {\footnotesize $\textcolor{color1}{3}$}}] (x3) {};
    \vertex[below = 15pt of r, color1, dot, minimum size=3pt, label = {right: {\footnotesize $\textcolor{color1}{4}$}}] (x4) {};
    \diagram*{
      (x1) -- (v12) -- (x2),
      (v12) -- [photon, edge label = {\scriptsize \;$\sigma_{j_{1}}\sigma_{j_{1}}$}, inner sep = 4pt] (b12),
      (b34) -- [photon, edge label = {\scriptsize \,$\Theta\,\sigma'_{j_{2}}\sigma'_{j_{2}}$}, inner sep = 4pt] (v34),
      (x3) --  [color1, ghost] (v34) -- [color1, ghost] (x4)
    };
  \end{feynman}
\end{tikzpicture} + \mathcal{O}\left(\frac{1}{n^2}\right)\label{ftensoratorderonen}
\end{align}
The \(1/n\) scaling of the diagrams follows from straightforward power counting of the respective summations: When $i=j$, the summation contributes factor of $n$, whereas when $i\neq j$ it yields a factor of $n^2$. Combined with the explicit coefficients on the right-hand side of~\eqref{ftensoratorderonen}, these factors produce an overall contribution of order \(1/n\). It is straightforward to verify that the first and second lines on the RHS of~\eqref{ftensoratorderonen} reproduce the corresponding lines on the RHS of~\eqref{eq:F}.
\end{proof}
In the same manner, the full Feynman rules can be used to reproduce the recursion relations governing higher-derivative versions of the NTK, as detailed below.
\begin{restatable}[]{theorem}{secondtheorem}
\label{theoremtwo}
The set of Feynman rules presented above, in conjunction with those defined in Appendix~\ref{app:general_feynman_rules}, reproduces the recursion relations governing the orthogonal dNTK and ddNTK tensors: $P$, $Q$ and $R$, $S$, $T$, $U$, respectively, at order \(1/n\).
\end{restatable}
\begin{proof} 
See Appendix~\ref{app:proofs_two}.
\end{proof} 
The preceding results establish that Feynman diagrams are complete at order \(1/n\). This construction extends straightforwardly to all orders in \(1/n\) by allowing the vertices in~\ref{feynmanrulesquarticntk} to admit an arbitrary number of external legs, without requiring any further modifications. The following theorem formalizes this extension.
\begin{restatable}[]{theorem}{thirdtheorem}
\label{theoremthree}
The Feynman rules defined above, augmented by the higher-order generalizations of the tensors $V$, $D$, $F$, $A$, $B$, $P$, $Q$, $R$, $S$, $T$, $U$, provide a complete characterization of the statistics of the orthogonal NTK and its descendants at arbitrary order in \(1/n\).
\end{restatable}
\begin{proof} 
See Appendix~\ref{app:proofs_three}.
\end{proof} 
This result demonstrates that our diagrammatic formalism enables systematic computation of preactivation and NTK statistics to arbitrary order in \(1/n\) from a concise set of rules. As a concrete example, we derive in Appendix~\ref{app:v6_tensor} the recursion relation for the $V_{6}$ tensor at order $1/n^2$, in contrast to the cumbersome algebraic manipulations demanded by direct methods.

\section{Applications}
\label{sec:applications-orthogonal}

In this section, we deduce a set of results for orthogonal neural networks from the Feynman-diagrammatic framework developed above.

\subsection{Criticality and stability at finite width}
Deep neural network outputs are highly sensitive to initialization: naive architectural choices can cause preactivations to explode or vanish exponentially. At infinite width, this stability is characterized by the susceptibility $\chi$, which controls the growth of both the NNGP and the NTK. This analysis extends to finite-width Gaussian-initialized networks, where criticality at infinite width enforces criticality at finite width for both preactivation and NTK statistics~\cite{guillen2025,banta2024}.

Next, we combine the ideas developed in~\cite{guillen2025} with the Feynman rules of Section~\ref{sec:feynman-diagrams} to show that
\begin{restatable}[]{theorem}{fourththeorem}
\label{theoremfour}
Criticality of the infinite-width NNGP and NTK implies criticality of orthogonal preactivation and NTK cumulants.
\end{restatable}
\begin{proof} 
The proof is twofold. First, as established in~\cite{guillen2025}, the stability analysis admits a bootstrap structure: criticality of lower-rank tensors suffices to control higher-rank tensors, whose behavior can then be analyzed systematically using the Feynman rules of Section~\ref{sec:feynman-diagrams} or Appendix~\ref{app:general_feynman_rules}. Second, at any fixed order in \(1/n\), orthogonality modifies the recursion of a given tensor only through the introduction of lower-rank tensors, as prescribed by Rule~4 in Section~4.1. This rule introduces Weingarten functions multypling diagrams generated by inequivalent permutations relative to the reference pairing. By definition, these contributions are subleading compared to the identity permutation. Since these lower-rank tensors are already assumed to be critical under the bootstrap hypothesis, the resulting stability analysis for orthogonal tensors coincides with that of the Gaussian case.
\end{proof} 

In Figure~\ref{fig:gradient_stability_plot}, we empirically verify this by sampling the NTK components of a $\tanh$ orthogonal MLP ($n=50$, $L=30$) with inputs drawn i.i.d. from the interval $(0,1)$ across varying $C_{W}$ (see Appendix~\ref{app:experimental_setup}). This illustrates the effectiveness of the diagrammatic framework for deriving results to all orders in \(1/n\).
\begin{figure*}
  \centering
  \includegraphics[width=0.8\textwidth,keepaspectratio]{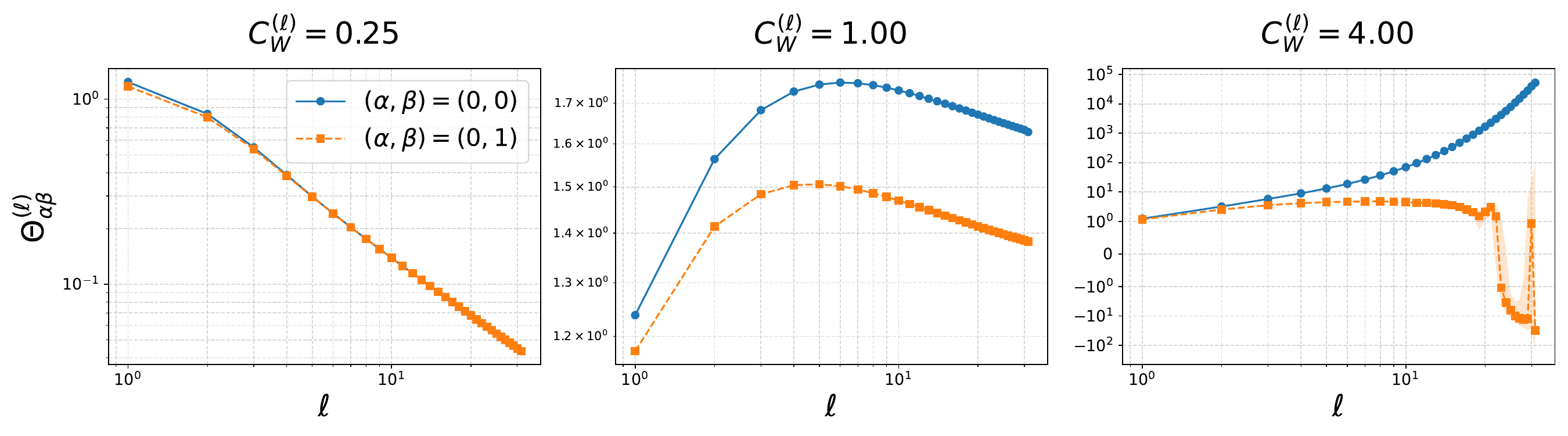}
  \vspace{-0.1cm}
  \caption{\emph{Gradient stability.} Monte Carlo estimates of the NTK \(\Theta_{\alpha\beta}^{(\ell)}\) for a \(\tanh\) orthogonal network (\(n=50\), \(L=30\)) across varying \(C_W\), with the critical case \(C_W=1\) shown in the center. Means are computed over 600 initializations; shaded regions denote standard errors (typically not visible). Results are consistent with Theorem~\ref{theoremfour}.}
  \label{fig:gradient_stability_plot}
\end{figure*}

\subsection{Single-input solutions}\label{subsec:single-input}
Once derived within the Feynman diagrammatic framework, the recursion relations can be solved for given initial conditions. We focus on the single-input setting with \(\tanh\) activation and solve the resulting recursions for the NTK tensors (see Appendix~\ref{app:analytical_recursion_relations} for an explicit derivation). We consider networks of width \(n=50\) and depth \(L=10\), initialized at criticality \(C_W=1\). Gaussian expectations are evaluated numerically in \textit{Mathematica}. We analyze normalized cumulants, defined as ratios to their infinite-width counterparts; for example, \(\tilde{D}^{(\ell)} = D^{(\ell)}/(K^{(\ell)}\Theta^{(\ell)})\) and \(\tilde{F}^{(\ell)} = F^{(\ell)}/(K^{(\ell)}\Theta^{(\ell)})\). The solutions are shown in Figure~\ref{fig:combined_tensors_plot} (a), where we observe excellent agreement with finite-width simulations, supporting the validity of our framework. See Appendix~\ref{app:solut-single-input} for a complete analysis, including NTK derivatives. Our results are consistent with the empirical findings of~\cite{day2023}.
 \begin{figure*}[tb]
  \centering
  \includegraphics[height=0.2\textheight,keepaspectratio]{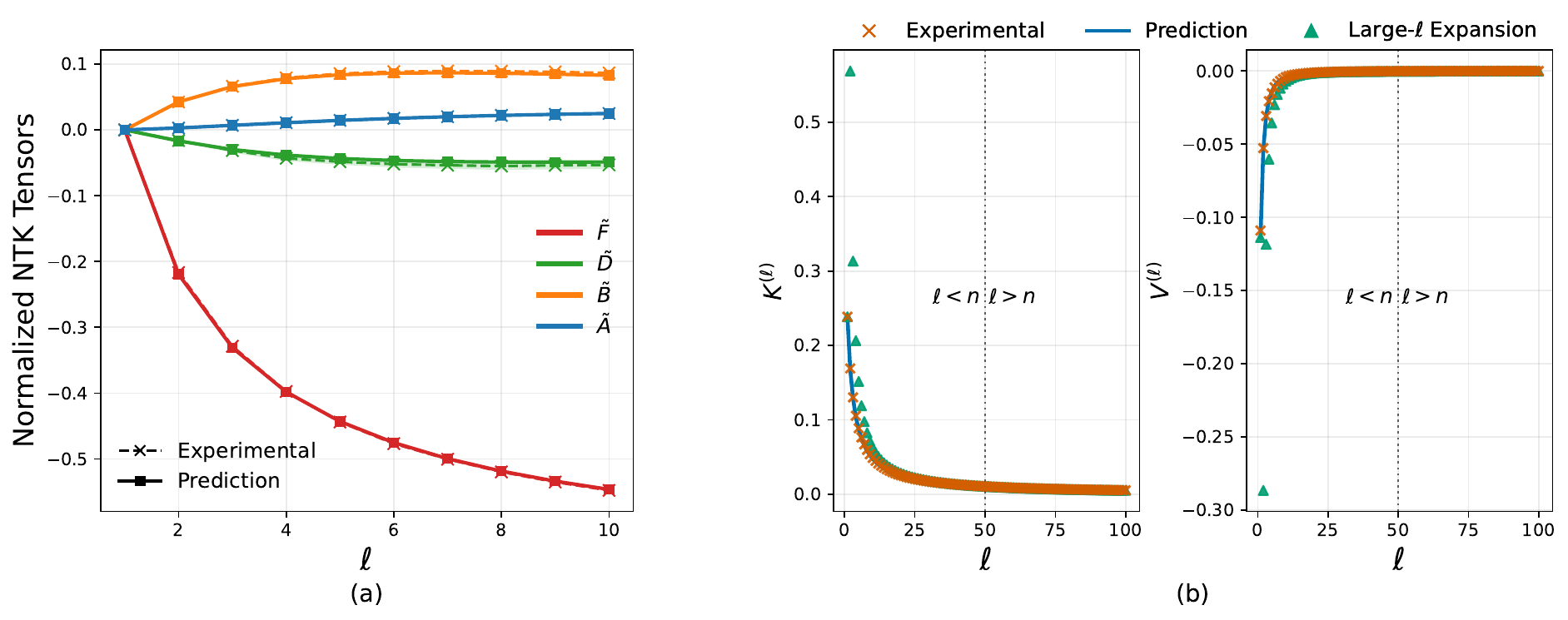}
  \vspace{-0.1cm}
\caption{\emph{Orthogonal saturation.} 
(a) Normalized NTK tensors from Monte Carlo simulations of \(\tanh\) orthogonal networks (\(n=50\), \(L=10\)) at criticality, compared with theoretical predictions. Means are computed over 600 initializations; shaded regions denote standard deviations (typically not visible), showing quantitative agreement. 
(b) Large-\(\ell\) expansions of the NNGP and quartic vertex \(V\), compared with exact solutions and simulations. While discrepancies appear at small \(\ell\), asymptotic predictions are in precise agreement at large \(\ell\). The theory captures the regime \(\ell > n\), consistent with~\cite{day2023}, where orthogonal networks exhibit saturation at large depth, unlike the Gaussian case.}
\label{fig:combined_tensors_plot}
\vspace{-0.5cm}
\end{figure*}


\subsection{Large-$\ell$ expansions}

The recursion relations derived above admit analytic solutions in the large-depth regime. In this limit, one can substitute the standard ansatz
\begin{align}\label{largelexpansion}
\mathcal{O}(\ell) = \ell^{-p_{\mathcal{O}}}\sum_{p,q=1}^{\infty} c^{\mathcal{O}}_{p,q}\frac{\log^{p}\ell}{\ell^{q}}
\end{align}
into the recursion equations and determine the coefficients $c^{\mathcal{O}}_{p,q}$ systematically once the initial conditions are fixed. We derive large-$\ell$ expansions for the NNGP and the quartic vertex $V$ in the single-input setting, and compare them with exact solutions and Monte Carlo simulations in Figure~\ref{fig:combined_tensors_plot} (b). The analytical expansion agrees well with the empirical observations for large $\ell$.
Complete derivations of the expansions~\eqref{largelexpansion} for the NTK tensors and their derivatives are provided in Appendix~\ref{app:large_l_expansion}.

\section{Acknowledgments}

We would like to thank Philipp Misof, Yonatan Kahn and Zhengkang (Kevin) Zhang for insightful and helpful discussions. This work was supported by the Wallenberg AI, Autonomous Systems and Software Program (WASP), funded by the Knut and Alice Wallenberg Foundation.

\renewcommand*{\bibfont}{\normalfont\footnotesize}
\printbibliography

@inproceedings{hu2019,
  ids           = {hu2020a},
  title         = {Provable {{Benefit}} of {{Orthogonal Initialization}} in {{Optimizing Deep Linear Networks}}},
  booktitle     = {International {{Conference}} on {{Learning Representations}}},
  author        = {Hu, Wei and Xiao, Lechao and Pennington, Jeffrey},
  year          = 2019,
  month         = sep,
  eprint        = {2001.05992},
  primaryclass  = {cs},
  urldate       = {2026-03-09},
  archiveprefix = {arXiv},
  langid        = {english}
}

@inproceedings{huang2021,
  ids           = {huang2021a},
  title         = {On the {{Neural Tangent Kernel}} of {{Deep Networks}} with {{Orthogonal Initialization}}},
  booktitle     = {Twenty-{{Ninth International Joint Conference}} on {{Artificial Intelligence}}},
  author        = {Huang, Wei and Du, Weitao and Xu, Richard Yi Da},
  year          = 2021,
  month         = aug,
  volume        = {3},
  eprint        = {2004.05867},
  primaryclass  = {cs},
  pages         = {2577--2583},
  issn          = {1045-0823},
  doi           = {10.24963/ijcai.2021/355},
  urldate       = {2026-03-09},
  archiveprefix = {arXiv},
  langid        = {english}
}

@inproceedings{pennington2017,
  ids           = {pennington2017a},
  title         = {Resurrecting the Sigmoid in Deep Learning through Dynamical Isometry: Theory and Practice},
  shorttitle    = {Resurrecting the Sigmoid in Deep Learning through Dynamical Isometry},
  booktitle     = {Advances in {{Neural Information Processing Systems}}},
  author        = {Pennington, Jeffrey and Schoenholz, Samuel and Ganguli, Surya},
  year          = 2017,
  volume        = {30},
  eprint        = {1711.04735},
  primaryclass  = {cs},
  publisher     = {Curran Associates, Inc.},
  urldate       = {2026-03-09},
  archiveprefix = {arXiv}
}

@inproceedings{pennington2018a,
  ids           = {pennington2018b},
  title         = {The Emergence of Spectral Universality in Deep Networks},
  booktitle     = {Proceedings of the {{Twenty-First International Conference}} on {{Artificial Intelligence}} and {{Statistics}}},
  author        = {Pennington, Jeffrey and Schoenholz, Samuel and Ganguli, Surya},
  year          = 2018,
  month         = mar,
  eprint        = {1802.09979},
  primaryclass  = {stat},
  pages         = {1924--1932},
  publisher     = {PMLR},
  issn          = {2640-3498},
  urldate       = {2026-03-09},
  archiveprefix = {arXiv},
  langid        = {english}
}

@inproceedings{xiao2018,
  title      = {Dynamical {{Isometry}} and a {{Mean Field Theory}} of {{CNNs}}: {{How}} to {{Train}} 10,000-{{Layer Vanilla Convolutional Neural Networks}}},
  shorttitle = {Dynamical {{Isometry}} and a {{Mean Field Theory}} of {{CNNs}}},
  booktitle  = {Proceedings of the 35th {{International Conference}} on {{Machine Learning}}},
  author     = {Xiao, Lechao and Bahri, Yasaman and {Sohl-Dickstein}, Jascha and Schoenholz, Samuel and Pennington, Jeffrey},
  year       = 2018,
  month      = jul,
  pages      = {5393--5402},
  publisher  = {PMLR},
  issn       = {2640-3498},
  urldate    = {2024-02-18},
  langid     = {english}
}

@inproceedings{mishkin2016,
  title         = {All You Need Is a Good Init},
  booktitle     = {International {{Conference}} on {{Learning Representations}} 2016},
  author        = {Mishkin, Dmytro and Matas, Jiri},
  year          = 2016,
  month         = feb,
  eprint        = {1511.06422},
  primaryclass  = {cs},
  publisher     = {arXiv},
  urldate       = {2024-03-13},
  archiveprefix = {arXiv}
}

@misc{saxe2014,
  title         = {Exact Solutions to the Nonlinear Dynamics of Learning in Deep Linear Neural Networks},
  author        = {Saxe, Andrew M. and McClelland, James L. and Ganguli, Surya},
  year          = 2014,
  month         = feb,
  number        = {arXiv:1312.6120},
  eprint        = {1312.6120},
  primaryclass  = {cond-mat, q-bio, stat},
  publisher     = {arXiv},
  urldate       = {2024-03-14},
  archiveprefix = {arXiv}
}

@inproceedings{jacot2018,
  ids           = {jacot2020},
  title         = {Neural {{Tangent Kernel}}: {{Convergence}} and {{Generalization}} in {{Neural Networks}}},
  shorttitle    = {Neural {{Tangent Kernel}}},
  booktitle     = {Advances in {{Neural Information Processing Systems}}},
  author        = {Jacot, Arthur and Gabriel, Franck and Hongler, Clement},
  year          = {2018},
  volume        = {31},
  eprint        = {1806.07572},
  publisher     = {Curran Associates, Inc.},
  urldate       = {2023-09-23},
  archiveprefix = {arXiv}
}

@article{banta2024,
  ids           = {banta2023},
  title         = {Structures of Neural Network Effective Theories},
  author        = {Banta, Ian and Cai, Tianji and Craig, Nathaniel and Zhang, Zhengkang},
  year          = {2024},
  month         = may,
  journal       = {Physical Review D},
  volume        = {109},
  number        = {10},
  eprint        = {2305.02334},
  primaryclass  = {cond-mat, physics:hep-ph, physics:hep-th, stat},
  pages         = {105007},
  doi           = {10.1103/PhysRevD.109.105007},
  urldate       = {2024-06-10},
  archiveprefix = {arXiv}
}

@book{roberts2022,
  ids           = {roberts2021a},
  title         = {The {{Principles}} of {{Deep Learning Theory}}: {{An Effective Theory Approach}} to {{Understanding Neural Networks}}},
  shorttitle    = {The {{Principles}} of {{Deep Learning Theory}}},
  author        = {Roberts, Daniel A. and Yaida, Sho},
  year          = {2022},
  eprint        = {2106.10165},
  publisher     = {Cambridge University Press},
  address       = {Cambridge},
  doi           = {10.1017/9781009023405},
  urldate       = {2024-12-18},
  archiveprefix = {arXiv},
  isbn          = {978-1-316-51933-2}
}

@inproceedings{he2015a,
  title         = {Delving {{Deep}} into {{Rectifiers}}: {{Surpassing Human-Level Performance}} on {{ImageNet Classification}}},
  shorttitle    = {Delving {{Deep}} into {{Rectifiers}}},
  booktitle     = {Proceedings of the {{IEEE International Conference}} on {{Computer Vision}}},
  author        = {He, Kaiming and Zhang, Xiangyu and Ren, Shaoqing and Sun, Jian},
  year          = {2015},
  eprint        = {1502.01852},
  primaryclass  = {cs},
  pages         = {1026--1034},
  urldate       = {2024-06-06},
  archiveprefix = {arXiv}
}

@inproceedings{glorot2010,
  title     = {Understanding the Difficulty of Training Deep Feedforward Neural Networks},
  booktitle = {Proceedings of the {{Thirteenth International Conference}} on {{Artificial Intelligence}} and {{Statistics}}},
  author    = {Glorot, Xavier and Bengio, Yoshua},
  year      = {2010},
  month     = mar,
  pages     = {249--256},
  publisher = {{JMLR Workshop and Conference Proceedings}},
  issn      = {1938-7228},
  urldate   = {2022-12-04},
  langid    = {english}
}

@inproceedings{yaida2020,
  title         = {Non-{{Gaussian}} Processes and Neural Networks at Finite Widths},
  booktitle     = {Proceedings of {{The First Mathematical}} and {{Scientific Machine Learning Conference}}},
  author        = {Yaida, Sho},
  year          = {2020},
  month         = aug,
  eprint        = {1910.00019},
  pages         = {165--192},
  publisher     = {PMLR},
  urldate       = {2023-09-27},
  archiveprefix = {arXiv},
  langid        = {english}
}

@article{halverson2021,
  title         = {Neural {{Networks}} and {{Quantum Field Theory}}},
  author        = {Halverson, James and Maiti, Anindita and Stoner, Keegan},
  year          = {2021},
  month         = sep,
  journal       = {Machine Learning: Science and Technology},
  volume        = {2},
  number        = {3},
  eprint        = {2008.08601},
  pages         = {035002},
  issn          = {2632-2153},
  doi           = {10.1088/2632-2153/abeca3},
  urldate       = {2021-08-06},
  archiveprefix = {arXiv}
}

@article{grosvenor2022,
  ids           = {grosvenor2021},
  title         = {The Edge of Chaos: Quantum Field Theory and Deep Neural Networks},
  shorttitle    = {The Edge of Chaos},
  author        = {Grosvenor, Kevin and Jefferson, Ro},
  year          = {2022},
  month         = mar,
  journal       = {SciPost Physics},
  volume        = {12},
  number        = {3},
  eprint        = {2109.13247},
  pages         = {081},
  issn          = {2542-4653},
  doi           = {10.21468/SciPostPhys.12.3.081},
  urldate       = {2024-07-30},
  archiveprefix = {arXiv},
  langid        = {english}
}

@article{demirtas2024,
  ids           = {demirtas2023},
  title         = {Neural Network Field Theories: Non-{{Gaussianity}}, Actions, and Locality},
  shorttitle    = {Neural Network Field Theories},
  author        = {Demirtas, Mehmet and Halverson, James and Maiti, Anindita and Schwartz, Matthew D and Stoner, Keegan},
  year          = {2024},
  month         = jan,
  journal       = {Machine Learning: Science and Technology},
  volume        = {5},
  number        = {1},
  eprint        = {2307.03223},
  primaryclass  = {hep-th},
  pages         = {015002},
  issn          = {2632-2153},
  doi           = {10.1088/2632-2153/ad17d3},
  urldate       = {2025-05-05},
  archiveprefix = {arXiv},
  langid        = {english}
}

@article{day2025,
  ids           = {day2023},
  title         = {Feature Learning and Generalization in Deep Networks with Orthogonal Weights},
  author        = {Day, Hannah and Kahn, Yonatan and Roberts, Daniel A},
  year          = 2025,
  month         = aug,
  journal       = {Machine Learning: Science and Technology},
  volume        = {6},
  number        = {3},
  eprint        = {2310.07765},
  primaryclass  = {hep-ph, physics:hep-th, stat},
  pages         = {035027},
  publisher     = {IOP Publishing},
  issn          = {2632-2153},
  doi           = {10.1088/2632-2153/adf278},
  urldate       = {2026-03-24},
  archiveprefix = {arXiv},
  langid        = {english}
}

@misc{maloney2022,
  title         = {A {{Solvable Model}} of {{Neural Scaling Laws}}},
  author        = {Maloney, Alexander and Roberts, Daniel A. and Sully, James},
  year          = {2022},
  month         = oct,
  number        = {arXiv:2210.16859},
  eprint        = {2210.16859},
  primaryclass  = {cs},
  publisher     = {arXiv},
  urldate       = {2025-01-22},
  archiveprefix = {arXiv}
}

@article{zhang2025,
  ids           = {zhang2024a},
  title         = {Neural Scaling Laws from Large-{{N}} Field Theory: Solvable Model beyond the Ridgeless Limit},
  shorttitle    = {Neural Scaling Laws from Large-{{N}} Field Theory},
  author        = {Zhang, Zhengkang},
  year          = {2025},
  month         = apr,
  journal       = {Machine Learning: Science and Technology},
  volume        = {6},
  number        = {2},
  eprint        = {2405.19398},
  primaryclass  = {hep-th},
  pages         = {025010},
  issn          = {2632-2153},
  doi           = {10.1088/2632-2153/adc872},
  urldate       = {2025-05-05},
  archiveprefix = {arXiv},
  langid        = {english}
}

@incollection{maiti2021,
  title         = {Symmetry-via-Duality: {{Invariant}} Neural Network Densities from Parameter-Space Correlators},
  shorttitle    = {Symmetry-via-{{Duality}}},
  booktitle     = {Machine Learning in Pure Mathematics and Theoretical Physics},
  author        = {Maiti, Anindita and Stoner, Keegan and Halverson, James},
  eprint        = {2106.00694},
  primaryclass  = {Chapter 8},
  pages         = {293--330},
  doi           = {10.1142/9781800613706_0008},
  archiveprefix = {arXiv},
  chapter       = {Chapter 8}
}

@misc{guillen2025,
  title         = {Finite-{{Width Neural Tangent Kernels}} from {{Feynman Diagrams}}},
  author        = {Guillen, Max and Misof, Philipp and Gerken, Jan E.},
  year          = 2025,
  month         = aug,
  number        = {arXiv:2508.11522},
  eprint        = {2508.11522},
  primaryclass  = {cs},
  publisher     = {arXiv},
  doi           = {10.48550/arXiv.2508.11522},
  urldate       = {2025-08-18},
  archiveprefix = {arXiv}
}

@article{weingarten1978,
  title     = {Asymptotic Behavior of Group Integrals in the Limit of Infinite Rank},
  author    = {Weingarten, Don},
  year      = 1978,
  month     = may,
  journal   = {Journal of Mathematical Physics},
  volume    = {19},
  number    = {5},
  pages     = {999--1001},
  publisher = {AIP Publishing},
  issn      = {0022-2488},
  doi       = {10.1063/1.523807},
  urldate   = {2026-03-24},
  langid    = {english}
}

@inproceedings{ngiam2010,
  title     = {Tiled Convolutional Neural Networks},
  booktitle = {Advances in {{Neural Information Processing Systems}}},
  author    = {Ngiam, Jiquan and Chen, Zhenghao and Chia, Daniel and Koh, Pang and Le, Quoc and Ng, Andrew},
  year      = 2010,
  volume    = {23},
  publisher = {Curran Associates, Inc.},
  urldate   = {2026-03-24}
}

@article{Collins2006jgn,
  author  = {Collins, Beno{\^i}t and {\'S}niady, Piotr},
  title   = {Integration with Respect to the Haar Measure on Unitary, Orthogonal and Symplectic Groups},
  doi     = {10.1007/s00220-006-1554-3},
  journal = {Communications in Mathematical Physics},
  volume  = {264},
  number  = {3},
  pages   = {773--795},
  year    = {2006}
}

@article{Collins2009,
  author  = {Collins, Beno{\^i}t and Matsumoto, Sho},
  title   = {On some properties of orthogonal Weingarten functions},
  doi     = {10.1063/1.3251304},
  journal = {Journal of Mathematical Physics},
  volume  = {50},
  number  = {11},
  pages   = {113516},
  year    = {2009}
}

\ifbool{includeapp}{

\newpage
\appendix
\onecolumn

\section{Orthogonal NTK tensors}
\label{app:orthogonal_NTK_tensors}

The statistics of the joint distribution of preactivations and the NTK in orthogonal neural networks are characterized by cumulants involving both quantities. At first order in \(1/n\), these statistics are fully determined by the two cumulants shown in~\eqref{eq:3} and~\eqref{eq:4}. In this appendix, we derive explicit recursion relations for the tensors $D$, $F$, $A$, and $B$ at all orders in \(1/n\), and then present their first-order \(1/n\) approximations.

\subsection{NTK-preactivation cross-correlator}
The cumulant defined in~\eqref{eq:3} admits the following convenient representation
\begin{align}
&\EE^{c}_{\theta}[z^{(\ell+1)}_{i_{1}}(x_{1})z^{(\ell+1)}_{i_{2}}(x_{2})\widehat{\Delta\Theta}^{(\ell+1)}_{i_{3}i_{4}}(x_{3},x_{4})]\nonumber\\
  &\qquad= \EE^{c}_{\theta}[z^{(\ell+1)}_{i_{1}}(x_{1})z^{(\ell+1)}_{i_{2}}(x_{2})\widehat{\Theta}^{(\ell+1)}_{i_{3}i_{4}}(x_{3},x_{4})]
   - \EE^{c}_{\theta}[z^{(\ell+1)}_{i_{1}}(x_{1})z^{(\ell+1)}_{i_{2}}(x_{2})]\EE^{c}_{\theta}[\widehat{\Theta}^{(\ell+1)}_{i_{3}i_{4}}(x_{3},x_{4})]
\end{align}
Substituting the preactivation recursion~\eqref{eq:7} and the NTK recursion~\eqref{ntkrecursion}, one finds that
\begin{align}
&\EE^{c}_{\theta}[z^{(\ell+1)}_{i_{1}}(x_{1})z^{(\ell+1)}_{i_{2}}(x_{2})\widehat{\Delta\Theta}^{(\ell+1)}_{i_{3}i_{4}}(x_{3},x_{4})]\nonumber\\
  &\qquad= (12)_{i}(34)_{i}\biggl\{\frac{\lambda^{(\ell+1)}_{W}C_{W}}{n^{2}}\sum_{j,k=1}^{n}\bigg(\EE[\sigma^{(\ell)}_{j,1}\sigma^{(\ell)}_{j,2}\sigma^{(\ell)}_{k,3}\sigma^{(\ell)}_{k,4}] - \EE[\sigma^{(\ell)}_{j,1}\sigma^{(\ell)}_{j,2}]\EE[\sigma^{(\ell)}_{k,3}\sigma^{(\ell)}_{k,4}]\bigg) \nonumber\\
  & + (C_{W})^{2}\sum_{j,k=1}^{n}\bigg(\EE\left[\left(\mathcal{W}[1,1]\sigma^{(\ell)}_{j,1}\sigma^{(\ell)}_{j,2} - \frac{1}{n^{2}}\EE[\sigma^{(\ell)}_{j,1}\sigma^{(\ell)}_{j,2}]\right)\sigma'^{(\ell)}_{k,3}\sigma'^{(\ell)}_{k,4}\widehat{\Theta}^{(\ell)}_{kk,34}\right] + 2\mathcal{W}[2]\EE[\sigma_{j,1}^{(\ell)}\sigma'^{(\ell)}_{j,3}\sigma_{k,2}^{(\ell)}\sigma'^{(\ell)}_{k,4}\widehat{\Theta}^{(\ell)}_{jk,34}]
  \bigg)\biggr\} \nonumber\\
  & + (13)(24)_{i}\biggl\{(C_{W})^{2}\sum_{j,k=1}^{n}\bigg(\mathcal{W}[1,1]\EE[\sigma_{j,1}^{(\ell)}\sigma'^{(\ell)}_{j,3}\sigma_{k,2}^{(\ell)}\sigma'^{(\ell)}_{k,4}\widehat{ \Theta}_{jk,34}]\nonumber\\
  & + \mathcal{W}[2]\EE[\sigma_{j,1}^{(\ell)}\sigma_{j,2}^{(\ell)}\sigma'^{(\ell)}_{k,3}\sigma'^{(\ell)}_{k,4}\widehat{\Theta}_{kk,34}] + \mathcal{W}[2]\EE[\sigma_{j,1}^{(\ell)}\sigma'^{(\ell)}_{j,4}\sigma_{k,2}^{(\ell)}\sigma'^{(\ell)}_{k,3}\widehat{ \Theta}_{jk,34}]\bigg)\biggr\}\nonumber\\
  & + (14)(23)_{i}\biggl\{(C_{W})^{2}\sum_{j,k=1}^{n}\bigg(\mathcal{W}[1,1]\EE[\sigma_{j,1}^{(\ell)}\sigma'^{(\ell)}_{j,4}\sigma_{k,2}^{(\ell)}\sigma'^{(\ell)}_{k,3}\widehat{ \Theta}_{jk,34}]\nonumber\\
  & + \mathcal{W}[2]\EE[\sigma_{j,1}^{(\ell)}\sigma_{j,2}^{(\ell)}\sigma'^{(\ell)}_{k,3}\sigma'^{(\ell)}_{k,4}\widehat{\Theta}_{kk,34}] + \mathcal{W}[2]\EE[\sigma_{j,1}^{(\ell)}\sigma'^{(\ell)}_{j,3}\sigma_{k,2}^{(\ell)}\sigma'^{(\ell)}_{k,4}\widehat{ \Theta}_{jk,34}]\bigg)\biggr\} \label{explicitcrosscorr}
\end{align}
Here $\mathcal{W}[1,1]$ and $\mathcal{W}[2]$ denote the $k=2$ Weingarten functions appearing in~\eqref{weingartenformula}, and we have used the index symmetry of the NTK sample labels. A direct comparison of~\eqref{explicitcrosscorr} and~\eqref{eq:3} then yields
\begin{align}
    \frac{1}{n}D^{(\ell+1)}_{1234} &= \frac{\lambda^{(\ell+1)}_{W}C_{W}}{n^{2}}\sum_{j,k=1}^{n}\bigg(\EE[\sigma^{(\ell)}_{j,1}\sigma^{(\ell)}_{j,2}\sigma^{(\ell)}_{k,3}\sigma^{(\ell)}_{k,4}] - \EE[\sigma^{(\ell)}_{j,1}\sigma^{(\ell)}_{j,2}]\EE[\sigma^{(\ell)}_{k,3}\sigma^{(\ell)}_{k,4}]\bigg) \nonumber\\
  & + (C_{W})^{2}\sum_{j,k=1}^{n}\bigg(\EE\left[\left(\mathcal{W}[1,1]\sigma^{(\ell)}_{j,1}\sigma^{(\ell)}_{j,2} - \frac{1}{n^{2}}\EE[\sigma^{(\ell)}_{j,1}\sigma^{(\ell)}_{j,2}]\right)\sigma'^{(\ell)}_{k,3}\sigma'^{(\ell)}_{k,4}\widehat{\Theta}^{(\ell)}_{kk,34}\right]\nonumber\\
  & + 2\mathcal{W}[2]\EE[\sigma_{j,1}^{(\ell)}\sigma'^{(\ell)}_{j,3}\sigma_{k,2}^{(\ell)}\sigma'^{(\ell)}_{k,4}\widehat{\Theta}^{(\ell)}_{jk,34}]
  \bigg)\label{dtensorfull}\\
  \frac{1}{n}F^{(\ell+1)}_{1324}  &= (C_{W})^{2}\sum_{j,k=1}^{n}\bigg(\mathcal{W}[1,1]\EE[\sigma_{j,1}^{(\ell)}\sigma'^{(\ell)}_{j,3}\sigma_{k,2}^{(\ell)}\sigma'^{(\ell)}_{k,4}\widehat{\Theta}^{(\ell)}_{jk,34}]\nonumber\\
  & + \mathcal{W}[2]\EE[\sigma_{j,1}^{(\ell)}\sigma_{j,2}^{(\ell)}\sigma'^{(\ell)}_{k,3}\sigma'^{(\ell)}_{k,4}\widehat{\Theta}^{(\ell)}_{kk,34}] + \mathcal{W}[2]\EE[\sigma_{j,1}^{(\ell)}\sigma'^{(\ell)}_{j,4}\sigma_{k,2}^{(\ell)}\sigma'^{(\ell)}_{k,3}\widehat{ \Theta}^{(\ell)}_{jk,34}]\bigg) \label{ftensorfull}
\end{align}
To obtain the first-order \(1/n\) corrections to these tensors, we expand the $k=2$ Weingarten functions in the large-$n$ regime. Explicitly,
\begin{align}
    \mathcal{W}[1,1] = \frac{n+1}{n(n-1)(n+2)} = \frac{1}{n^{2}} + \frac{2}{n^{4}} - \frac{2}{n^{5}} + \mathcal{O}\left(\frac{1}{n^{6}}\right) \label{w11expanded}\\
    \mathcal{W}[2] = -\frac{1}{n(n-1)(n+2)} = -\frac{1}{n^{3}} + \frac{1}{n^{4}} - \frac{3}{n^{5}} + \mathcal{O}\left(\frac{1}{n^{6}}\right) \label{w2expanded}
\end{align}
Upon substituting~\eqref{w11expanded} and~\eqref{w2expanded} into~\eqref{dtensorfull} and~\eqref{ftensorfull}, we obtain
\begin{align}
  \frac{1}{n}D^{(\ell+1)}_{1234} &= \frac{\lambda_{W}^{(\ell+1)}C_{W}}{n^{2}}\sum_{j,k=1}^{n}\EE[\widehat{\Delta K}^{(\ell)}_{j,12}\widehat{\Delta K}^{(\ell)}_{k,34}] + \frac{(C_{W})^{2}}{n^{2}}\sum_{j,k=1}^{n}\EE[\widehat{\Delta K}^{(\ell)}_{j,12}\sigma'^{(\ell)}_{k,3}\sigma'^{(\ell)}_{k,4}]\Theta^{(\ell)}_{34}\nonumber\\
  & + \frac{(C_{W})^{2}}{n^{2}}\sum_{j,k=1}^{n}\EE[\widehat{\Delta K}^{(\ell)}_{j,12}\sigma'^{(\ell)}_{k,3}\sigma'^{(\ell)}_{k,4}\widehat{\Delta \Theta}^{(\ell)}_{kk,34}] -  2\frac{(C_{W})^{2}}{n^{3}}\sum_{j=1}^{n}\EE[\sigma_{j,1}^{(\ell)}\sigma'^{(\ell)}_{j,3}\sigma_{j,2}^{(\ell)}\sigma'^{(\ell)}_{j,4}]\Theta^{(\ell)}_{34} \nonumber\\
  & -  2\frac{(C_{W})^{2}}{n^{3}}\sum_{j,k=1}^{n}\EE[\sigma_{j,1}^{(\ell)}\sigma'^{(\ell)}_{j,3}\sigma_{k,2}^{(\ell)}\sigma'^{(\ell)}_{k,4}\widehat{\Delta \Theta}^{(\ell)}_{jk,34}] + \mathcal{O}\left(\frac{1}{n^{2}}\right) \\
  \frac{1}{n}F^{(\ell+1)}_{1324}  &= \frac{(C_{W})^{2}}{n^{2}}\bigg[\sum_{j=1}^{n}\EE[\sigma_{j,1}^{(\ell)}\sigma'^{(\ell)}_{j,3}\sigma_{j,2}^{(\ell)}\sigma'^{(\ell)}_{j,4}]\Theta^{(\ell)}_{34} + \sum_{j,k=1}^{n}\EE[\sigma_{j,1}^{(\ell)}\sigma'^{(\ell)}_{j,3}\sigma_{k,2}^{(\ell)}\sigma'^{(\ell)}_{k,4}\widehat{\Delta \Theta}^{(\ell)}_{jk,34}]\bigg] \nonumber\\
  & - \frac{(C_{W})^{2}}{n^{3}}\bigg[\sum_{j,k=1}^{n}\EE[\sigma_{j,1}^{(\ell)}\sigma_{j,2}^{(\ell)}\sigma'^{(\ell)}_{k,3}\sigma'^{(\ell)}_{k,4}]\Theta^{(\ell)}_{34} + \sum_{j,k=1}^{n}\EE[\sigma_{j,1}^{(\ell)}\sigma_{j,2}^{(\ell)}\sigma'^{(\ell)}_{k,3}\sigma'^{(\ell)}_{k,4}\widehat{\Delta \Theta}^{(\ell)}_{jk,34}]\bigg]\nonumber\\
  & - \frac{(C_{W})^{2}}{n^{3}}\bigg[\sum_{j=1}^{n}\EE[\sigma_{j,1}^{(\ell)}\sigma'^{(\ell)}_{j,4}\sigma_{j,2}^{(\ell)}\sigma'^{(\ell)}_{j,3}]\Theta^{(\ell)}_{34} + \sum_{j,k=1}^{n}\EE[\sigma_{j,1}^{(\ell)}\sigma'^{(\ell)}_{j,4}\sigma_{k,2}^{(\ell)}\sigma'^{(\ell)}_{k,3}\widehat{\Delta \Theta}^{(\ell)}_{jk,34}]\bigg] + \mathcal{O}\left(\frac{1}{n^{2}}\right)
\end{align}
Here we introduce the tensor $\widehat{\Delta K}^{(\ell)}_{i,\alpha\beta} = \sigma_{i,\alpha}^{(\ell)}\sigma_{i,\beta}^{(\ell)} - \EE[\sigma_{i,\alpha}^{(\ell)}\sigma_{i,\beta}^{(\ell)}]$, and decompose the NTK tensor into its mean $\Theta$ and fluctuations $\widehat{\Delta \Theta}$. One can then directly apply the effective field theory techniques developed in~\cite{roberts2021a} to expand non-Gaussian expectation values with an interaction controlled by the four-point vertex $V_{4}$. After lengthy and technically involved algebraic computations, one arrives at
\begin{align}
  D^{(\ell +1)}_{12\textcolor{color1}{34}} &= C_{W}\langle \widehat{\Delta G}^{(\ell)}_{12}\widehat{\Delta \Omega}_{34}^{(\ell +1)}\rangle_{K^{(\ell)}} \nonumber\\
  & + \frac{C_{W}}{4}\sum_{\beta_{1},\beta_{2},\beta_{3},\beta_{4} \in \{1,2,3,4\}}V^{(\ell)}_{(\beta_{1}\beta_{2})(\beta_{3}\beta_{4})}\Biggl\langle \frac{d^{2}(\widehat{\Delta G}^{(\ell)}_{12})}{d z^{(\ell)}_{\beta_{1}}d z^{(\ell)}_{\beta_{2}}}\Biggr\rangle_{K^{(\ell)}}\Biggl\langle \frac{d^{2}(\widehat{\Delta \Omega}^{(\ell +1)}_{34})}{d z^{(\ell)}_{\beta_{1}}d z^{(\ell)}_{\beta_{2}}}\Biggr\rangle_{K^{(\ell)}}\nonumber\\
  & + (C_{W})^{2}\sum_{\beta_{1},\beta_{2} \in \{1,2,3,4\}}\Biggl\langle \frac{d^{2}(\widehat{\Delta G}^{(\ell)}_{12})}{d z^{(\ell)}_{\beta_{1}}d z^{(\ell)}_{\beta_{2}}}\Biggr\rangle_{K^{(\ell)}} \langle \sigma'^{(\ell)}_{\textcolor{color1}{3}}\sigma'^{(\ell)}_{\textcolor{color1}{4}}\rangle_{K^{(\ell)}}D^{(\ell)}_{\beta_{1}\beta_{2}\textcolor{color1}{34}} + \mathcal{O}\left(\frac{1}{n}\right)\label{drecursion}\,,\\
  F^{(\ell+1)}_{1\textcolor{color1}{3}2\textcolor{color1}{4}} &=(C_{W})^{2}\left[ \langle \sigma^{(\ell)}_{1}\sigma^{(\ell)}_{2}\sigma'^{(\ell)}_{\textcolor{color1}{3}}\sigma'^{(\ell)}_{\textcolor{color1}{4}} \rangle_{K^{(\ell)}} - \langle \sigma^{(\ell)}_{1}\sigma^{(\ell)}_{2}\rangle_{K^{(\ell)}}\langle \sigma'^{(\ell)}_{\textcolor{color1}{3}}\sigma'^{(\ell)}_{\textcolor{color1}{4}} \rangle_{K^{(\ell)}} \right]\Theta^{(\ell)}_{\textcolor{color1}{3}\textcolor{color1}{4}}\nonumber\\
  &\hspace{-0.5cm}+(C_{W})^{2}\!\!\!\sum_{\alpha,\beta,\gamma,\delta=1}^{4}\Biggl\langle \frac{d(\sigma^{(\ell)}_{1}\sigma'^{(\ell)}_{\textcolor{color1}{3}})}{dz^{(\ell)}_{\gamma}} \Biggr\rangle_{K^{(\ell)}}\Biggl\langle \frac{d(\sigma_{2}^{(\ell)}\sigma'^{(\ell)}_{\textcolor{color1}{4}})}{dz^{(\ell)}_{\delta}} \Biggr\rangle_{K^{(\ell)}}F^{(\ell)}_{\gamma\textcolor{color1}{3}\delta\textcolor{color1}{4}} + \mathcal{O}\left(\frac{1}{n}\right)\label{frecursion} \,,
\end{align}
where  $\widehat{\Delta G}_{\alpha\beta}^{(\ell)} = \sigma_{i,\alpha}^{(\ell)}\sigma_{i,\beta}^{(\ell)} - \langle\sigma_{i,\alpha}^{(\ell)}\sigma_{i,\beta}^{(\ell)}\rangle_{K^{(\ell)}}$, and $\widehat{\Omega}^{(\ell+1)}_{i,\alpha\beta} = \lambda^{(\ell+1)}_{W}\sigma^{(\ell)}_{i,\alpha}\sigma^{(\ell)}_{i,\beta} + C_{W}\Theta_{\alpha\beta}^{(\ell)}\sigma'^{(\ell)}_{i,\alpha}\sigma'^{(\ell)}_{i,\beta}$ with $\widehat{\Delta\Omega}_{i,\alpha\beta}^{(\ell+1)}=\widehat{\Omega}_{i,\alpha\beta}^{(\ell+1)}-\langle \widehat{\Omega}_{i,\alpha\beta}^{(\ell+1)} \rangle_{K^{(\ell)}}$.

We observe that the recursion relation for the $D$ tensor coincides with its Gaussian counterpart, while the recursion relation for the $F$ tensor acquires an additional term proportional to the NTK mean. The expression in~\eqref{eq:F} is obtained from~\eqref{frecursion} by performing integrations-by-part inside the Gaussian expectation values of derivative terms.

\subsection{NTK variance}
The NTK variance admits the compact representation
\begin{align}
&\EE^{c}_{\theta}[\widehat{\Delta \Theta}^{(\ell+1)}_{i_{1}i_{2}}(\textcolor{color1}{x_{1}},\textcolor{color1}{x_{2}})\widehat{\Delta \Theta}^{(\ell+1)}_{i_{3}i_{4}}(\textcolor{color2}{x_{3}},\textcolor{color2}{x_{4}})]\nonumber\\
&\qquad=\frac{1}{n}\!\!\left(\! A^{(\ell+1)}_{\textcolor{color1}{12}\textcolor{color2}{34}}(12)(34)_{i}{+}B^{(\ell+1)}_{\textcolor{color1}{1}\textcolor{color2}{3}\textcolor{color1}{2}\textcolor{color2}{4}}(13)(24)_{i}{+}B^{(\ell+1)}_{\textcolor{color1}{1}\textcolor{color2}{4}\textcolor{color1}{2}\textcolor{color2}{3}}(14)(23)_{i} \!\right)\,,\label{eq:4}
\end{align}
where the tensors $A$ and $B$ are fully determined by $\hat{K}$ and $\hat{\Theta}$. To see this, we rewrite the cumulant~\eqref{eq:4} in the equivalent form
\begin{align}
&\EE^{c}_{\theta}[\widehat{\Delta \Theta}^{(\ell+1)}_{i_{1}i_{2}}(x_{1},x_{2})\widehat{\Delta \Theta}^{(\ell+1)}_{i_{3}i_{4}}(x_{3},x_{4})]\nonumber\\
  &\qquad= \EE^{c}_{\theta}[\widehat{\Theta}^{(\ell+1)}_{i_{1}i_{2}}(x_{1},x_{2})\widehat{\Theta}^{(\ell+1)}_{i_{3}i_{4}}(x_{3},x_{4})] - \EE^{c}_{\theta}[\widehat{\Theta}^{(\ell+1)}_{i_{1}i_{2}}(x_{1},x_{2})]\EE^{c}_{\theta}[\widehat{\Theta}^{(\ell+1)}_{i_{3}i_{4}}(x_{3},x_{4})]
\end{align}
Substituting the recursion of the NTK tensor from~\eqref{ntkrecursion}, we obtain that
\begin{align}
&\EE^{c}_{\theta}[\widehat{\Delta \Theta}^{(\ell+1)}_{i_{1}i_{2}}(x_{1},x_{2})\widehat{\Delta \Theta}^{(\ell+1)}_{i_{3}i_{4}}(x_{3},x_{4})]\nonumber\\
  &\qquad= (12)(34)_{i}\biggl\{\frac{(\lambda^{(\ell+1)}_{W})^{2}}{n^{2}}\sum_{j,k}\EE[\widehat{\Delta K}^{(\ell)}_{j,12}\widehat{\Delta K}^{(\ell)}_{j,34}] + \frac{\lambda_{W}^{(\ell+1)}C_{W}}{n^{2}}\sum_{j,k=1}^{n}\EE[\widehat{\Delta K}_{j,12}^{(\ell)}\sigma'^{(\ell)}_{k,3}\sigma'^{(\ell)}_{k,4}\widehat{\Theta}^{(\ell)}_{kk,34}]\nonumber\\
  & + \frac{\lambda_{W}^{(\ell+1)}C_{W}}{n^{2}}\sum_{j,k=1}^{n}\EE[\widehat{\Delta K}_{j,34}^{(\ell)}\sigma'^{(\ell)}_{k,1}\sigma'^{(\ell)}_{k,2}\widehat{\Theta}^{(\ell)}_{kk,12}] + (C_{W})^{2}\bigg[\sum_{j,k=1}^{n}\mathcal{W}[1,1]\EE[\sigma'^{(\ell)}_{j,1}\sigma'^{(\ell)}_{j,2}\widehat{\Theta}_{jj,12}\sigma'^{(\ell)}_{k,3}\sigma'^{(\ell)}_{k,4}\widehat{\Theta}^{(\ell)}_{kk,34}] \nonumber\\
  & - \frac{1}{n^{2}}\sum_{j,k=1}^{n}\EE[\sigma'^{(\ell)}_{j,1}\sigma'^{(\ell)}_{j,2}\widehat{\Theta}^{(\ell)}_{jj,12}]\EE[\sigma'^{(\ell)}_{k,3}\sigma'^{(\ell)}_{k,4}\widehat{\Theta}^{(\ell)}_{kk,34}] + \sum_{j,k=1}^{n}\mathcal{W}[2]\EE[\sigma'^{(\ell)}_{j,1}\sigma'^{(\ell)}_{k,2}\widehat{\Theta}^{(\ell)}_{jk,12}\sigma'^{(\ell)}_{j,3}\sigma'^{(\ell)}_{k,4}\widehat{\Theta}^{(\ell)}_{jk,34}] \nonumber\\
  & + \sum_{j,k=1}^{n}\mathcal{W}[2]\EE[\sigma'^{(\ell)}_{j,1}\sigma'^{(\ell)}_{k,2}\widehat{\Theta}^{(\ell)}_{jk,12}\sigma'^{(\ell)}_{k,3}\sigma'^{(\ell)}_{j,4}\widehat{\Theta}^{(\ell)}_{kj,34}]
  \bigg]\biggr\}\nonumber\\
  & + (13)(24)_{i}(C_{W})^{2}\biggl\{\sum_{j,k=1}^{n}\mathcal{W}[2]\EE[\sigma'^{(\ell)}_{j,1}\sigma'^{(\ell)}_{j,2}\widehat{\Theta}^{(\ell)}_{jj,12}\sigma'^{(\ell)}_{k,3}\sigma'^{(\ell)}_{k,4}\widehat{\Theta}^{(\ell)}_{kk,34}]\nonumber\\
  & + \sum_{j,k=1}^{n}\mathcal{W}[1,1]\EE[\sigma'^{(\ell)}_{j,1}\sigma'^{(\ell)}_{k,2}\widehat{\Theta}^{(\ell)}_{jk,12}\sigma'^{(\ell)}_{j,3}\sigma'^{(\ell)}_{k,4}\widehat{\Theta}^{(\ell)}_{jk,34}] + \sum_{j,k=1}^{n}\mathcal{W}[2]\EE[\sigma'^{(\ell)}_{j,1}\sigma'^{(\ell)}_{k,2}\widehat{\Theta}^{(\ell)}_{jk,12}\sigma'^{(\ell)}_{k,3}\sigma'^{(\ell)}_{j,4}\widehat{\Theta}^{(\ell)}_{kj,34}]\biggr\}\nonumber\\
  & + (14)(23)_{i}(C_{W})^{2}\biggl\{\sum_{j,k=1}^{n}\mathcal{W}[2]\EE[\sigma'^{(\ell)}_{j,1}\sigma'^{(\ell)}_{j,2}\widehat{\Theta}^{(\ell)}_{jj,12}\sigma'^{(\ell)}_{k,3}\sigma'^{(\ell)}_{k,4}\widehat{\Theta}^{(\ell)}_{kk,34}]\nonumber\\
  & + \sum_{j,k=1}^{n}\mathcal{W}[2]\EE[\sigma'^{(\ell)}_{j,1}\sigma'^{(\ell)}_{k,2}\widehat{\Theta}^{(\ell)}_{jk,12}\sigma'^{(\ell)}_{j,3}\sigma'^{(\ell)}_{k,4}\widehat{\Theta}^{(\ell)}_{jk,34}] + \sum_{j,k=1}^{n}\mathcal{W}[1,1]\EE[\sigma'^{(\ell)}_{j,1}\sigma'^{(\ell)}_{k,2}\widehat{\Theta}^{(\ell)}_{jk,12}\sigma'^{(\ell)}_{k,3}\sigma'^{(\ell)}_{j,4}\widehat{\Theta}^{(\ell)}_{kj,34}]\biggr\}
\end{align}
A direct comparison with \eqref{eq:4} then yields
\begin{align}
\frac{1}{n}A^{(\ell+1)}_{1234} &= \frac{(\lambda^{(\ell+1)}_{W})^{2}}{n^{2}}\sum_{j,k}\EE[\widehat{\Delta K}^{(\ell)}_{j,12}\widehat{\Delta K}^{(\ell)}_{j,34}] + \frac{\lambda_{W}^{(\ell+1)}C_{W}}{n^{2}}\sum_{j,k=1}^{n}\EE[\widehat{\Delta K}_{j,12}^{(\ell)}\sigma'^{(\ell)}_{k,3}\sigma'^{(\ell)}_{k,4}\widehat{\Theta}^{(\ell)}_{kk,34}]\nonumber\\
  & + \frac{\lambda_{W}^{(\ell+1)}C_{W}}{n^{2}}\sum_{j,k=1}^{n}\EE[\widehat{\Delta K}_{j,34}^{(\ell)}\sigma'^{(\ell)}_{k,1}\sigma'^{(\ell)}_{k,2}\widehat{\Theta}^{(\ell)}_{kk,12}] + (C_{W})^{2}\bigg[\sum_{j,k=1}^{n}\mathcal{W}[1,1]\EE[\sigma'^{(\ell)}_{j,1}\sigma'^{(\ell)}_{j,2}\widehat{\Theta}^{(\ell)}_{jj,12}\sigma'^{(\ell)}_{k,3}\sigma'^{(\ell)}_{k,4}\widehat{\Theta}^{(\ell)}_{kk,34}] \nonumber\\
  & - \frac{1}{n^{2}}\sum_{j,k=1}^{n}\EE[\sigma'^{(\ell)}_{j,1}\sigma'^{(\ell)}_{j,2}\widehat{\Theta}^{(\ell)}_{jj,12}]\EE[\sigma'^{(\ell)}_{k,3}\sigma'^{(\ell)}_{k,4}\widehat{\Theta}^{(\ell)}_{kk,34}] + \sum_{j,k=1}^{n}\mathcal{W}[2]\EE[\sigma'^{(\ell)}_{j,1}\sigma'^{(\ell)}_{k,2}\widehat{\Theta}^{(\ell)}_{jk,12}\sigma'^{(\ell)}_{j,3}\sigma'^{(\ell)}_{k,4}\widehat{\Theta}^{(\ell)}_{jk,34}] \nonumber\\
  & + \sum_{j,k=1}^{n}\mathcal{W}[2]\EE[\sigma'^{(\ell)}_{j,1}\sigma'^{(\ell)}_{k,2}\widehat{\Theta}^{(\ell)}_{jk,12}\sigma'^{(\ell)}_{k,3}\sigma'^{(\ell)}_{j,4}\widehat{\Theta}^{(\ell)}_{kj,34}]
  \bigg]\\
\frac{1}{n}B^{(\ell+1)}_{1324} &= (C_{W})^{2}\biggl\{\sum_{j,k=1}^{n}\mathcal{W}[2]\EE[\sigma'^{(\ell)}_{j,1}\sigma'^{(\ell)}_{j,2}\widehat{\Theta}^{(\ell)}_{jj,12}\sigma'^{(\ell)}_{k,3}\sigma'^{(\ell)}_{k,4}\widehat{\Theta}^{(\ell)}_{kk,34}]\nonumber\\
  & + \sum_{j,k=1}^{n}\mathcal{W}[1,1]\EE[\sigma'^{(\ell)}_{j,1}\sigma'^{(\ell)}_{k,2}\widehat{\Theta}^{(\ell)}_{jk,12}\sigma'^{(\ell)}_{j,3}\sigma'^{(\ell)}_{k,4}\widehat{\Theta}^{(\ell)}_{jk,34}] + \sum_{j,k=1}^{n}\mathcal{W}[2]\EE[\sigma'^{(\ell)}_{j,1}\sigma'^{(\ell)}_{k,2}\widehat{\Theta}^{(\ell)}_{jk,12}\sigma'^{(\ell)}_{k,3}\sigma'^{(\ell)}_{j,4}\widehat{\Theta}^{(\ell)}_{kj,34}]\biggr\}
\end{align}
The use of the \(1/n\)-expansions \eqref{w11expanded}, \eqref{w2expanded} allows the following simplifications
\begin{align}
\frac{1}{n}A^{(\ell+1)}_{1234} &= \frac{1}{n^{2}}\sum_{j,k}\EE[\widehat{\Delta \Omega}^{(\ell+1)}_{j,12}\widehat{\Delta \Omega}^{(\ell+1)}_{j,34}] + \frac{C_{W}}{n^{2}}\sum_{j,k=1}^{n}\EE[\widehat{\Delta \Omega}_{j,12}^{(\ell+1)}\sigma'^{(\ell)}_{k,3}\sigma'^{(\ell)}_{k,4}\widehat{\Delta \Theta}^{(\ell)}_{kk,34}]\nonumber\\
  & + \frac{C_{W}}{n^{2}}\sum_{j,k=1}^{n}\EE[\widehat{\Delta \Omega}_{j,34}^{(\ell+1)}\sigma'^{(\ell)}_{k,1}\sigma'^{(\ell)}_{k,2}\widehat{\Delta \Theta}^{(\ell)}_{kk,12}] + \frac{(C_{W})^{2}}{n^{2}}\bigg[\sum_{j,k=1}^{n}\EE[\sigma'^{(\ell)}_{j,1}\sigma'^{(\ell)}_{j,2}\widehat{\Delta \Theta}^{(\ell)}_{jj,12}\sigma'^{(\ell)}_{k,3}\sigma'^{(\ell)}_{k,4}\widehat{\Delta \Theta}^{(\ell)}_{kk,34}] \nonumber\\
  & - \sum_{j,k=1}^{n}\EE[\sigma'^{(\ell)}_{j,1}\sigma'^{(\ell)}_{j,2}\widehat{\Delta \Theta}^{(\ell)}_{jj,12}]\EE[\sigma'^{(\ell)}_{k,3}\sigma'^{(\ell)}_{k,4}\widehat{\Delta \Theta}^{(\ell)}_{kk,34}] - \frac{2}{n}\sum_{j=1}^{n}\EE[\sigma'^{(\ell)}_{j,1}\sigma'^{(\ell)}_{j,2}\sigma'^{(\ell)}_{j,3}\sigma'^{(\ell)}_{j,4}]\Theta^{(\ell)}_{12}\Theta^{(\ell)}_{34} \nonumber\\
  & - \frac{2}{n}\sum_{j=1}^{n}\EE[\sigma'^{(\ell)}_{j,1}\sigma'^{(\ell)}_{j,2}\sigma'^{(\ell)}_{j,3}\sigma'^{(\ell)}_{j,4}\widehat{\Delta \Theta}^{(\ell)}_{jj,34}]\Theta^{(\ell)}_{12} - \frac{2}{n}\sum_{j=1}^{n}\EE[\sigma'^{(\ell)}_{j,1}\sigma'^{(\ell)}_{j,2}\widehat{\Delta \Theta}^{(\ell)}_{jj,12}\sigma'^{(\ell)}_{j,3}\sigma'^{(\ell)}_{j,4}]\Theta^{(\ell)}_{34}\nonumber\\
  & - \frac{1}{n}\sum_{j,k=1}^{n}\EE[\sigma'^{(\ell)}_{j,1}\sigma'^{(\ell)}_{k,2}\widehat{\Delta \Theta}^{(\ell)}_{jk,12}\sigma'^{(\ell)}_{j,3}\sigma'^{(\ell)}_{k,4}\widehat{\Delta \Theta}^{(\ell)}_{jk,34}] - \frac{1}{n}\sum_{j,k=1}^{n}\EE[\sigma'^{(\ell)}_{j,1}\sigma'^{(\ell)}_{k,2}\widehat{\Delta \Theta}^{(\ell)}_{jk,12}\sigma'^{(\ell)}_{k,3}\sigma'^{(\ell)}_{j,4}\widehat{\Delta \Theta}^{(\ell)}_{kj,34}]
  \bigg]\nonumber\\
  & + \mathcal{O}\left(\frac{1}{n^{2}}\right)\label{atensorexpr}\\
\frac{1}{n}B^{(\ell+1)}_{1324} &= \frac{(C_{W})^{2}}{n^{2}}\bigg[-\frac{1}{n}\sum_{j,k=1}^{n}\EE[\sigma'^{(\ell)}_{j,1}\sigma'^{(\ell)}_{j,2}\sigma'^{(\ell)}_{k,3}\sigma'^{(\ell)}_{k,4}]\Theta^{(\ell)}_{12}\Theta^{(\ell)}_{34} -\frac{1}{n}\sum_{j,k=1}^{n}\EE[\sigma'^{(\ell)}_{j,1}\sigma'^{(\ell)}_{j,2}\sigma'^{(\ell)}_{k,3}\sigma'^{(\ell)}_{k,4}\widehat{\Delta \Theta}^{(\ell)}_{kk,34}]\Theta^{(\ell)}_{12}\nonumber\\
& -\frac{1}{n}\sum_{j,k=1}^{n}\EE[\sigma'^{(\ell)}_{j,1}\sigma'^{(\ell)}_{j,2}\widehat{\Delta \Theta}^{(\ell)}_{jj,12}\sigma'^{(\ell)}_{k,3}\sigma'^{(\ell)}_{k,4}]\Theta^{(\ell)}_{34} -\frac{1}{n}\sum_{j,k=1}^{n}\EE[\sigma'^{(\ell)}_{j,1}\sigma'^{(\ell)}_{j,2}\widehat{\Delta \Theta}^{(\ell)}_{jj,12}\sigma'^{(\ell)}_{k,3}\sigma'^{(\ell)}_{k,4}\widehat{\Delta \Theta}^{(\ell)}_{kk,34}]\nonumber\\ 
 & + \sum_{j=1}^{n}\EE[\sigma'^{(\ell)}_{j,1}\sigma'^{(\ell)}_{j,2}\sigma'^{(\ell)}_{j,3}\sigma'^{(\ell)}_{j,4}]\Theta^{(\ell)}_{12}\Theta^{(\ell)}_{34} + \sum_{j=1}^{n}\EE[\sigma'^{(\ell)}_{j,1}\sigma'^{(\ell)}_{j,2}\widehat{\Delta \Theta}^{(\ell)}_{jj,12}\sigma'^{(\ell)}_{j,3}\sigma'^{(\ell)}_{j,4}]\Theta^{(\ell)}_{34} \nonumber\\
 & + \sum_{j=1}^{n}\EE[\sigma'^{(\ell)}_{j,1}\sigma'^{(\ell)}_{j,2}\sigma'^{(\ell)}_{j,3}\sigma'^{(\ell)}_{j,4}\widehat{\Delta \Theta}^{(\ell)}_{jj,34}]\Theta^{(\ell)}_{12}  + \sum_{j,k=1}^{n}\EE[\sigma'^{(\ell)}_{j,1}\sigma'^{(\ell)}_{k,2}\widehat{\Delta \Theta}^{(\ell)}_{jk,12}\sigma'^{(\ell)}_{j,3}\sigma'^{(\ell)}_{k,4}\widehat{\Delta \Theta}^{(\ell)}_{jk,34}] \nonumber\\
 & - \frac{1}{n}\sum_{j=1}^{n}\EE[\sigma'^{(\ell)}_{j,1}\sigma'^{(\ell)}_{j,2}\sigma'^{(\ell)}_{j,3}\sigma'^{(\ell)}_{j,4}]\Theta^{(\ell)}_{12}\Theta^{(\ell)}_{34} - \frac{1}{n}\sum_{j=1}^{n}\EE[\sigma'^{(\ell)}_{j,1}\sigma'^{(\ell)}_{j,2}\widehat{\Delta \Theta}^{(\ell)}_{jj,12}\sigma'^{(\ell)}_{j,3}\sigma'^{(\ell)}_{j,4}]\Theta^{(\ell)}_{34}\nonumber\\
 & - \frac{1}{n}\sum_{j=1}^{n}\EE[\sigma'^{(\ell)}_{j,1}\sigma'^{(\ell)}_{j,2}\sigma'^{(\ell)}_{j,3}\sigma'^{(\ell)}_{j,4}\widehat{\Delta \Theta}^{(\ell)}_{jj,34}]\Theta^{(\ell)}_{12} - \frac{1}{n}\sum_{j,k=1}^{n}\EE[\sigma'^{(\ell)}_{j,1}\sigma'^{(\ell)}_{k,2}\widehat{\Delta \Theta}^{(\ell)}_{jk,12}\sigma'^{(\ell)}_{k,3}\sigma'^{(\ell)}_{j,4}\widehat{\Delta \Theta}^{(\ell)}_{kj,34}]\bigg] \nonumber\\ 
 &+ \mathcal{O}\left(\frac{1}{n^{2}}\right)\label{btensorexpr}
\end{align}
Here, we again decompose the NTK tensor into its mean $\Theta$ and fluctuations $\widehat{\Delta \Theta}$. These expressions can be further simplified by noting that not all contributions arising from the off-diagonal entries of the Weingarten matrix appear at order \(1/n\). In particular, the terms proportional to \(1/n\) inside the square brackets in~\eqref{atensorexpr} scale at most linearly with $n$, as follows from~\eqref{eq:3} and~\eqref{eq:4}, yielding an overall contribution of order $\mathcal{O}\left(\frac{1}{n^{2}}\right)$. This observation leads to the formula
\begin{align}
\frac{1}{n}A^{(\ell+1)}_{1234} &= \frac{1}{n^{2}}\sum_{j,k}\EE[\widehat{\Delta \Omega}^{(\ell+1)}_{j,12}\widehat{\Delta \Omega}^{(\ell+1)}_{j,34}] + \frac{C_{W}}{n^{2}}\sum_{j,k=1}^{n}\EE[\widehat{\Delta \Omega}_{j,12}^{(\ell+1)}\sigma'^{(\ell)}_{k,3}\sigma'^{(\ell)}_{k,4}\widehat{\Delta \Theta}^{(\ell)}_{kk,34}]\nonumber\\
  & + \frac{C_{W}}{n^{2}}\sum_{j,k=1}^{n}\EE[\widehat{\Delta \Omega}_{j,34}^{(\ell+1)}\sigma'^{(\ell)}_{k,1}\sigma'^{(\ell)}_{k,2}\widehat{\Delta \Theta}^{(\ell)}_{kk,12}] + \frac{(C_{W})^{2}}{n^{2}}\bigg[\sum_{j,k=1}^{n}\EE[\sigma'^{(\ell)}_{j,1}\sigma'^{(\ell)}_{j,2}\widehat{\Delta \Theta}^{(\ell)}_{jj,12}\sigma'^{(\ell)}_{k,3}\sigma'^{(\ell)}_{k,4}\widehat{\Delta \Theta}^{(\ell)}_{kk,34}] \nonumber\\
  & - \sum_{j,k=1}^{n}\EE[\sigma'^{(\ell)}_{j,1}\sigma'^{(\ell)}_{j,2}\widehat{\Delta \Theta}^{(\ell)}_{jj,12}]\EE[\sigma'^{(\ell)}_{k,3}\sigma'^{(\ell)}_{k,4}\widehat{\Delta \Theta}^{(\ell)}_{kk,34}]\bigg] + \mathcal{O}\left(\frac{1}{n^{2}}\right)\label{firstordera}
\end{align}  
Similarly, the only term that can yield a non-zero contribution at order $\frac{1}{n}$ in~\eqref{btensorexpr} is the first term, $\EE[\sigma'^{(\ell)}_{j,1}\sigma'^{(\ell)}_{j,2}\sigma'^{(\ell)}_{k,3}\sigma'^{(\ell)}_{k,4}]\Theta^{(\ell)}_{12}\Theta^{(\ell)}_{34}$, through its off-diagonal components. Likewise, the terms of order $n^{0}$ that vanish at order $\frac{1}{n}$ are precisely those containing a single NTK fluctuation inside the expectation value, in agreement with~\eqref{eq:3}. We therefore obtain
\begin{align}
\frac{1}{n}B^{(\ell+1)}_{1324} &= \frac{(C_{W})^{2}}{n^{2}}\bigg[-\frac{1}{n}\sum_{j,k=1}^{n}\EE[\sigma'^{(\ell)}_{j,1}\sigma'^{(\ell)}_{j,2}\sigma'^{(\ell)}_{k,3}\sigma'^{(\ell)}_{k,4}]\Theta^{(\ell)}_{12}\Theta^{(\ell)}_{34} + \sum_{j=1}^{n}\EE[\sigma'^{(\ell)}_{j,1}\sigma'^{(\ell)}_{j,2}\sigma'^{(\ell)}_{j,3}\sigma'^{(\ell)}_{j,4}]\Theta^{(\ell)}_{12}\Theta^{(\ell)}_{34} \nonumber\\
& + \sum_{j,k=1}^{n}\EE[\sigma'^{(\ell)}_{j,1}\sigma'^{(\ell)}_{k,2}\widehat{\Delta \Theta}^{(\ell)}_{jk,12}\sigma'^{(\ell)}_{j,3}\sigma'^{(\ell)}_{k,4}\widehat{\Delta \Theta}^{(\ell)}_{jk,34}] \bigg] + \mathcal{O}\left(\frac{1}{n^{2}}\right)\label{firstorderb}
\end{align}
One can apply the effective field theory techniques of~\cite{roberts2021a} to compute the first-order \(1/n\) corrections to the expectation values appearing in~\eqref{firstordera} and~\eqref{firstorderb}. Using~\eqref{eq:3} and~\eqref{eq:4}, together with extensive path-integral manipulations and algebraic identities, one finds that
\begin{eqnarray}
	A^{(\ell +1)}_{\textcolor{color1}{12}\textcolor{color2}{34}} &=& \langle \widehat{\Delta \Omega}^{(\ell + 1)}_{12}\widehat{\Delta \Omega}_{34}^{(\ell +1)}\rangle_{K^{(\ell)}} \nonumber\\
	&& + \frac{1}{4}\sum_{\beta_{1},\beta_{2},\beta_{3},\beta_{4} \in \{1,2,3,4\}}V^{(\ell)}_{(\beta_{1}\beta_{2})(\beta_{3}\beta_{4})}\Biggl\langle \frac{d^{2}(\widehat{\Delta \Omega}^{(\ell + 1)}_{12})}{d z^{(\ell)}_{\beta_{1}}d z^{(\ell)}_{\beta_{2}}}\Biggr\rangle_{K^{(\ell)}}\Biggl\langle \frac{d^{2}(\widehat{\Delta \Omega}^{(\ell +1)}_{34})}{d z^{(\ell)}_{\beta_{1}}d z^{(\ell)}_{\beta_{2}}}\Biggr\rangle_{K^{(\ell)}}\nonumber\\
	&& + \frac{C_{W}}{2}\sum_{\beta_{1},\beta_{2} \in \{1,2,3,4\}}\Biggl\langle \frac{d^{2}(\widehat{\Delta \Omega}^{(\ell + 1)}_{12})}{d z^{(\ell)}_{\beta_{1}}d z^{(\ell)}_{\beta_{2}}}\Biggr\rangle_{K^{(\ell)}} \langle \sigma'^{(\ell)}_{\textcolor{color2}{3}}\sigma'^{(\ell)}_{\textcolor{color2}{4}}\rangle_{K^{(\ell)}}D^{(\ell)}_{\beta_{1}\beta_{2}\textcolor{color2}{34}} \nonumber\\
	&& + \frac{C_{W}}{2}\sum_{\beta_{3},\beta_{4} \in \{1,2,3,4\}}\Biggl\langle \frac{d^{2}(\widehat{\Delta \Omega}^{(\ell + 1)}_{34})}{d z^{(\ell)}_{\beta_{3}}d z^{(\ell)}_{\beta_{4}}}\Biggr\rangle_{K^{(\ell)}} \langle \sigma'^{(\ell)}_{\textcolor{color1}{1}}\sigma'^{(\ell)}_{\textcolor{color1}{2}}\rangle_{K^{(\ell)}}D^{(\ell)}_{\textcolor{color1}{12}\beta_{3}\beta_{4}}\nonumber\\
	&& + (C_{W})^{2}\langle \sigma'^{(\ell)}_{\textcolor{color1}{1}}\sigma'^{(\ell)}_{\textcolor{color1}{2}}\rangle_{K^{(\ell)}} \langle \sigma'^{(\ell)}_{\textcolor{color2}{3}}\sigma'^{(\ell)}_{\textcolor{color2}{4}}\rangle_{K^{(\ell)}}A^{(\ell)}_{\textcolor{color1}{12}\textcolor{color2}{34}} + \mathcal{O}\left(\frac{1}{n}\right)\,,
    \label{eq:A_recursion_algebraic}\\
B^{(\ell +1)}_{\textcolor{color1}{1}\textcolor{color2}{3}\textcolor{color1}{2}\textcolor{color2}{4}} &=& (C_{W})^{2}\bigg[\langle \sigma'^{(\ell)}_{\textcolor{color1}{1}}\sigma'^{(\ell)}_{\textcolor{color1}{2}}\sigma'^{(\ell)}_{\textcolor{color2}{3}}\sigma'^{(\ell)}_{\textcolor{color2}{4}}\rangle_{K^{(\ell)}} - \langle\sigma'^{(\ell)}_{\textcolor{color1}{1}}\sigma'^{(\ell)}_{\textcolor{color1}{2}}\rangle_{K^{(\ell)}}\langle\sigma'^{(\ell)}_{\textcolor{color2}{3}}\sigma'^{(\ell)}_{\textcolor{color2}{4}}\rangle_{K^{(\ell)}} \bigg]\Theta^{(\ell)}_{\textcolor{color1}{12}}\Theta^{(\ell)}_{\textcolor{color2}{34}}\nonumber\\
&& + (C_{W})^{2}\langle \sigma'^{(\ell)}_{\textcolor{color1}{1}}\sigma'^{(\ell)}_{\textcolor{color2}{3}}\rangle_{K^{(\ell)}}\langle \sigma'^{(\ell)}_{\textcolor{color1}{2}}\sigma'^{(\ell)}_{\textcolor{color2}{4}}\rangle_{K^{(\ell)}}B^{(\ell)}_{\textcolor{color1}{1}\textcolor{color2}{3}\textcolor{color1}{2}\textcolor{color2}{4}} + \mathcal{O}\left(\frac{1}{n}\right)\,.
\label{eq:B_recursion_algebraic}
\end{eqnarray}
Here we observe that the recursion relation for the tensor $A$ coincides with its Gaussian counterpart, while that for the tensor $B$ acquires an additional term proportional to the NTK mean.

\section{Generalized Feynman rules}
\label{app:general_feynman_rules}
Without loss of generality, consider the channel $(12)_{i}(34)_{i}\ldots (2k\,2k-1)_{i}$, corresponding to the pairing $(12)(34)\ldots (2k\,2k-1)$. The Feynman rules reproducing tensors with this channel structure can be organized into two groups. The first group implements the orthogonality constraints as follows:
\allowdisplaybreaks
\begin{enumerate}
    \item Preactivations, NTKs, dNTKs and ddNTKs are represented by external lines, as illustrated below.
\begin{align}
    z_{\alpha} \equiv
    \begin{tikzpicture}[baseline=-0.1cm]
      \begin{feynman}
        \vertex (l) {};
        \vertex[right = 0pt of l, dot, minimum size=3pt, label = {left: {\footnotesize $\alpha^{}$}}] (x1) {};
        \vertex[right = 30pt of x1, dot, minimum size=0pt] (b) {};
        \diagram*{
          (x1),
          (x1) -- [inner sep = 4pt] (b) 
        };
      \end{feynman}
    \end{tikzpicture}
    &\qquad \widehat{\Delta \Theta}_{\textcolor{color1}{\alpha\beta}} \equiv
    \begin{tikzpicture}[baseline=0.05cm]
      \begin{feynman}
        \vertex (l) {};
        \vertex[right = 0pt of l, color1, dot, minimum size=3pt, label = {left: {\footnotesize $\textcolor{color1}{\alpha^{}}$}}] (x1) {};
        \vertex[above = 8pt of l, color1, dot, minimum size=3pt, label = {left: {\footnotesize $\textcolor{color1}{\beta^{}}$}}] (x2) {};
        \vertex[right = 30pt of x1, dot, minimum size=0pt] (b) {};
        \vertex[right = 30pt of x2, dot, minimum size=0pt] (bb) {};
        \diagram*{
          (x1), (x2),
          (x1) -- [color1, ghost, inner sep = 4pt] (b),
          (x2) -- [color1, ghost, inner sep = 4pt] (bb)
        };
      \end{feynman}
    \end{tikzpicture}\nonumber\\
    \widehat{\mathrm{d} \Theta}_{\textcolor{color3}{\delta_0}\textcolor{color1}{\delta_1}\textcolor{color2}{\delta_2}} \equiv \begin{tikzpicture}[baseline=(bb)]
    \begin{feynman}
          \vertex (l) {};
          \vertex[above = 12pt of l, color1, dot, minimum size=3pt, label = {left: {\footnotesize $\textcolor{color1}{\delta_1^{}}$}}] (x1) {};
          \vertex[above = 24pt of l, color2, dot, minimum size=3pt, label = {left: {\footnotesize $\textcolor{color2}{\delta_2^{}}$}}] (x2) {};
          \vertex[above = 0pt of l, color3, dot, minimum size=3pt, label = {left: {\footnotesize $\textcolor{color3}{\delta_0}^{}$}}] (x3) {};
          \vertex[right = 30pt of x1, dot, minimum size=0pt] (bb) {};
          \vertex[right = 30pt of x2, dot, minimum size=0pt] (bbb) {};
          \vertex[right = 30pt of x3, dot, minimum size=0pt] (b) {};
          \diagram*{
            (x1), (x2), (x3),
            (x3) -- [color1color2dntknew, inner sep = 4pt] (b),
            (x1) -- [color1, ghost, inner sep = 4pt] (bb),
            (x2) -- [color2, ghost, inner sep = 4pt] (bbb)    
          };
        \end{feynman}
      \end{tikzpicture}
   \qquad \widehat{\mathrm{dd_I} \Theta}_{\textcolor{color4}{\delta_0}\textcolor{color2}{\delta_1}\textcolor{color1}{\delta_2}\textcolor{color6}{\delta_3}} &\equiv \begin{tikzpicture}[baseline=(b)]
        \begin{feynman}
          \vertex (l) {};
          \vertex[above = 12pt of l, color2, dot, minimum size=3pt, label = {left: {\footnotesize $\textcolor{color2}{\delta_1^{}}$}}] (x1) {};
          \vertex[above = 24pt of l, color1, dot, minimum size=3pt, label = {left: {\footnotesize $\textcolor{color1}{\delta_2^{}}$}}] (x2) {};
          \vertex[above = 36pt of l, color6, dot, minimum size=3pt, label = {left: {\footnotesize $\textcolor{color6}{\delta_3^{}}$}}] (x3) {};
          \vertex[right = 0pt of l, color4, dot, minimum size=3pt, label = {left: {\footnotesize $\textcolor{color4}{\delta_0^{}}$}}] (x4) {};
          \vertex[right = 30pt of x1, dot, minimum size=0pt] (b) {};
          \vertex[right = 30pt of x2, dot, minimum size=0pt] (bb) {};
          \vertex[right = 30pt of x3, dot, minimum size=0pt] (bbb) {};
          \vertex[right = 30pt of x4, dot, minimum size=0pt] (bbbb) {};
          \diagram*{
            (x1), (x2), (x3), (x4),
            (x1) -- [color2, ghost, inner sep = 4pt] (b),
            (x2) -- [color1, ghost, inner sep = 4pt] (bb),
            (x3) -- [color6, ghost, inner sep = 4pt] (bbb),
            (x4) -- [color6color1color2ddntknew, inner sep = 4pt] (bbbb)
          };
        \end{feynman}
      \end{tikzpicture}
    \qquad \widehat{\mathrm{dd_{II}} \Theta}_{\textcolor{color5}{\delta_1}\textcolor{color7}{\delta_2}\textcolor{color1}{\delta_3}\textcolor{color2}{\delta_4}} \equiv \begin{tikzpicture}[baseline=(bbbb)]
        \begin{feynman}
          \vertex (l) {};
          \vertex[above = 24pt of l, color1, dot, minimum size=3pt, label = {left: {\footnotesize $\textcolor{color1}{\delta_3^{}}$}}] (x1) {};
          \vertex[above = 36pt of l, color2, dot, minimum size=3pt, label = {left: {\footnotesize $\textcolor{color2}{\delta_4^{}}$}}] (x2) {};
          \vertex[right = 0pt of l, color5, dot, minimum size=3pt, label = {left: {\footnotesize $\textcolor{color5}{\delta_1^{}}$}}] (x3) {};
          \vertex[above = 12pt of l, color7, dot, minimum size=3pt, label = {left: {\footnotesize $\textcolor{color7}{\delta_2^{}}$}}] (x4) {};
          \vertex[right = 30pt of x1, dot, minimum size=0pt] (b) {};
          \vertex[right = 30pt of x2, dot, minimum size=0pt] (bb) {};
          \vertex[right = 30pt of x3, dot, minimum size=0pt] (bbb) {};
          \vertex[right = 30pt of x4, dot, minimum size=0pt] (bbbb) {};
          \diagram*{
            (x1), (x2), (x3), (x4),
            (x1) -- [color1, ghost, inner sep = 4pt] (b),
            (x2) -- [color2, ghost, inner sep = 4pt] (bb),
            (x3) -- [color6color1ddntknew, inner sep = 4pt] (bbb),
            (x4) -- [color6color2ddntknew, inner sep = 4pt] (bbbb)
          };
        \end{feynman}
      \end{tikzpicture}\label{eq:6}
  \end{align}
In the first line of~\eqref{eq:6}, a colored line represents a single NTK label. Distinct colors are used for external dotted lines associated with different NTKs. In the second line, we introduce the diagrammatic representations of the dNTK: $\widehat{\mathrm{d} \Theta}_{\delta_0\delta_1\delta_2}
    =\sum_{\mu\nu}\frac{d^2 z_{\delta_0}}{d\theta_{\mu}d\theta_{\nu}}\frac{d z_{\delta_1}}{d\theta_\mu}\frac{d z_{\delta_2}}{d\theta_\nu}$, the dd$_{\text{I}}$NTK: $
    \widehat{\mathrm{dd_I} \Theta}_{\delta_0\delta_1\delta_2\delta_3}
    =\sum_{\mu\nu\rho}\frac{d^3 z_{\delta_0}}{d\theta_{\mu}d\theta_{\nu}d\theta_\rho}\frac{d z_{\delta_1}}{d\theta_\mu}\frac{d z_{\delta_2}}{d\theta_\nu}\frac{d z_{\delta_3}}{d\theta_\rho}
    $, and the dd$_{\text{II}}$NTK: $\widehat{\mathrm{dd_{II}} \Theta}_{\delta_1\delta_2\delta_3\delta_4}
    =\sum_{\mu\nu\rho}\frac{d^2 z_{\delta_1}}{d\theta_{\mu}d\theta_{\nu}}\frac{d^2 z_{\delta_2}}{d\theta_{\rho}d\theta_{\nu}}\frac{d z_{\delta_3}}{d\theta_\mu}\frac{d z_{\delta_4}}{d\theta_\rho}
    $. 
    The colors of the dotted lines correspond to the $\theta$-indices appearing in these definitions. Since the $\theta$-indices are contracted in pairs, distinct colors encode the corresponding contraction pattern. In the diagrammatic notation, triple lines represent third derivatives, double lines represent second derivatives, and single lines represent first derivatives.

\item Define the cubic vertices as
\begingroup
\allowdisplaybreaks
\begin{align}\label{orthogonalityvertex}
 \begin{tikzpicture}[baseline=(b)]
      \begin{feynman}
        \vertex (l) {};
        \vertex[below = 15pt of l, dot, minimum size=3pt, label = {left: {\footnotesize $\alpha^{c}$}}] (x1) {};
        \vertex[above = 15pt of l, dot, minimum size=3pt, label = {left: {\footnotesize $\beta^{c}$}}] (x2) {};
        \vertex[right = 10pt of l, dot, minimum size=0pt] (v12) {};
        \vertex[right = 33pt of v12, dot, minimum size=0pt] (b) {};
        \diagram*{
          (x1) --  (v12) --  (x2),
          (v12) -- [photon, edge label = {\scriptsize \;$\sigma_{i,\alpha}^{(\ell)}\sigma_{i,\beta}^{(\ell)}$}, inner sep = 4pt] (b) 
        };
      \end{feynman}
    \end{tikzpicture} &\sim C_{W} & 
    \begin{tikzpicture}[baseline=(b)]
      \begin{feynman}
        \vertex (l) {};
        \vertex[below = 15pt of l, color1, dot, minimum size=3pt, label = {left: {\footnotesize $\textcolor{color1}{\alpha^{}}$}}] (x1) {};
        \vertex[above = 15pt of l, color1, dot, minimum size=3pt, label = {left: {\footnotesize $\textcolor{color1}{\beta^{}}$}}] (x2) {};
        \vertex[right = 10pt of l, dot, minimum size=0pt] (v12) {};
        \vertex[right = 30pt of v12, dot, minimum size=0pt] (b) {};
        \diagram*{
          (x1) --  [color1, ghost] (v12) --  [color1, ghost] (x2),
          (v12) -- [black, photon, edge label = {\scriptsize \;$\sigma_{i,\alpha}^{(\ell)}\sigma_{i,\beta}^{(\ell)}$}, inner sep = 4pt] (b) 
        };
      \end{feynman}
    \end{tikzpicture} &\sim 1  &
    \begin{tikzpicture}[baseline=(b)]
      \begin{feynman}
        \vertex (l) {};
        \vertex[below = 15pt of l, color1, dot, minimum size=3pt, label = {left: {\footnotesize $\textcolor{color1}{\alpha}^{c}$}}] (x1) {};
        \vertex[above = 15pt of l, color1, dot, minimum size=3pt, label = {left: {\footnotesize $\textcolor{color1}{\beta}^{c}$}}] (x2) {};
        \vertex[right = 10pt of l, dot, minimum size=0pt] (v12) {};
        \vertex[right = 30pt of v12, dot, minimum size=0pt] (b) {};
        \diagram*{
          (x1) --  [color1, ghost] (v12) --  [color1, ghost] (x2),
          (v12) -- [black, photon, edge label = {\scriptsize \;$\Theta^{(\ell)}_{\alpha\beta}\sigma'^{(\ell)}_{i,\alpha}\sigma'^{(\ell)}_{i,\beta}$}, inner sep = 4pt] (b) 
        };
      \end{feynman}
    \end{tikzpicture} &\sim C_{W}  \nonumber\\
    \begin{tikzpicture}[baseline=(b)]
      \begin{feynman}
        \vertex (l) {};
        \vertex[below = 15pt of l, color1, dot, minimum size=3pt, label = {left: {\footnotesize $\textcolor{color1}{\alpha}^{c}$}}] (x1) {};
        \vertex[above = 15pt of l, color1, dot, minimum size=3pt, label = {left: {\footnotesize $\textcolor{color1}{\beta}^{c}$}}] (x2) {};
        \vertex[right = 10pt of l, dot, minimum size=0pt] (v12) {};
        \vertex[right = 30pt of v12, dot, minimum size=0pt] (b) {};
        \diagram*{
          (x1) --  [color1, ghost] (v12) --  [color1, ghost] (x2),
          (v12) -- [color1doubghost, edge label = {\scriptsize \;$\sigma'^{(\ell)}_{i,\textcolor{color1}{\alpha}}\sigma'^{(\ell)}_{i,\textcolor{color1}{\beta}}$}, inner sep = 4pt] (b) 
        };
      \end{feynman}
    \end{tikzpicture}&\sim C_{W}  &
    \begin{tikzpicture}[baseline=(b)]
      \begin{feynman}
        \vertex (l) {};
        \vertex[below = 15pt of l, dot, minimum size=3pt, label = {left: {\footnotesize $\alpha^{c}$}}] (x1) {};
        \vertex[above = 15pt of l, color1, dot, minimum size=3pt, label = {left: {\footnotesize $\textcolor{color1}{\beta}^{c}$}}] (x2) {};
        \vertex[right = 10pt of l, dot, minimum size=0pt] (v12) {};
        \vertex[right = 33pt of v12, dot, minimum size=0pt] (b) {};
        \diagram*{
          (x1) --  (v12) --  [color1, ghost] (x2),
          (v12) -- [color1blackghost, edge label = {\scriptsize \;$\sigma^{(\ell)}_{i,\alpha}\sigma'^{(\ell)}_{i,\textcolor{color1}{\beta}}$}, inner sep = 4pt] (b) 
        };
      \end{feynman}
    \end{tikzpicture} &\sim C_{W}  
    &\begin{tikzpicture}[baseline=(b)]
      \begin{feynman}
        \vertex (l) {};
        \vertex[below = 15pt of l, color1, dot, minimum size=3pt, label = {left: {\footnotesize $\textcolor{color1}{\alpha}^{c}$}}] (x1) {};
        \vertex[above = 15pt of l, color2, dot, minimum size=3pt, label = {left: {\footnotesize $\textcolor{color2}{\beta}^{c}$}}] (x2) {};
        \vertex[right = 10pt of l, dot, minimum size=0pt] (v12) {};
        \vertex[right = 30pt of v12, dot, minimum size=0pt] (b) {};
        \diagram*{
          (x1) --  [color1, ghost] (v12) --  [color2, ghost] (x2),
          (v12) -- [color2color1ghost, edge label = {\scriptsize \;$\sigma'^{(\ell)}_{i,\textcolor{color1}{\alpha}}\sigma'^{(\ell)}_{i,\textcolor{color2}{\beta}}$}, inner sep = 4pt] (b) 
        };
      \end{feynman}
    \end{tikzpicture} &\sim C_{W} \nonumber\\
    \begin{tikzpicture}[baseline=(b)]
      \begin{feynman}
        \vertex (l) {};
        \vertex[below = 15pt of l, color3, dot, minimum size=3pt, label = {below: {\footnotesize $\textcolor{color3}{\alpha}^{c}$}}] (x1) {};
        \vertex[above = 15pt of l, dot, minimum size=3pt, label = {above: {\footnotesize $\beta^{c}$}}] (x2) {};
        \vertex[right = 10pt of l, dot, minimum size=0pt] (v12) {};
        \vertex[right = 30pt of v12, dot, minimum size=0pt] (b) {};
        \diagram*{
          (x1) --  [color2color1dntknew] (v12) -- (x2),
          (v12) -- [color1color2ghost, edge label = {\scriptsize \;$\sigma''^{(\ell)}_{i,\textcolor{color3}{\alpha}}\sigma^{(\ell)}_{i,\beta}$}, inner sep = 4pt] (b) 
        };
      \end{feynman}
    \end{tikzpicture} &\sim C_{W} &
     \begin{tikzpicture}[baseline=(b)]
      \begin{feynman}
        \vertex (l) {};
        \vertex[below = 15pt of l, color3, dot, minimum size=3pt, label = {below: {\footnotesize $\textcolor{color3}{\alpha}^{c}$}}] (x1) {};
        \vertex[above = 15pt of l, dot, minimum size=3pt, label = {above: {\footnotesize $\beta^{c}$}}] (x2) {};
        \vertex[right = 10pt of l, dot, minimum size=0pt] (v12) {};
        \vertex[right = 33pt of v12, dot, minimum size=0pt] (b) {};
        \diagram*{
          (x1) --  [color2color1dntknew] (v12) -- (x2),
          (v12) -- [blackcolor1color2ghost, edge label = {\scriptsize \;$\sigma'^{(\ell)}_{i,\textcolor{color3}{\alpha}}\sigma^{(\ell)}_{i,\beta}$}, inner sep = 4pt] (b) 
        };
      \end{feynman}
    \end{tikzpicture} &\sim C_{W}
    &  \begin{tikzpicture}[baseline=(b)]
      \begin{feynman}
        \vertex (l) {};
        \vertex[below = 15pt of l, color3, dot, minimum size=3pt, label = {below: {\footnotesize $\textcolor{color3}{\alpha}^{}$}}] (x1) {};
        \vertex[above = 15pt of l, color1, dot, minimum size=3pt, label = {above: {\footnotesize $\textcolor{color1}{\beta^{}}$}}] (x2) {};
        \vertex[right = 10pt of l, dot, minimum size=0pt] (v12) {};
        \vertex[right = 33pt of v12, dot, minimum size=0pt] (b) {};
        \diagram*{
          (x1) --  [color2color1dntknew] (v12) -- [color1, ghost] (x2),
          (v12) -- [blackcolor2ghost, edge label = {\scriptsize \;$\sigma'^{(\ell)}_{i,\textcolor{color2}{\alpha}}\sigma^{(\ell)}_{i,\beta}$}, inner sep = 4pt] (b) 
        };
      \end{feynman}
    \end{tikzpicture} &\sim 1  \nonumber\\
    \begin{tikzpicture}[baseline=(b)]
      \begin{feynman}
        \vertex (l) {};
        \vertex[below = 15pt of l, color3, dot, minimum size=3pt, label = {below: {\footnotesize $\textcolor{color3}{\alpha}^{c}$}}] (x1) {};
        \vertex[above = 15pt of l, color1, dot, minimum size=3pt, label = {above: {\footnotesize $\textcolor{color1}{\beta}^{c}$}}] (x2) {};
        \vertex[right = 10pt of l, dot, minimum size=0pt] (v12) {};
        \vertex[right = 33pt of v12, dot, minimum size=0pt] (b) {};
        \diagram*{
          (x1) --  [color2color1dntknew] (v12) -- [color1, ghost] (x2),
          (v12) -- [blackcolor2ghost, edge label = {\scriptsize \hspace{2mm} \;$\Theta^{(\ell)}_{\textcolor{color1}{\alpha\beta}}\sigma''^{(\ell)}_{i,\textcolor{color2}{\alpha}}\sigma'^{(\ell)}_{i,\beta}$}, inner sep = 4pt] (b) 
        };
      \end{feynman}
    \end{tikzpicture} &\sim C_{W} &  
    \begin{tikzpicture}[baseline=(b)]
      \begin{feynman}
        \vertex (l) {};
        \vertex[below = 15pt of l, color3, dot, minimum size=3pt, label = {below: {\footnotesize $\textcolor{color3}{\alpha}^{c}$}}] (x1) {};
        \vertex[above = 15pt of l, color1, dot, minimum size=3pt, label = {above: {\footnotesize $\textcolor{color1}{\beta}^{c}$}}] (x2) {};
        \vertex[right = 10pt of l, dot, minimum size=0pt] (v12) {};
        \vertex[right = 33pt of v12, dot, minimum size=0pt] (b) {};
        \diagram*{
          (x1) --  [color2color1dntknew] (v12) -- [color1, ghost] (x2),
          (v12) -- [ntkdntkcolor1color2, edge label = {\scriptsize \;$\sigma'^{(\ell)}_{i,\textcolor{color3}{\alpha}}\sigma'^{(\ell)}_{i,\textcolor{color1}{\beta}}$}, inner sep = 4pt] (b) 
        };
      \end{feynman}
    \end{tikzpicture} &\sim C_{W} &
    \begin{tikzpicture}[baseline=(b)]
      \begin{feynman}
        \vertex (l) {};
        \vertex[below = 15pt of l, color7, dot, minimum size=3pt, label = {below: {\footnotesize $\textcolor{color7}{\alpha}^{c}$}}] (x1) {};
        \vertex[above = 15pt of l, color1, dot, minimum size=3pt, label = {above: {\footnotesize $\textcolor{color1}{\beta}^{c}$}}] (x2) {};
        \vertex[right = 10pt of l, dot, minimum size=0pt] (v12) {};
        \vertex[right = 33pt of v12, dot, minimum size=0pt] (b) {};
        \diagram*{
          (x1) --  [color6color2ddntknew] (v12) -- [color1, ghost] (x2),
          (v12) -- [ntkdntkblackcolor1color2color6, edge label = {\scriptsize \;$\sigma''^{(\ell)}_{i,\textcolor{color7}{\alpha}}\sigma'^{(\ell)}_{i,\textcolor{color1}{\beta}}$}, inner sep = 4pt] (b) 
        };
      \end{feynman}
    \end{tikzpicture} &\sim C_{W} \nonumber\\
     \begin{tikzpicture}[baseline=(b)]
      \begin{feynman}
        \vertex (l) {};
        \vertex[below = 15pt of l, color7, dot, minimum size=3pt, label = {below: {\footnotesize $\textcolor{color7}{\alpha}^{c}$}}] (x1) {};
        \vertex[above = 15pt of l, color1, dot, minimum size=3pt, label = {above: {\footnotesize $\textcolor{color1}{\beta}^{c}$}}] (x2) {};
        \vertex[right = 10pt of l, dot, minimum size=0pt] (v12) {};
        \vertex[right = 33pt of v12, dot, minimum size=0pt] (b) {};
        \diagram*{
          (x1) --  [color6color2ddntknew] (v12) -- [color1, ghost] (x2),
          (v12) -- [ntkdntkcolor1color2color6, edge label = {\scriptsize \;$\sigma'^{(\ell)}_{i,\textcolor{color7}{\alpha}}\sigma'^{(\ell)}_{i,\textcolor{color1}{\beta}}$}, inner sep = 4pt] (b) 
        };
      \end{feynman}
    \end{tikzpicture} &\sim C_{W} &    
     \begin{tikzpicture}[baseline=(b)]
      \begin{feynman}
        \vertex (l) {};
        \vertex[below = 15pt of l, color7, dot, minimum size=3pt, label = {below: {\footnotesize $\textcolor{color7}{\alpha^{}}$}}] (x1) {};
        \vertex[above = 15pt of l, color5, dot, minimum size=3pt, label = {above: {\footnotesize $\textcolor{color5}{\beta^{}}$}}] (x2) {};
        \vertex[right = 10pt of l, dot, minimum size=0pt] (v12) {};
        \vertex[right = 33pt of v12, dot, minimum size=0pt] (b) {};
        \diagram*{
          (x1) --  [color6color2ddntknew] (v12) -- [color6color1ddntknew] (x2),
          (v12) -- [color1color2ghost, edge label = {\scriptsize \;$\sigma'^{(\ell)}_{i,\textcolor{color2}{\alpha}}\sigma'^{(\ell)}_{i,\textcolor{color1}{\beta}}$}, inner sep = 4pt] (b) 
        };
      \end{feynman}
    \end{tikzpicture} &\sim 1 
    & \begin{tikzpicture}[baseline=(b)]
      \begin{feynman}
        \vertex (l) {};
        \vertex[below = 15pt of l, color7, dot, minimum size=3pt, label = {below: {\footnotesize $\textcolor{color7}{\alpha}^{c}$}}] (x1) {};
        \vertex[above = 15pt of l, color5, dot, minimum size=3pt, label = {above: {\footnotesize $\textcolor{color5}{\beta}^{c}$}}] (x2) {};
        \vertex[right = 10pt of l, dot, minimum size=0pt] (v12) {};
        \vertex[right = 33pt of v12, dot, minimum size=0pt] (b) {};
        \diagram*{
          (x1) --  [color6color2ddntknew] (v12) -- [color6color1ddntknew] (x2),
          (v12) -- [color1color2ghost, edge label = {\scriptsize \hspace{2mm} \;$\Theta^{(\ell)}_{\textcolor{color6}{\alpha\beta}}\sigma''^{(\ell)}_{i,\textcolor{color2}{\alpha}}\sigma''^{(\ell)}_{i,\textcolor{color1}{\beta}}$}, inner sep = 4pt] (b) 
        };
      \end{feynman}
    \end{tikzpicture} &\sim C_{W} \nonumber\\
     \begin{tikzpicture}[baseline=(b)]
      \begin{feynman}
        \vertex (l) {};
        \vertex[below = 15pt of l, color7, dot, minimum size=3pt, label = {below: {\footnotesize $\textcolor{color7}{\alpha}^{c}$}}] (x1) {};
        \vertex[above = 15pt of l, color5, dot, minimum size=3pt, label = {above: {\footnotesize $\textcolor{color5}{\beta}^{c}$}}] (x2) {};
        \vertex[right = 10pt of l, dot, minimum size=0pt] (v12) {};
        \vertex[right = 33pt of v12, dot, minimum size=0pt] (b) {};
        \diagram*{
          (x1) --  [color6color2ddntknew] (v12) -- [color6color1ddntknew] (x2),
          (v12) -- [ddntkddntkcolor1color2ghost, edge label = {\scriptsize \;$\sigma'^{(\ell)}_{i,\textcolor{color7}{\alpha}}\sigma'^{(\ell)}_{i,\textcolor{color5}{\beta}}$}, inner sep = 4pt] (b) 
        };
      \end{feynman}
    \end{tikzpicture} &\sim C_{W} & 
    \begin{tikzpicture}[baseline=(b)]
      \begin{feynman}
        \vertex (l) {};
        \vertex[below = 15pt of l, color4, dot, minimum size=3pt, label = {below: {\footnotesize $\textcolor{color4}{\alpha^{}}$}}] (x1) {};
        \vertex[above = 15pt of l, color6, dot, minimum size=3pt, label = {above: {\footnotesize $\textcolor{color6}{\beta^{}}$}}] (x2) {};
        \vertex[right = 10pt of l, dot, minimum size=0pt] (v12) {};
        \vertex[right = 33pt of v12, dot, minimum size=0pt] (b) {};
        \diagram*{
          (x1) --  [color6color1color2ddntknew] (v12) -- [color6, ghost] (x2),
          (v12) -- [color1color2ghost, edge label = {\scriptsize \;$\sigma''^{(\ell)}_{i,\textcolor{color3}{\alpha}}\sigma^{(\ell)}_{i,\beta}$}, inner sep = 4pt] (b) 
        };
      \end{feynman}
    \end{tikzpicture} &\sim 1 & 
    \begin{tikzpicture}[baseline=(b)]
      \begin{feynman}
        \vertex (l) {};
        \vertex[below = 15pt of l, color4, dot, minimum size=3pt, label = {below: {\footnotesize $\textcolor{color4}{\alpha}^{c}$}}] (x1) {};
        \vertex[above = 15pt of l, color6, dot, minimum size=3pt, label = {above: {\footnotesize $\textcolor{color6}{\beta}^{c}$}}] (x2) {};
        \vertex[right = 10pt of l, dot, minimum size=0pt] (v12) {};
        \vertex[right = 33pt of v12, dot, minimum size=0pt] (b) {};
        \diagram*{
          (x1) --  [color6color1color2ddntknew] (v12) -- [color6, ghost] (x2),
          (v12) -- [color1color2ghost, edge label = {\scriptsize \hspace{2mm} \;$\Theta^{(\ell)}_{\textcolor{color6}{\alpha\beta}}\sigma'''^{(\ell)}_{i,\textcolor{color3}{\alpha}}\sigma'^{(\ell)}_{i,\beta}$}, inner sep = 4pt] (b) 
        };
      \end{feynman}
    \end{tikzpicture} &\sim C_{W}  \nonumber\\
    \begin{tikzpicture}[baseline=(b)]
      \begin{feynman}
        \vertex (l) {};
        \vertex[below = 15pt of l, color4, dot, minimum size=3pt, label = {below: {\footnotesize $\textcolor{color4}{\alpha^{}}$}}] (x1) {};
        \vertex[above = 15pt of l, color6, dot, minimum size=3pt, label = {above: {\footnotesize $\textcolor{color6}{\beta^{}}$}}] (x2) {};
        \vertex[right = 10pt of l, dot, minimum size=0pt] (v12) {};
        \vertex[right = 33pt of v12, dot, minimum size=0pt] (b) {};
        \diagram*{
          (x1) --  [color6color1color2ddntknew] (v12) -- [color6, ghost] (x2),
          (v12) -- [blackcolor1color2ghost, edge label = {\scriptsize \;$\sigma'^{(\ell)}_{i,\textcolor{color3}{\alpha}}\sigma^{(\ell)}_{i,\beta}$}, inner sep = 4pt] (b) 
        };
      \end{feynman}
    \end{tikzpicture} &\sim 1 &
    \begin{tikzpicture}[baseline=(b)]
      \begin{feynman}
        \vertex (l) {};
        \vertex[below = 15pt of l, color4, dot, minimum size=3pt, label = {below: {\footnotesize $\textcolor{color4}{\alpha}^{c}$}}] (x1) {};
        \vertex[above = 15pt of l, color6, dot, minimum size=3pt, label = {above: {\footnotesize $\textcolor{color6}{\beta}^{c}$}}] (x2) {};
        \vertex[right = 10pt of l, dot, minimum size=0pt] (v12) {};
        \vertex[right = 33pt of v12, dot, minimum size=0pt] (b) {};
        \diagram*{
          (x1) --  [color6color1color2ddntknew] (v12) -- [color6, ghost] (x2),
          (v12) -- [blackcolor1color2ghost, edge label = {\scriptsize \hspace{1mm}\;$\Theta^{(\ell)}_{\textcolor{color6}{\alpha\beta}}\sigma''^{(\ell)}_{i,\textcolor{color3}{\alpha}}\sigma'^{(\ell)}_{i,\beta}$}, inner sep = 4pt] (b) 
        };
      \end{feynman}
    \end{tikzpicture} &\sim C_{W} & & & &
  \end{align}
  \endgroup
Here the superscript $c$, referred to as the orthogonality charge, keeps track of the label's orthogonal character, and lines that do not end in a dot represent internal lines.

\item Draw a square propagator connecting internal lines in all possible ways, consistent with the chosen pairing. The square represents the full expectation value. 
\begin{align}
\begin{tikzpicture}[baseline=(b)]
      \begin{feynman}
        \vertex (l) {};
        \vertex[right = 0pt of l, label = {above: {\textcolor{color1}{$ $}}}] (v12) {};
        \vertex[right = 32pt of v12, squareblob] (b) {};
        \vertex[right = 32pt of b, label = {above: {\textcolor{color1}{$ $}}}] (v34) {};
        \vertex[right = 8pt of v34] (r) {};
        \diagram*{
          (v12) -- [photon] (b) -- [photon] (v34)
        };
        \vertex[below = 15pt of b] {{\scriptsize $\EE[\,\cdot\,]$}};
      \end{feynman}
\end{tikzpicture}
\end{align}
This procedure generates distinct diagram types, both connected and disconnected. The connected diagrams are further classified as $s$-class diagrams, defined by the number $s$ of square propagators appearing in the diagram. The square propagator obeys an additional selection rule: all k-class diagrams formed from pairs of external lines corresponding to different object types vanish.



\item For each $s$-class diagram, generate all inequivalent permutations of its $2m$ external labels carrying orthogonality charge. Multiply each resulting diagram by \(1/n\) for every uncharged pairing, and by the appropriate $m$-class Weingarten function $\mathcal{W}$, determined by the relative ordering $\tau$ of the diagram's labels with respect to the original pairing $\pi = (12)(34)\ldots (2k\,2k-1)$: $\mathcal{W}[\tau,\pi] = \mathcal{W}[\ell (\tau \circ \pi)]$, where $\circ$ denotes ordinary permutation multiplication and $\ell(e)$ denotes the cycle length of $e$.

\item Multiply each $s$-class contribution by the M\"obius coefficient $(-1)^{s-1}(s-1)!$, and sum over all classes.
\end{enumerate}
The second group implements the effective field theory techniques developed in~\cite{roberts2021a}, applied to the square propagator in the diagrammatic construction of the previous step, through the following set of Feynman rules analogous to those introduced in~\cite{guillen2025}:
\begin{enumerate}
    \setcounter{enumi}{5} 
\item We define the bare propagator as
\begin{equation}
    \langle\hspace{3mm}\rangle_{K^{(\ell)}} \equiv
    \begin{tikzpicture}[baseline=-0.1cm]
      \begin{feynman}
        \tikzfeynmanset{every blob = {/tikz/fill=white!50, /tikz/minimum size=15pt}}
        \vertex (l) {};
        \vertex[above = 0pt of l, dot, minimum size=0pt] (v12) {};
        \vertex[above = 0pt of v12, blob] (b) {};
        \vertex[above = 0pt of b, dot, minimum size=0pt] (v34) {};
        \diagram*{
          (v12) -- (b) -- (v34), 
        };
      \end{feynman}
    \end{tikzpicture}   
  \end{equation}
  where $\langle \hspace{3mm}\rangle_{K^{(\ell)}}$ denotes a zero-mean Gaussian expectation with covariance specified by $K^{(\ell)}$. The expectation value is taken over the decorations of the internal lines attached to the propagator, which obeys the same selection rules described in~\cite{guillen2025}. These rules are summarized as follows:
\begin{enumerate}[label=(\alph*)]
\item Propagators may only connect to internal lines emanating from cubic vertices or from the internal quartic vertices introduced below. In particular, propagators cannot be directly connected to other propagators.

\item Dotted lines attached to a propagator do not enter the Gaussian expectation value, as they carry no decorations.

\item Each preactivation line decorated with $z_{i}$ acts as a derivative with respect to $z_{i}$ acting on the argument of the Gaussian expectation value.

\item The neural indices of all internal lines connected to a propagator must be identical.

\item If both dotted and dashed lines of the same color are attached to the propagator, they must appear in pairs carrying the same sample index. The two lines in each pair attach to different vertices. Moreover, if both vertices are drawn in the orientation specified in the Feynman rules, the top-to-bottom ordering of the sample indices (and therefore the colors) of the lines connected to the two vertices must coincide.

\item A pair of dashed lines of the same color connected to the propagator contributes a factor $\Theta_{\alpha\beta}$ when the two lines attach to different vertices, where $\alpha$ and $\beta$ denote the sample indices of the pair.
\end{enumerate}
  \item Quartic vertices are defined analogously, following~\cite{guillen2025}. Explicitly,
  \begin{align}
     \begin{tikzpicture}[baseline=(b)]
      \begin{feynman}
        \vertex (l) {};
        \vertex[below = 15pt of l, dot, minimum size=3pt, label = {left: {\footnotesize $\alpha_{1}$}}] (x1) {};
        \vertex[above = 15pt of l, dot, minimum size=3pt, label = {left: {\footnotesize $\alpha_{2}$}}] (x2) {};
        \vertex[right = 20pt of l, quarticblob] (b) {};
        \vertex[below = 25pt of b] {\scriptsize $\frac{1}{n}V_{\alpha_{1}\alpha_{2}\alpha_{3}\alpha_{4}}^{(\ell+1)}$}; 
        \vertex[right = 20pt of b] (r) {};
        \vertex[above = 15pt of r, dot, minimum size=3pt, label = {right: {\footnotesize $\alpha_{3}$}}] (x3) {};
        \vertex[below = 15pt of r, dot, minimum size=3pt, label = {right: {\footnotesize $\alpha_{4}$}}] (x4) {};
        \diagram*{
          (x1) -- (b) -- (x2), 
          (x3) -- (b) -- (x4)
        };
      \end{feynman}
    \end{tikzpicture}
     & &
    \begin{tikzpicture}[baseline=(b)]
      \begin{feynman}
        \vertex (l) {};
        \vertex[below = 15pt of l, dot, minimum size=3pt, label = {left: {\footnotesize $\alpha_{1}$}}] (x1) {};
        \vertex[above = 15pt of l, dot, minimum size=3pt, label = {left: {\footnotesize $\alpha_{2}$}}] (x2) {};
        \vertex[right = 20pt of l, quarticblob] (b) {};
        \vertex[below = 25pt of b] {\scriptsize $\frac{1}{n}D_{\alpha_{1}\alpha_{2}\textcolor{color1}{\alpha_{3}\alpha_{4}}}^{(\ell+1)}$}; 
        \vertex[right = 20pt of b] (r) {};
        \vertex[above = 15pt of r, dot, color1, minimum size=3pt, label = {right: {\footnotesize $\textcolor{color1}{\alpha_{3}}$}}] (x3) {};
        \vertex[below = 15pt of r, dot, color1, minimum size=3pt, label = {right: {\footnotesize $\textcolor{color1}{\alpha_{4}}$}}] (x4) {};
        \diagram*{
          (x1) -- (b) -- (x2), 
          (x3) -- [color1, ghost] (b) -- [color1, ghost] (x4)
        };
      \end{feynman}
    \end{tikzpicture}
     & &
    \begin{tikzpicture}[baseline=(b)]
      \begin{feynman}
        \vertex (l) {};
        \vertex[below = 15pt of l, dot, minimum size=3pt, label = {left: {\footnotesize $\alpha_{1}$}}] (x1) {};
        \vertex[above = 15pt of l, dot, color1, minimum size=3pt, label = {left: {\footnotesize $\textcolor{color1}{\alpha_{3}}$}}] (x2) {};
        \vertex[right = 20pt of l, quarticblob] (b) {};
        \vertex[below = 25pt of b] {\scriptsize $\frac{1}{n}F_{\alpha_{1}\textcolor{color1}{\alpha_{3}}\alpha_{2}\textcolor{color1}{\alpha_{4}}}^{(\ell+1)}$}; 
        \vertex[right = 20pt of b] (r) {};
        \vertex[above = 15pt of r, dot, minimum size=3pt, label = {right: {\footnotesize $\alpha_{2}$}}] (x3) {};
        \vertex[below = 15pt of r, dot, color1, minimum size=3pt, label = {right: {\footnotesize $\textcolor{color1}{\alpha_{4}}$}}] (x4) {};
        \diagram*{
          (x1) -- (b) -- [color1, ghost] (x2), 
          (x3) -- (b) -- [color1, ghost] (x4)
        };
      \end{feynman}
    \end{tikzpicture}
   \nonumber\\
    \begin{tikzpicture}[baseline=(b)]
      \begin{feynman}
        \vertex (l) {};
        \vertex[below = 15pt of l, dot, color2, minimum size=3pt, label = {left: {\footnotesize $\textcolor{color2}{\alpha_{1}}$}}] (x1) {};
        \vertex[above = 15pt of l, dot, color2, minimum size=3pt, label = {left: {\footnotesize $\textcolor{color2}{\alpha_{2}}$}}] (x2) {};
        \vertex[right = 20pt of l, quarticblob] (b) {};
        \vertex[below = 25pt of b] {\scriptsize $\frac{1}{n}A_{\textcolor{color2}{\alpha_{1}}\textcolor{color2}{\alpha_{2}}\textcolor{color1}{\alpha_{3}}\textcolor{color1}{\alpha_{4}}}^{(\ell+1)}$}; 
        \vertex[right = 20pt of b] (r) {};
        \vertex[above = 15pt of r, dot, color1, minimum size=3pt, label = {right: {\footnotesize $\textcolor{color1}{\alpha_{3}}$}}] (x3) {};
        \vertex[below = 15pt of r, dot, color1, minimum size=3pt, label = {right: {\footnotesize $\textcolor{color1}{\alpha_{4}}$}}] (x4) {};
        \diagram*{
          (x1) -- [color2, ghost] (b) -- [color2, ghost] (x2), 
          (x3) -- [color1, ghost] (b) -- [color1, ghost] (x4)
        };
      \end{feynman}
    \end{tikzpicture}
    & &
    \begin{tikzpicture}[baseline=(b)]
      \begin{feynman}
        \vertex (l) {};
        \vertex[below = 15pt of l, dot, color2, minimum size=3pt, label = {left: {\footnotesize $\textcolor{color2}{\alpha_{1}}$}}] (x1) {};
        \vertex[above = 15pt of l, dot, color1, minimum size=3pt, label = {left: {\footnotesize $\textcolor{color1}{\alpha_{3}}$}}] (x2) {};
        \vertex[right = 20pt of l, quarticblob] (b) {};
        \vertex[below = 25pt of b] {\scriptsize $\frac{1}{n}B_{\textcolor{color2}{\alpha_{1}}\textcolor{color1}{\alpha_{3}}\textcolor{color2}{\alpha_{2}}\textcolor{color1}{\alpha_{4}}}^{(\ell+1)}$}; 
        \vertex[right = 20pt of b] (r) {};
        \vertex[above = 15pt of r, dot, color2, minimum size=3pt, label = {right: {\footnotesize $\textcolor{color2}{\alpha_{2}}$}}] (x3) {};
        \vertex[below = 15pt of r, dot, color1, minimum size=3pt, label = {right: {\footnotesize $\textcolor{color1}{\alpha_{4}}$}}] (x4) {};
        \diagram*{
          (x1) -- [color2, ghost] (b) -- [color1, ghost] (x2), 
          (x3) -- [color2, ghost] (b) -- [color1, ghost] (x4)
        };
      \end{feynman}
    \end{tikzpicture} & &
    \begin{tikzpicture}[baseline=(b)]
      \begin{feynman}
        \vertex (l) {};
        \vertex[below = 16pt of l,color1, dot, minimum size=3pt, label = {below: {\footnotesize $\textcolor{color1}{\alpha_1^{}}$}}] (x1) {};
        \vertex[above = 16pt of l, color2, dot, minimum size=3pt, label = {above: {\footnotesize $\textcolor{color2}{\alpha_2^{}}$}}] (x2) {};
        \vertex[right = 8pt of l, dot, minimum size=0pt] (v12) {};
        \vertex[below = 25pt of v12, dot, minimum size=0pt] (v12d) {};
        \vertex[right = 8pt of v12d, dot, minimum size=0pt] (v12dd) {};
        \vertex[right = 20pt of v12, quarticblob] (b) {};
         \vertex[below = 25pt of b] {\scriptsize $\frac{1}{n}P_{\textcolor{color3}{\alpha_{3}}\textcolor{color1}{\alpha_{1}}\textcolor{color2}{\alpha_{2}}\alpha_{4}}^{(\ell+1)}$}; 
        \vertex[right = 20pt of b, dot, minimum size=0pt] (v34) {};
        \vertex[below = 25pt of v34, dot, minimum size=0pt] (v34d) {};
        \vertex[left = 8pt of v34d, dot, minimum size=0pt] (v34dd) {};
        \vertex[right = 8pt of v34] (r) {};
        \vertex[above = 16pt of r, color3, dot, minimum size=3pt, label = {above: {\footnotesize $\textcolor{color3}{\alpha_3^{}}$}}] (x3) {};
        \vertex[below = 16pt of r, dot, minimum size=3pt, label = {below: {\footnotesize $\alpha_4^{}$}}] (x4) {};
        \diagram*{
          (x1) -- [color1, ghost] (b) -- [color2, ghost] (x2), 
          (x3) -- [color1color2dntknew, ghost] (b) -- (x4)
        };
      \end{feynman}
    \end{tikzpicture}
    \nonumber\\
    \begin{tikzpicture}[baseline=(b)]
      \begin{feynman}
        \vertex (l) {};
        \vertex[below = 16pt of l, color1!40!color2, dot, minimum size=3pt, label = {below: {\footnotesize $\textcolor{color3}{\alpha_1^{}}$}}] (x1) {};
        \vertex[above = 16pt of l, color1, dot, minimum size=3pt, label = {above: {\footnotesize $\textcolor{color1}{\alpha_3^{}}$}}] (x2) {};
        \vertex[right = 8pt of l, dot, minimum size=0pt] (v12) {};
        \vertex[below = 25pt of v12, dot, minimum size=0pt] (v12d) {};
        \vertex[right = 8pt of v12d, dot, minimum size=0pt] (v12dd) {};
        \vertex[right = 20pt of v12, quarticblob] (b) {};
        \vertex[below = 25pt of b] {\scriptsize $\frac{1}{n}Q_{\textcolor{color3}{\alpha_{1}}\textcolor{color2}{\alpha_{2}}\textcolor{color1}{\alpha_{3}}\alpha_{4}}^{(\ell+1)}$}; 
        \vertex[right = 20pt of b, dot, minimum size=0pt] (v34) {};
        \vertex[below = 25pt of v34, dot, minimum size=0pt] (v34d) {};
        \vertex[left = 8pt of v34d, dot, minimum size=0pt] (v34dd) {};
        \vertex[right = 8pt of v34] (r) {};
        \vertex[above = 16pt of r, color2, dot, minimum size=3pt, label = {above: {\footnotesize $\textcolor{color2}{\alpha_2^{}}$}}] (x3) {};
        \vertex[below = 16pt of r, dot, minimum size=3pt, label = {below: {\footnotesize $\alpha_4^{}$}}] (x4) {};
        \diagram*{
          (x1) -- [color2color1dntknew] (b) -- [color1, ghost] (x2), 
          (x3) -- [color2, ghost] (b) -- (x4)
        };
      \end{feynman}
    \end{tikzpicture}
    & &
    \begin{tikzpicture}[baseline=(b)]
      \begin{feynman}
        \vertex (l) {};
        \vertex[below = 16pt of l,brown, dot, minimum size=3pt, label = {below: {\footnotesize $\textcolor{color4}{\alpha_1^{}}$}}] (x1) {};
        \vertex[above = 16pt of l, color6, dot, minimum size=3pt, label = {above: {\footnotesize $\textcolor{color6}{\alpha_2^{}}$}}] (x2) {};
        \vertex[right = 8pt of l, dot, minimum size=0pt] (v12) {};
        \vertex[below = 25pt of v12, dot, minimum size=0pt] (v12d) {};
        \vertex[right = 8pt of v12d, dot, minimum size=0pt] (v12dd) {};
        \vertex[right = 20pt of v12, quarticblob] (b) {};
        \vertex[below = 25pt of b] {\scriptsize $\frac{1}{n}R_{\textcolor{color4}{\alpha_{1}}\textcolor{color6}{\alpha_{2}}\textcolor{color2}{\alpha_{3}}\textcolor{color1}{\alpha_{4}}}^{(\ell+1)}$}; 
        \vertex[right = 20pt of b, dot, minimum size=0pt] (v34) {};
        \vertex[below = 25pt of v34, dot, minimum size=0pt] (v34d) {};
        \vertex[left = 8pt of v34d, dot, minimum size=0pt] (v34dd) {};
        \vertex[right = 8pt of v34] (r) {};
        \vertex[above = 16pt of r, color2, dot, minimum size=3pt, label = {above: {\footnotesize $\textcolor{color2}{\alpha_3^{}}$}}] (x3) {};
        \vertex[below = 16pt of r, color1, dot, minimum size=3pt, label = {below: {\footnotesize $\textcolor{color1}{\alpha_4^{}}$}}] (x4) {};
        \diagram*{
          (x1) -- [color6color1color2ddntknew] (b) -- [color6, ghost] (x2), 
          (x3) -- [color2, ghost] (b) -- [color1, ghost](x4)
        };
      \end{feynman}
    \end{tikzpicture} 
    & &
    \begin{tikzpicture}[baseline=(b)]
      \begin{feynman}
        \vertex (l) {};
        \vertex[below = 16pt of l,color7, dot, minimum size=3pt, label = {below: {\footnotesize $\textcolor{color7}{\alpha_1^{}}$}}] (x1) {};
        \vertex[above = 16pt of l, color5, dot, minimum size=3pt, label = {above: {\footnotesize $\textcolor{color5}{\alpha_2^{}}$}}] (x2) {};
        \vertex[right = 8pt of l, dot, minimum size=0pt] (v12) {};
        \vertex[below = 25pt of v12, dot, minimum size=0pt] (v12d) {};
        \vertex[right = 8pt of v12d, dot, minimum size=0pt] (v12dd) {};
        \vertex[right = 20pt of v12, quarticblob] (b) {};
        \vertex[below = 25pt of b] {\scriptsize $\frac{1}{n}S_{\textcolor{color7}{\alpha_{1}}\textcolor{color5}{\alpha_{2}}\textcolor{color2}{\alpha_{3}}\textcolor{color1}{\alpha_{4}}}^{(\ell+1)}$}; 
        \vertex[right = 20pt of b, dot, minimum size=0pt] (v34) {};
        \vertex[below = 25pt of v34, dot, minimum size=0pt] (v34d) {};
        \vertex[left = 8pt of v34d, dot, minimum size=0pt] (v34dd) {};
        \vertex[right = 8pt of v34] (r) {};
        \vertex[above = 16pt of r, color2, dot, minimum size=3pt, label = {above: {\footnotesize $\textcolor{color2}{\alpha_3^{}}$}}] (x3) {};
        \vertex[below = 16pt of r, color1, dot, minimum size=3pt, label = {below: {\footnotesize $\textcolor{color1}{\alpha_4^{}}$}}] (x4) {};
        \diagram*{
          (x1) -- [color6color2ddntknew] (b) -- [color6color1ddntknew] (x2), 
          (x3) -- [color2, ghost] (b) -- [color1, ghost] (x4)
        };
      \end{feynman}
    \end{tikzpicture}
    \nonumber\\
    \begin{tikzpicture}[baseline=(b)]
      \begin{feynman}
        \vertex (l) {};
        \vertex[below = 16pt of l, color7, dot, minimum size=3pt, label = {below: {\footnotesize $\textcolor{color7}{\alpha_1^{}}$}}] (x1) {};
        \vertex[above = 16pt of l, color2, dot, minimum size=3pt, label = {above: {\footnotesize $\textcolor{color2}{\alpha_3^{}}$}}] (x2) {};
        \vertex[right = 8pt of l, dot, minimum size=0pt] (v12) {};
        \vertex[below = 25pt of v12, dot, minimum size=0pt] (v12d) {};
        \vertex[right = 8pt of v12d, dot, minimum size=0pt] (v12dd) {};
        \vertex[right = 20pt of v12, quarticblob] (b) {};
        \vertex[below = 25pt of b] {\scriptsize $\frac{1}{n}T_{\textcolor{color7}{\alpha_{1}}\textcolor{color2}{\alpha_{3}}\textcolor{color1}{\alpha_{4}}\textcolor{color5}{\alpha_{2}}}^{(\ell+1)}$}; 
        \vertex[right = 20pt of b, dot, minimum size=0pt] (v34) {};
        \vertex[below = 25pt of v34, dot, minimum size=0pt] (v34d) {};
        \vertex[left = 8pt of v34d, dot, minimum size=0pt] (v34dd) {};
        \vertex[right = 8pt of v34] (r) {};
        \vertex[above = 16pt of r, color5, dot, minimum size=3pt, label = {above: {\footnotesize $\textcolor{color5}{\alpha_2^{}}$}}] (x3) {};
        \vertex[below = 16pt of r, color1, dot, minimum size=3pt, label = {below: {\footnotesize $\textcolor{color1}{\alpha_4^{}}$}}] (x4) {};
        \diagram*{
          (x1) -- [color6color2ddntk2new] (b) -- [color2, ghost] (x2), 
          (x3) -- [color6color1ddntk2new, ghost] (b) -- [color1, ghost] (x4)
        };
      \end{feynman}
    \end{tikzpicture} 
    & &
    \begin{tikzpicture}[baseline=(b)]
      \begin{feynman}
        \vertex (l) {};
        \vertex[below = 16pt of l, color7, dot, minimum size=3pt, label = {below: {\footnotesize $\textcolor{color7}{\alpha_1^{}}$}}] (x1) {};
        \vertex[above = 16pt of l, color1, dot, minimum size=3pt, label = {above: {\footnotesize $\textcolor{color1}{\alpha_4^{}}$}}] (x2) {};
        \vertex[right = 8pt of l, dot, minimum size=0pt] (v12) {};
        \vertex[below = 25pt of v12, dot, minimum size=0pt] (v12d) {};
        \vertex[right = 8pt of v12d, dot, minimum size=0pt] (v12dd) {};
        \vertex[right = 20pt of v12, quarticblob] (b) {};
        \vertex[below = 25pt of b] {\scriptsize $\frac{1}{n}U_{\textcolor{color7}{\alpha_{1}}\textcolor{color1}{\alpha_{4}}\textcolor{color5}{\alpha_{2}}\textcolor{color2}{\alpha_{3}}}^{(\ell+1)}$}; 
        \vertex[right = 20pt of b, dot, minimum size=0pt] (v34) {};
        \vertex[below = 25pt of v34, dot, minimum size=0pt] (v34d) {};
        \vertex[left = 8pt of v34d, dot, minimum size=0pt] (v34dd) {};
        \vertex[right = 8pt of v34] (r) {};
        \vertex[above = 16pt of r, color5, dot, minimum size=3pt, label = {above: {\footnotesize $\textcolor{color5}{\alpha_2^{}}$}}] (x3) {};
        \vertex[below = 16pt of r, color2, dot, minimum size=3pt, label = {below: {\footnotesize $\textcolor{color2}{\alpha_3^{}}$}}] (x4) {};
        \diagram*{
          (x1) -- [color6color2ddntk2new] (b) -- [color1, ghost] (x2), 
          (x3) -- [color6color1ddntk2new, ghost] (b) -- [color2, ghost] (x4)
        };
      \end{feynman}
    \end{tikzpicture} 
    \label{feynmanrulesquartic}
  \end{align}
\item Higher-order NTK and preactivation tensors are introduced via a natural generalization of the vertices in~\eqref{feynmanrulesquartic}.
\item The square propagator decomposes into all connected and disconnected diagrams built from the bare propagator, quartic vertices, and higher-order vertices, with internal lines remaining undotted. This decomposition respects the selection rules (a)-(f).
\end{enumerate}
 



\section{Feynman rules in action}
\label{app:v_b4_f4}

In this appendix, we explicitly apply the Feynman rules (1)-(9) of Section~\ref{sec:feynman-diagrams} to the tensors $V_{4}$, $F$, and $B$.

\subsection{The quartic vertex: $V_{4}$}
The Feynman rules (1)-(5) generate the following diagrams for the quartic vertex $V_{4}$:
\begin{align}
  \begin{tikzpicture}[baseline=(b)]
      \begin{feynman}
        \vertex (l) {};
        \vertex[below = 15pt of l, dot, minimum size=3pt, label = {left: {\footnotesize $1$}}] (x1) {};
        \vertex[above = 15pt of l, dot, minimum size=3pt, label = {left: {\footnotesize $2$}}] (x2) {};
        \vertex[right = 20pt of l, quarticblob] (b) {};
        \vertex[right = 20pt of b] (r) {};
        \vertex[above = 15pt of r, dot, minimum size=3pt, label = {right: {\footnotesize $3$}}] (x3) {};
        \vertex[below = 15pt of r, dot, minimum size=3pt, label = {right: {\footnotesize $4$}}] (x4) {};
        \diagram*{
          (x1) -- (b) -- (x2), 
          (x3) -- (b) -- (x4)
        };
      \end{feynman}
    \end{tikzpicture}
\!\!& = 
\mathcal{W}[1,1]\,\!\!\sum_{j,k}\!\!
\begin{tikzpicture}[baseline=(b)]
\begin{feynman}
\vertex (l) {};
\vertex[below = 15pt of l, dot, minimum size=3pt, label = {left: {\footnotesize $1^{c}$}}] (x1) {};
\vertex[above = 15pt of l, dot, minimum size=3pt, label = {left: {\footnotesize $2^{c}$}}] (x2) {};
\vertex[right = 8pt of l, dot, minimum size=0pt] (v12) {};
\vertex[right = 30pt of v12, squareblob] (b) {};
\vertex[right = 30pt of b, dot, minimum size=0pt] (v34) {};
\vertex[right = 8pt of v34] (r) {};
\vertex[above = 15pt of r, dot, minimum size=3pt, label = {right: {\footnotesize $3^{c}$}}] (x3) {};
\vertex[below = 15pt of r, dot, minimum size=3pt, label = {right: {\footnotesize $4^{c}$}}] (x4) {};
\diagram*{
	(x1) -- (v12) -- (x2),
	(v12) -- [photon, edge label = {\scriptsize \;$\sigma_{j}\sigma_{j}$}, inner sep = 4pt] (b) -- [photon, edge label = {\scriptsize \,$\sigma_{k}\sigma_{k}$}, inner sep = 4pt] (v34), 
	 (x3) -- (v34) -- (x4)
};
\end{feynman}
\end{tikzpicture} +
\mathcal{W}[2]\,\!\!\sum_{j,k}\!\!
\begin{tikzpicture}[baseline=(b)]
\begin{feynman}
\vertex (l) {};
\vertex[below = 15pt of l, dot, minimum size=3pt, label = {left: {\footnotesize $1^{c}$}}] (x1) {};
\vertex[above = 15pt of l, dot, minimum size=3pt, label = {left: {\footnotesize $3^{c}$}}] (x2) {};
\vertex[right = 8pt of l, dot, minimum size=0pt] (v12) {};
\vertex[right = 30pt of v12, squareblob] (b) {};
\vertex[right = 30pt of b, dot, minimum size=0pt] (v34) {};
\vertex[right = 8pt of v34] (r) {};
\vertex[above = 15pt of r, dot, minimum size=3pt, label = {right: {\footnotesize $2^{c}$}}] (x3) {};
\vertex[below = 15pt of r, dot, minimum size=3pt, label = {right: {\footnotesize $4^{c}$}}] (x4) {};
\diagram*{
	(x1) -- (v12) -- (x2),
	(v12) -- [photon, edge label = {\scriptsize \;$\sigma_{j}\sigma_{j}$}, inner sep = 4pt] (b) -- [photon, edge label = {\scriptsize \,$\sigma_{k}\sigma_{k}$}, inner sep = 4pt] (v34), 
	 (x3) -- (v34) -- (x4)
};
\end{feynman}
\end{tikzpicture}\nonumber\\
& + \mathcal{W}[2]\,\!\!\sum_{j,k}\!\!
\begin{tikzpicture}[baseline=(b)]
\begin{feynman}
\vertex (l) {};
\vertex[below = 15pt of l, dot, minimum size=3pt, label = {left: {\footnotesize $1^{c}$}}] (x1) {};
\vertex[above = 15pt of l, dot, minimum size=3pt, label = {left: {\footnotesize $4^{c}$}}] (x2) {};
\vertex[right = 8pt of l, dot, minimum size=0pt] (v12) {};
\vertex[right = 30pt of v12, squareblob] (b) {};
\vertex[right = 30pt of b, dot, minimum size=0pt] (v34) {};
\vertex[right = 8pt of v34] (r) {};
\vertex[above = 15pt of r, dot, minimum size=3pt, label = {right: {\footnotesize $2^{c}$}}] (x3) {};
\vertex[below = 15pt of r, dot, minimum size=3pt, label = {right: {\footnotesize $3^{c}$}}] (x4) {};
\diagram*{
	(x1) -- (v12) -- (x2),
	(v12) -- [photon, edge label = {\scriptsize \;$\sigma_{j}\sigma_{j}$}, inner sep = 4pt] (b) -- [photon, edge label = {\scriptsize \,$\sigma_{k}\sigma_{k}$}, inner sep = 4pt] (v34), 
	 (x3) -- (v34) -- (x4)
};
\end{feynman}
\end{tikzpicture}
- (\mathcal{W}[1])^{2}\sum_{j,k}\begin{tikzpicture}[baseline=(b)]
  \begin{feynman}
    \vertex (l) {};
    \vertex[below = 15pt of l, dot, minimum size=3pt, label = {left: {\footnotesize $1^{c}$}}] (x1) {};
    \vertex[above = 15pt of l, , dot, minimum size=3pt, label = {left: {\footnotesize $2^{c}$}}] (x2) {};
    \vertex[right = 8pt of l, dot, minimum size=0pt] (v12) {};
    \vertex[right = 35pt of v12, squareblob] (b12) {};
    \vertex[right = 5pt of b12, dot, minimum size=0pt] (w12) {};
    \vertex[above = 1pt of w12, label = {above: {\scriptsize \hspace{-10pt} $ $}}] (w12u) {};
    \vertex[below = 8pt of w12] (w12d) {};
    \vertex[left = 10pt of w12d] (w12dl) {};
    \tikzfeynmanset{every blob = {/tikz/fill=gray!50, /tikz/minimum size=0pt}}
    \vertex[right = 3pt of w12, blob , minimum size = 0pt] (b) {};
    \vertex[right = 3pt of b, dot, minimum size=0pt] (w34) {};
    \vertex[right = 16pt of w34, squareblob] (b34) {};
    \vertex[above = 1pt of w34, label = {above: {\scriptsize \hspace{10pt} $ $}}] (w34u) {};
    \vertex[below = 8pt of w34] (w34d) {};
    \vertex[right = 10pt of w34d] (w34dr) {};
    \vertex[right = 35pt of b34, dot, minimum size=0pt] (v34) {};
    \vertex[right = 8pt of v34] (r) {};
    \vertex[above = 15pt of r, dot, minimum size=3pt, label = {right: {\footnotesize $3^{c}$}}] (x3) {};
    \vertex[below = 15pt of r, dot, minimum size=3pt, label = {right: {\footnotesize $4^{c}$}}] (x4) {};
    \diagram*{
      (x1) -- (v12) -- (x2),
      (v12) -- [photon, edge label = {\scriptsize \;$\sigma_{j}\sigma_{j}$}, inner sep = 4pt] (b12),
      (b34) -- [photon, edge label = {\scriptsize \,$\sigma_{k}\sigma_{k}$}, inner sep = 4pt] (v34),
      (x3) --  (v34) -- (x4)
    };
  \end{feynman}
\end{tikzpicture}\label{quarticvertexsquare}
\end{align}
The first three terms in \eqref{quarticvertexsquare} are 1-class diagrams, as they are rendered fully connected by the square propagator. They arise from the three inequivalent pairings obtained by permuting the reference pairing $(12)(34)$, namely $\{(12)(34),(13)(24),(14)(23)\}$. The $k=2$ Weingarten functions are selected according to the relative ordering between the pairing $\pi$ appearing in the diagram and the reference pairing $\tau$: $\mathcal{W}[\pi,\tau] = \mathcal{W}[1,1]$ for $\pi=\tau$, $\mathcal{W}[\pi,\tau] = \mathcal{W}[2]$ for $\pi\neq\tau$. The coefficient multiplying each contribution is fixed by the Möbius formula evaluated on the single-block partition $s=1$. 

By contrast, the last diagram in \eqref{quarticvertexsquare} belongs to the 2-class, since the two cubic vertices are disconnected. Each subdiagram admits a unique pairing, yielding the $k=1$ Weingarten function $W[1]$. The negative coefficient associated with this contribution follows from the Möbius formula applied to the two-block partition $s=2$.

We now apply the second set of Feynman rules, (6)–(9), to the square propagators in~\eqref{quarticvertexsquare}. Since we are interested in the first-order $\frac{1}{n}$ correction to the quartic vertex, we approximate the Weingarten functions by their leading terms in the $\frac{1}{n}$ expansions~\eqref{w11expanded} and~\eqref{w2expanded}. Consequently,~\eqref{quarticvertexsquare} simplifies to  
\begin{align}
 + \mathcal{O}\left(\frac{1}{n^{2}}\right)\label{quarticvertexsquareapprox}
\end{align}
To evaluate~\eqref{quarticvertexsquareapprox}, we distinguish two cases. When $i=j$, the only nontrivial contribution arises from the first line of~\eqref{quarticvertexsquareapprox}, which reads
\begin{align}
     \frac{1}{n^{2}}\sum_{j}\!\!
%
\label{fourpointdiagonal}
\end{align}
When $i\neq j$, the first line of~\eqref{quarticvertexsquareapprox} yields the following diagram:
\begin{align}
 \frac{1}{n^{2}}\sum_{j_1, j_2}\!
%
\label{fourpointoffdiagonalone}
\end{align}
whereas the second line produces the following diagram:
\begin{align}
- 
\frac{1}{n^{3}}\left(
\sum_{j_1, j_2}\!
%
 \right) + \mathcal{O}\left(\frac{1}{n^{2}}\right)
\nonumber\\
\end{align}
The \(1/n\) scaling of the diagrams on the right-hand side can be verified by straightforward power counting: the summation in the first diagram yields a factor of $n$, while the summations in the remaining diagrams yield factors of $n^{2}$, which combine with the explicit prefactors to produce an overall \(1/n\) contribution. Consequently, the expression~\eqref{v4coneovern} reproduces the recursion relation for the quartic vertex shown in~\eqref{fourpointrecursion}, previously derived in~\cite{banta2023} through direct algebraic manipulations.

\subsection{Joint NTK-preactivation cumulant: $F$}

We already applied the Feynman rules of Section~\ref{sec:feynman-diagrams} to $F$ in the main text; for completeness, we repeat the derivation here.

The Feynman rules (1)-(5) produce the following diagrams for the tensor $F$:
\begin{align}
  \begin{tikzpicture}[baseline=(b)]
      \begin{feynman}
        \vertex (l) {};
        \vertex[below = 15pt of l, dot, minimum size=3pt, label = {left: {\footnotesize $1$}}] (x1) {};
        \vertex[above = 15pt of l, color1, dot, minimum size=3pt, label = {left: {\footnotesize $\textcolor{color1}{3}$}}] (x2) {};
        \vertex[right = 20pt of l, quarticblob] (b) {};
        \vertex[right = 20pt of b] (r) {};
        \vertex[above = 15pt of r, dot, minimum size=3pt, label = {right: {\footnotesize $2$}}] (x3) {};
        \vertex[below = 15pt of r, color1, dot, minimum size=3pt, label = {right: {\footnotesize $\textcolor{color1}{4}$}}] (x4) {};
        \diagram*{
          (x1) -- (b) -- [color1, ghost] (x2), 
          (x3) -- (b) -- [color1, ghost](x4)
        };
      \end{feynman}
    \end{tikzpicture}
\!\!& = 
\mathcal{W}[1,1]\,\!\!\sum_{j,k}\!\!
\begin{tikzpicture}[baseline=(b)]
\begin{feynman}
\vertex (l) {};
\vertex[below = 15pt of l, dot, minimum size=3pt, label = {left: {\footnotesize $1^{c}$}}] (x1) {};
\vertex[above = 15pt of l, color1, dot, minimum size=3pt, label = {left: {\footnotesize $\textcolor{color1}{3}^{c}$}}] (x2) {};
\vertex[right = 8pt of l, dot, minimum size=0pt] (v12) {};
\vertex[right = 30pt of v12, squareblob] (b) {};
\vertex[right = 30pt of b, dot, minimum size=0pt] (v34) {};
\vertex[right = 8pt of v34] (r) {};
\vertex[above = 15pt of r, dot, minimum size=3pt, label = {right: {\footnotesize $2^{c}$}}] (x3) {};
\vertex[below = 15pt of r, color1, dot, minimum size=3pt, label = {right: {\footnotesize $\textcolor{color1}{4}^{c}$}}] (x4) {};
\diagram*{
	(x1) -- (v12) -- [color1, ghost](x2),
	(v12) -- [color1blackghost, edge label = {\scriptsize \;$\sigma_{j}\sigma'_{j}$}, inner sep = 4pt] (b) -- [blackcolor1ghost, edge label = {\scriptsize \,$\sigma_{k}\sigma'_{k}$}, inner sep = 4pt] (v34), 
	 (x3) -- (v34) -- [color1, ghost](x4)
};
\end{feynman}
\end{tikzpicture} +
\mathcal{W}[2]\,\!\!\sum_{j,k}\!\!
\begin{tikzpicture}[baseline=(b)]
\begin{feynman}
\vertex (l) {};
\vertex[below = 15pt of l, dot, minimum size=3pt, label = {left: {\footnotesize $1^{c}$}}] (x1) {};
\vertex[above = 15pt of l, dot, minimum size=3pt, label = {left: {\footnotesize $2^{c}$}}] (x2) {};
\vertex[right = 8pt of l, dot, minimum size=0pt] (v12) {};
\vertex[right = 30pt of v12, squareblob] (b) {};
\vertex[right = 35pt of b, dot, minimum size=0pt] (v34) {};
\vertex[right = 8pt of v34] (r) {};
\vertex[above = 15pt of r, color1, dot, minimum size=3pt, label = {right: {\footnotesize $\textcolor{color1}{3}^{c}$}}] (x3) {};
\vertex[below = 15pt of r, color1, dot, minimum size=3pt, label = {right: {\footnotesize $\textcolor{color1}{4}^{c}$}}] (x4) {};
\diagram*{
	(x1) -- (v12) -- (x2),
	(v12) -- [photon, edge label = {\scriptsize \;$\sigma_{j}\sigma_{j}$}, inner sep = 4pt] (b) -- [photon, edge label = {\scriptsize \,$\Theta\, \sigma'_{k}\sigma'_{k}$}, inner sep = 4pt] (v34), 
	 (x3) -- [color1, ghost] (v34) -- [color1, ghost] (x4)
};
\end{feynman}
\end{tikzpicture}\nonumber\\
& + \mathcal{W}[2]\,\!\!\sum_{j,k}\!\!
\begin{tikzpicture}[baseline=(b)]
\begin{feynman}
\vertex (l) {};
\vertex[below = 15pt of l, dot, minimum size=3pt, label = {left: {\footnotesize $1^{c}$}}] (x1) {};
\vertex[above = 15pt of l, dot, minimum size=3pt, label = {left: {\footnotesize $2^{c}$}}] (x2) {};
\vertex[right = 8pt of l, dot, minimum size=0pt] (v12) {};
\vertex[right = 30pt of v12, squareblob] (b) {};
\vertex[right = 30pt of b, dot, minimum size=0pt] (v34) {};
\vertex[right = 8pt of v34] (r) {};
\vertex[above = 15pt of r, color1, dot, minimum size=3pt, label = {right: {\footnotesize $\textcolor{color1}{3}^{c}$}}] (x3) {};
\vertex[below = 15pt of r, color1, dot, minimum size=3pt, label = {right: {\footnotesize $\textcolor{color1}{4}^{c}$}}] (x4) {};
\diagram*{
	(x1) -- (v12) -- (x2),
	(v12) -- [photon, edge label = {\scriptsize \;$\sigma_{j}\sigma_{j}$}, inner sep = 4pt] (b) -- [color1doubghost, edge label = {\scriptsize \,$\sigma'_{k}\sigma'_{k}$}, inner sep = 4pt] (v34), 
	 (x3) -- [color1, ghost] (v34) -- [color1, ghost] (x4)
};
\end{feynman}
\end{tikzpicture}
+ \mathcal{W}[2]\,\!\!\sum_{j,k}\!\!
\begin{tikzpicture}[baseline=(b)]
\begin{feynman}
\vertex (l) {};
\vertex[below = 15pt of l, dot, minimum size=3pt, label = {left: {\footnotesize $1^{c}$}}] (x1) {};
\vertex[above = 15pt of l, color1, dot, minimum size=3pt, label = {left: {\footnotesize $\textcolor{color1}{4}^{c}$}}] (x2) {};
\vertex[right = 8pt of l, dot, minimum size=0pt] (v12) {};
\vertex[right = 30pt of v12, squareblob] (b) {};
\vertex[right = 30pt of b, dot, minimum size=0pt] (v34) {};
\vertex[right = 8pt of v34] (r) {};
\vertex[above = 15pt of r, dot, minimum size=3pt, label = {right: {\footnotesize $2^{c}$}}] (x3) {};
\vertex[below = 15pt of r, color1, dot, minimum size=3pt, label = {right: {\footnotesize $\textcolor{color1}{3}^{c}$}}] (x4) {};
\diagram*{
	(x1) -- (v12) -- [color1, ghost] (x2),
	(v12) -- [color1blackghost, edge label = {\scriptsize \;$\sigma_{j}\sigma'_{j}$}, inner sep = 4pt] (b) -- [blackcolor1ghost, edge label = {\scriptsize \,$\sigma_{k}\sigma'_{k}$}, inner sep = 4pt] (v34), 
	 (x3) -- (v34) -- [color1, ghost] (x4)
};
\end{feynman}
\end{tikzpicture}
\label{ftensorsquarepropapp}
\end{align}
The absence of 2-class diagrams in~\eqref{ftensorsquarepropapp} follows from the fact that the square propagator vanishes when acting on a pair of lines of different type. The 1-class diagrams are generated from the inequivalent permutations of the reference pairing $(13)(24)$, namely $\{(13)(24),(12)(34),(14)(23)\}$. The corresponding $k=2$ Weingarten functions multiplying the subdiagrams are determined by the relative ordering of the reference pairing $\tau$ and the pairing $\pi$ appearing in the subdiagram: $\mathcal{W}[\pi,\tau]=\mathcal{W}[1,1]$ for $\pi=\tau$, $\mathcal{W}[\pi,\tau]=\mathcal{W}[2]$ for $\pi\neq\tau$. The overall factor of 1 follows from the M\"obius relation for a single-block partition.

At first order in \(1/n\), the Weingarten functions can be approximated by their leading terms in the expansions \eqref{w11expanded} and \eqref{w2expanded}. In this manner, \eqref{ftensorsquarepropapp} simplifies to
\begin{align}
  \begin{tikzpicture}[baseline=(b)]
      \begin{feynman}
        \vertex (l) {};
        \vertex[below = 15pt of l, dot, minimum size=3pt, label = {left: {\footnotesize $1$}}] (x1) {};
        \vertex[above = 15pt of l, color1, dot, minimum size=3pt, label = {left: {\footnotesize $\textcolor{color1}{3}$}}] (x2) {};
        \vertex[right = 20pt of l, quarticblob] (b) {};
        \vertex[right = 20pt of b] (r) {};
        \vertex[above = 15pt of r, dot, minimum size=3pt, label = {right: {\footnotesize $2$}}] (x3) {};
        \vertex[below = 15pt of r, color1, dot, minimum size=3pt, label = {right: {\footnotesize $\textcolor{color1}{4}$}}] (x4) {};
        \diagram*{
          (x1) -- (b) -- [color1, ghost] (x2), 
          (x3) -- (b) -- [color1, ghost](x4)
        };
      \end{feynman}
    \end{tikzpicture}
\!\!& = 
\frac{1}{n^{2}}\,\!\!\sum_{j,k}\!\!
\begin{tikzpicture}[baseline=(b)]
\begin{feynman}
\vertex (l) {};
\vertex[below = 15pt of l, dot, minimum size=3pt, label = {left: {\footnotesize $1^{c}$}}] (x1) {};
\vertex[above = 15pt of l, color1, dot, minimum size=3pt, label = {left: {\footnotesize $\textcolor{color1}{3}^{c}$}}] (x2) {};
\vertex[right = 8pt of l, dot, minimum size=0pt] (v12) {};
\vertex[right = 30pt of v12, squareblob] (b) {};
\vertex[right = 30pt of b, dot, minimum size=0pt] (v34) {};
\vertex[right = 8pt of v34] (r) {};
\vertex[above = 15pt of r, dot, minimum size=3pt, label = {right: {\footnotesize $2^{c}$}}] (x3) {};
\vertex[below = 15pt of r, color1, dot, minimum size=3pt, label = {right: {\footnotesize $\textcolor{color1}{4}^{c}$}}] (x4) {};
\diagram*{
	(x1) -- (v12) -- [color1, ghost](x2),
	(v12) -- [color1blackghost, edge label = {\scriptsize \;$\sigma_{j}\sigma'_{j}$}, inner sep = 4pt] (b) -- [blackcolor1ghost, edge label = {\scriptsize \,$\sigma_{k}\sigma'_{k}$}, inner sep = 4pt] (v34), 
	 (x3) -- (v34) -- [color1, ghost](x4)
};
\end{feynman}
\end{tikzpicture} -\frac{1}{n^{3}}\sum_{j,k}\!\!
\begin{tikzpicture}[baseline=(b)]
\begin{feynman}
\vertex (l) {};
\vertex[below = 15pt of l, dot, minimum size=3pt, label = {left: {\footnotesize $1^{c}$}}] (x1) {};
\vertex[above = 15pt of l, dot, minimum size=3pt, label = {left: {\footnotesize $2^{c}$}}] (x2) {};
\vertex[right = 8pt of l, dot, minimum size=0pt] (v12) {};
\vertex[right = 30pt of v12, squareblob] (b) {};
\vertex[right = 35pt of b, dot, minimum size=0pt] (v34) {};
\vertex[right = 8pt of v34] (r) {};
\vertex[above = 15pt of r, color1, dot, minimum size=3pt, label = {right: {\footnotesize $\textcolor{color1}{3}^{c}$}}] (x3) {};
\vertex[below = 15pt of r, color1, dot, minimum size=3pt, label = {right: {\footnotesize $\textcolor{color1}{4}^{c}$}}] (x4) {};
\diagram*{
	(x1) -- (v12) -- (x2),
	(v12) -- [photon, edge label = {\scriptsize \;$\sigma_{j}\sigma_{j}$}, inner sep = 4pt] (b) -- [photon, edge label = {\scriptsize \,$\Theta\, \sigma'_{k}\sigma'_{k}$}, inner sep = 4pt] (v34), 
	 (x3) -- [color1, ghost] (v34) -- [color1, ghost] (x4)
};
\end{feynman}
\end{tikzpicture}\nonumber\\
& -\frac{1}{n^3}\,\!\!\sum_{j,k}\!\!
\begin{tikzpicture}[baseline=(b)]
\begin{feynman}
\vertex (l) {};
\vertex[below = 15pt of l, dot, minimum size=3pt, label = {left: {\footnotesize $1^{c}$}}] (x1) {};
\vertex[above = 15pt of l, dot, minimum size=3pt, label = {left: {\footnotesize $2^{c}$}}] (x2) {};
\vertex[right = 8pt of l, dot, minimum size=0pt] (v12) {};
\vertex[right = 30pt of v12, squareblob] (b) {};
\vertex[right = 30pt of b, dot, minimum size=0pt] (v34) {};
\vertex[right = 8pt of v34] (r) {};
\vertex[above = 15pt of r, color1, dot, minimum size=3pt, label = {right: {\footnotesize $\textcolor{color1}{3}^{c}$}}] (x3) {};
\vertex[below = 15pt of r, color1, dot, minimum size=3pt, label = {right: {\footnotesize $\textcolor{color1}{4}^{c}$}}] (x4) {};
\diagram*{
	(x1) -- (v12) -- (x2),
	(v12) -- [photon, edge label = {\scriptsize \;$\sigma_{j}\sigma_{j}$}, inner sep = 4pt] (b) -- [color1doubghost, edge label = {\scriptsize \,$\sigma'_{k}\sigma'_{k}$}, inner sep = 4pt] (v34), 
	 (x3) -- [color1, ghost] (v34) -- [color1, ghost] (x4)
};
\end{feynman}
\end{tikzpicture} -\frac{1}{n^{3}}\,\!\!\sum_{j,k}\!\!
\begin{tikzpicture}[baseline=(b)]
\begin{feynman}
\vertex (l) {};
\vertex[below = 15pt of l, dot, minimum size=3pt, label = {left: {\footnotesize $1^{c}$}}] (x1) {};
\vertex[above = 15pt of l, color1, dot, minimum size=3pt, label = {left: {\footnotesize $\textcolor{color1}{4}^{c}$}}] (x2) {};
\vertex[right = 8pt of l, dot, minimum size=0pt] (v12) {};
\vertex[right = 30pt of v12, squareblob] (b) {};
\vertex[right = 30pt of b, dot, minimum size=0pt] (v34) {};
\vertex[right = 8pt of v34] (r) {};
\vertex[above = 15pt of r, dot, minimum size=3pt, label = {right: {\footnotesize $2^{c}$}}] (x3) {};
\vertex[below = 15pt of r, color1, dot, minimum size=3pt, label = {right: {\footnotesize $\textcolor{color1}{3}^{c}$}}] (x4) {};
\diagram*{
	(x1) -- (v12) -- [color1, ghost] (x2),
	(v12) -- [color1blackghost, edge label = {\scriptsize \;$\sigma_{j}\sigma'_{j}$}, inner sep = 4pt] (b) -- [blackcolor1ghost, edge label = {\scriptsize \,$\sigma_{k}\sigma'_{k}$}, inner sep = 4pt] (v34), 
	 (x3) -- (v34) -- [color1, ghost] (x4)
};
\end{feynman}
\end{tikzpicture} + \mathcal{O}\left(\frac{1}{n^{2}}\right)
\label{ftensorsquarepropexpandedapp}
\end{align}
We now apply the second set of Feynman rules (6)–(9). As an immediate consequence, the last terms in~\eqref{ftensorsquarepropexpandedapp} vanish at order \(1/n\). A nonzero contribution would require the corresponding diagrams to produce a factor of $n^{2}$, which occurs when $j\neq k$ and a square propagator decomposes into two bare propagators forming disconnected subdiagrams. However, such configurations are forbidden by the selection rules (a)–(f), as each subdiagram violates color conservation. Consequently, we are left with
\begin{align}
  \begin{tikzpicture}[baseline=(b)]
      \begin{feynman}
        \vertex (l) {};
        \vertex[below = 15pt of l, dot, minimum size=3pt, label = {left: {\footnotesize $1$}}] (x1) {};
        \vertex[above = 15pt of l, color1, dot, minimum size=3pt, label = {left: {\footnotesize $\textcolor{color1}{3}$}}] (x2) {};
        \vertex[right = 20pt of l, quarticblob] (b) {};
        \vertex[right = 20pt of b] (r) {};
        \vertex[above = 15pt of r, dot, minimum size=3pt, label = {right: {\footnotesize $2$}}] (x3) {};
        \vertex[below = 15pt of r, color1, dot, minimum size=3pt, label = {right: {\footnotesize $\textcolor{color1}{4}$}}] (x4) {};
        \diagram*{
          (x1) -- (b) -- [color1, ghost] (x2), 
          (x3) -- (b) -- [color1, ghost](x4)
        };
      \end{feynman}
    \end{tikzpicture}
\!\!& = 
\frac{1}{n^{2}}\,\!\!\sum_{j,k}\!\!
\begin{tikzpicture}[baseline=(b)]
\begin{feynman}
\vertex (l) {};
\vertex[below = 15pt of l, dot, minimum size=3pt, label = {left: {\footnotesize $1^{c}$}}] (x1) {};
\vertex[above = 15pt of l, color1, dot, minimum size=3pt, label = {left: {\footnotesize $\textcolor{color1}{3}^{c}$}}] (x2) {};
\vertex[right = 8pt of l, dot, minimum size=0pt] (v12) {};
\vertex[right = 30pt of v12, squareblob] (b) {};
\vertex[right = 30pt of b, dot, minimum size=0pt] (v34) {};
\vertex[right = 8pt of v34] (r) {};
\vertex[above = 15pt of r, dot, minimum size=3pt, label = {right: {\footnotesize $2^{c}$}}] (x3) {};
\vertex[below = 15pt of r, color1, dot, minimum size=3pt, label = {right: {\footnotesize $\textcolor{color1}{4}^{c}$}}] (x4) {};
\diagram*{
	(x1) -- (v12) -- [color1, ghost](x2),
	(v12) -- [color1blackghost, edge label = {\scriptsize \;$\sigma_{j}\sigma'_{j}$}, inner sep = 4pt] (b) -- [blackcolor1ghost, edge label = {\scriptsize \,$\sigma_{k}\sigma'_{k}$}, inner sep = 4pt] (v34), 
	 (x3) -- (v34) -- [color1, ghost](x4)
};
\end{feynman}
\end{tikzpicture} -\frac{1}{n^3}\,\!\!\sum_{j,k}\!\!
\begin{tikzpicture}[baseline=(b)]
\begin{feynman}
\vertex (l) {};
\vertex[below = 15pt of l, dot, minimum size=3pt, label = {left: {\footnotesize $1^{c}$}}] (x1) {};
\vertex[above = 15pt of l, dot, minimum size=3pt, label = {left: {\footnotesize $2^{c}$}}] (x2) {};
\vertex[right = 8pt of l, dot, minimum size=0pt] (v12) {};
\vertex[right = 30pt of v12, squareblob] (b) {};
\vertex[right = 35pt of b, dot, minimum size=0pt] (v34) {};
\vertex[right = 8pt of v34] (r) {};
\vertex[above = 15pt of r, color1, dot, minimum size=3pt, label = {right: {\footnotesize $\textcolor{color1}{3}^{c}$}}] (x3) {};
\vertex[below = 15pt of r, color1, dot, minimum size=3pt, label = {right: {\footnotesize $\textcolor{color1}{4}^{c}$}}] (x4) {};
\diagram*{
	(x1) -- (v12) -- (x2),
	(v12) -- [photon, edge label = {\scriptsize \;$\sigma_{j}\sigma_{j}$}, inner sep = 4pt] (b) -- [photon, edge label = {\scriptsize \,$\Theta\,\sigma'_{k}\sigma'_{k}$}, inner sep = 4pt] (v34), 
	 (x3) -- [color1, ghost] (v34) -- [color1, ghost] (x4)
};
\end{feynman}
\end{tikzpicture}+ \mathcal{O}\left(\frac{1}{n^{2}}\right)\nonumber\\
\label{ftensorsquarepropexpandedremovedapp}
\end{align}
We now expand the diagrams in~\eqref{ftensorsquarepropexpandedremovedapp} in terms of the bare propagator and quartic vertices by analyzing the diagonal and off-diagonal neural components of each subdiagram, together with the selection rules (a)-(f). When $i=j$, the first diagram in~\eqref{ftensorsquarepropexpandedremovedapp} contributes as
\begin{align}
\frac{1}{n^{2}}\sum_{j}
\begin{tikzpicture}[baseline=(b)]
\tikzfeynmanset{every blob = {/tikz/fill=white!50, /tikz/minimum size=15pt}}
\begin{feynman}
\vertex (l) {};
\vertex[below = 15pt of l, dot, minimum size=3pt, label = {left: {\footnotesize $1$}}] (x1) {};
\vertex[above = 15pt of l, color1, dot, minimum size=3pt, label = {left: {\footnotesize $\textcolor{color1}{3}$}}] (x2) {};
\vertex[right = 8pt of l, dot, minimum size=0pt] (v12) {};
\vertex[right = 30pt of v12, blob] (b) {};
\vertex[right = 30pt of b, dot, minimum size=0pt] (v34) {};
\vertex[right = 8pt of v34] (r) {};
\vertex[above = 15pt of r, dot, minimum size=3pt, label = {right: {\footnotesize $2$}}] (x3) {};
\vertex[below = 15pt of r, color1, dot, minimum size=3pt, label = {right: {\footnotesize $\textcolor{color1}{4}$}}] (x4) {};
\diagram*{
	(x1) -- (v12) -- [color1, ghost] (x2),
	(v12) -- [color1blackghost, edge label = {\scriptsize \;$\sigma_{j}\sigma'_{j}$}, inner sep = 4pt] (b) -- [blackcolor1ghost, edge label = {\scriptsize \,$\sigma_{j}\sigma'_{j}$}, inner sep = 4pt] (v34), 
	 (x3) -- (v34) -- [color1, ghost] (x4)
};
\end{feynman}
\end{tikzpicture}\label{ftensorbarepropijsameapp}
\end{align}
When $i\neq j$, the first diagram reduces to
\begin{align}
\frac{1}{n^{2}}\sum_{j_1, j_2}\!
\begin{tikzpicture}[baseline=(b)]
  \tikzfeynmanset{every blob = {/tikz/fill=white!50, /tikz/minimum size=15pt}}
  \begin{feynman}
    \vertex (l) {};
    \vertex[below = 15pt of l, dot, minimum size=3pt, label = {left: {\footnotesize $1$}}] (x1) {};
    \vertex[above = 15pt of l, color1, dot, minimum size=3pt, label = {left: {\footnotesize $\textcolor{color1}{3}$}}] (x2) {};
    \vertex[right = 8pt of l, dot, minimum size=0pt] (v12) {};
    \vertex[right = 35pt of v12, blob] (b12) {};
    \vertex[right = 16pt of b12, dot, minimum size=0pt] (w12) {};
    \vertex[above = 1pt of w12, label = {above: {\scriptsize \hspace{-10pt} $z_{j_1}$}}] (w12u) {};
    \vertex[below = 8pt of w12] (w12d) {};
    \vertex[left = 10pt of w12d] (w12dl) {};
    \tikzfeynmanset{every blob = {/tikz/fill=gray!50, /tikz/minimum size=15pt}}
    \vertex[right = 3pt of w12, blob , minimum size = 6pt] (b) {};
    \vertex[right = 3pt of b, dot, minimum size=0pt] (w34) {};
    \tikzfeynmanset{every blob = {/tikz/fill=white!50, /tikz/minimum size=15pt}}
    \vertex[right = 16pt of w34, blob] (b34) {};
    \vertex[above = 1pt of w34, label = {above: {\scriptsize \hspace{10pt} $z_{j_2}$}}] (w34u) {};
    \vertex[below = 8pt of w34] (w34d) {};
    \vertex[right = 10pt of w34d] (w34dr) {};
    \vertex[right = 35pt of b34, dot, minimum size=0pt] (v34) {};
    \vertex[right = 8pt of v34] (r) {};
    \vertex[above = 15pt of r, dot, minimum size=3pt, label = {right: {\footnotesize $2$}}] (x3) {};
    \vertex[below = 15pt of r, color1, dot, minimum size=3pt, label = {right: {\footnotesize $\textcolor{color1}{4}$}}] (x4) {};
    \diagram*{
      (x1) -- (v12) -- [color1, ghost] (x2),
      (v12) -- [color1blackghost, edge label = {\scriptsize \;$\sigma_{j_{1}}\sigma'_{j_{1}}$}, inner sep = 4pt] (b12) -- [color1, ghost, quarter left] (w12) -- [quarter left] (b12),
      (b34) -- [color1, ghost, quarter left] (w34) -- [quarter left] (b34) -- [blackcolor1ghost, edge label = {\scriptsize \,$\sigma_{j_{2}}\sigma'_{j_{2}}$}, inner sep = 4pt] (v34),
      (x3) --  (v34) -- [color1, ghost] (x4)
    };
    \draw [decoration={brace}, decorate] (w34dr) -- (w12dl) node [pos=0.5, below = 1pt] {\scriptsize $\frac{1}{n}F_4^{(\ell)}$};
  \end{feynman}
\end{tikzpicture}\label{ftensorbarepropijdiffoneapp}
\end{align}
whereas the second diagram becomes
\begin{align}
-\frac{1}{n^3}
\sum_{j_1, j_2}\!
\begin{tikzpicture}[baseline=(b)]
  \tikzfeynmanset{every blob = {/tikz/fill=white!50, /tikz/minimum size=15pt}}
  \begin{feynman}
    \vertex (l) {};
    \vertex[below = 15pt of l, dot, minimum size=3pt, label = {left: {\footnotesize $1$}}] (x1) {};
    \vertex[above = 15pt of l, , dot, minimum size=3pt, label = {left: {\footnotesize $2$}}] (x2) {};
    \vertex[right = 8pt of l, dot, minimum size=0pt] (v12) {};
    \vertex[right = 35pt of v12, blob] (b12) {};
    \vertex[right = 5pt of b12, dot, minimum size=0pt] (w12) {};
    \vertex[above = 1pt of w12, label = {above: {\scriptsize \hspace{-10pt} $ $}}] (w12u) {};
    \vertex[below = 8pt of w12] (w12d) {};
    \vertex[left = 10pt of w12d] (w12dl) {};
    \tikzfeynmanset{every blob = {/tikz/fill=gray!50, /tikz/minimum size=0pt}}
    \vertex[right = 3pt of w12, blob , minimum size = 0pt] (b) {};
    \vertex[right = 3pt of b, dot, minimum size=0pt] (w34) {};
    \tikzfeynmanset{every blob = {/tikz/fill=white!50, /tikz/minimum size=15pt}}
    \vertex[right = 16pt of w34, blob] (b34) {};
    \vertex[above = 1pt of w34, label = {above: {\scriptsize \hspace{10pt} $ $}}] (w34u) {};
    \vertex[below = 8pt of w34] (w34d) {};
    \vertex[right = 10pt of w34d] (w34dr) {};
    \vertex[right = 35pt of b34, dot, minimum size=0pt] (v34) {};
    \vertex[right = 8pt of v34] (r) {};
    \vertex[above = 15pt of r, color1, dot, minimum size=3pt, label = {right: {\footnotesize $\textcolor{color1}{3}$}}] (x3) {};
    \vertex[below = 15pt of r, color1, dot, minimum size=3pt, label = {right: {\footnotesize $\textcolor{color1}{4}$}}] (x4) {};
    \diagram*{
      (x1) -- (v12) -- (x2),
      (v12) -- [photon, edge label = {\scriptsize \;$\sigma_{j_{1}}\sigma_{j_{1}}$}, inner sep = 4pt] (b12),
      (b34) -- [photon, edge label = {\scriptsize \,$\Theta\,\sigma'_{j_{2}}\sigma'_{j_{2}}$}, inner sep = 4pt] (v34),
      (x3) --  [color1, ghost] (v34) -- [color1, ghost] (x4)
    };
  \end{feynman}
\end{tikzpicture}\label{ftensorbarepropijdifftwoapp}
\end{align}
Substituting~\eqref{ftensorbarepropijsameapp}, \eqref{ftensorbarepropijdiffoneapp}, and~\eqref{ftensorbarepropijdifftwoapp} into~\eqref{ftensorsquarepropexpandedremovedapp}, we obtain
\begin{align}
  \begin{tikzpicture}[baseline=(b)]
      \begin{feynman}
        \vertex (l) {};
        \vertex[below = 15pt of l, dot, minimum size=3pt, label = {left: {\footnotesize $1$}}] (x1) {};
        \vertex[above = 15pt of l, dot, color1, minimum size=3pt, label = {left: {\footnotesize $\textcolor{color1}{3}$}}] (x2) {};
        \vertex[right = 20pt of l, quarticblob] (b) {};
        \vertex[right = 20pt of b] (r) {};
        \vertex[above = 15pt of r, dot, minimum size=3pt, label = {right: {\footnotesize $2$}}] (x3) {};
        \vertex[below = 15pt of r, dot, color1, minimum size=3pt, label = {right: {\footnotesize $\textcolor{color1}{4}$}}] (x4) {};
        \diagram*{
          (x1) -- (b) -- [color1, ghost] (x2), 
          (x3) -- (b) -- [color1, ghost] (x4)
        };
      \end{feynman}
    \end{tikzpicture}
\!\!&= 
\frac{1}{n^{2}}\sum_{j}
\begin{tikzpicture}[baseline=(b)]
\tikzfeynmanset{every blob = {/tikz/fill=white!50, /tikz/minimum size=15pt}}
\begin{feynman}
\vertex (l) {};
\vertex[below = 15pt of l, dot, minimum size=3pt, label = {left: {\footnotesize $1$}}] (x1) {};
\vertex[above = 15pt of l, color1, dot, minimum size=3pt, label = {left: {\footnotesize $\textcolor{color1}{3}$}}] (x2) {};
\vertex[right = 8pt of l, dot, minimum size=0pt] (v12) {};
\vertex[right = 30pt of v12, blob] (b) {};
\vertex[right = 30pt of b, dot, minimum size=0pt] (v34) {};
\vertex[right = 8pt of v34] (r) {};
\vertex[above = 15pt of r, dot, minimum size=3pt, label = {right: {\footnotesize $2$}}] (x3) {};
\vertex[below = 15pt of r, color1, dot, minimum size=3pt, label = {right: {\footnotesize $\textcolor{color1}{4}$}}] (x4) {};
\diagram*{
	(x1) -- (v12) -- [color1, ghost] (x2),
	(v12) -- [color1blackghost, edge label = {\scriptsize \;$\sigma_{j}\sigma'_{j}$}, inner sep = 4pt] (b) -- [blackcolor1ghost, edge label = {\scriptsize \,$\sigma_{j}\sigma'_{j}$}, inner sep = 4pt] (v34), 
	 (x3) -- (v34) -- [color1, ghost] (x4)
};
\end{feynman}
\end{tikzpicture}
+ \frac{1}{n^{2}}\sum_{j_1, j_2}\!
\begin{tikzpicture}[baseline=(b)]
  \tikzfeynmanset{every blob = {/tikz/fill=white!50, /tikz/minimum size=15pt}}
  \begin{feynman}
    \vertex (l) {};
    \vertex[below = 15pt of l, dot, minimum size=3pt, label = {left: {\footnotesize $1$}}] (x1) {};
    \vertex[above = 15pt of l, color1, dot, minimum size=3pt, label = {left: {\footnotesize $\textcolor{color1}{3}$}}] (x2) {};
    \vertex[right = 8pt of l, dot, minimum size=0pt] (v12) {};
    \vertex[right = 35pt of v12, blob] (b12) {};
    \vertex[right = 16pt of b12, dot, minimum size=0pt] (w12) {};
    \vertex[above = 1pt of w12, label = {above: {\scriptsize \hspace{-10pt} $z_{j_1}$}}] (w12u) {};
    \vertex[below = 8pt of w12] (w12d) {};
    \vertex[left = 10pt of w12d] (w12dl) {};
    \tikzfeynmanset{every blob = {/tikz/fill=gray!50, /tikz/minimum size=15pt}}
    \vertex[right = 3pt of w12, blob , minimum size = 6pt] (b) {};
    \vertex[right = 3pt of b, dot, minimum size=0pt] (w34) {};
    \tikzfeynmanset{every blob = {/tikz/fill=white!50, /tikz/minimum size=15pt}}
    \vertex[right = 16pt of w34, blob] (b34) {};
    \vertex[above = 1pt of w34, label = {above: {\scriptsize \hspace{10pt} $z_{j_2}$}}] (w34u) {};
    \vertex[below = 8pt of w34] (w34d) {};
    \vertex[right = 10pt of w34d] (w34dr) {};
    \vertex[right = 35pt of b34, dot, minimum size=0pt] (v34) {};
    \vertex[right = 8pt of v34] (r) {};
    \vertex[above = 15pt of r, dot, minimum size=3pt, label = {right: {\footnotesize $2$}}] (x3) {};
    \vertex[below = 15pt of r, color1, dot, minimum size=3pt, label = {right: {\footnotesize $\textcolor{color1}{4}$}}] (x4) {};
    \diagram*{
      (x1) -- (v12) -- [color1, ghost] (x2),
      (v12) -- [color1blackghost, edge label = {\scriptsize \;$\sigma_{j_{1}}\sigma'_{j_{1}}$}, inner sep = 4pt] (b12) -- [color1, ghost, quarter left] (w12) -- [quarter left] (b12),
      (b34) -- [color1, ghost, quarter left] (w34) -- [quarter left] (b34) -- [blackcolor1ghost, edge label = {\scriptsize \,$\sigma_{j_{2}}\sigma'_{j_{2}}$}, inner sep = 4pt] (v34),
      (x3) --  (v34) -- [color1, ghost] (x4)
    };
    \draw [decoration={brace}, decorate] (w34dr) -- (w12dl) node [pos=0.5, below = 1pt] {\scriptsize $\frac{1}{n}F_4^{(\ell)}$};
  \end{feynman}
\end{tikzpicture}\nonumber\\
&- 
\frac{1}{n^{3}}
\sum_{j_1, j_2}\!
\begin{tikzpicture}[baseline=(b)]
  \tikzfeynmanset{every blob = {/tikz/fill=white!50, /tikz/minimum size=15pt}}
  \begin{feynman}
    \vertex (l) {};
    \vertex[below = 15pt of l, dot, minimum size=3pt, label = {left: {\footnotesize $1$}}] (x1) {};
    \vertex[above = 15pt of l, , dot, minimum size=3pt, label = {left: {\footnotesize $2$}}] (x2) {};
    \vertex[right = 8pt of l, dot, minimum size=0pt] (v12) {};
    \vertex[right = 35pt of v12, blob] (b12) {};
    \vertex[right = 5pt of b12, dot, minimum size=0pt] (w12) {};
    \vertex[above = 1pt of w12, label = {above: {\scriptsize \hspace{-10pt} $ $}}] (w12u) {};
    \vertex[below = 8pt of w12] (w12d) {};
    \vertex[left = 10pt of w12d] (w12dl) {};
    \tikzfeynmanset{every blob = {/tikz/fill=gray!50, /tikz/minimum size=0pt}}
    \vertex[right = 3pt of w12, blob , minimum size = 0pt] (b) {};
    \vertex[right = 3pt of b, dot, minimum size=0pt] (w34) {};
    \tikzfeynmanset{every blob = {/tikz/fill=white!50, /tikz/minimum size=15pt}}
    \vertex[right = 16pt of w34, blob] (b34) {};
    \vertex[above = 1pt of w34, label = {above: {\scriptsize \hspace{10pt} $ $}}] (w34u) {};
    \vertex[below = 8pt of w34] (w34d) {};
    \vertex[right = 10pt of w34d] (w34dr) {};
    \vertex[right = 35pt of b34, dot, minimum size=0pt] (v34) {};
    \vertex[right = 8pt of v34] (r) {};
    \vertex[above = 15pt of r, color1, dot, minimum size=3pt, label = {right: {\footnotesize $\textcolor{color1}{3}$}}] (x3) {};
    \vertex[below = 15pt of r, color1, dot, minimum size=3pt, label = {right: {\footnotesize $\textcolor{color1}{4}$}}] (x4) {};
    \diagram*{
      (x1) -- (v12) -- (x2),
      (v12) -- [photon, edge label = {\scriptsize \;$\sigma_{j_{1}}\sigma_{j_{1}}$}, inner sep = 4pt] (b12),
      (b34) -- [photon, edge label = {\scriptsize \,$\Theta\,\sigma'_{j_{2}}\sigma'_{j_{2}}$}, inner sep = 4pt] (v34),
      (x3) --  [color1, ghost] (v34) -- [color1, ghost] (x4)
    };
  \end{feynman}
\end{tikzpicture} + \mathcal{O}\left(\frac{1}{n^2}\right)\label{ftensoratorderonenapp}
\end{align}
The \(1/n\) scaling of the diagrams follows from straightforward power counting of the respective summations: When $i=j$, the summation contributes factor of $n$, whereas when $i\neq j$ it yields a factor of $n^2$. Combined with the explicit coefficients on the right-hand side of~\eqref{ftensoratorderonenapp}, these factors produce an overall contribution of order \(1/n^2\). It is straightforward to verify that the first and second lines on the RHS of~\eqref{ftensoratorderonenapp} reproduce the corresponding lines on the RHS of~\eqref{eq:F}.

\subsection{NTK variance: $B$}
The Feynman rules (1)-(5) yield the following diagrams for the tensor $B$:
\begin{align}
  \begin{tikzpicture}[baseline=(b)]
      \begin{feynman}
        \vertex (l) {};
        \vertex[below = 15pt of l, color2, dot, minimum size=3pt, label = {left: {\footnotesize $\textcolor{color2}{1}$}}] (x1) {};
        \vertex[above = 15pt of l, color1, dot, minimum size=3pt, label = {left: {\footnotesize $\textcolor{color1}{3}$}}] (x2) {};
        \vertex[right = 20pt of l, quarticblob] (b) {};
        \vertex[right = 20pt of b] (r) {};
        \vertex[above = 15pt of r, color2, dot, minimum size=3pt, label = {right: {\footnotesize $\textcolor{color2}{2}$}}] (x3) {};
        \vertex[below = 15pt of r, color1, dot, minimum size=3pt, label = {right: {\footnotesize $\textcolor{color1}{4}$}}] (x4) {};
        \diagram*{
          (x1) -- [color2, ghost] (b) -- [color1, ghost] (x2), 
          (x3) -- [color2, ghost] (b) -- [color1, ghost](x4)
        };
      \end{feynman}
    \end{tikzpicture}
\!\!& = 
\mathcal{W}[1,1]\,\!\!\sum_{j,k}\!\!
\begin{tikzpicture}[baseline=(b)]
\begin{feynman}
\vertex (l) {};
\vertex[below = 15pt of l, color2, dot, minimum size=3pt, label = {left: {\footnotesize $\textcolor{color2}{1}^{c}$}}] (x1) {};
\vertex[above = 15pt of l, color1, dot, minimum size=3pt, label = {left: {\footnotesize $\textcolor{color1}{3}^{c}$}}] (x2) {};
\vertex[right = 8pt of l, dot, minimum size=0pt] (v12) {};
\vertex[right = 30pt of v12, squareblob] (b) {};
\vertex[right = 30pt of b, dot, minimum size=0pt] (v34) {};
\vertex[right = 8pt of v34] (r) {};
\vertex[above = 15pt of r, color2, dot, minimum size=3pt, label = {right: {\footnotesize $\textcolor{color2}{2}^{c}$}}] (x3) {};
\vertex[below = 15pt of r, color1, dot, minimum size=3pt, label = {right: {\footnotesize $\textcolor{color1}{4}^{c}$}}] (x4) {};
\diagram*{
	(x1) -- [color2, ghost] (v12) -- [color1, ghost](x2),
	(v12) -- [color1color2ghost, edge label = {\scriptsize \;$\sigma'_{j}\sigma'_{j}$}, inner sep = 4pt] (b) -- [color2color1ghost, edge label = {\scriptsize \,$\sigma'_{k}\sigma'_{k}$}, inner sep = 4pt] (v34), 
	 (x3) -- [color2, ghost] (v34) -- [color1, ghost](x4)
};
\end{feynman}
\end{tikzpicture} +
\mathcal{W}[2]\,\!\!\sum_{j,k}\!\!
\begin{tikzpicture}[baseline=(b)]
\begin{feynman}
\vertex (l) {};
\vertex[below = 15pt of l, color2, dot, minimum size=3pt, label = {left: {\footnotesize $\textcolor{color2}{1}^{c}$}}] (x1) {};
\vertex[above = 15pt of l,color2,  dot, minimum size=3pt, label = {left: {\footnotesize $\textcolor{color2}{2}^{c}$}}] (x2) {};
\vertex[right = 8pt of l, dot, minimum size=0pt] (v12) {};
\vertex[right = 35pt of v12, squareblob] (b) {};
\vertex[right = 35pt of b, dot, minimum size=0pt] (v34) {};
\vertex[right = 8pt of v34] (r) {};
\vertex[above = 15pt of r, color1, dot, minimum size=3pt, label = {right: {\footnotesize $\textcolor{color1}{3}^{c}$}}] (x3) {};
\vertex[below = 15pt of r, color1, dot, minimum size=3pt, label = {right: {\footnotesize $\textcolor{color1}{4}^{c}$}}] (x4) {};
\diagram*{
	(x1) -- [color2, ghost] (v12) -- [color2, ghost] (x2),
	(v12) -- [photon, edge label = {\scriptsize \;$\Theta\,\sigma'_{j}\sigma'_{j}$}, inner sep = 4pt] (b) -- [photon, edge label = {\scriptsize \,$\Theta\,\sigma'_{k}\sigma'_{k}$}, inner sep = 4pt] (v34), 
	 (x3) -- [color1, ghost] (v34) -- [color1, ghost] (x4)
};
\end{feynman}
\end{tikzpicture} \nonumber\\
& + \mathcal{W}[2]\,\!\!\sum_{j,k}\!\!
\begin{tikzpicture}[baseline=(b)]
\begin{feynman}
\vertex (l) {};
\vertex[below = 15pt of l, color2, dot, minimum size=3pt, label = {left: {\footnotesize $\textcolor{color2}{1}^{c}$}}] (x1) {};
\vertex[above = 15pt of l,color2,  dot, minimum size=3pt, label = {left: {\footnotesize $\textcolor{color2}{2}^{c}$}}] (x2) {};
\vertex[right = 8pt of l, dot, minimum size=0pt] (v12) {};
\vertex[right = 35pt of v12, squareblob] (b) {};
\vertex[right = 30pt of b, dot, minimum size=0pt] (v34) {};
\vertex[right = 8pt of v34] (r) {};
\vertex[above = 15pt of r, color1, dot, minimum size=3pt, label = {right: {\footnotesize $\textcolor{color1}{3}^{c}$}}] (x3) {};
\vertex[below = 15pt of r, color1, dot, minimum size=3pt, label = {right: {\footnotesize $\textcolor{color1}{4}^{c}$}}] (x4) {};
\diagram*{
	(x1) -- [color2, ghost] (v12) -- [color2, ghost] (x2),
	(v12) -- [photon, edge label = {\scriptsize \;$\Theta\,\sigma'_{j}\sigma'_{j}$}, inner sep = 4pt] (b) -- [color1doubghost, edge label = {\scriptsize \,$\sigma'_{k}\sigma'_{k}$}, inner sep = 4pt] (v34), 
	 (x3) -- [color1, ghost] (v34) -- [color1, ghost] (x4)
};
\end{feynman}
\end{tikzpicture}
+ \mathcal{W}[2]\,\!\!\sum_{j,k}\!\!
\begin{tikzpicture}[baseline=(b)]
\begin{feynman}
\vertex (l) {};
\vertex[below = 15pt of l, color2, dot, minimum size=3pt, label = {left: {\footnotesize $\textcolor{color2}{1}^{c}$}}] (x1) {};
\vertex[above = 15pt of l,color2,  dot, minimum size=3pt, label = {left: {\footnotesize $\textcolor{color2}{2}^{c}$}}] (x2) {};
\vertex[right = 8pt of l, dot, minimum size=0pt] (v12) {};
\vertex[right = 30pt of v12, squareblob] (b) {};
\vertex[right = 35pt of b, dot, minimum size=0pt] (v34) {};
\vertex[right = 8pt of v34] (r) {};
\vertex[above = 15pt of r, color1, dot, minimum size=3pt, label = {right: {\footnotesize $\textcolor{color1}{3}^{c}$}}] (x3) {};
\vertex[below = 15pt of r, color1, dot, minimum size=3pt, label = {right: {\footnotesize $\textcolor{color1}{4}^{c}$}}] (x4) {};
\diagram*{
	(x1) -- [color2, ghost] (v12) -- [color2, ghost] (x2),
	(v12) -- [color2doubghost, edge label = {\scriptsize \;$\sigma'_{j}\sigma'_{j}$}, inner sep = 4pt] (b) -- [photon, edge label = {\scriptsize \,$\Theta\,\sigma'_{k}\sigma'_{k}$}, inner sep = 4pt] (v34), 
	 (x3) -- [color1, ghost] (v34) -- [color1, ghost] (x4)
};
\end{feynman}
\end{tikzpicture}\nonumber\\
& + \mathcal{W}[2]\,\!\!\sum_{j,k}\!\!
\begin{tikzpicture}[baseline=(b)]
\begin{feynman}
\vertex (l) {};
\vertex[below = 15pt of l, color2, dot, minimum size=3pt, label = {left: {\footnotesize $\textcolor{color2}{1}^{c}$}}] (x1) {};
\vertex[above = 15pt of l,color2,  dot, minimum size=3pt, label = {left: {\footnotesize $\textcolor{color2}{2}^{c}$}}] (x2) {};
\vertex[right = 8pt of l, dot, minimum size=0pt] (v12) {};
\vertex[right = 30pt of v12, squareblob] (b) {};
\vertex[right = 30pt of b, dot, minimum size=0pt] (v34) {};
\vertex[right = 8pt of v34] (r) {};
\vertex[above = 15pt of r, color1, dot, minimum size=3pt, label = {right: {\footnotesize $\textcolor{color1}{3}^{c}$}}] (x3) {};
\vertex[below = 15pt of r, color1, dot, minimum size=3pt, label = {right: {\footnotesize $\textcolor{color1}{4}^{c}$}}] (x4) {};
\diagram*{
	(x1) -- [color2, ghost] (v12) -- [color2, ghost] (x2),
	(v12) -- [color2doubghost, edge label = {\scriptsize \;$\sigma'_{j}\sigma'_{j}$}, inner sep = 4pt] (b) -- [color1doubghost, edge label = {\scriptsize \,$\sigma'_{k}\sigma'_{k}$}, inner sep = 4pt] (v34), 
	 (x3) -- [color1, ghost] (v34) -- [color1, ghost] (x4)
};
\end{feynman}
\end{tikzpicture}
+ \mathcal{W}[2]\,\!\!\sum_{j,k}\!\!
\begin{tikzpicture}[baseline=(b)]
\begin{feynman}
\vertex (l) {};
\vertex[below = 15pt of l, color2, dot, minimum size=3pt, label = {left: {\footnotesize $\textcolor{color2}{1}^{c}$}}] (x1) {};
\vertex[above = 15pt of l, color1, dot, minimum size=3pt, label = {left: {\footnotesize $\textcolor{color1}{4}^{c}$}}] (x2) {};
\vertex[right = 8pt of l, dot, minimum size=0pt] (v12) {};
\vertex[right = 30pt of v12, squareblob] (b) {};
\vertex[right = 30pt of b, dot, minimum size=0pt] (v34) {};
\vertex[right = 8pt of v34] (r) {};
\vertex[above = 15pt of r, color2, dot, minimum size=3pt, label = {right: {\footnotesize $\textcolor{color2}{2}^{c}$}}] (x3) {};
\vertex[below = 15pt of r, color1, dot, minimum size=3pt, label = {right: {\footnotesize $\textcolor{color1}{3}^{c}$}}] (x4) {};
\diagram*{
	(x1) -- [color2, ghost](v12) -- [color1, ghost] (x2),
	(v12) -- [color1color2ghost, edge label = {\scriptsize \;$\sigma'_{j}\sigma'_{j}$}, inner sep = 4pt] (b) -- [color2color1ghost, edge label = {\scriptsize \,$\sigma'_{k}\sigma'_{k}$}, inner sep = 4pt] (v34), 
	 (x3) -- [color2,  ghost] (v34) -- [color1, ghost] (x4)
};
\end{feynman}
\end{tikzpicture}
\label{btensorsquareprop}
\end{align}
In \eqref{btensorsquareprop}, the application of Feyman rule (3) eliminates the presence of 2-class diagrams. Applying rule (4) to the inequivalent permutations of the reference pairing $(13)(24)$, namely $\{(13)(24)\}$, $\{(12)(34)\}$, $\{(14)(23)\}$ generates the 1-class diagrams shown in~\eqref{btensorsquareprop}. The corresponding Weingarten coefficients multiplying each subdiagram are determined by the relative ordering between the reference pairing $\tau$ and the pairing $\pi$ defining the subdiagram: $\mathcal{W}[\pi,\tau]=\mathcal{W}[1,1]$ for $\pi=\tau$, $\mathcal{W}[\pi,\tau]=\mathcal{W}[2]$ for $\pi\neq\tau$. The overall factor of 1 follows from the M\"obius formula for a single-block partition.

Restricting the analysis to first order in \(1/n\), the expansions~\eqref{w11expanded} and~\eqref{w2expanded} reduce~\eqref{btensorsquareprop} to
\begin{align}
 + \mathcal{O}\left(\frac{1}{n^{2}}\right)
\label{btensorsquarepropexpanded}
\end{align}
We now apply the second set of Feynman rules (6)-(9). In particular, the selection rules (a)-(f) eliminate the last four terms in~\eqref{btensorsquarepropexpanded}. To cancel the prefactor \(1/n^3\), the corresponding diagrams would have to contribute a factor of $n^2$, which could only arise from two disconnected bare propagators. However, such subdiagrams violate the color-conservation property of the propagator. Consequently, we are left with
\begin{align}
  %
 + \mathcal{O}\left(\frac{1}{n^{2}}\right)
\label{btensorsquarepropexpandedtwo}
\end{align}
Applying Feynman rules (6)–(9) to the diagonal and off-diagonal neural components on the right-hand side of~\eqref{btensorsquarepropexpandedtwo} generates the following diagrams: When $i=j$, the first term yields

\begin{align}
\frac{1}{n^{2}}\sum_{j}\!\!
%
\label{btensorexpandedijsame}
\end{align}
When $i\neq j$, the first term produces
\begin{align}
\frac{1}{n^{2}}\sum_{j_1, j_2}\!
%
\label{btensorexpandedijdiffone}
\end{align}
whereas the second term becomes
\begin{align}
-\frac{1}{n^{3}}\sum_{j_1, j_2}\!
%
\right) + \mathcal{O}\left(\frac{1}{n^2}\right)
\label{btensorexpandedfinal}
\end{align}
The linear \(1/n\) scaling of the right-hand side of~\eqref{btensorexpandedfinal} follows from simple power counting of the associated summations. Diagrams involving a single neural index contribute a factor of $n$, whereas those involving two distinct neural indices contribute a factor of $n^2$. Combined with the explicit prefactors in~\eqref{btensorexpandedfinal}, each term on the right-hand side yields an overall contribution of order \(1/n\). It is straightforward to verify that the first and second lines on the RHS of~\eqref{btensorexpandedfinal} reproduce the corresponding lines on the RHS of~\eqref{eq:B_recursion_algebraic}.


\section{Feynman rules at leading order in \(1/n\)}
\label{app:simplified_feynman_rules}
We present an alternative formulation of the Feynman rules introduced in Section~\ref{sec:feynman-diagrams}, valid at leading order in \(1/n\), which closely parallels the diagrammatic framework of~\cite{guillen2025} in the Gaussian setting. For simplicity, we restrict attention to the statistics of preactivations and the NTK.
We adopt the same conventions and notation as in~\cite{guillen2025}:

\begin{enumerate}
\item Preactivations and NTKs are represented by external lines, as illustrated below.
\begin{align}
    &z_{\alpha} \equiv
    \begin{tikzpicture}[baseline=-0.1cm]
      \begin{feynman}
        \vertex (l) {};
        \vertex[right = 0pt of l, dot, minimum size=3pt, label = {left: {\footnotesize $\alpha^{}$}}] (x1) {};
        \vertex[right = 30pt of x1, dot, minimum size=0pt] (b) {};
        \diagram*{
          (x1),
          (x1) -- [inner sep = 4pt] (b) 
        };
      \end{feynman}
    \end{tikzpicture}
    \qquad\widehat{\Delta \Theta}_{\textcolor{color1}{\alpha\beta}} \equiv
    \begin{tikzpicture}[baseline=0.05cm]
      \begin{feynman}
        \vertex (l) {};
        \vertex[right = 0pt of l, color1, dot, minimum size=3pt, label = {left: {\footnotesize $\textcolor{color1}{\alpha^{}}$}}] (x1) {};
        \vertex[above = 8pt of l, color1, dot, minimum size=3pt, label = {left: {\footnotesize $\textcolor{color1}{\beta^{}}$}}] (x2) {};
        \vertex[right = 30pt of x1, dot, minimum size=0pt] (b) {};
        \vertex[right = 30pt of x2, dot, minimum size=0pt] (bb) {};
        \diagram*{
          (x1), (x2),
          (x1) -- [color1, ghost, inner sep = 4pt] (b),
          (x2) -- [color1, ghost, inner sep = 4pt] (bb)
        };
      \end{feynman}
    \end{tikzpicture}
  \end{align}
where a colored line corresponds to a single NTK label.

\item The propagator is represented by
\begin{equation}
    \langle\hspace{3mm}\rangle_{K^{(\ell)}} \equiv
    \begin{tikzpicture}[baseline=-0.1cm]
      \begin{feynman}
        \tikzfeynmanset{every blob = {/tikz/fill=white!50, /tikz/minimum size=15pt}}
        \vertex (l) {};
        \vertex[above = 0pt of l, dot, minimum size=0pt] (v12) {};
        \vertex[above = 0pt of v12, blob] (b) {};
        \vertex[above = 0pt of b, dot, minimum size=0pt] (v34) {};
        \diagram*{
          (v12) -- (b) -- (v34), 
        };
      \end{feynman}
    \end{tikzpicture}   
  \end{equation}
  where $\langle \hspace{3mm}\rangle_{K^{(\ell)}}$ denotes a zero-mean Gaussian expectation with covariance specified by $K^{(\ell)}$. The expectation value is taken over the decorations of the internal lines attached to the propagator, which satisfies the set of selection rules (a)-(f) listed in Rule (6) of Section~\ref{sec:feynman-diagrams}.

  \item Cubic vertices are defined as in~\cite{guillen2025}. Explicitly,
   \begin{align}
 \begin{tikzpicture}[baseline=(b)]
      \begin{feynman}
        \vertex (l) {};
        \vertex[below = 15pt of l, dot, minimum size=3pt, label = {left: {\footnotesize $\alpha^{}$}}] (x1) {};
        \vertex[above = 15pt of l, dot, minimum size=3pt, label = {left: {\footnotesize $\beta^{}$}}] (x2) {};
        \vertex[right = 10pt of l, dot, minimum size=0pt] (v12) {};
        \vertex[right = 33pt of v12, dot, minimum size=0pt] (b) {};
        \diagram*{
          (x1) --  (v12) --  (x2),
          (v12) -- [photon, edge label = {\scriptsize \;$\widehat{\Delta G}_{i,\alpha\beta}^{(\ell)}$}, inner sep = 4pt] (b) 
        };
      \end{feynman}
    \end{tikzpicture}\hspace{-0.1cm} \sim \frac{C_{W}}{n}    
  \hspace{0.1cm} ,& \quad 
    \begin{tikzpicture}[baseline=(b)]
      \begin{feynman}
        \vertex (l) {};
        \vertex[below = 15pt of l, color1, dot, minimum size=3pt, label = {left: {\footnotesize $\textcolor{color1}{\alpha^{}}$}}] (x1) {};
        \vertex[above = 15pt of l, color1, dot, minimum size=3pt, label = {left: {\footnotesize $\textcolor{color1}{\beta^{}}$}}] (x2) {};
        \vertex[right = 10pt of l, dot, minimum size=0pt] (v12) {};
        \vertex[right = 30pt of v12, dot, minimum size=0pt] (b) {};
        \diagram*{
          (x1) --  [color1, ghost] (v12) --  [color1, ghost] (x2),
          (v12) -- [black, photon, edge label = {\scriptsize \;$\widehat{\Delta \Omega}_{i,\alpha\beta}^{(\ell+1)}$}, inner sep = 4pt] (b) 
        };
      \end{feynman}
    \end{tikzpicture}\hspace{-0.1cm} \sim \frac{1}{n}     
    \hspace{0.1cm} ,\quad
    \begin{tikzpicture}[baseline=(b)]
      \begin{feynman}
        \vertex (l) {};
        \vertex[below = 15pt of l, dot, minimum size=3pt, label = {left: {\footnotesize $\alpha^{}$}}] (x1) {};
        \vertex[above = 15pt of l, color1, dot, minimum size=3pt, label = {left: {\footnotesize $\textcolor{color1}{\beta^{}}$}}] (x2) {};
        \vertex[right = 10pt of l, dot, minimum size=0pt] (v12) {};
        \vertex[right = 33pt of v12, dot, minimum size=0pt] (b) {};
        \diagram*{
          (x1) --  (v12) --  [color1, ghost] (x2),
          (v12) -- [color1blackghost, edge label = {\scriptsize \;$\sigma^{(\ell)}_{i,\alpha}\sigma'^{(\ell)}_{i,\textcolor{color1}{\beta}}$}, inner sep = 4pt] (b) 
        };
      \end{feynman}
    \end{tikzpicture}\hspace{-0.1cm} \sim \frac{C_{W}}{n}  
    \hspace{0.1cm} , \nonumber\\
    \begin{tikzpicture}[baseline=(b)]
      \begin{feynman}
        \vertex (l) {};
        \vertex[below = 15pt of l, color1, dot, minimum size=3pt, label = {left: {\footnotesize $\textcolor{color1}{\alpha^{}}$}}] (x1) {};
        \vertex[above = 15pt of l, color1, dot, minimum size=3pt, label = {left: {\footnotesize $\textcolor{color1}{\beta^{}}$}}] (x2) {};
        \vertex[right = 10pt of l, dot, minimum size=0pt] (v12) {};
        \vertex[right = 30pt of v12, dot, minimum size=0pt] (b) {};
        \diagram*{
          (x1) --  [color1, ghost] (v12) --  [color1, ghost] (x2),
          (v12) -- [color1doubghost, edge label = {\scriptsize \;$\sigma'^{(\ell)}_{i,\textcolor{color1}{\alpha}}\sigma'^{(\ell)}_{i,\textcolor{color1}{\beta}}$}, inner sep = 4pt] (b) 
        };
      \end{feynman}
    \end{tikzpicture}\hspace{-0.1cm} \sim \frac{C_{W}}{n} \hspace{0.1cm}  ,& \quad
    \begin{tikzpicture}[baseline=(b)]
      \begin{feynman}
        \vertex (l) {};
        \vertex[below = 15pt of l, color1, dot, minimum size=3pt, label = {left: {\footnotesize $\textcolor{color1}{\alpha^{}}$}}] (x1) {};
        \vertex[above = 15pt of l, color2, dot, minimum size=3pt, label = {left: {\footnotesize $\textcolor{color2}{\beta^{}}$}}] (x2) {};
        \vertex[right = 10pt of l, dot, minimum size=0pt] (v12) {};
        \vertex[right = 30pt of v12, dot, minimum size=0pt] (b) {};
        \diagram*{
          (x1) --  [color1, ghost] (v12) --  [color2, ghost] (x2),
          (v12) -- [color2color1ghost, edge label = {\scriptsize \;$\sigma'^{(\ell)}_{i,\textcolor{color1}{\alpha}}\sigma'^{(\ell)}_{i,\textcolor{color2}{\beta}}$}, inner sep = 4pt] (b) 
        };
      \end{feynman}
    \end{tikzpicture}\hspace{-0.1cm} \sim \frac{C_{W}}{n} \label{feynmanrulescubicapp}
  \end{align}
   where $\widehat{\Omega}^{(\ell+1)}_{i,\alpha\beta} = \sigma^{(\ell)}_{i,\alpha}\sigma^{(\ell)}_{i,\beta} + C_{W}\Theta_{\alpha\beta}^{(\ell)}\sigma'^{(\ell)}_{i,\alpha}\sigma'^{(\ell)}_{i,\beta}$ and $\widehat{\Delta\Omega}_{i,\alpha\beta}^{(\ell+1)}=\widehat{\Omega}_{i,\alpha\beta}^{(\ell+1)}-\langle \widehat{\Omega}_{i,\alpha\beta}^{(\ell+1)} \rangle_{K^{(\ell)}}$. Lines without a dot at one end are internal lines. 
   
   \item Quartic vertices are defined analogously, following~\cite{guillen2025}. Explicitly,
  \begin{align}
     \begin{tikzpicture}[baseline=(b)]
      \begin{feynman}
        \vertex (l) {};
        \vertex[below = 15pt of l, dot, minimum size=3pt, label = {left: {\footnotesize $\alpha_{1}$}}] (x1) {};
        \vertex[above = 15pt of l, dot, minimum size=3pt, label = {left: {\footnotesize $\alpha_{2}$}}] (x2) {};
        \vertex[right = 20pt of l, quarticblob] (b) {};
        \vertex[below = 25pt of b] {\scriptsize $\frac{1}{n\*}V_{\alpha_{1}\alpha_{2}\alpha_{3}\alpha_{4}}^{(\ell+1)}$}; 
        \vertex[right = 20pt of b] (r) {};
        \vertex[above = 15pt of r, dot, minimum size=3pt, label = {right: {\footnotesize $\alpha_{3}$}}] (x3) {};
        \vertex[below = 15pt of r, dot, minimum size=3pt, label = {right: {\footnotesize $\alpha_{4}$}}] (x4) {};
        \diagram*{
          (x1) -- (b) -- (x2), 
          (x3) -- (b) -- (x4)
        };
      \end{feynman}
    \end{tikzpicture}
    \hspace{0.5cm} , & \quad
    \begin{tikzpicture}[baseline=(b)]
      \begin{feynman}
        \vertex (l) {};
        \vertex[below = 15pt of l, dot, minimum size=3pt, label = {left: {\footnotesize $\alpha_{1}$}}] (x1) {};
        \vertex[above = 15pt of l, dot, minimum size=3pt, label = {left: {\footnotesize $\alpha_{2}$}}] (x2) {};
        \vertex[right = 20pt of l, quarticblob] (b) {};
        \vertex[below = 25pt of b] {\scriptsize $\frac{1}{n\*}D_{\alpha_{1}\alpha_{2}\textcolor{color1}{\alpha_{3}\alpha_{4}}}^{(\ell+1)}$}; 
        \vertex[right = 20pt of b] (r) {};
        \vertex[above = 15pt of r, dot, color1, minimum size=3pt, label = {right: {\footnotesize $\textcolor{color1}{\alpha_{3}}$}}] (x3) {};
        \vertex[below = 15pt of r, dot, color1, minimum size=3pt, label = {right: {\footnotesize $\textcolor{color1}{\alpha_{4}}$}}] (x4) {};
        \diagram*{
          (x1) -- (b) -- (x2), 
          (x3) -- [color1, ghost] (b) -- [color1, ghost] (x4)
        };
      \end{feynman}
    \end{tikzpicture}
    \hspace{0.5cm} , \quad
    \begin{tikzpicture}[baseline=(b)]
      \begin{feynman}
        \vertex (l) {};
        \vertex[below = 15pt of l, dot, minimum size=3pt, label = {left: {\footnotesize $\alpha_{1}$}}] (x1) {};
        \vertex[above = 15pt of l, dot, color1, minimum size=3pt, label = {left: {\footnotesize $\textcolor{color1}{\alpha_{3}}$}}] (x2) {};
        \vertex[right = 20pt of l, quarticblob] (b) {};
        \vertex[below = 25pt of b] {\scriptsize $\frac{1}{n\*}F_{\alpha_{1}\textcolor{color1}{\alpha_{3}}\alpha_{2}\textcolor{color1}{\alpha_{4}}}^{(\ell+1)}$}; 
        \vertex[right = 20pt of b] (r) {};
        \vertex[above = 15pt of r, dot, minimum size=3pt, label = {right: {\footnotesize $\alpha_{2}$}}] (x3) {};
        \vertex[below = 15pt of r, dot, color1, minimum size=3pt, label = {right: {\footnotesize $\textcolor{color1}{\alpha_{4}}$}}] (x4) {};
        \diagram*{
          (x1) -- (b) -- [color1, ghost] (x2), 
          (x3) -- (b) -- [color1, ghost] (x4)
        };
      \end{feynman}
    \end{tikzpicture}
    \hspace{0.5cm} , \nonumber\\
    \begin{tikzpicture}[baseline=(b)]
      \begin{feynman}
        \vertex (l) {};
        \vertex[below = 15pt of l, dot, color2, minimum size=3pt, label = {left: {\footnotesize $\textcolor{color2}{\alpha_{1}}$}}] (x1) {};
        \vertex[above = 15pt of l, dot, color2, minimum size=3pt, label = {left: {\footnotesize $\textcolor{color2}{\alpha_{2}}$}}] (x2) {};
        \vertex[right = 20pt of l, quarticblob] (b) {};
        \vertex[below = 25pt of b] {\scriptsize $\frac{1}{n\*}A_{\textcolor{color2}{\alpha_{1}}\textcolor{color2}{\alpha_{2}}\textcolor{color1}{\alpha_{3}}\textcolor{color1}{\alpha_{4}}}^{(\ell+1)}$}; 
        \vertex[right = 20pt of b] (r) {};
        \vertex[above = 15pt of r, dot, color1, minimum size=3pt, label = {right: {\footnotesize $\textcolor{color1}{\alpha_{3}}$}}] (x3) {};
        \vertex[below = 15pt of r, dot, color1, minimum size=3pt, label = {right: {\footnotesize $\textcolor{color1}{\alpha_{4}}$}}] (x4) {};
        \diagram*{
          (x1) -- [color2, ghost] (b) -- [color2, ghost] (x2), 
          (x3) -- [color1, ghost] (b) -- [color1, ghost] (x4)
        };
      \end{feynman}
    \end{tikzpicture}
    \hspace{0.5cm} ,& \quad
    \begin{tikzpicture}[baseline=(b)]
      \begin{feynman}
        \vertex (l) {};
        \vertex[below = 15pt of l, dot, color2, minimum size=3pt, label = {left: {\footnotesize $\textcolor{color2}{\alpha_{1}}$}}] (x1) {};
        \vertex[above = 15pt of l, dot, color1, minimum size=3pt, label = {left: {\footnotesize $\textcolor{color1}{\alpha_{3}}$}}] (x2) {};
        \vertex[right = 20pt of l, quarticblob] (b) {};
        \vertex[below = 25pt of b] {\scriptsize $\frac{1}{n\*}B_{\textcolor{color2}{\alpha_{1}}\textcolor{color1}{\alpha_{3}}\textcolor{color2}{\alpha_{2}}\textcolor{color1}{\alpha_{4}}}^{(\ell+1)}$}; 
        \vertex[right = 20pt of b] (r) {};
        \vertex[above = 15pt of r, dot, color2, minimum size=3pt, label = {right: {\footnotesize $\textcolor{color2}{\alpha_{2}}$}}] (x3) {};
        \vertex[below = 15pt of r, dot, color1, minimum size=3pt, label = {right: {\footnotesize $\textcolor{color1}{\alpha_{4}}$}}] (x4) {};
        \diagram*{
          (x1) -- [color2, ghost] (b) -- [color1, ghost] (x2), 
          (x3) -- [color2, ghost] (b) -- [color1, ghost] (x4)
        };
      \end{feynman}
    \end{tikzpicture}
    \label{feynmanrulesquarticapp}
  \end{align}

 \item The orthogonal diagram describing the $2m$-point cumulant for the reference pairing $\pi=(12)(34)\cdots(2m-1\,\,2m)$ is obtained from the Gaussian one, using the Feynman rules \eqref{feynmanrulescubicapp} and \eqref{feynmanrulesquarticapp}, by summing over all reconnections of the external labels:
\begin{equation}\label{leadingorderorthogonal}
V^{\mathrm{orth}}_{2m,\pi}
=
V^{\mathrm{gauss}}_{2m,\pi}
+
\sum_{\lambda\vdash m,\ \lambda\neq (1,\dots,1)}
\frac{\beta_\lambda}{n^{\,m-\ell(\lambda)}}
\sum_{\tau\in\mathcal C_\lambda(\pi)}
\prod_{j=1}^{\ell(\lambda)}
V^{\mathrm{conn}}_{2\lambda_j}
\end{equation}
Here:
\begin{itemize}
    \item $\lambda=(\lambda_1,\dots,\lambda_{\ell(\lambda)})\vdash m$ is a partition of $m$, and $\ell(\lambda)$ is its number of parts;
    \item $\mathcal C_\lambda(\pi)$ denotes the set of pairings $\tau$ such that, when compared with $\pi$, the external labels split into blocks of sizes $2\lambda_1,\dots,2\lambda_{\ell(\lambda)}$;
    \item each block of size $2\lambda_j$ contributes a connected correlator $V^{\mathrm{conn}}_{2\lambda_j}$, with $V^{\mathrm{conn}}_2$ given by 
    \begin{align}\label{orthogonalityvertexsimp}
    \begin{tikzpicture}[baseline=(b)]
      \begin{feynman}
        \vertex (l) {};
        \vertex[below = 15pt of l, dot, minimum size=3pt, label = {left: {\footnotesize $\alpha^{}$}}] (x1) {};
        \vertex[above = 15pt of l, dot, minimum size=3pt, label = {left: {\footnotesize $\beta^{}$}}] (x2) {};
        \vertex[right = 10pt of l, dot, minimum size=0pt] (v12) {};
        \vertex[right = 33pt of v12, dot, minimum size=0pt] (b) {};
        \diagram*{
          (x1) --  (v12) --  (x2),
          (v12) -- [photon, edge label = {\scriptsize \;$\sigma_{i,\alpha}^{(\ell)}\sigma_{i,\beta}^{(\ell)}$}, inner sep = 4pt] (b) 
        };
      \end{feynman} 
      \end{tikzpicture}\hspace{-0.1cm} \sim \frac{C_{W}}{n} \quad, \quad    
    \begin{tikzpicture}[baseline=(b)]
      \begin{feynman}
        \vertex (l) {};
        \vertex[below = 15pt of l, color1, dot, minimum size=3pt, label = {left: {\footnotesize $\textcolor{color1}{\alpha}^{}$}}] (x1) {};
        \vertex[above = 15pt of l, color1, dot, minimum size=3pt, label = {left: {\footnotesize $\textcolor{color1}{\beta}^{}$}}] (x2) {};
        \vertex[right = 10pt of l, dot, minimum size=0pt] (v12) {};
        \vertex[right = 40pt of v12, dot, minimum size=0pt] (b) {};
        \diagram*{
          (x1) --  [color1, ghost] (v12) --  [color1, ghost](x2),
          (v12) -- [photon, edge label = {\scriptsize \;$\Theta^{(\ell)}_{\alpha\beta}\sigma'^{(\ell)}_{i,\alpha}\sigma'^{(\ell)}_{i,\beta}$}, inner sep = 4pt] (b) 
        };
      \end{feynman} 
      \end{tikzpicture}\hspace{-0.1cm} \sim \frac{C_{W}}{n}
    \end{align}
    \item  $\beta_\lambda$ is the leading coefficient of the orthogonal Weingarten function for pairings in the class $\lambda$, namely
    \begin{equation}
    \mathcal W[\tau,\pi]
    =
    \frac{\beta_\lambda}{n^{\,m-\ell(\lambda)}}+O\!\left(n^{-(m-\ell(\lambda)+1)}\right),
    \qquad \tau\in\mathcal C_\lambda(\pi).
    \end{equation}
\end{itemize}
\vspace{0.5em}
Its general expression is given by
\begin{align}
\beta_\lambda
=
\prod_{i=1}^{\ell}
(-1)^{\lambda_i - 1} \, C_{\lambda_i - 1}
\end{align}
where $C_{k}$ is the $k$-th Catalan number, $C_{k} = \frac{1}{k+1}\binom{2k}{k}$. For low orders one finds
\begin{align}
&\beta_{(2)}=-1, \nonumber\\
&\beta_{(2,1)}=-1,\qquad \beta_{(3)}=2, \nonumber\\
&\beta_{(2,1,1)}=-1,\qquad \beta_{(2,2)}=1,\qquad \beta_{(3,1)}=2,\qquad \beta_{(4)}=-5.
\end{align}
\end{enumerate}
We next apply these rules to derive the recursion relations for the tensors $F$ and $B$, recovering the expressions previously obtained from the general Feynman rules of Section~\ref{sec:feynman-diagrams}, namely~\eqref{ftensoratorderonen} and~\eqref{btensorexpandedfinal}.

The layer-$(\ell+1)$ tensor on the LHS of~\eqref{eq:F} is represented by the third quartic vertex in~\eqref{feynmanrulesquarticapp}. We enumerate all Feynman diagrams compatible with these external lines at order \(1/n\). Since the fifth Feynman rule mixes pairing orderings, the analysis splits into two cases.

For the ordering 1324, the Gaussian contraction scheme applies: the only admissible cubic vertex connecting a dotted (NTK) line with a solid (preactivation) line is the third vertex in~\eqref{feynmanrulescubicapp}. This introduces two internal solid-dashed lines weighted by ${\sigma}_{i}^{(\ell)}{\sigma'}_{i}^{(\ell)}$, which may carry either identical or distinct channel indices. In the former case, the lines are connected by a single propagator; in the latter, by two propagators and an internal $F$-tensor at layer $\ell$.

For all other orderings, the fifth Feynman rule applies and requires subtracting the contributions associated with the orderings 1234 and 1423. The pairings $(12)(34)$ and $(14)(23)$, when composed with the reference pairing $(13)(24)$, yield the partition $\lambda=(2)$, with length $\ell(\lambda)=1$. Consequently, these contributions scale as \(1/n\), with coefficient $\beta_{(2)}=-1$. Each pairing contributes a factor $V_{2}$, leading to a product of two such vertices. By the selection rules (a)-(f) of Section~\ref{sec:feynman-diagrams}, the $(14)(23)$ term vanishes at order \(1/n\) due to violation of color-line preservation. The only non-vanishing contribution therefore arises from $(12)(34)$, yielding a product of two cubic vertices: one involving NTK lines and one involving two preactivation lines, as specified in~\eqref{orthogonalityvertexsimp}.

The diagrams corresponding to~\eqref{eq:F} are thus given by
\begin{align}
  \begin{tikzpicture}[baseline=(b)]
      \begin{feynman}
        \vertex (l) {};
        \vertex[below = 15pt of l, dot, minimum size=3pt, label = {left: {\footnotesize $1$}}] (x1) {};
        \vertex[above = 15pt of l, dot, color1, minimum size=3pt, label = {left: {\footnotesize $\textcolor{color1}{3}$}}] (x2) {};
        \vertex[right = 20pt of l, quarticblob] (b) {};
        \vertex[right = 20pt of b] (r) {};
        \vertex[above = 15pt of r, dot, minimum size=3pt, label = {right: {\footnotesize $2$}}] (x3) {};
        \vertex[below = 15pt of r, dot, color1, minimum size=3pt, label = {right: {\footnotesize $\textcolor{color1}{4}$}}] (x4) {};
        \diagram*{
          (x1) -- (b) -- [color1, ghost] (x2), 
          (x3) -- (b) -- [color1, ghost] (x4)
        };
      \end{feynman}
    \end{tikzpicture}
\!\!&= 
\!\!\sum_{j}\!\!
\begin{tikzpicture}[baseline=(b)]
\tikzfeynmanset{every blob = {/tikz/fill=white!50, /tikz/minimum size=15pt}}
\begin{feynman}
\vertex (l) {};
\vertex[below = 15pt of l, dot, minimum size=3pt, label = {left: {\footnotesize $1$}}] (x1) {};
\vertex[above = 15pt of l, color1, dot, minimum size=3pt, label = {left: {\footnotesize $\textcolor{color1}{3}$}}] (x2) {};
\vertex[right = 8pt of l, dot, minimum size=0pt] (v12) {};
\vertex[right = 30pt of v12, blob] (b) {};
\vertex[right = 30pt of b, dot, minimum size=0pt] (v34) {};
\vertex[right = 8pt of v34] (r) {};
\vertex[above = 15pt of r, dot, minimum size=3pt, label = {right: {\footnotesize $2$}}] (x3) {};
\vertex[below = 15pt of r, color1, dot, minimum size=3pt, label = {right: {\footnotesize $\textcolor{color1}{4}$}}] (x4) {};
\diagram*{
	(x1) -- (v12) -- [color1, ghost] (x2),
	(v12) -- [color1blackghost, edge label = {\scriptsize \;$\sigma_{j}\sigma'_{j}$}, inner sep = 4pt] (b) -- [blackcolor1ghost, edge label = {\scriptsize \,$\sigma_{j}\sigma'_{j}$}, inner sep = 4pt] (v34), 
	 (x3) -- (v34) -- [color1, ghost] (x4)
};
\end{feynman}
\end{tikzpicture}
- 
\frac{1}{n}\left(
\sum_{j_1, j_2}\!
\begin{tikzpicture}[baseline=(b)]
  \tikzfeynmanset{every blob = {/tikz/fill=white!50, /tikz/minimum size=15pt}}
  \begin{feynman}
    \vertex (l) {};
    \vertex[below = 15pt of l, dot, minimum size=3pt, label = {left: {\footnotesize $1$}}] (x1) {};
    \vertex[above = 15pt of l, , dot, minimum size=3pt, label = {left: {\footnotesize $2$}}] (x2) {};
    \vertex[right = 8pt of l, dot, minimum size=0pt] (v12) {};
    \vertex[right = 35pt of v12, blob] (b12) {};
    \vertex[right = 5pt of b12, dot, minimum size=0pt] (w12) {};
    \vertex[above = 1pt of w12, label = {above: {\scriptsize \hspace{-10pt} $ $}}] (w12u) {};
    \vertex[below = 8pt of w12] (w12d) {};
    \vertex[left = 10pt of w12d] (w12dl) {};
    \tikzfeynmanset{every blob = {/tikz/fill=gray!50, /tikz/minimum size=0pt}}
    \vertex[right = 3pt of w12, blob , minimum size = 0pt] (b) {};
    \vertex[right = 3pt of b, dot, minimum size=0pt] (w34) {};
    \tikzfeynmanset{every blob = {/tikz/fill=white!50, /tikz/minimum size=15pt}}
    \vertex[right = 16pt of w34, blob] (b34) {};
    \vertex[above = 1pt of w34, label = {above: {\scriptsize \hspace{10pt} $ $}}] (w34u) {};
    \vertex[below = 8pt of w34] (w34d) {};
    \vertex[right = 10pt of w34d] (w34dr) {};
    \vertex[right = 35pt of b34, dot, minimum size=0pt] (v34) {};
    \vertex[right = 8pt of v34] (r) {};
    \vertex[above = 15pt of r, color1, dot, minimum size=3pt, label = {right: {\footnotesize $\textcolor{color1}{3}$}}] (x3) {};
    \vertex[below = 15pt of r, color1, dot, minimum size=3pt, label = {right: {\footnotesize $\textcolor{color1}{4}$}}] (x4) {};
    \diagram*{
      (x1) -- (v12) -- (x2),
      (v12) -- [photon, edge label = {\scriptsize \;$\sigma_{j_{1}}\sigma_{j_{1}}$}, inner sep = 4pt] (b12),
      (b34) -- [photon, edge label = {\scriptsize \,$\sigma'_{j_{2}}\sigma'_{j_{2}}$}, inner sep = 4pt] (v34),
      (x3) --  [color1, ghost] (v34) -- [color1, ghost] (x4)
    };
  \end{feynman}
\end{tikzpicture}
\right) \nonumber\\
&  + \sum_{j_1, j_2}\!
\begin{tikzpicture}[baseline=(b)]
  \tikzfeynmanset{every blob = {/tikz/fill=white!50, /tikz/minimum size=15pt}}
  \begin{feynman}
    \vertex (l) {};
    \vertex[below = 15pt of l, dot, minimum size=3pt, label = {left: {\footnotesize $1$}}] (x1) {};
    \vertex[above = 15pt of l, color1, dot, minimum size=3pt, label = {left: {\footnotesize $\textcolor{color1}{3}$}}] (x2) {};
    \vertex[right = 8pt of l, dot, minimum size=0pt] (v12) {};
    \vertex[right = 35pt of v12, blob] (b12) {};
    \vertex[right = 16pt of b12, dot, minimum size=0pt] (w12) {};
    \vertex[above = 1pt of w12, label = {above: {\scriptsize \hspace{-10pt} $\phi_{j_1}$}}] (w12u) {};
    \vertex[below = 8pt of w12] (w12d) {};
    \vertex[left = 10pt of w12d] (w12dl) {};
    \tikzfeynmanset{every blob = {/tikz/fill=gray!50, /tikz/minimum size=15pt}}
    \vertex[right = 3pt of w12, blob , minimum size = 6pt] (b) {};
    \vertex[right = 3pt of b, dot, minimum size=0pt] (w34) {};
    \tikzfeynmanset{every blob = {/tikz/fill=white!50, /tikz/minimum size=15pt}}
    \vertex[right = 16pt of w34, blob] (b34) {};
    \vertex[above = 1pt of w34, label = {above: {\scriptsize \hspace{10pt} $\phi_{j_2}$}}] (w34u) {};
    \vertex[below = 8pt of w34] (w34d) {};
    \vertex[right = 10pt of w34d] (w34dr) {};
    \vertex[right = 35pt of b34, dot, minimum size=0pt] (v34) {};
    \vertex[right = 8pt of v34] (r) {};
    \vertex[above = 15pt of r, dot, minimum size=3pt, label = {right: {\footnotesize $2$}}] (x3) {};
    \vertex[below = 15pt of r, color1, dot, minimum size=3pt, label = {right: {\footnotesize $\textcolor{color1}{4}$}}] (x4) {};
    \diagram*{
      (x1) -- (v12) -- [color1, ghost] (x2),
      (v12) -- [color1blackghost, edge label = {\scriptsize \;$\sigma_{j_{1}}\sigma'_{j_{1}}$}, inner sep = 4pt] (b12) -- [color1, ghost, quarter left] (w12) -- [quarter left] (b12),
      (b34) -- [color1, ghost, quarter left] (w34) -- [quarter left] (b34) -- [blackcolor1ghost, edge label = {\scriptsize \,$\sigma_{j_{2}}\sigma'_{j_{2}}$}, inner sep = 4pt] (v34),
      (x3) --  (v34) -- [color1, ghost] (x4)
    };
    \draw [decoration={brace}, decorate] (w34dr) -- (w12dl) node [pos=0.5, below = 1pt] {\scriptsize $\frac{1}{n\*}F_4^{(\ell)}$};
  \end{feynman}
\end{tikzpicture} \label{ftensorfeynmandiagram}
\end{align}
The sums over channel indices are included in~\eqref{ftensorfeynmandiagram}. These are the only diagrams contributing at order \(1/n\); all others are either higher order or excluded by the selection rules.


The recursion relation for the tensor $B$ at order \(1/n\) follows analogously. In the Gaussian sector, ordering 1324, the only compatible cubic vertex is the fifth vertex in~\eqref{feynmanrulescubicapp}, yielding two diagrams: one with identical channel indices connected by a single propagator, and one with distinct indices connected by two propagators and an internal $B$-tensor at layer $\ell$. For the remaining orderings, the 1423 contribution vanishes by color-line preservation. The 1234 ordering, constructed from the cubic vertex in~\eqref{orthogonalityvertexsimp}, instead yields a nonvanishing contribution. The associated prefactor is determined by the cycle structure of the pairing $(12)(34)$ relative to $(13)(24)$, namely $\lambda (2)$ with $\beta_{(2)}=-1$. We therefore obtain
\begin{align}
  \begin{tikzpicture}[baseline=(b)]
      \begin{feynman}
        \vertex (l) {};
        \vertex[below = 15pt of l, dot, color1, minimum size=3pt, label = {left: {\footnotesize $\textcolor{color1}{1}$}}] (x1) {};
        \vertex[above = 15pt of l, dot, color2, minimum size=3pt, label = {left: {\footnotesize $\textcolor{color2}{3}$}}] (x2) {};
        \vertex[right = 20pt of l, quarticblob] (b) {};
        \vertex[right = 20pt of b] (r) {};
        \vertex[above = 15pt of r, dot, color1, minimum size=3pt, label = {right: {\footnotesize $\textcolor{color1}{2}$}}] (x3) {};
        \vertex[below = 15pt of r, dot, color2, minimum size=3pt, label = {right: {\footnotesize $\textcolor{color2}{4}$}}] (x4) {};
        \diagram*{
          (x1) -- [color1, ghost](b) -- [color2, ghost] (x2), 
          (x3) -- [color1, ghost](b) -- [color2, ghost] (x4)
        };
      \end{feynman}
    \end{tikzpicture}
\!\!&= 
\!\!\sum_{j}\!\!
\begin{tikzpicture}[baseline=(b)]
\tikzfeynmanset{every blob = {/tikz/fill=white!50, /tikz/minimum size=15pt}}
\begin{feynman}
\vertex (l) {};
\vertex[below = 15pt of l, dot, color1, minimum size=3pt, label = {left: {\footnotesize $\textcolor{color1}{1}$}}] (x1) {};
\vertex[above = 15pt of l, color2, dot, minimum size=3pt, label = {left: {\footnotesize $\textcolor{color2}{3}$}}] (x2) {};
\vertex[right = 8pt of l, dot, minimum size=0pt] (v12) {};
\vertex[right = 30pt of v12, blob] (b) {};
\vertex[right = 30pt of b, dot, minimum size=0pt] (v34) {};
\vertex[right = 8pt of v34] (r) {};
\vertex[above = 15pt of r, dot, color1, minimum size=3pt, label = {right: {\footnotesize $\textcolor{color1}{2}$}}] (x3) {};
\vertex[below = 15pt of r, dot, color2, minimum size=3pt, label = {right: {\footnotesize $\textcolor{color2}{4}$}}] (x4) {};
\diagram*{
	(x1) -- [color1, ghost] (v12) -- [color2, ghost] (x2),
	(v12) -- [color2color1ghost, edge label = {\scriptsize \;$\sigma'_{j}\sigma'_{j}$}, inner sep = 4pt] (b) -- [color1color2ghost, edge label = {\scriptsize \,$\sigma'_{j}\sigma'_{j}$}, inner sep = 4pt] (v34), 
	 (x3) -- [color1, ghost](v34) -- [color2, ghost] (x4)
};
\end{feynman}
\end{tikzpicture} 
+
\sum_{j_1, j_2}\!
\begin{tikzpicture}[baseline=(b)]
  \tikzfeynmanset{every blob = {/tikz/fill=white!50, /tikz/minimum size=15pt}}
  \begin{feynman}
    \vertex (l) {};
    \vertex[below = 15pt of l, dot, color1, minimum size=3pt, label = {left: {\footnotesize $\textcolor{color1}{1}$}}] (x1) {};
    \vertex[above = 15pt of l, color2, dot, minimum size=3pt, label = {left: {\footnotesize $\textcolor{color2}{3}$}}] (x2) {};
    \vertex[right = 8pt of l, dot, minimum size=0pt] (v12) {};
    \vertex[right = 35pt of v12, blob] (b12) {};
    \vertex[right = 16pt of b12, dot, minimum size=0pt] (w12) {};
    \vertex[above = 1pt of w12, label = {above: {\scriptsize \hspace{-10pt} $ $}}] (w12u) {};
    \vertex[below = 8pt of w12] (w12d) {};
    \vertex[left = 10pt of w12d] (w12dl) {};
    \tikzfeynmanset{every blob = {/tikz/fill=gray!50, /tikz/minimum size=15pt}}
    \vertex[right = 3pt of w12, blob , minimum size = 6pt] (b) {};
    \vertex[right = 3pt of b, dot, minimum size=0pt] (w34) {};
    \tikzfeynmanset{every blob = {/tikz/fill=white!50, /tikz/minimum size=15pt}}
    \vertex[right = 16pt of w34, blob] (b34) {};
    \vertex[above = 1pt of w34, label = {above: {\scriptsize \hspace{10pt} $ $}}] (w34u) {};
    \vertex[below = 8pt of w34] (w34d) {};
    \vertex[right = 10pt of w34d] (w34dr) {};
    \vertex[right = 35pt of b34, dot, minimum size=0pt] (v34) {};
    \vertex[right = 8pt of v34] (r) {};
    \vertex[above = 15pt of r, dot, color1, minimum size=3pt, label = {right: {\footnotesize $\textcolor{color1}{2}$}}] (x3) {};
    \vertex[below = 15pt of r, dot, color2, minimum size=3pt, label = {right: {\footnotesize $\textcolor{color2}{4}$}}] (x4) {};
    \diagram*{
      (x1) -- [color1, ghost] (v12) -- [color2, ghost] (x2),
      (v12) -- [color2color1ghost, edge label = {\scriptsize \;$\sigma_{j_{1}}\sigma'_{j_{1}}$}, inner sep = 4pt] (b12) -- [color2, ghost, quarter left] (w12) -- [color1, ghost, quarter left] (b12),
      (b34) -- [color2, ghost, quarter left] (w34) -- [color1, ghost, quarter left] (b34) -- [color1color2ghost, edge label = {\scriptsize \,$\sigma_{j_{2}}\sigma'_{j_{2}}$}, inner sep = 4pt] (v34),
      (x3) -- [color1, ghost] (v34) -- [color2, ghost] (x4)
    };
    \draw [decoration={brace}, decorate] (w34dr) -- (w12dl) node [pos=0.5, below = 1pt] {\scriptsize $\frac{1}{n\*}B_4^{(\ell)}$};
  \end{feynman}
\end{tikzpicture}\nonumber\\
&- 
\frac{1}{n}\left(
\sum_{j_1, j_2}\!
\begin{tikzpicture}[baseline=(b)]
  \tikzfeynmanset{every blob = {/tikz/fill=white!50, /tikz/minimum size=15pt}}
  \begin{feynman}
    \vertex (l) {};
    \vertex[below = 15pt of l, dot, color1, minimum size=3pt, label = {left: {\footnotesize $\textcolor{color1}{1}$}}] (x1) {};
    \vertex[above = 15pt of l, , dot, color1, minimum size=3pt, label = {left: {\footnotesize $\textcolor{color1}{2}$}}] (x2) {};
    \vertex[right = 8pt of l, dot, minimum size=0pt] (v12) {};
    \vertex[right = 35pt of v12, blob] (b12) {};
    \vertex[right = 5pt of b12, dot, minimum size=0pt] (w12) {};
    \vertex[above = 1pt of w12, label = {above: {\scriptsize \hspace{-10pt} $ $}}] (w12u) {};
    \vertex[below = 8pt of w12] (w12d) {};
    \vertex[left = 10pt of w12d] (w12dl) {};
    \tikzfeynmanset{every blob = {/tikz/fill=gray!50, /tikz/minimum size=0pt}}
    \vertex[right = 3pt of w12, blob , minimum size = 0pt] (b) {};
    \vertex[right = 3pt of b, dot, minimum size=0pt] (w34) {};
    \tikzfeynmanset{every blob = {/tikz/fill=white!50, /tikz/minimum size=15pt}}
    \vertex[right = 16pt of w34, blob] (b34) {};
    \vertex[above = 1pt of w34, label = {above: {\scriptsize \hspace{10pt} $ $}}] (w34u) {};
    \vertex[below = 8pt of w34] (w34d) {};
    \vertex[right = 10pt of w34d] (w34dr) {};
    \vertex[right = 35pt of b34, dot, minimum size=0pt] (v34) {};
    \vertex[right = 8pt of v34] (r) {};
    \vertex[above = 15pt of r, color2, dot, minimum size=3pt, label = {right: {\footnotesize $\textcolor{color2}{3}$}}] (x3) {};
    \vertex[below = 15pt of r, color2, dot, minimum size=3pt, label = {right: {\footnotesize $\textcolor{color2}{4}$}}] (x4) {};
    \diagram*{
      (x1) --  [color1, ghost] (v12) --  [color1, ghost] (x2),
      (v12) -- [color1doubghost, edge label = {\scriptsize \;$\sigma'_{j_{1}}\sigma'_{j_{1}}$}, inner sep = 4pt] (b12),
      (b34) -- [color2doubghost, edge label = {\scriptsize \,$\sigma'_{j_{2}}\sigma'_{j_{2}}$}, inner sep = 4pt] (v34),
      (x3) --  [color2, ghost] (v34) -- [color2, ghost] (x4)
    };
  \end{feynman}
\end{tikzpicture}
\right) \label{btensorfeynmandiagram}
\end{align}
The sums over channel indices are included in~\eqref{btensorfeynmandiagram}. These exhaust all contributions at order \(1/n\); all remaining diagrams are either higher order or ruled out by the selection rules.

We have verified that the leading-order Feynman rules (1)-(5) remain valid at order \(1/n^2\), as shown in Appendix~\ref{app:v6_tensor}.

A useful application of the leading-order \(1/n\) Feynman rules (1)-(5) is an alternative proof of Theorem~\ref{theoremfour}. At this order, orthogonality enters only through the introduction of lower-rank tensors relative to the tensor under consideration, as prescribed by rule (5). Consequently, criticality can be determined by analyzing the first term in~\eqref{leadingorderorthogonal}, corresponding to the Gaussian-like sector. In this way, network stability follows as in the Gaussian case. In particular, if the infinite-width NNGP and NTK are at criticality, then the orthogonal network is also at criticality.

\section{Proof of main theorems}
\label{app:proof_main_theorems}

\subsection{Proof of Theorem \ref{theoremone}}\label{app:proofs_one}

\firsttheorem*

This statement was already proved for the tensors $F$ and $B$ in Appendix~\ref{app:v_b4_f4}. To show that the Feynman rules of Appendix~\ref{app:general_feynman_rules} also determine the recursion relations for the tensors $D$ and $A$ at order \(1/n\), it suffices to argue that any corrections arising from orthogonality are at least of order \(1/n^{2}\), as can be inferred from~\eqref{drecursion} and~\eqref{eq:A_recursion_algebraic}. We now verify this case by case.

\subsubsection{Joint NTK-preactivation cumulant}\label{proofford}

\begin{proof}
Orthogonality enters only through Feynman rule (4), which introduces the Weingarten function $\mathcal{W}[2]$ for pairings different from the reference pairing. For the tensor $D$, the reference pairing is $\tau=(12)(\textcolor{color1}{34})$. The diagrams associated with pairings inequivalent to $\tau$ are therefore given by
\begin{align}
 \mathcal{W}[2]\,\!\!\sum_{j,k}\!\!
\begin{tikzpicture}[baseline=(b)]
\begin{feynman}
\vertex (l) {};
\vertex[below = 15pt of l, dot, minimum size=3pt, label = {left: {\footnotesize $1^{c}$}}] (x1) {};
\vertex[above = 15pt of l, color1, dot, minimum size=3pt, label = {left: {\footnotesize $\textcolor{color1}{3}^{c}$}}] (x2) {};
\vertex[right = 8pt of l, dot, minimum size=0pt] (v12) {};
\vertex[right = 30pt of v12, squareblob] (b) {};
\vertex[right = 30pt of b, dot, minimum size=0pt] (v34) {};
\vertex[right = 8pt of v34] (r) {};
\vertex[above = 15pt of r, dot, minimum size=3pt, label = {right: {\footnotesize $2^{c}$}}] (x3) {};
\vertex[below = 15pt of r, color1, dot, minimum size=3pt, label = {right: {\footnotesize $\textcolor{color1}{4}^{c}$}}] (x4) {};
\diagram*{
	(x1) -- (v12) -- [color1, ghost](x2),
	(v12) -- [color1blackghost, edge label = {\scriptsize \;$\sigma_{j}\sigma'_{j}$}, inner sep = 4pt] (b) -- [blackcolor1ghost, edge label = {\scriptsize \,$\sigma_{k}\sigma'_{k}$}, inner sep = 4pt] (v34), 
	 (x3) -- (v34) -- [color1, ghost](x4)
};
\end{feynman}
\end{tikzpicture} + \mathcal{W}[2]\,\!\!\sum_{j,k}\!\!
\begin{tikzpicture}[baseline=(b)]
\begin{feynman}
\vertex (l) {};
\vertex[below = 15pt of l, dot, minimum size=3pt, label = {left: {\footnotesize $1^{c}$}}] (x1) {};
\vertex[above = 15pt of l, color1, dot, minimum size=3pt, label = {left: {\footnotesize $\textcolor{color1}{4}^{c}$}}] (x2) {};
\vertex[right = 8pt of l, dot, minimum size=0pt] (v12) {};
\vertex[right = 30pt of v12, squareblob] (b) {};
\vertex[right = 30pt of b, dot, minimum size=0pt] (v34) {};
\vertex[right = 8pt of v34] (r) {};
\vertex[above = 15pt of r, dot, minimum size=3pt, label = {right: {\footnotesize $2^{c}$}}] (x3) {};
\vertex[below = 15pt of r, color1, dot, minimum size=3pt, label = {right: {\footnotesize $\textcolor{color1}{3}^{c}$}}] (x4) {};
\diagram*{
	(x1) -- (v12) -- [color1, ghost] (x2),
	(v12) -- [color1blackghost, edge label = {\scriptsize \;$\sigma_{j}\sigma'_{j}$}, inner sep = 4pt] (b) -- [blackcolor1ghost, edge label = {\scriptsize \,$\sigma_{k}\sigma'_{k}$}, inner sep = 4pt] (v34), 
	 (x3) -- (v34) -- [color1, ghost] (x4)
};
\end{feynman}
\end{tikzpicture}
\end{align}
Since the leading term of $\mathcal{W}[2]$ scales as \(1/n^{3}\), these contributions can survive at order \(1/n\) only if the corresponding diagrams produce factors of $n^{2}$. This would occur if the square propagator decomposed into two bare propagators with distinct neural indices $j$ and $k$. However, such configurations violate color conservation in the propagator, as specified by the selection rules (a)–(f). Consequently, these diagrams vanish at order \(1/n\), implying that the recursion relation for the orthogonal tensor $D$ coincides with its Gaussian counterpart at order \(1/n\).
\end{proof}

\subsubsection{NTK variance}

\begin{proof}
Similarly to~\eqref{proofford}, orthogonality induces the following diagrams for the tensor $A$ with reference pairing (\textcolor{color1}{12})(\textcolor{color2}{34}):
\begin{align}
\mathcal{W}[2]\,\!\!\sum_{j,k}\!\!
\begin{tikzpicture}[baseline=(b)]
\begin{feynman}
\vertex (l) {};
\vertex[below = 15pt of l, color2, dot, minimum size=3pt, label = {left: {\footnotesize $\textcolor{color2}{1}^{c}$}}] (x1) {};
\vertex[above = 15pt of l, color1, dot, minimum size=3pt, label = {left: {\footnotesize $\textcolor{color1}{3}^{c}$}}] (x2) {};
\vertex[right = 8pt of l, dot, minimum size=0pt] (v12) {};
\vertex[right = 30pt of v12, squareblob] (b) {};
\vertex[right = 30pt of b, dot, minimum size=0pt] (v34) {};
\vertex[right = 8pt of v34] (r) {};
\vertex[above = 15pt of r, color2, dot, minimum size=3pt, label = {right: {\footnotesize $\textcolor{color2}{2}^{c}$}}] (x3) {};
\vertex[below = 15pt of r, color1, dot, minimum size=3pt, label = {right: {\footnotesize $\textcolor{color1}{4}^{c}$}}] (x4) {};
\diagram*{
	(x1) [color2, ghost] -- (v12) -- [color1, ghost](x2),
	(v12) -- [color1color2ghost, edge label = {\scriptsize \;$\sigma_{j}\sigma'_{j}$}, inner sep = 4pt] (b) -- [color2color1ghost, edge label = {\scriptsize \,$\sigma_{k}\sigma'_{k}$}, inner sep = 4pt] (v34), 
	 (x3) -- [color2, ghost] (v34) -- [color1, ghost](x4)
};
\end{feynman}
\end{tikzpicture} + \mathcal{W}[2]\,\!\!\sum_{j,k}\!\!
\begin{tikzpicture}[baseline=(b)]
\begin{feynman}
\vertex (l) {};
\vertex[below = 15pt of l, color2, dot, minimum size=3pt, label = {left: {\footnotesize $\textcolor{color2}{1}^{c}$}}] (x1) {};
\vertex[above = 15pt of l, color1, dot, minimum size=3pt, label = {left: {\footnotesize $\textcolor{color1}{4}^{c}$}}] (x2) {};
\vertex[right = 8pt of l, dot, minimum size=0pt] (v12) {};
\vertex[right = 30pt of v12, squareblob] (b) {};
\vertex[right = 30pt of b, dot, minimum size=0pt] (v34) {};
\vertex[right = 8pt of v34] (r) {};
\vertex[above = 15pt of r, color2, dot, minimum size=3pt, label = {right: {\footnotesize $\textcolor{color2}{2}^{c}$}}] (x3) {};
\vertex[below = 15pt of r, color1, dot, minimum size=3pt, label = {right: {\footnotesize $\textcolor{color1}{3}^{c}$}}] (x4) {};
\diagram*{
	(x1) -- [color2, ghost](v12) -- [color1, ghost] (x2),
	(v12) -- [color1color2ghost, edge label = {\scriptsize \;$\sigma_{j}\sigma'_{j}$}, inner sep = 4pt] (b) -- [color2color1ghost, edge label = {\scriptsize \,$\sigma_{k}\sigma'_{k}$}, inner sep = 4pt] (v34), 
	 (x3) -- [color2,  ghost] (v34) -- [color1, ghost] (x4)
};
\end{feynman}
\end{tikzpicture}
\end{align}
Once again, the leading scaling $\mathcal{W}[2] \sim 1/n^3$ requires the diagrams to contribute a factor $n^{2}$. This would occur if the square propagator decomposed into two bare propagators with distinct neural indices $j$ and $k$. However, such configurations violate color conservation in the propagators and therefore vanish at order \(1/n\). Consequently, the recursion relation for the orthogonal tensor $A$ coincides with its Gaussian counterpart at order \(1/n\).
\end{proof}

\subsection{Proof of Theorem~\ref{theoremtwo}}\label{app:proofs_two}

\secondtheorem*

\begin{proof}
We prove this statement by showing that the recursion relations for the dNTK and ddNTK tensors coincide with those derived in the Gaussian case. As discussed in the main text, this result is further supported by the numerical simulations presented in Section~\ref{subsec:single-input} for the single-input setting.

To see this, note that orthogonality enters through Feynman rule (4), which introduces the Weingarten function $W[2]$ multiplying the 1-class diagrams obtained from inequivalent permutations with respect to the reference pairing. Since the leading scaling of $W[2]$ is \(1/n^3\), such contributions could survive at order \(1/n\) only if the square propagator decomposed into two bare propagators, so that the corresponding Gaussian expectation values with distinct neural indices produce a factor of $n^{2}$. However, for the preactivation-dNTK cumulant and the mean of the dNTK this configuration is forbidden, as it violates color conservation in the bare propagators. Indeed, since the dNTK and ddNTKs are represented by colored lines as in~\eqref{eq:6}, and the cubic vertices involving these lines always include an internal line carrying at least one color \eqref{orthogonalityvertex}, no bare propagator can attach to such vertices. Consequently, the diagrams proportional to $\mathcal{W}[2]$ in the recursion relations of the dNTK and ddNTK tensors decouple at order \(1/n\). Therefore, the recursion relations for the orthogonal dNTK and ddNTK tensors coincide with those obtained in the Gaussian setup.
\end{proof}

\subsection{Proof of Theorem~\ref{theoremthree}}
\label{app:proofs_three}

\thirdtheorem*

\begin{proof}
We begin by observing that the restriction to cubic vertices with two external legs and one internal line in our Feynman rules follows from two structural properties of orthogonal weights: the vanishing of all odd moments, and the fact that weight contractions always occur in pairs of neural indices.

The square propagator introduced in Section~\ref{sec:feynman-diagrams} generates all possible pairings of internal lines. This is equivalent to summing over all connected and disconnected diagrams with an even number of external legs. In parallel, the permutation of external labels carrying orthogonal charge, weighted by the corresponding Weingarten functions, implements the orthogonal contractions of the weight parameters appearing in the definition of the observables. The signs associated with each diagram match those prescribed by the Möbius relation (Rule (5)), thereby reproducing the standard combinatorial structure of cumulants.

Concretely, applying this procedure to $n$ preactivations with a fixed pairing structure generates all even-point diagrams, each weighted by the appropriate Weingarten function. Since these weights encode orthogonal contractions, the square propagators can be identified with expectation values over the weights multiplying the corresponding activations. This allows one to rewrite the full expression as products of parameter expectation values of preactivations $z_{\alpha,i} = \sum_{j}W_{ij}\sigma_{\alpha,j}$, each weighted by Möbius factors determined by the associated set partition. This coincides with the standard definition of the cumulant of $n$ preactivations. 

As an illustration, consider six preactivations with reference pairing $(12)(34)(56)$. Following the discussion above, the square propagator generates all possible even-point connected and disconnected diagrams, weighted by the corresponding Weingarten functions $\mathcal{W}[1,1,1]$, $\mathcal{W}[2,1]$, $\mathcal{W}[3]$, $\mathcal{W}[2]$, $\mathcal{W}[1,1]$, $\mathcal{W}[1]$ (see Appendix~\ref{v6-diagrammatic-derivation} for details). These weights and their associated index contractions coincide with those obtained from the orthogonality of the parameters, allowing one to rewrite the resulting contributions as products of expectation values such as
\begin{align}
\mathbb{E}[z_{1}z_{2}z_{3}z_{4}z_{5}z_{6}] \quad, \quad\mathbb{E}[z_{1}z_{2}z_{3}z_{4}]&\mathbb{E}[z_{5}z_{6}]\quad, \quad\mathbb{E}[z_{1}z_{2}z_{5}z_{6}]\mathbb{E}[z_{3}z_{4}]\quad,\nonumber\\
\mathbb{E}[z_{3}z_{4}z_{5}z_{6}]\mathbb{E}[z_{1}z_{2}] \quad&, \quad\mathbb{E}[z_{1}z_{2}]\mathbb{E}[z_{3}z_{4}]\mathbb{E}[z_{5}z_{6}]
\end{align}
Upon attaching the corresponding Möbius prefactors and summing over all partitions, one recovers the contribution of the channel $(12)(34)(56)$ to the standard sextic cumulant.

The use of the multileg definitions~\eqref{eq:6}, together with the Feynman rules with coloured lines in~\eqref{orthogonalityvertex}, allows one to straightforwardly extend this argument to composite observables such as the NTK, dNTK, and ddNTK, which are multilinear in the weights.

The second set of Feynman rules implements an effective field theory expansion for non-Gaussian probability distributions. In this framework, the square propagator (full expectation value) decomposes into a Gaussian (free) propagator and interaction vertices arising from the \(1/n\)-expansion of the distribution. As shown in prior work~\cite{guillen2025}, this expansion is governed by a unique cubic vertex structure obtained by Gaussianizing the argument of the square propagator, together with higher-order interaction vertices. In particular, no new cubic vertices appear at higher orders, highlighting the uniqueness of our Feynman rules.

Combining these observations, Rules (1)-(5) reproduce the cumulant structure with orthogonal contractions, while Rules (6)-(9) generate the \(1/n\)-expansion of the expectation values entering the cumulant in terms of Gaussian propagators and higher-order interaction vertices. This completes the proof.
\end{proof}

\section{Analytical recursion relations: Single-input case}
\label{app:analytical_recursion_relations}

In this appendix, we list the recursion relations obtained from the Feynman diagram approach in the single-input setting for the NTK, dNTK, and ddNTK tensors. The expressions for the tensors $D$, $F$, $A$ and $B$ follow directly from~\eqref{drecursion}, \eqref{ftensoratorderonen}, \eqref{eq:A_recursion_algebraic}, and~\eqref{eq:B_recursion_algebraic}, respectively. The relations for the tensors $P$, $Q$, $R$, $S$, $T$, $U$ can be derived analogously through the systematic application of the Feynman rules in Appendix~\ref{app:general_feynman_rules}. Since these relations coincide with their Gaussian counterparts at order \(1/n\), one may also use the results of~\cite{roberts2021a}. Explicitly, we obtain

\begin{align}
D^{(\ell+1)}
&= \chi_{\perp}^{(\ell)} \chi_{\parallel}^{(\ell)} D^{(\ell)}
   + \left( \frac{\lambda_{W}^{(\ell+1)}}{C_{W}} \right)
     \Big[
        C_{W}^{2}\,
        \langle \sigma(z)\sigma(z)\sigma(z)\sigma(z)\rangle_{K^{(\ell)}}
        - \big(C_{W} g^{(\ell)}\big)^{2}
        + \big(\chi_{\parallel}^{(\ell)}\big)^{2} V^{(\ell)}
     \Big] \nonumber\\
&\quad
   + \Theta^{(\ell)}
     \Big[
        C_{W}^{2}\,
        \langle \sigma(z)\sigma(z)\sigma'(z)\sigma'(z)\rangle_{K^{(\ell)}}
        - C_{W} g^{(\ell)}\, \chi_{\perp}^{(\ell)}
        + 2 h^{(\ell)}\, \chi_{\parallel}^{(\ell)}\, V^{(\ell)}
     \Big], \label{dsingleinput}
\\[1.2em]
F^{(\ell+1)}
&= \big(\chi_{\parallel}^{(\ell)}\big)^{2} F^{(\ell)}
   + C_{W}^{2}\,\Big(
     \langle \sigma(z)\sigma(z)\sigma'(z)\sigma'(z)\rangle_{K^{(\ell)}} - \langle \sigma(z)\sigma(z)\rangle_{K^{(\ell)}}\langle\sigma'(z)\sigma'(z)\rangle_{K^{(\ell)}}\Big)
     \,\Theta^{(\ell)}, \label{fsingleinput}
\\[1.2em]
A^{(\ell+1)}
&= \big(\chi_{\perp}^{(\ell)}\big)^{2} A^{(\ell)}
   + \left( \frac{\lambda_{W}^{(\ell+1)}}{C_{W}} \right)^{2}
     \Big[
        C_{W}^{2}\,
        \langle \sigma(z)\sigma(z)\sigma(z)\sigma(z)\rangle_{K^{(\ell)}}
        - \big(C_{W} g^{(\ell)}\big)^{2}
        + \big(\chi_{\parallel}^{(\ell)}\big)^{2} V^{(\ell)}
     \Big] \nonumber\\
&\quad
   + 2 \left( \frac{\lambda_{W}^{(\ell+1)}}{C_{W}} \right)
     \Theta^{(\ell)}
     \Big[
        C_{W}^{2}\,
        \langle \sigma(z)\sigma(z)\sigma'(z)\sigma'(z)\rangle_{K^{(\ell)}}
        - C_{W} g^{(\ell)}\, \chi_{\perp}^{(\ell)}
        + 2 h^{(\ell)}\, \chi_{\parallel}^{(\ell)}\, V^{(\ell)}
     \Big] \nonumber\\
&\quad
   + 2 \left( \frac{\lambda_{W}^{(\ell+1)}}{C_{W}} \right)
     \chi_{\perp}^{(\ell)}\chi_{\parallel}^{(\ell)} D^{(\ell)}
   + 4 h^{(\ell)} \chi_{\perp}^{(\ell)} \Theta^{(\ell)} D^{(\ell)} \nonumber\\
&\quad
   + \big(\Theta^{(\ell)}\big)^{2}
     \Big[
        C_{W}^{2}\,
        \langle \sigma'(z)\sigma'(z)\sigma'(z)\sigma'(z)\rangle_{K^{(\ell)}}
        - \big(\chi_{\perp}^{(\ell)}\big)^{2}
        + \big(2 h^{(\ell)}\big)^{2} V^{(\ell)}
     \Big], \label{asingleinput} 
\\[1.8em]
B^{(\ell+1)}
&= \big(\chi_{\perp}^{(\ell)}\big)^{2} B^{(\ell)}
   + C_{W}^{2}\,\Big(
     \langle \sigma'(z)\sigma'(z)\sigma'(z)\sigma'(z)\rangle_{K^{(\ell)}} - \langle \sigma'(z)\sigma'(z)\rangle_{K^{(\ell)}}\langle\sigma'(z)\sigma'(z)\rangle_{K^{(\ell)}}
     \Big)\big(\Theta^{(\ell)}\big)^{2}, \label{bsingleinput}
\\[1.2em]
P^{(\ell+1)}
&= C_{W}^{2}\,
   \big\langle \sigma''(z)\sigma'(z)\sigma'(z)\sigma(z) \big\rangle_{K^{(\ell)}}\,
   (\Theta^{(\ell)})^{2}
 + C_{W}\, \chi_{\perp}^{(\ell)}\,
   \big\langle \sigma''(z)\sigma(z) \big\rangle_{K^{(\ell)}}\, B^{(\ell)} 
\nonumber\\[4pt]
&\quad
 + \Big[
      C_{W}\, \chi_{\perp}^{(\ell)}\,
      \big\langle \sigma''(z)\sigma(z) \big\rangle_{K^{(\ell)}}
      + \big(\chi_{\perp}^{(\ell)}\big)^{2}
   \Big] P^{(\ell)}, \label{psingleinput} 
\\[1.2em]
Q^{(\ell+1)}
&= C_{W}^{2}\,
   \big\langle \sigma''(z)\sigma'(z)\sigma'(z)\sigma(z) \big\rangle_{K^{(\ell)}}\,
   (\Theta^{(\ell)})^{2}
 + \frac{\lambda_{W}^{(\ell+1)}}{C_{W}}\, F^{(\ell+1)}
 + 2 h^{(\ell)}\, \chi_{\parallel}^{(\ell)}\, \Theta^{(\ell)}\, F^{(\ell)}
\nonumber\\[4pt]
&\quad
 + \Big[
      C_{W}\, \chi_{\perp}^{(\ell)}\,
      \big\langle \sigma''(z)\sigma(z) \big\rangle_{K^{(\ell)}}
      + \big(\chi_{\perp}^{(\ell)}\big)^{2}
   \Big] Q^{(\ell)} \label{qsingleinput} \\
R^{(\ell+1)}
&= \big(\chi_{\perp}^{(\ell)}\big)^{2} R^{(\ell)}
 + \lambda_{W}^{(\ell+1)}C_{W} \langle \sigma''(z)\sigma'(z)\sigma'(z)\sigma(z)\rangle_{K^{(\ell)}}
       (\Theta^{(\ell)})^{2}
    \nonumber\\[4pt]
&\quad  + C_{W}^{2}\,
       \langle \sigma'''(z)\sigma'(z)\sigma'(z)\sigma'(z)\rangle_{K^{(\ell)}}
       (\Theta^{(\ell)})^{3}
\nonumber\\[4pt]
&\quad
 + \chi_{\perp}^{(\ell)}
   \Big(
      \lambda^{(\ell+1)}_{W}\,
      \langle \sigma''(z)\sigma(z)\rangle_{K^{(\ell)}}\,
      + C_{W}\Theta^{(\ell)}\langle \sigma'''(z)\sigma'(z)\rangle_{K^{(\ell)}}\Big)\Big(B^{(\ell)} + P^{(\ell)}\Big)\nonumber\\[4pt]
&\quad
 + \chi_{\perp}^{(\ell)}
   \Big(
      \lambda^{(\ell+1)}_{W}\,
      \langle \sigma'(z)\sigma'(z)\rangle_{K^{(\ell)}}\,
      + C_{W}\Theta^{(\ell)}\langle \sigma''(z)\sigma''(z)\rangle_{K^{(\ell)}}\Big)P^{(\ell)} \label{rsingleinput},
\\[1.2em]
S^{(\ell+1)}
&= \big(\chi_{\perp}^{(\ell)}\big)^{2} S^{(\ell)}
 + \lambda_{W}^{(\ell+1)}C_{W} \langle \sigma'(z)\sigma'(z)\sigma'(z)\sigma'(z)\rangle_{K^{(\ell)}}\,
      (\Theta^{(\ell)})^{2}
    \nonumber\\[4pt]
&\quad + C_{W}^{2}\,
      \langle \sigma''(z)\sigma''(z)\sigma'(z)\sigma'(z)\rangle_{K^{(\ell)}}\,
      (\Theta^{(\ell)})^{3}
\nonumber\\[4pt]
&\quad
 + \chi_{\perp}^{(\ell)}
   \Big[
      \lambda_{W}^{(\ell+1)}
      \langle \sigma'(z)\sigma'(z)\rangle_{K^{(\ell)}} 
      + C_{W}\Theta^{(\ell)}\langle \sigma''(z)\sigma''(z)\rangle_{K^{(\ell)}}\Big] B^{(\ell)}, \label{ssingleinput} 
\\[1.2em]
T^{(\ell+1)}
&= \big(\chi_{\perp}^{(\ell)}\big)^{2} T^{(\ell)}
 + 2C_{W}\lambda_{W}^{(\ell+1)}\,
   \langle \sigma''(z)\sigma'(z)\sigma'(z)\sigma(z)\rangle_{K^{(\ell)}}\,
   (\Theta^{(\ell)})^{2}
 \nonumber\\[4pt]
&\quad + C_{W}^{2}\,
   \langle \sigma''(z)\sigma''(z)\sigma'(z)\sigma'(z)\rangle_{K^{(\ell)}}\,
   (\Theta^{(\ell)})^{3}
\nonumber\\[4pt]
&\quad
 + (\lambda_{W}^{(\ell+1)})^{2}\,
   \Theta^{(\ell)}\langle \sigma'(z)\sigma'(z)\sigma(z)\sigma(z)\rangle_{K^{(\ell)}}
\nonumber\\[4pt]
&\quad + \Big(\lambda_{W}^{(\ell+1)}\langle z \sigma'(z)\sigma(z)\rangle_{K^{(\ell)}} + C_{W}\Theta^{(\ell)}\langle z\sigma''(z)\sigma'(z)\rangle_{K^{(\ell)}}\Big)^{2}\frac{F^{(\ell)}}{(K^{(\ell)})^{2}}\nonumber\\[4pt]
&\quad
+ 2\chi_{\perp}^{(\ell)}\bigg[\lambda_{W}^{(\ell+1)}\Big(\langle \sigma''(z)\sigma(z)\rangle_{K^{(\ell)}} + \langle \sigma'(z)\sigma'(z)\rangle_{K^{(\ell)}}\Big) \nonumber\\[4pt]
&\quad + C_{W}\Theta^{(\ell)}\Big(\langle \sigma'''(z)\sigma'(z)\rangle_{K^{(\ell)}} + \langle \sigma''(z)\sigma''(z)\rangle_{K^{(\ell)}}\Big)\bigg]Q^{(\ell)}, \label{tsingleinput} 
\\[1.2em]
U^{(\ell+1)}
&= \big(\chi_{\perp}^{(\ell)}\big)^{2} U^{(\ell)}
 + C_{W}^{2}\,
   \langle \sigma''(z)\sigma''(z)\sigma'(z)\sigma'(z)\rangle_{K^{(\ell)}}\,
   (\Theta^{(\ell)})^{3}, \label{usingleinput} 
\end{align}
where the susceptibilities $\chi^{(\ell)}_{||}$, $\chi^{(\ell)}_{\perp}$, and the auxiliary functions $g^{(\ell)}(K)$, $h^{(\ell)}(K)$ are defined as follows:
\begin{align}
\chi^{(\ell)}_{||} = \frac{C_{W}}{K^{(\ell)}}\langle z^{(\ell)}\sigma^{(\ell)}(z)\sigma'^{(\ell)}(z)\rangle_{K^{(\ell)}}, & \quad \chi^{(\ell)}_{\perp} = C_{W}\langle \sigma'^{(\ell)}(z)\sigma'^{(\ell)}(z)\rangle_{K^{(\ell)}} \nonumber\\
\quad h^{(\ell)}(K) = \frac{C_{W}}{4(K^{(\ell)})^{2}}\langle((z^{(\ell)})^{2}-K^{(\ell)})(\sigma'^{(\ell)}(z)\sigma'^{(\ell)}(z))\rangle_{K^{(\ell)}}, & \quad g^{(\ell)}(K) = \langle \sigma^{(\ell)}\sigma^{(\ell)}\rangle_{K^{(\ell)}}
\end{align}

\subsection{Solution to the single-input recursions}\label{app:solut-single-input}
In this subsection, we analyze the solutions of the recursions~\eqref{dsingleinput}–\eqref{usingleinput}, shown in Figure~\ref{fig:all_tensors_plots}. As in Section~\ref{subsec:single-input}, we consider a square neural network with $\tanh$ activation, width $n=50$, and depth $L=10$, and perform both numerical and symbolic computations in \textit{Mathematica}. The input vector $x_{0}\in\mathbb{R}^{50}$ is taken to be 
\begin{align}\label{n50-inputvector}
x_0 &=
\begin{pmatrix}
0.934738 \\
0.26696 \\
0.784097 \\
0.656448 \\
0.305308 \\
0.401958 \\
0.894594 \\
0.0559893 \\
0.000643274 \\
0.0274513
\end{pmatrix}
\oplus
\begin{pmatrix}
0.377754 \\
0.127474 \\
0.879907 \\
0.710555 \\
0.509949 \\
0.312682 \\
0.0854376 \\
0.869372 \\
0.114232 \\
0.0851646
\end{pmatrix}
\oplus
\begin{pmatrix}
0.254697 \\
0.560475 \\
0.508664 \\
0.0271565 \\
0.426426 \\
0.457646 \\
0.913778 \\
0.40436 \\
0.407187 \\
0.0644401
\end{pmatrix}
\oplus
\begin{pmatrix}
0.256718 \\
0.869761 \\
0.0406222 \\
0.431362 \\
0.906228 \\
0.55979 \\
0.275852 \\
0.553722 \\
0.235762 \\
0.751627
\end{pmatrix}
\oplus
\begin{pmatrix}
0.178558 \\
0.411167 \\
0.100846 \\
0.220264 \\
0.215917 \\
0.490943 \\
0.596323 \\
0.0799147 \\
0.205998 \\
0.0372218
\end{pmatrix}\,.
\end{align}

The normalized tensors are defined by
\begin{align}   
 \tilde{B}^{(\ell)} &= \frac{B^{(\ell)}}{\left(\Theta^{(\ell)}\right)^{2}} , &
 \tilde{A}^{(\ell)} &= \frac{A^{(\ell)}}{\left(\Theta^{(\ell)}\right)^{2}} , &
 \tilde{P}^{(\ell)} &= \frac{P^{(\ell)}}{\left(\Theta^{(\ell)}\right)^{2}}, &
 \tilde{Q}^{(\ell)} &= \frac{Q^{(\ell)}}{\left(\Theta^{(\ell)}\right)^{2}},  \nonumber\\
 \tilde{R}^{(\ell)} &= \frac{R^{(\ell)}K^{(\ell)}}{\left(\Theta^{(\ell)}\right)^{3}}, &
 \tilde{S}^{(\ell)} &= \frac{S^{(\ell)}K^{(\ell)}}{\left(\Theta^{(\ell)}\right)^{3}}, &
 \tilde{T}^{(\ell)} &= \frac{T^{(\ell)}K^{(\ell)}}{\left(\Theta^{(\ell)}\right)^{3}}, &
 \tilde{U}^{(\ell)} &= \frac{U^{(\ell)}K^{(\ell)}}{\left(\Theta^{(\ell)}\right)^{3}}\,.
\end{align}

We report results for the normalized tensors in the following order: the NTK-mixed tensors $\tilde{D}$ and $\tilde{F}$ (Figure~\ref{fig:D_F_tensors_plot}); the NTK variance tensors $\tilde{A}$ and $\tilde{B}$ (Figure~\ref{fig:A_B_tensors_plot}); the dNTK tensors $\tilde{P}$ and $\tilde{Q}$ (Figure~\ref{fig:P_Q_tensors_plot}); the dd$_{\text{I}}$NTK tensor $\tilde{R}$ (Figure~\ref{fig:R_tensor_plot}); and the dd$_{\text{II}}$NTK tensors $\tilde{S}$, $\tilde{T}$, and $\tilde{U}$ (Figure~\ref{fig:S_T_U_tensors_plot}). The corresponding large-$\ell$ expansions, derived in Appendix~\ref{app:large_l_expansion}, are shown for comparison.
\begin{figure*}[tb]
  \centering
  \includegraphics[width=1.0\textwidth,keepaspectratio]{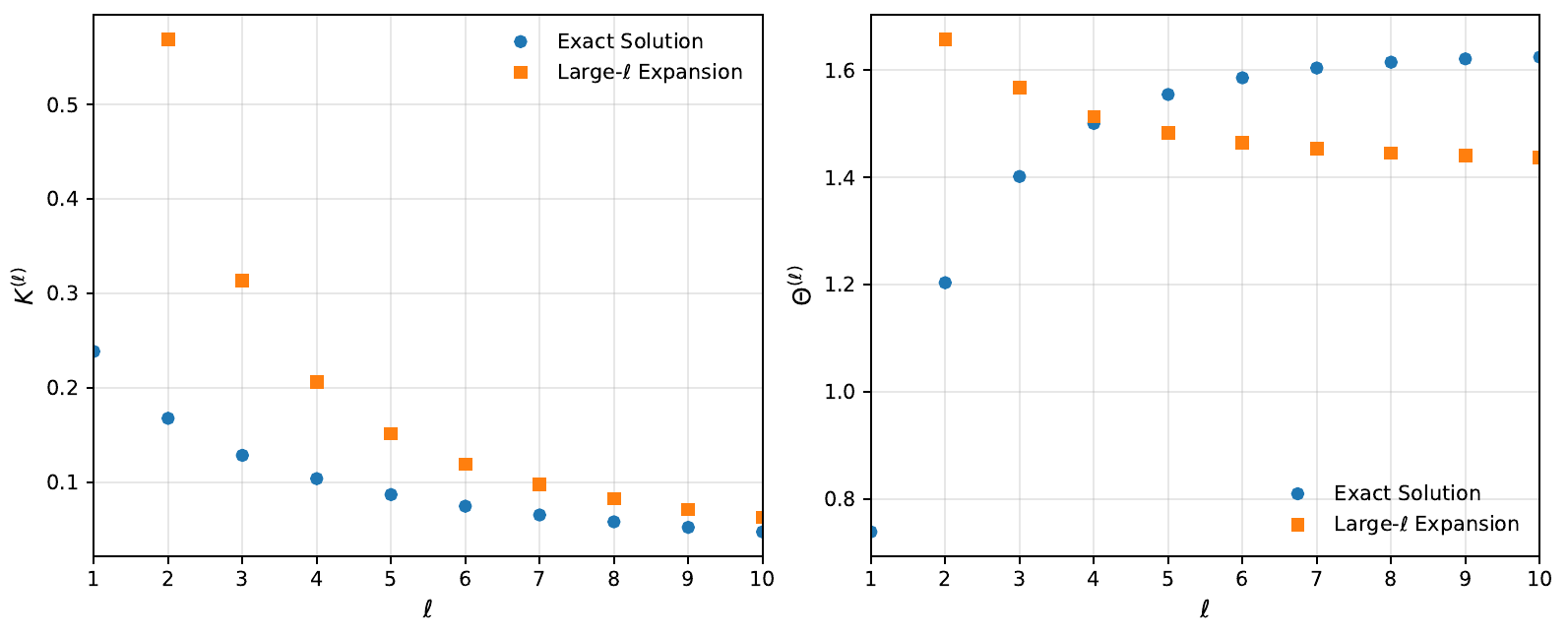}
  \vspace{-0.1cm}
  \caption{\emph{The NNGP and NTK.} Solution of the single-input recursion relations~\eqref{nngprecursion} and~\eqref{ntkrecursionleadingorder}. We consider a $\tanh$ network of width $n=50$ and depth $L=10$ with inputs drawn from $(0,1)$. Blue points denote the exact solutions of~\eqref{nngplexpansion} and~\eqref{ntkrecursionleadingorder}, while orange boxes represent the large-$\ell$ expansions~\eqref{Klargel} and~\eqref{Thetalargel}, evaluated at integer $\ell$. The tensor magnitudes remain smaller than their Gaussian counterparts and exhibit early-layer saturation.}
 \label{fig:NNGP_NTK_tensors_plots}
\end{figure*}
\begin{figure*}[tb]
  \centering
  \includegraphics[width=1.0\textwidth,height=1.0\textheight,keepaspectratio]{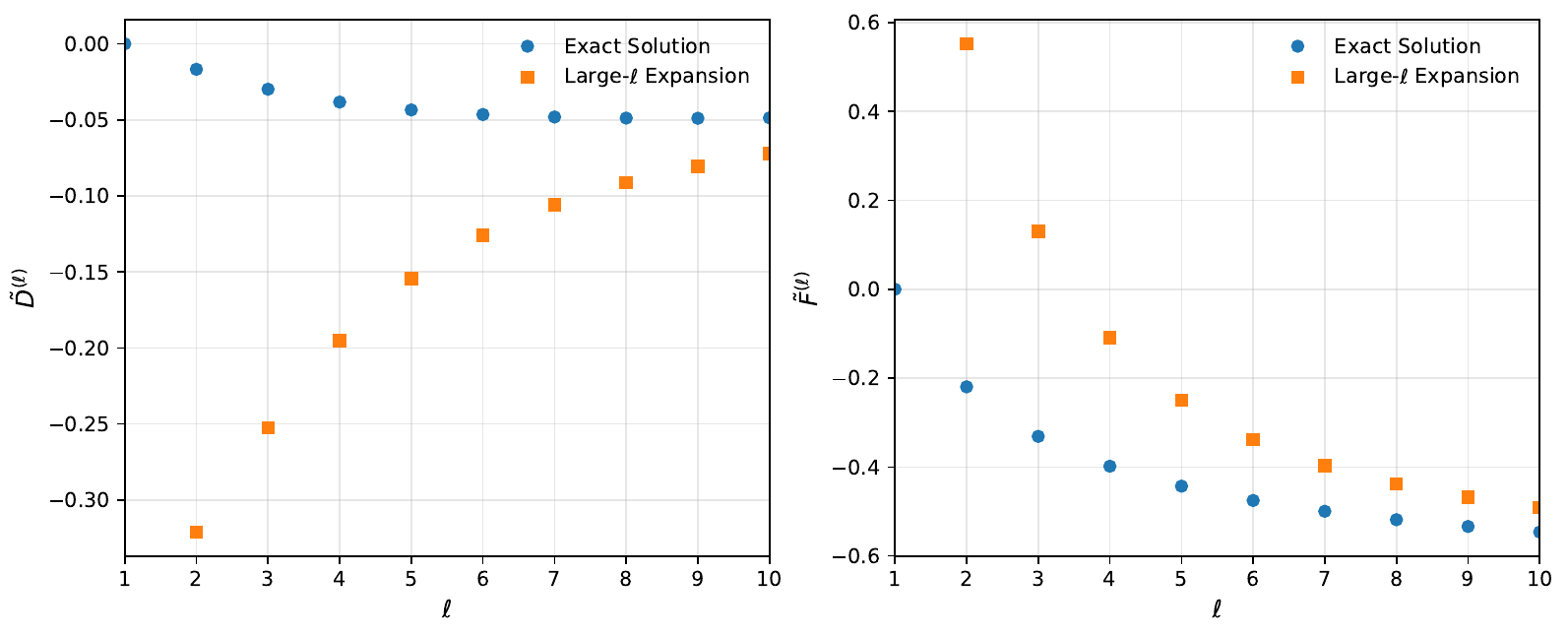}
  \vspace{-0.1cm}
  \caption{\emph{The NTK mixed tensors.} Solution of the single-input recursion relations~\eqref{dsingleinput} and~\eqref{fsingleinput}. We consider a $\tanh$ network of width $n=50$ and depth $L=10$ with inputs drawn from $(0,1)$. Blue points denote the exact solutions of~\eqref{dsingleinput} and~\eqref{fsingleinput}, while orange boxes represent the large-$\ell$ expansions~\eqref{Dlargel} and~\eqref{Flargel}, evaluated at integer $\ell$. The tensor magnitudes remain smaller than their Gaussian counterparts and exhibit early-layer saturation.}
  \label{fig:D_F_tensors_plot}
\end{figure*}
\begin{figure*}[tb]
  \centering
  \includegraphics[width=1.0\textwidth,height=0.5\textheight,keepaspectratio]{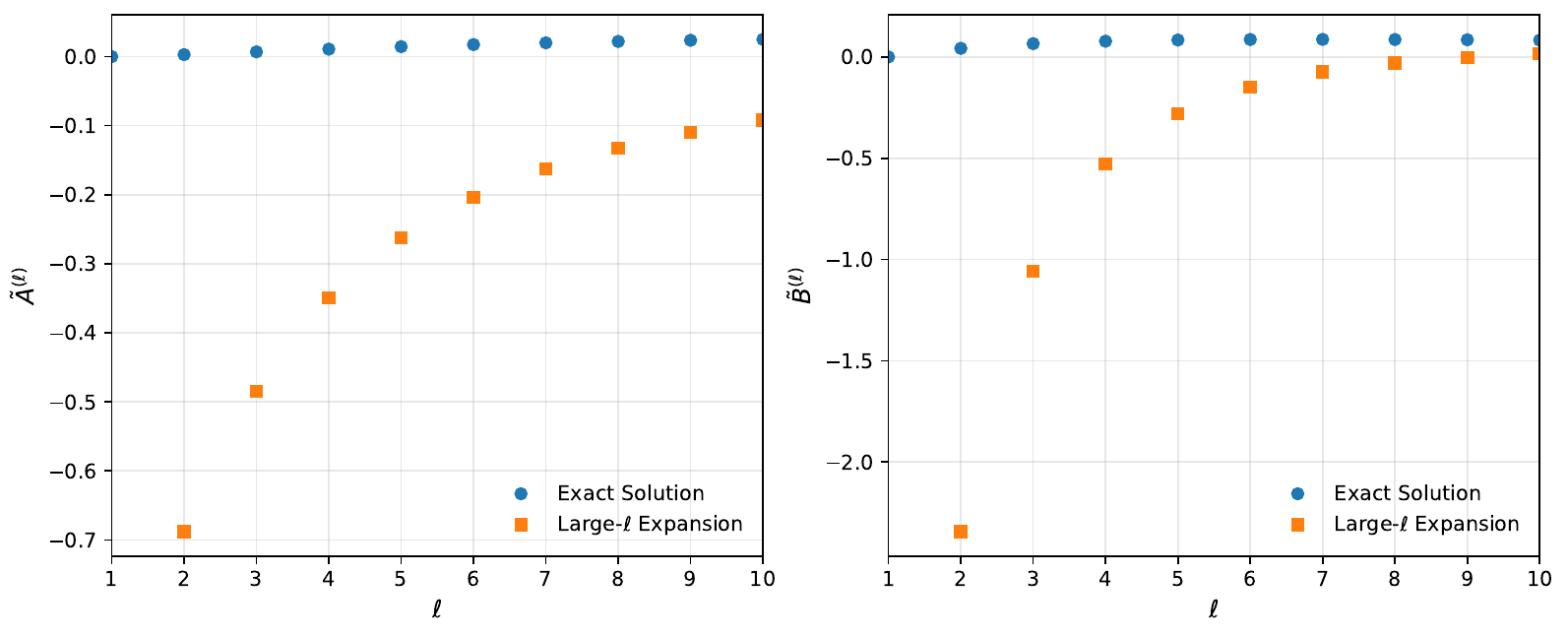}
  \vspace{-0.1cm}
  \caption{\emph{The NTK variance tensors.} Solution of the single-input recursion relations~\eqref{asingleinput} and~\eqref{bsingleinput}. We consider a $\tanh$ network of width $n=50$ and depth $L=10$ with inputs drawn from $(0,1)$. Blue points denote the exact solutions of~\eqref{asingleinput} and~\eqref{bsingleinput}, while orange boxes represent the large-$\ell$ expansion~\eqref{Alargel} and~\eqref{Blargel}, evaluated at integer $\ell$. The magnitude of the tensors remains smaller than its Gaussian counterpart and exhibits early-layer saturation.}
  \label{fig:A_B_tensors_plot}
\end{figure*}
\begin{figure*}[tb]
  \centering
  \includegraphics[width=1.0\textwidth,height=0.5\textheight,keepaspectratio]{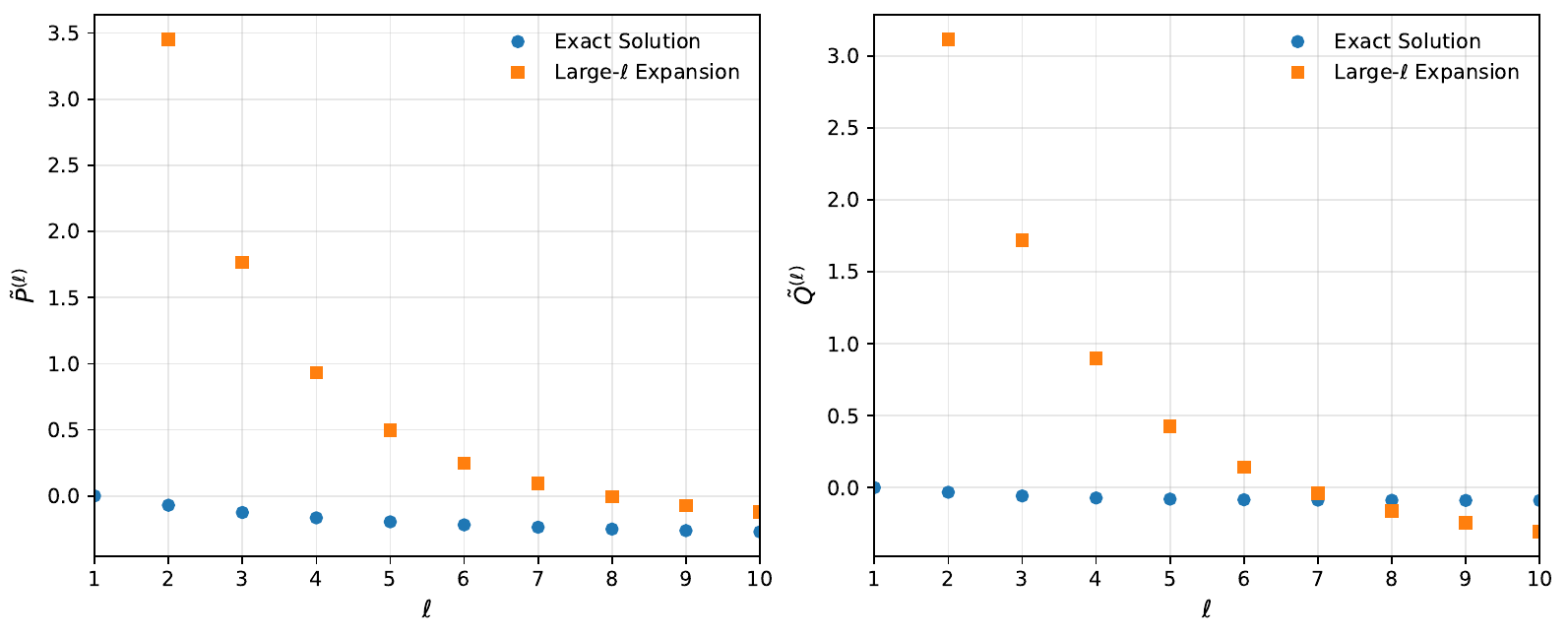}
  \vspace{-0.1cm}
  \caption{\emph{The dNTK tensors.} Solutions of the single-input recursion relations~\eqref{psingleinput} and~\ref{qsingleinput}. We consider a $\tanh$ network of width $n=50$ and depth $L=10$ with inputs drawn from $(0,1)$. Blue points denote the exact solution of~\eqref{psingleinput} and~\ref{qsingleinput}, while orange boxes represent the large-$\ell$ expansions~\eqref{Plargel} and~\eqref{Qlargel}, evaluated at integer $\ell$. The magnitude of the tensors remains smaller than its Gaussian counterpart and exhibits early-layer saturation.}
  \label{fig:P_Q_tensors_plot}
\end{figure*}
\begin{figure*}[tb]
  \centering
  \includegraphics[width=.5\textwidth,height=0.5\textheight,keepaspectratio]{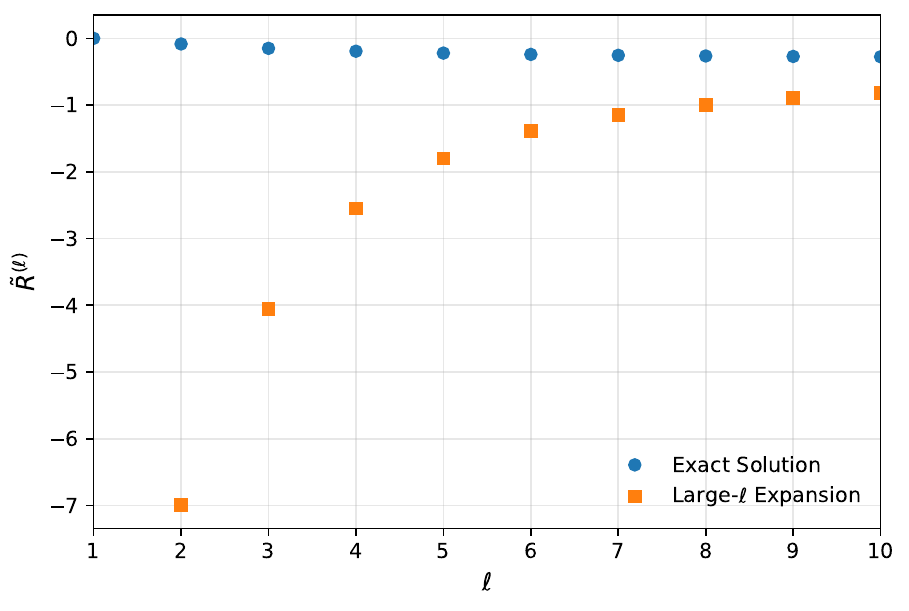}
  \vspace{-0.1cm}
  \caption{\emph{The d$_{\text{I}}$NTK tensor.} Solution of the single-input recursion relation~\eqref{rsingleinput}. We consider a $\tanh$ network of width $n=50$ and depth $L=10$ with inputs drawn from $(0,1)$. Blue points denote the exact solution of~\eqref{rsingleinput}, while orange boxes represent the large-$\ell$ expansion~\eqref{Rlargel}, evaluated at integer $\ell$. The magnitude of the tensor remains smaller than its Gaussian counterpart and exhibits early-layer saturation.}
  \label{fig:R_tensor_plot}
\end{figure*}
\begin{figure*}[tb]
  \centering
  \includegraphics[width=1.0\textwidth,height=0.5\textheight,keepaspectratio]{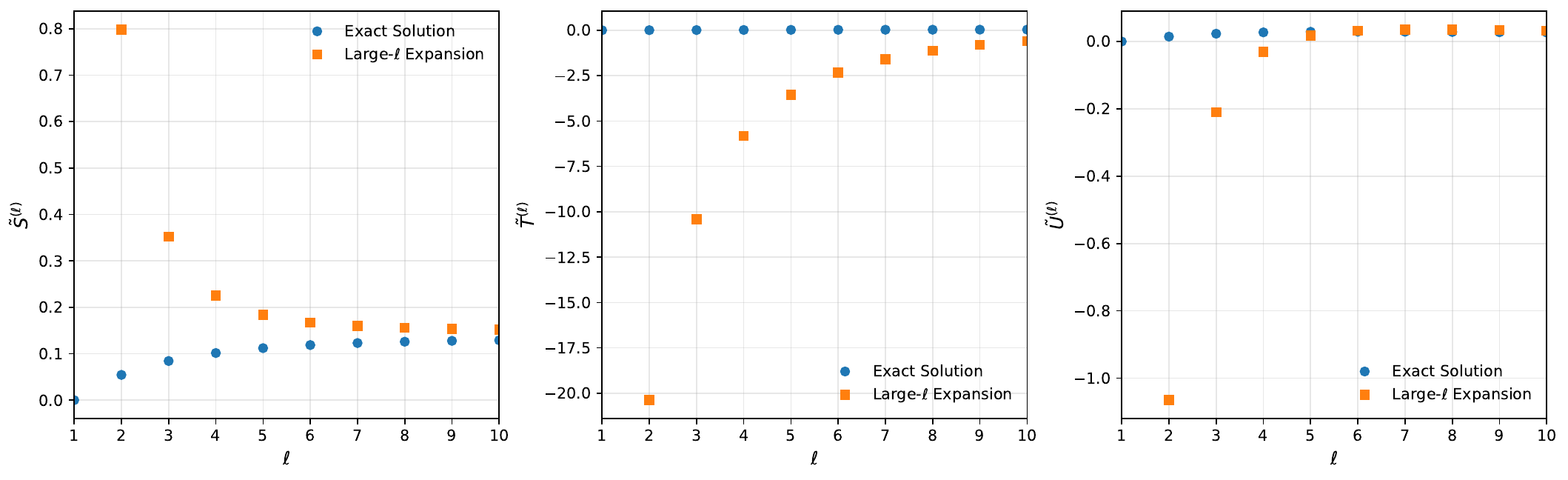}
  \vspace{-0.1cm}
  \caption{\emph{The dd$_{\text{II}}$NTK tensors.} Solutions to the single-input recursion relations~\eqref{ssingleinput}, \eqref{tsingleinput}, and~\eqref{usingleinput}. We consider a $\tanh$ network with width $n=50$ and depth $L=10$, with inputs drawn from $(0,1)$. Blue points show the exact solutions of~\eqref{ssingleinput}, \eqref{tsingleinput}, and~\eqref{usingleinput}, while orange boxes represent the corresponding large-$\ell$ expansions~\eqref{Slargel}, \eqref{Tlargel}, and~\eqref{Ulargel}, evaluated at integer $\ell$. The tensor magnitude remains below its Gaussian counterpart and exhibits early-layer saturation.}
  \label{fig:S_T_U_tensors_plot}
\end{figure*}

For ease of comparison, we provide a combined plot of all NTK, dNTK, and ddNTK tensors in Figure~\ref{fig:all_tensors_plots}. The results show that all tensors attain smaller magnitudes than their Gaussian counterparts and exhibit early-layer saturation, consistent with the empirical findings of~\cite{day2023}, which are quantitatively reproduced here.

\begin{figure*}[tb]
  \centering
  \includegraphics[width=0.7\textwidth,height=0.5\textheight,keepaspectratio]{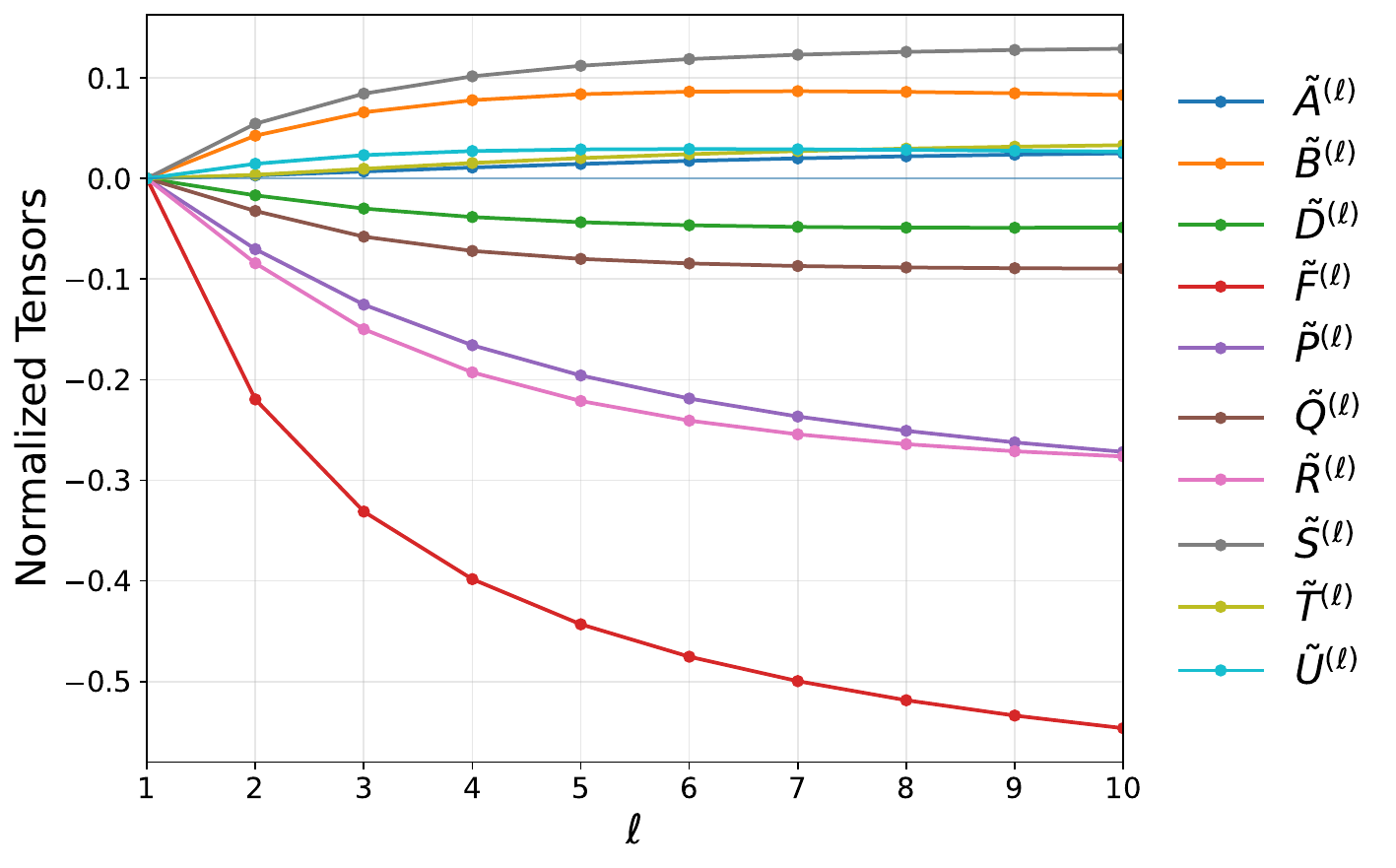}
  \vspace{-0.1cm}
  \caption{\emph{Orthogonal saturation.} Exact solutions of the single-input recursion relations at order $1/n$ (see Theorems~\ref{theoremone} and~\ref{theoremtwo}). We consider a $\tanh$ network of width $n=50$ with inputs drawn from the real interval $(0,1)$. All tensors attain smaller values than their Gaussian counterparts and exhibit early-layer saturation. These predictions are consistent with the empirical results of~\cite{day2023}.}\label{fig:all_tensors_plots}
\end{figure*}

\section{The $V_{6}$ tensor at order \(1/n^2\)}
\label{app:v6_tensor}

In this appendix, we compute the recursion relation governing the layer evolution of the sextic vertex $V_{6}$ using both algebraic and diagrammatic methods. 

\subsection{Algebraic derivation}
\label{sec:v6-algebraic-derivation}
The sextic vertex is defined from the six-point cumulant of the preactivations as follows
\begin{align}\label{six-point-full}
\mathbb{E}^{c}_{\theta}[z_{i_1,\alpha_1}^{(\ell+1)}&\ldots z_{i_6,\alpha_6}^{(\ell+1)}] = \mathbb{E}_{\theta}[z_{i_1,\alpha_1}^{(\ell+1)}\ldots z_{i_6,\alpha_6}^{(\ell+1)}] \nonumber\\
& - \bigg[\mathbb{E}_{\theta}[z_{i_1,\alpha_1}^{(\ell+1)}\ldots z_{i_4,\alpha_4}^{(\ell+1)}]\mathbb{E}_{\theta}[z_{i_5,\alpha_5}^{(\ell+1)}z_{i_6,\alpha_6}^{(\ell+1)}] + \text{14 other (4,2) subdivisions}\bigg]\nonumber\\
& + 2\bigg[\mathbb{E}_{\theta}[z_{i_1,\alpha_1}^{(\ell+1)}z_{i_2,\alpha_2}^{(\ell+1)}]\mathbb{E}_{\theta}[z_{i_3,\alpha_3}^{(\ell+1)}z_{i_4,\alpha_4}^{(\ell+1)}]\mathbb{E}_{\theta}[z_{i_5,\alpha_5}^{(\ell+1)}z_{i_6,\alpha_6}^{(\ell+1)}] + \text{14 other (2,2,2) subdivisions}\bigg]\nonumber\\
&= \frac{1}{n_{\ell}^2}\bigg[\delta_{i_1 i_2}\delta_{i_3 i_4}\delta_{i_5 i_6} V^{(\ell+1)}_{\alpha_1 \alpha_2 \alpha_3 \alpha_4 \alpha_5 \alpha_6} + \text{14 other (2,2,2) subdivisions} \bigg]
\end{align}
To obtain an explicit expression for $V^{(\ell+1)}_{\alpha_1 \alpha_2 \alpha_3 \alpha_4 \alpha_5 \alpha_6}$, we evaluate the expectation values of the orthogonal weights defining the preactivations $z_{i,\alpha}$. In the case of four weights, we introduce the basis $e_{1}=(12)(34)$, $e_{2} = (13)(24)$, $e_{3} = (14)(23)$. The general definition of the Weingarten function, $\mathcal{W}[\tau,\pi]=\mathcal{W}[\ell(\tau\circ \pi)]$, where $\circ$ denotes ordinary permutation multiplication and $\ell(e)$ denotes the cycle length of $e$, implies that $\mathcal{W}[e_{i},e_{j}]=\mathcal{W}[1,1]$ if $e_{i} = e_{j}$, and $\mathcal{W}[e_{i},e_{j}]=\mathcal{W}[2]$ if $e_{i}\neq e_{j}$. Therefore, in the basis $(e_{1},e_{2},e_{3})$, the $k=2$ Weingarten matrix takes the form
\begin{align}\label{k2wmatrix}
M &=
\begin{pmatrix}
\mathcal{W}[1,1] & \mathcal{W}[2] & \mathcal{W}[2] \\
\mathcal{W}[2] & \mathcal{W}[1,1] & \mathcal{W}[2]\\
\mathcal{W}[2] & \mathcal{W}[2] & \mathcal{W}[1,1]
\end{pmatrix}
\end{align}

For six weights, the corresponding basis consists of 15 elements, namely
\begin{align}
\begin{array}{lll}
\tilde{e}_1  = (12)(34)(56), & \tilde{e}_2  = (12)(35)(46), & \tilde{e}_3  = (12)(36)(45), \\[4pt] \tilde{e}_4  = (13)(24)(56), & \tilde{e}_5  = (13)(25)(46), & \tilde{e}_6  = (13)(26)(45), \\[4pt] \tilde{e}_7  = (14)(23)(56), & \tilde{e}_8  = (14)(25)(36), & \tilde{e}_9  = (14)(26)(35), \\[4pt] \tilde{e}_{10} = (15)(23)(46), & \tilde{e}_{11} = (15)(24)(36), & \tilde{e}_{12} = (15)(26)(34), \\[4pt] \tilde{e}_{13} = (16)(23)(45), & \tilde{e}_{14} = (16)(24)(35), & \tilde{e}_{15} = (16)(25)(34).\label{k3basis}
\end{array}
\end{align}
In this case, the cycle lengths can take only three possible values, namely $(1,1,1)$, $(2,1)$, and $(3)$, depending on the structure resulting from the product of three permutations. Consequently, in the basis~\eqref{k3basis}, the $k=3$ Weingarten matrix is represented by a $15\times 15$ matrix whose first row reads
\begin{align}\label{k3wmatrix}
\begin{aligned}
\big(
&\mathcal{W}[1,1,1],\;
\mathcal{W}[2,1],\; \mathcal{W}[2,1],\; \mathcal{W}[2,1],\;
\mathcal{W}[3],\;  \mathcal{W}[3],\; \mathcal{W}[2,1],\\
& \mathcal{W}[3],\; \mathcal{W}[3],\;
\mathcal{W}[3],\; \mathcal{W}[3],\;
\mathcal{W}[2,1],\; \mathcal{W}[3],\; \mathcal{W}[3],\; \mathcal{W}[2,1]
\big)
\end{aligned}
\end{align}
Using~\eqref{k2wmatrix} and~\eqref{k3wmatrix}, we then obtain
\begin{align}\label{v6algebraic}
V^{(\ell+1)}_{\alpha_1 \alpha_2 \alpha_3 \alpha_4 \alpha_5 \alpha_6} &= V^{\mathcal{G}(\ell+1)}_{\alpha_1 \alpha_2 \alpha_3 \alpha_4 \alpha_5 \alpha_6} + \Delta V^{(\ell+1)}_{\alpha_1 \alpha_2 \alpha_3 \alpha_4 \alpha_5 \alpha_6}
\end{align}
where $V^{\mathcal{G}(\ell+1)}_{\alpha_1 \alpha_2 \alpha_3 \alpha_4 \alpha_5 \alpha_6}$ is defined as 
\begin{align}\label{vg6}
\frac{1}{n^{2}_{\ell}} &V^{\mathcal{G}(\ell+1)}_{\alpha_{1}\alpha_{2}\alpha_{3}\alpha_{4}\alpha_{5}\alpha_{6}} = (C^{(\ell+1)}_{W})^{3}\sum_{i,j,k=1}^{n}\bigg[\mathcal{W}[1,1,1]\mathbb{E}[\sigma^{(\ell)}_{i,1}\sigma^{(\ell)}_{i,2}\sigma^{(\ell)}_{j,3}\sigma^{(\ell)}_{j,4}\sigma^{(\ell)}_{k,5}\sigma^{(\ell)}_{k,6}] \nonumber\\
& - \mathcal{W}[1,1]\cdot\mathcal{W}[1]\bigg(\mathbb{E}[\sigma^{(\ell)}_{i,1}\sigma^{(\ell)}_{i,2}\sigma^{(\ell)}_{j,3}\sigma^{(\ell)}_{j,4}]\mathbb{E}[\sigma^{(\ell)}_{k,5}\sigma^{(\ell)}_{k,6}] + \mathbb{E}[\sigma^{(\ell)}_{j,3}\sigma^{(\ell)}_{j,4}\sigma^{(\ell)}_{k,5}\sigma^{(\ell)}_{k,6}]\mathbb{E}[\sigma^{(\ell)}_{i,1}\sigma^{(\ell)}_{i,2}] \nonumber\\
& + \mathbb{E}[\sigma^{(\ell)}_{i,1}\sigma^{(\ell)}_{i,2}\sigma^{(\ell)}_{k,5}\sigma^{(\ell)}_{k,6}]\mathbb{E}[\sigma^{(\ell)}_{j,3}\sigma^{(\ell)}_{j,4}]\bigg) + 2(\mathcal{W}[1])^{3}\mathbb{E}[\sigma^{(\ell)}_{i,1}\sigma^{(\ell)}_{i,2}]\mathbb{E}[\sigma^{(\ell)}_{j,3}\sigma^{(\ell)}_{j,4}]\mathbb{E}[\sigma^{(\ell)}_{k,5}\sigma^{(\ell)}_{k,6}]\bigg]
\end{align}
and $\Delta V^{(\ell+1)}_{\alpha_1 \alpha_2 \alpha_3 \alpha_4 \alpha_5 \alpha_6}$ is given by
\begin{align}\label{deltav6}
\frac{1}{n^{2}_{\ell}}\Delta &V^{(\ell+1)}_{\alpha_1 \alpha_2 \alpha_3 \alpha_4 \alpha_5 \alpha_6} = (C^{(\ell+1)}_{W})^{3}\sum_{i,j,k=1}^{n}\bigg[\mathcal{W}[2,1]\mathbb{E}[\sigma^{(\ell)}_{\alpha_{1},i}\sigma^{(\ell)}_{\alpha_{2},i}\sigma^{(\ell)}_{\alpha_{3},j}\sigma^{(\ell)}_{\alpha_{4},k}\sigma^{(\ell)}_{\alpha_{5},j}\sigma^{(\ell)}_{\alpha_{6},k}]\nonumber\\
& + \mathcal{W}[2,1]\mathbb{E}[\sigma^{(\ell)}_{\alpha_{1},i}\sigma^{(\ell)}_{\alpha_{2},i}\sigma^{(\ell)}_{\alpha_{3},j}\sigma^{(\ell)}_{\alpha_{4},k}\sigma^{(\ell)}_{\alpha_{5},k}\sigma^{(\ell)}_{\alpha_{6},j}] + \mathcal{W}[2,1]\mathbb{E}[\sigma^{(\ell)}_{\alpha_{1},i}\sigma^{(\ell)}_{\alpha_{2},j}\sigma^{(\ell)}_{\alpha_{3},i}\sigma^{(\ell)}_{\alpha_{4},j}\sigma^{(\ell)}_{\alpha_{5},k}\sigma^{(\ell)}_{\alpha_{6},k}] \nonumber\\
& + \mathcal{W}[3]\mathbb{E}[\sigma^{(\ell)}_{\alpha_{1},i}\sigma^{(\ell)}_{\alpha_{2},j}\sigma^{(\ell)}_{\alpha_{3},i}\sigma^{(\ell)}_{\alpha_{4},k}\sigma^{(\ell)}_{\alpha_{5},j}\sigma^{(\ell)}_{\alpha_{6},k}] + \mathcal{W}[3]\mathbb{E}[\sigma^{(\ell)}_{\alpha_{1},i}\sigma^{(\ell)}_{\alpha_{2},j}\sigma^{(\ell)}_{\alpha_{3},i}\sigma^{(\ell)}_{\alpha_{4},k}\sigma^{(\ell)}_{\alpha_{5},k}\sigma^{(\ell)}_{\alpha_{6},j}] \nonumber\\
& + \mathcal{W}[2,1]\mathbb{E}[\sigma^{(\ell)}_{\alpha_{1},i}\sigma^{(\ell)}_{\alpha_{2},j}\sigma^{(\ell)}_{\alpha_{3},j}\sigma^{(\ell)}_{\alpha_{4},i}\sigma^{(\ell)}_{\alpha_{5},k}\sigma^{(\ell)}_{\alpha_{6},k}] + \mathcal{W}[3]\mathbb{E}[\sigma^{(\ell)}_{\alpha_{1},i}\sigma^{(\ell)}_{\alpha_{2},j}\sigma^{(\ell)}_{\alpha_{3},k}\sigma^{(\ell)}_{\alpha_{4},i}\sigma^{(\ell)}_{\alpha_{5},j}\sigma^{(\ell)}_{\alpha_{6},k}]\nonumber\\
& + \mathcal{W}[3]\mathbb{E}[\sigma^{(\ell)}_{\alpha_{1},i}\sigma^{(\ell)}_{\alpha_{2},j}\sigma^{(\ell)}_{\alpha_{3},k}\sigma^{(\ell)}_{\alpha_{4},i}\sigma^{(\ell)}_{\alpha_{5},k}\sigma^{(\ell)}_{\alpha_{6},j}] + \mathcal{W}[3]\mathbb{E}[\sigma^{(\ell)}_{\alpha_{1},i}\sigma^{(\ell)}_{\alpha_{2},j}\sigma^{(\ell)}_{\alpha_{3},j}\sigma^{(\ell)}_{\alpha_{4},k}\sigma^{(\ell)}_{\alpha_{5},i}\sigma^{(\ell)}_{\alpha_{6},k}] \nonumber\\
& + \mathcal{W}[3]\mathbb{E}[\sigma^{(\ell)}_{\alpha_{1},i}\sigma^{(\ell)}_{\alpha_{2},j}\sigma^{(\ell)}_{\alpha_{3},k}\sigma^{(\ell)}_{\alpha_{4},j}\sigma^{(\ell)}_{\alpha_{5},i}\sigma^{(\ell)}_{\alpha_{6},k}] + \mathcal{W}[2,1]\mathbb{E}[\sigma^{(\ell)}_{\alpha_{1},i}\sigma^{(\ell)}_{\alpha_{2},j}\sigma^{(\ell)}_{\alpha_{3},k}\sigma^{(\ell)}_{\alpha_{4},k}\sigma^{(\ell)}_{\alpha_{5},i}\sigma^{(\ell)}_{\alpha_{6},j}]\nonumber\\
& + \mathcal{W}[3]\mathbb{E}[\sigma^{(\ell)}_{\alpha_{1},i}\sigma^{(\ell)}_{\alpha_{2},j}\sigma^{(\ell)}_{\alpha_{3},j}\sigma^{(\ell)}_{\alpha_{4},k}\sigma^{(\ell)}_{\alpha_{5},k}\sigma^{(\ell)}_{\alpha_{6},i}] + \mathcal{W}[3]\mathbb{E}[\sigma^{(\ell)}_{\alpha_{1},i}\sigma^{(\ell)}_{\alpha_{2},j}\sigma^{(\ell)}_{\alpha_{3},k}\sigma^{(\ell)}_{\alpha_{4},j}\sigma^{(\ell)}_{\alpha_{5},k}\sigma^{(\ell)}_{\alpha_{6},i}]\nonumber\\
& + \mathcal{W}[2,1]\mathbb{E}[\sigma^{(\ell)}_{\alpha_{1},i}\sigma^{(\ell)}_{\alpha_{2},j}\sigma^{(\ell)}_{\alpha_{3},k}\sigma^{(\ell)}_{\alpha_{4},k}\sigma^{(\ell)}_{\alpha_{5},j}\sigma^{(\ell)}_{\alpha_{6},i}]\nonumber\\
& - \mathcal{W}[2]\cdot \mathcal{W}[1]\bigg(\mathbb{E}[\sigma^{(\ell)}_{\alpha_1,i}\sigma^{(\ell)}_{\alpha_2,j}\sigma^{(\ell)}_{\alpha_3,i}\sigma^{(\ell)}_{\alpha_4,j}] + \mathbb{E}[\sigma^{(\ell)}_{\alpha_1,i}\sigma^{(\ell)}_{\alpha_2,j}\sigma^{(\ell)}_{\alpha_3,j}\sigma^{(\ell)}_{\alpha_4,i}]\bigg)\mathbb{E}[\sigma^{(\ell)}_{\alpha_5,k}\sigma^{(\ell)}_{\alpha_6,k}]\nonumber\\
& - \mathcal{W}[2]\cdot \mathcal{W}[1]\bigg(\mathbb{E}[\sigma^{(\ell)}_{\alpha_3,i}\sigma^{(\ell)}_{\alpha_4,j}\sigma^{(\ell)}_{\alpha_5,i}\sigma^{(\ell)}_{\alpha_6,j}] + \mathbb{E}[\sigma^{(\ell)}_{\alpha_3,i}\sigma^{(\ell)}_{\alpha_4,j}\sigma^{(\ell)}_{\alpha_5,j}\sigma^{(\ell)}_{\alpha_6,i}]\bigg)\mathbb{E}[\sigma^{(\ell)}_{\alpha_1,k}\sigma^{(\ell)}_{\alpha_2,k}]\nonumber\\
& - \mathcal{W}[2]\mathcal{W}[1]\bigg(\mathbb{E}[\sigma^{(\ell)}_{\alpha_1,i}\sigma^{(\ell)}_{\alpha_2,j}\sigma^{(\ell)}_{\alpha_5,i}\sigma^{(\ell)}_{\alpha_6,j}] + \mathbb{E}[\sigma^{(\ell)}_{\alpha_1,i}\sigma^{(\ell)}_{\alpha_2,j}\sigma^{(\ell)}_{\alpha_5,j}\sigma^{(\ell)}_{\alpha_6,i}]\bigg)\mathbb{E}[\sigma^{(\ell)}_{\alpha_3,k}\sigma^{(\ell)}_{\alpha_4,k}]\bigg]
\end{align}
The term $V^{\mathcal{G}(\ell+1)}_{\alpha_{1}\alpha_{2}\alpha_{3}\alpha_{4}\alpha_{5}\alpha_{6}}$ captures the contribution of the diagonal components of the Weingarten matrices and reduces to the Gaussian case in the infinite-width limit, whereas $\Delta V^{(\ell+1)}_{\alpha_{1}\alpha_{2}\alpha_{3}\alpha_{4}\alpha_{5}\alpha_{6}}$ encodes the off-diagonal contributions.

We now compute $V^{(\ell+1)}_{\alpha_{1}\alpha_{2}\alpha_{3}\alpha_{4}\alpha_{5}\alpha_{6}}$ at order \(1/n^2\). To this end, we expand the $k=3$ Weingarten functions as follows
\begin{align}\label{k3wfunctions}
\begin{aligned}
\mathcal{W}[1,1,1] &= \frac{1}{n^{3}} + \frac{6}{n^{5}} + O\!\left(\frac{1}{n^{6}}\right),\\[4pt]
\mathcal{W}[2,1]   &= -\,\frac{1}{n^{4}} + \frac{1}{n^{5}} + O\!\left(\frac{1}{n^{6}}\right),\\[4pt]
\mathcal{W}[3]     &= \frac{2}{n^{5}} + O\!\left(\frac{1}{n^{6}}\right).
\end{aligned}
\end{align}
Combining~\eqref{k3wfunctions}, \eqref{w11expanded} and~\eqref{w2expanded}, we obtain
\begin{align}\label{v6ordern2}
\frac{1}{n_{\ell}^{2}}V^{\mathcal{G}(\ell+1)}_{\alpha_{1}\alpha_{2}\alpha_{3}\alpha_{4}\alpha_{5}\alpha_{6}} &= \frac{1}{n_{\ell}^{2}}V^{G(\ell+1)}_{\alpha_{1}\alpha_{2}\alpha_{3}\alpha_{4}\alpha_{5}\alpha_{6}} + \frac{6}{n_{\ell}^{2}}\langle\sigma^{(\ell)}_{\alpha_{1}}\sigma^{(\ell)}_{\alpha_{2}}\rangle_{K^{(\ell)}}\langle\sigma^{(\ell)}_{\alpha_{3}}\sigma^{(\ell)}_{\alpha_{4}}\rangle_{K^{(\ell)}}\langle\sigma^{(\ell)}_{\alpha_{5}}\sigma^{(\ell)}_{\alpha_{6}}\rangle_{K^{(\ell)}} \nonumber\\
& - \frac{3\cdot 2}{n_{\ell}^{2}}\langle\sigma^{(\ell)}_{\alpha_{1}}\sigma^{(\ell)}_{\alpha_{2}}\rangle_{K^{(\ell)}}\langle\sigma^{(\ell)}_{\alpha_{3}}\sigma^{(\ell)}_{\alpha_{4}}\rangle_{K^{(\ell)}}\langle\sigma^{(\ell)}_{\alpha_{5}}\sigma^{(\ell)}_{\alpha_{6}}\rangle_{K^{(\ell)}}\nonumber\\
&= \frac{1}{n_{\ell}^{2}}\mathcal{V}^{G(\ell+1)}_{\alpha_{1}\alpha_{2}\alpha_{3}\alpha_{4}\alpha_{5}\alpha_{6}}
\end{align}
The factor 6 in the first equality of \eqref{v6ordern2} originates from $\mathcal{W}[1,1,1]$ in \eqref{vg6}, while the factor $2\cdot 3$ arises from the three contributions of the form $\mathcal{W}[1,1]\cdot \mathcal{W}[1]$. The tensor $V^{G(\ell+1)}_{\alpha_{1}\alpha_{2}\alpha_{3}\alpha_{4}\alpha_{5}\alpha_{6}}$ denotes the Gaussian six-point cumulant at order $\frac{1}{n^{2}}$. The corresponding recursion relation was derived explicitly in~\cite{banta2023} and reads
\allowdisplaybreaks
\begin{align}
\frac{1}{n_{\ell}^{2}}V^{G(\ell+1)}_{\alpha_{1}\alpha_{2}\alpha_{3}\alpha_{4}\alpha_{5}\alpha_{6}} &=\frac{\bigl(C_W^{(\ell+1)}\bigr)^{\*3}}{n_{\ell}^2} \,
\Bigl\langle \widehat{\Delta G}^{(\ell)}_{\alpha_{1}\alpha_{2}} \,\widehat{\Delta G}^{(\ell)}_{\alpha_{3}\alpha_{4}} \,\widehat{\Delta G}^{(\ell)}_{\alpha_{5}\alpha_{6}} \Bigr\rangle_{K^{(\ell)}} \nonumber\\
&\hspace{-60pt} +\,\frac{\bigl(C_W^{(\ell+1)}\bigr)^{\*3}}{4\,n_{\ell} n_{\ell-1}} \!\!\!\!\sum_{\beta_{i} \in \{\alpha_{1},\ldots,\alpha_{6}\}} \!\!\!\!\!\!\!\!V^{(\ell)}_{\beta_{1}\beta_{2}\beta_{3}\beta_{4}} \Biggl[ 
\biggl\langle \frac{d^2 \bigl( \widehat{\Delta G}^{(\ell)}_{\alpha_{1}\alpha_{2}} \, \widehat{\Delta G}^{(\ell)}_{\alpha_{5}\alpha_{6}} \bigr)}{d z^{(\ell)}_{\beta_{1}} \,d z^{(\ell)}_{\beta_{2}}} \biggr\rangle_{\!\!K^{(\ell)}}
\biggl\langle \frac{d^2 \widehat{\Delta G}^{(\ell)}_{\alpha_{3}\alpha_{4}}}{d z^{(\ell)}_{\beta_{3}} \,d z^{(\ell)}_{\beta_{4}} } \biggr\rangle_{\!\!K^{(\ell)}}
\! \nonumber\\[10pt]
&\hspace{-60pt} + \, \biggl\langle \frac{d^2 \bigl( \widehat{\Delta G}^{(\ell)}_{\alpha_{1}\alpha_{2}} \, \widehat{\Delta G}^{(\ell)}_{\alpha_{3}\alpha_{4}} \bigr)}{d z^{(\ell)}_{\beta_{1}} \,d z^{(\ell)}_{\beta_{2}} } \biggr\rangle_{\!\!K^{(\ell)}}
\biggl\langle \frac{d^2 \widehat{\Delta G}^{(\ell)}_{\alpha_{5}\alpha_{6}}}{d z^{(\ell)}_{\beta_{3}} \,d z^{(\ell)}_{\beta_{4}} } \biggr\rangle_{\!\!K^{(\ell)}}\!\!+  \biggl\langle \frac{d^2 \bigl( \widehat{\Delta G}^{(\ell)}_{\alpha_{3}\alpha_{4}} \, \widehat{\Delta G}^{(\ell)}_{\alpha_{5}\alpha_{6}} \bigr)}{d z^{(\ell)}_{\beta_{1}} \,d z^{(\ell)}_{\beta_{2}}} \biggr\rangle_{\!\!K^{(\ell)}}
\biggl\langle \frac{d^2 \widehat{\Delta G}^{(\ell)}_{\alpha_{1}\alpha_{2}}}{d z^{(\ell)}_{\beta_{3}} \,d z^{(\ell)}_{\beta_{4}} } \biggr\rangle_{\!\!K^{(\ell)}}\Biggr]\nonumber\\
&\hspace{-60pt} + \, \frac{\bigl(C_W^{(\ell+1)}\bigr)^{\*3}}{8\,n_{\ell-1}^2}
\!\!\!\!\sum_{\beta_{i}\in \{\alpha_{1},\ldots,\alpha_{6}\}} \!\!\!\!\!\!\!\! V^{(\ell)}_{\beta_{1}\beta_{2}\beta_{3}\beta_{4}\beta_{5}\beta_{6}}
\biggl\langle \frac{d^2 \widehat{\Delta G}^{(\ell)}_{\alpha_{1}\alpha_{2}}}{d z^{(\ell)}_{\beta_{1}} \, d z^{(\ell)}_{\beta_{2}}} \biggr\rangle_{\!\!K^{(\ell)}}
\biggl\langle \frac{d^2 \widehat{\Delta G}^{(\ell)}_{\alpha_{3}\alpha_{4}}}{d z^{(\ell)}_{\beta_{3}} \, d z^{(\ell)}_{\beta_{4}}} \biggr\rangle_{\!\!K^{(\ell)}}
\biggl\langle \frac{d^2 \widehat{\Delta G}^{(\ell)}_{\alpha_{5}\alpha_{6}}}{d z^{(\ell)}_{\beta_{5}} \, d z^{(\ell)}_{\beta_{6}}} \biggr\rangle_{\!\!K^{(\ell)}} \nonumber\\
& \hspace{-60pt} + \, \frac{\bigl(C_W^{(\ell+1)}\bigr)^{\*3}}{16\, n_{\ell-1}^2} \!\!\!\!\sum_{\beta_{i} \in \{\alpha_{1},\ldots \alpha_{6}\}} \!\!\!\!\!\!\!\!V^{(\ell)}_{\beta_{1}\beta_{2}\beta_{3}\beta_{4}} \,V^{(\ell)}_{\beta_{5}\beta_{6}\beta_{7}\beta_{8}} 
\Biggl[
\biggl\langle \frac{d^4 \widehat{\Delta G}^{(\ell)}_{\alpha_{3}\alpha_{4}}}{d z^{(\ell)}_{\beta_{3}} \, d z^{(\ell)}_{\beta_{4}} \, d z^{(\ell)}_{\beta_{7}}  \, d z^{(\ell)}_{\beta_{8}} } \biggr\rangle_{\!\!K^{(\ell)}}
\biggl\langle \frac{d^2 \widehat{\Delta G}^{(\ell)}_{\alpha_{1}\alpha_{2}}}{d z^{(\ell)}_{\beta_{1}} \, d z^{(\ell)}_{\beta_{2}}} \biggr\rangle_{\!\!K^{(\ell)}}
\biggl\langle \frac{d^2 \widehat{\Delta G}^{(\ell)}_{\alpha_{5}\alpha_{6}}}{d z^{(\ell)}_{\beta_{5}} \, d z^{(\ell)}_{\beta_{6}}} \biggr\rangle_{\!\!K^{(\ell)}} 
\nonumber\\[10pt]
&\hspace{-60pt}
+\, \biggl\langle \frac{d^4 \widehat{\Delta G}^{(\ell)}_{\alpha_{1}\alpha_{2}}}{d z^{(\ell)}_{\beta_{3}} \, d z^{(\ell)}_{\beta_{4}} \, d z^{(\ell)}_{\beta_{7}}  \, d z^{(\ell)}_{\beta_{8}} } \biggr\rangle_{\!\!K^{(\ell)}}
\biggl\langle \frac{d^2 \widehat{\Delta G}^{(\ell)}_{\alpha_{3}\alpha_{4}}}{d z^{(\ell)}_{\beta_{1}} \, d z^{(\ell)}_{\beta_{2}}} \biggr\rangle_{\!\!K^{(\ell)}}
\biggl\langle \frac{d^2 \widehat{\Delta G}^{(\ell)}_{\alpha_{5}\alpha_{6}}}{d z^{(\ell)}_{\beta_{5}} \, d z^{(\ell)}_{\beta_{6}}} \biggr\rangle_{\!\!K^{(\ell)}} 
\nonumber\\[10pt]
&\hspace{-60pt}
+\, \biggl\langle \frac{d^4 \widehat{\Delta G}^{(\ell)}_{\alpha_{5}\alpha_{6}}}{d z^{(\ell)}_{\beta_{3}} \, d z^{(\ell)}_{\beta_{4}} \, d z^{(\ell)}_{\beta_{7}}  \, d z^{(\ell)}_{\beta_{8}} } \biggr\rangle_{\!\!K^{(\ell)}}
\biggl\langle \frac{d^2 \widehat{\Delta G}^{(\ell)}_{\alpha_{1}\alpha_{2}}}{d z^{(\ell)}_{\beta_{1}} \, d z^{(\ell)}_{\beta_{2}}} \biggr\rangle_{\!\!K^{(\ell)}}
\biggl\langle \frac{d^2 \widehat{\Delta G}^{(\ell)}_{\alpha_{3}\alpha_{4}}}{d z^{(\ell)}_{\beta_{5}} \, d z^{(\ell)}_{\beta_{6}}} \biggr\rangle_{\!\!K^{(\ell)}} \Biggr]\,.
\end{align}
where $\widehat{\Delta G}^{(\ell)}_{\alpha\beta} = \sigma^{(\ell)}_{\alpha}\sigma^{(\ell)}_{\beta} - \langle\sigma^{(\ell)}_{\alpha}\sigma^{(\ell)}_{\beta}\rangle_{K^{(\ell)}}$. 

Likewise, the off-diagonal contribution $\Delta \mathcal{V}^{(\ell+1)}_{\alpha_{1}\alpha_{2}\alpha_{3}\alpha_{4}\alpha_{5}\alpha_{6}}$ can be systematically computed as follows: 

\begin{enumerate}
\item In the $\frac{1}{n}$-expansion of $\mathcal{W}[2,1]$, only the terms of order $\frac{1}{n^{5}}$ and $\frac{1}{n^{4}}$ contribute nontrivially. The former becomes relevant when all neural indices are distinct, yielding
\begin{align}
  \frac{1}{n_{\ell}^{2}}\bigg[\langle \sigma^{(\ell)}_{\alpha_{1}}\sigma^{(\ell)}_{\alpha_{2}}\rangle_{K^{(\ell)}}\langle \sigma^{(\ell)}_{\alpha_{3}}\sigma^{(\ell)}_{\alpha_{5}}\rangle_{K^{(\ell)}}\langle \sigma^{(\ell)}_{\alpha_{4}}\sigma^{(\ell)}_{\alpha_{6}}\rangle_{K^{(\ell)}} + \langle \sigma^{(\ell)}_{\alpha_{1}}\sigma^{(\ell)}_{\alpha_{2}}\rangle_{K^{(\ell)}}\langle \sigma^{(\ell)}_{\alpha_{3}}\sigma^{(\ell)}_{\alpha_{6}}\rangle_{K^{(\ell)}}\langle \sigma^{(\ell)}_{\alpha_{4}}\sigma^{(\ell)}_{\alpha_{5}}\rangle_{K^{(\ell)}}\nonumber\\
+ \langle \sigma^{(\ell)}_{\alpha_{1}}\sigma^{(\ell)}_{\alpha_{3}}\rangle_{K^{(\ell)}}\langle \sigma^{(\ell)}_{\alpha_{2}}\sigma^{(\ell)}_{\alpha_{4}}\rangle_{K^{(\ell)}}\langle \sigma^{(\ell)}_{\alpha_{5}}\sigma^{(\ell)}_{\alpha_{6}}\rangle_{K^{(\ell)}} + \langle \sigma^{(\ell)}_{\alpha_{1}}\sigma^{(\ell)}_{\alpha_{4}}\rangle_{K^{(\ell)}}\langle \sigma^{(\ell)}_{\alpha_{2}}\sigma^{(\ell)}_{\alpha_{3}}\rangle_{K^{(\ell)}}\langle \sigma^{(\ell)}_{\alpha_{5}}\sigma^{(\ell)}_{\alpha_{6}}\rangle_{K^{(\ell)}}\nonumber\\
+ \langle \sigma^{(\ell)}_{\alpha_{1}}\sigma^{(\ell)}_{\alpha_{5}}\rangle_{K^{(\ell)}}\langle \sigma^{(\ell)}_{\alpha_{2}}\sigma^{(\ell)}_{\alpha_{6}}\rangle_{K^{(\ell)}}\langle \sigma^{(\ell)}_{\alpha_{3}}\sigma^{(\ell)}_{\alpha_{4}}\rangle_{K^{(\ell)}} + \langle \sigma^{(\ell)}_{\alpha_{1}}\sigma^{(\ell)}_{\alpha_{6}}\rangle_{K^{(\ell)}}\langle \sigma^{(\ell)}_{\alpha_{2}}\sigma^{(\ell)}_{\alpha_{5}}\rangle_{K^{(\ell)}}\langle \sigma^{(\ell)}_{\alpha_{3}}\sigma^{(\ell)}_{\alpha_{4}}\rangle_{K^{(\ell)}}\bigg]
\end{align}

The latter splits into two cases:
\begin{itemize}
\item The first case arises when all neural indices are distinct: 
\begin{align}
\frac{3}{n_{\ell}^{2}}\bigg[\langle \sigma^{(\ell)}_{\alpha_{1}}\sigma^{(\ell)}_{\alpha_{2}}\rangle_{K^{(\ell)}}\langle \sigma^{(\ell)}_{\alpha_{3}}\sigma^{(\ell)}_{\alpha_{5}}\rangle_{K^{(\ell)}}\langle \sigma^{(\ell)}_{\alpha_{4}}\sigma^{(\ell)}_{\alpha_{6}}\rangle_{K^{(\ell)}} + \langle \sigma^{(\ell)}_{\alpha_{1}}\sigma^{(\ell)}_{\alpha_{2}}\rangle_{K^{(\ell)}}\langle \sigma^{(\ell)}_{\alpha_{3}}\sigma^{(\ell)}_{\alpha_{6}}\rangle_{K^{(\ell)}}\langle \sigma^{(\ell)}_{\alpha_{4}}\sigma^{(\ell)}_{\alpha_{5}}\rangle_{K^{(\ell)}}\nonumber\\
+ \langle \sigma^{(\ell)}_{\alpha_{1}}\sigma^{(\ell)}_{\alpha_{3}}\rangle_{K^{(\ell)}}\langle \sigma^{(\ell)}_{\alpha_{2}}\sigma^{(\ell)}_{\alpha_{4}}\rangle_{K^{(\ell)}}\langle \sigma^{(\ell)}_{\alpha_{5}}\sigma^{(\ell)}_{\alpha_{6}}\rangle_{K^{(\ell)}} + \langle \sigma^{(\ell)}_{\alpha_{1}}\sigma^{(\ell)}_{\alpha_{4}}\rangle_{K^{(\ell)}}\langle \sigma^{(\ell)}_{\alpha_{2}}\sigma^{(\ell)}_{\alpha_{3}}\rangle_{K^{(\ell)}}\langle \sigma^{(\ell)}_{\alpha_{5}}\sigma^{(\ell)}_{\alpha_{6}}\rangle_{K^{(\ell)}}\nonumber\\
+ \langle \sigma^{(\ell)}_{\alpha_{1}}\sigma^{(\ell)}_{\alpha_{5}}\rangle_{K^{(\ell)}}\langle \sigma^{(\ell)}_{\alpha_{2}}\sigma^{(\ell)}_{\alpha_{6}}\rangle_{K^{(\ell)}}\langle \sigma^{(\ell)}_{\alpha_{3}}\sigma^{(\ell)}_{\alpha_{4}}\rangle_{K^{(\ell)}} + \langle \sigma^{(\ell)}_{\alpha_{1}}\sigma^{(\ell)}_{\alpha_{6}}\rangle_{K^{(\ell)}}\langle \sigma^{(\ell)}_{\alpha_{2}}\sigma^{(\ell)}_{\alpha_{5}}\rangle_{K^{(\ell)}}\langle \sigma^{(\ell)}_{\alpha_{3}}\sigma^{(\ell)}_{\alpha_{4}}\rangle_{K^{(\ell)}}\bigg]
\end{align}

\item The second case occurs when exactly two neural indices coincide. In this case, we find two types of contributions: one that does not involve the four-point cumulant,
\begin{align}
-\frac{1}{n_{\ell}^{2}}\bigg[\langle \sigma^{(\ell)}_{\alpha_{1}}\sigma^{(\ell)}_{\alpha_{2}}\sigma^{(\ell)}_{\alpha_{3}}\sigma^{(\ell)}_{\alpha_{5}}\rangle_{K^{(\ell)}}\langle \sigma^{(\ell)}_{\alpha_{4}}\sigma^{(\ell)}_{\alpha_{6}}\rangle_{K^{(\ell)}} + \langle \sigma^{(\ell)}_{\alpha_{1}}\sigma^{(\ell)}_{\alpha_{2}}\sigma^{(\ell)}_{\alpha_{4}}\sigma^{(\ell)}_{\alpha_{6}}\rangle_{K^{(\ell)}}\langle \sigma^{(\ell)}_{\alpha_{3}}\sigma^{(\ell)}_{\alpha_{5}}\rangle_{K^{(\ell)}}\nonumber\\
 + \langle \sigma^{(\ell)}_{\alpha_{3}}\sigma^{(\ell)}_{\alpha_{5}}\sigma^{(\ell)}_{\alpha_{4}}\sigma^{(\ell)}_{\alpha_{6}}\rangle_{K^{(\ell)}}\langle \sigma^{(\ell)}_{\alpha_{1}}\sigma^{(\ell)}_{\alpha_{2}}\rangle_{K^{(\ell)}} + \langle \sigma^{(\ell)}_{\alpha_{1}}\sigma^{(\ell)}_{\alpha_{2}}\sigma^{(\ell)}_{\alpha_{3}}\sigma^{(\ell)}_{\alpha_{6}}\rangle_{K^{(\ell)}}\langle \sigma^{(\ell)}_{\alpha_{4}}\sigma^{(\ell)}_{\alpha_{5}}\rangle_{K^{(\ell)}}\nonumber\\
+ \langle \sigma^{(\ell)}_{\alpha_{1}}\sigma^{(\ell)}_{\alpha_{2}}\sigma^{(\ell)}_{\alpha_{4}}\sigma^{(\ell)}_{\alpha_{5}}\rangle_{K^{(\ell)}}\langle \sigma^{(\ell)}_{\alpha_{3}}\sigma^{(\ell)}_{\alpha_{6}}\rangle_{K^{(\ell)}}
+ \langle \sigma^{(\ell)}_{\alpha_{4}}\sigma^{(\ell)}_{\alpha_{5}}\sigma^{(\ell)}_{\alpha_{3}}\sigma^{(\ell)}_{\alpha_{6}}\rangle_{K^{(\ell)}}\langle \sigma^{(\ell)}_{\alpha_{1}}\sigma^{(\ell)}_{\alpha_{2}}\rangle_{K^{(\ell)}}\nonumber\\
+ \langle \sigma^{(\ell)}_{\alpha_{1}}\sigma^{(\ell)}_{\alpha_{3}}\sigma^{(\ell)}_{\alpha_{2}}\sigma^{(\ell)}_{\alpha_{4}}\rangle_{K^{(\ell)}}\langle \sigma^{(\ell)}_{\alpha_{5}}\sigma^{(\ell)}_{\alpha_{6}}\rangle_{K^{(\ell)}} + \langle \sigma^{(\ell)}_{\alpha_{1}}\sigma^{(\ell)}_{\alpha_{3}}\sigma^{(\ell)}_{\alpha_{5}}\sigma^{(\ell)}_{\alpha_{6}}\rangle_{K^{(\ell)}}\langle \sigma^{(\ell)}_{\alpha_{2}}\sigma^{(\ell)}_{\alpha_{4}}\rangle_{K^{(\ell)}} \nonumber\\
+ \langle \sigma^{(\ell)}_{\alpha_{2}}\sigma^{(\ell)}_{\alpha_{4}}\sigma^{(\ell)}_{\alpha_{5}}\sigma^{(\ell)}_{\alpha_{6}}\rangle_{K^{(\ell)}}\langle \sigma^{(\ell)}_{\alpha_{1}}\sigma^{(\ell)}_{\alpha_{3}}\rangle_{K^{(\ell)}}
+ \langle \sigma^{(\ell)}_{\alpha_{1}}\sigma^{(\ell)}_{\alpha_{4}}\sigma^{(\ell)}_{\alpha_{2}}\sigma^{(\ell)}_{\alpha_{3}}\rangle_{K^{(\ell)}}\langle \sigma^{(\ell)}_{\alpha_{5}}\sigma^{(\ell)}_{\alpha_{6}}\rangle_{K^{(\ell)}}\nonumber\\ 
+ \langle \sigma^{(\ell)}_{\alpha_{1}}\sigma^{(\ell)}_{\alpha_{4}}\sigma^{(\ell)}_{\alpha_{5}}\sigma^{(\ell)}_{\alpha_{6}}\rangle_{K^{(\ell)}}\langle \sigma^{(\ell)}_{\alpha_{2}}\sigma^{(\ell)}_{\alpha_{3}}\rangle_{K^{(\ell)}} + \langle \sigma^{(\ell)}_{\alpha_{2}}\sigma^{(\ell)}_{\alpha_{3}}\sigma^{(\ell)}_{\alpha_{5}}\sigma^{(\ell)}_{\alpha_{6}}\rangle_{K^{(\ell)}}\langle \sigma^{(\ell)}_{\alpha_{1}}\sigma^{(\ell)}_{\alpha_{4}}\rangle_{K^{(\ell)}}\nonumber\\
+ \langle \sigma^{(\ell)}_{\alpha_{1}}\sigma^{(\ell)}_{\alpha_{5}}\sigma^{(\ell)}_{\alpha_{2}}\sigma^{(\ell)}_{\alpha_{6}}\rangle_{K^{(\ell)}}\langle \sigma^{(\ell)}_{\alpha_{3}}\sigma^{(\ell)}_{\alpha_{4}}\rangle_{K^{(\ell)}} + \langle \sigma^{(\ell)}_{\alpha_{1}}\sigma^{(\ell)}_{\alpha_{5}}\sigma^{(\ell)}_{\alpha_{3}}\sigma^{(\ell)}_{\alpha_{4}}\rangle_{K^{(\ell)}}\langle \sigma^{(\ell)}_{\alpha_{2}}\sigma^{(\ell)}_{\alpha_{6}}\rangle_{K^{(\ell)}}\nonumber\\ + \langle \sigma^{(\ell)}_{\alpha_{2}}\sigma^{(\ell)}_{\alpha_{6}}\sigma^{(\ell)}_{\alpha_{3}}\sigma^{(\ell)}_{\alpha_{4}}\rangle_{K^{(\ell)}}\langle \sigma^{(\ell)}_{\alpha_{1}}\sigma^{(\ell)}_{\alpha_{5}}\rangle_{K^{(\ell)}} + \langle \sigma^{(\ell)}_{\alpha_{1}}\sigma^{(\ell)}_{\alpha_{6}}\sigma^{(\ell)}_{\alpha_{2}}\sigma^{(\ell)}_{\alpha_{5}}\rangle_{K^{(\ell)}}\langle \sigma^{(\ell)}_{\alpha_{3}}\sigma^{(\ell)}_{\alpha_{4}}\rangle_{K^{(\ell)}}\nonumber\\
 + \langle \sigma^{(\ell)}_{\alpha_{1}}\sigma^{(\ell)}_{\alpha_{6}}\sigma^{(\ell)}_{\alpha_{3}}\sigma^{(\ell)}_{\alpha_{4}}\rangle_{K^{(\ell)}}\langle \sigma^{(\ell)}_{\alpha_{2}}\sigma^{(\ell)}_{\alpha_{5}}\rangle_{K^{(\ell)}} + \langle \sigma^{(\ell)}_{\alpha_{2}}\sigma^{(\ell)}_{\alpha_{5}}\sigma^{(\ell)}_{\alpha_{3}}\sigma^{(\ell)}_{\alpha_{4}}\rangle_{K^{(\ell)}}\langle \sigma^{(\ell)}_{\alpha_{1}}\sigma^{(\ell)}_{\alpha_{6}}\rangle_{K^{(\ell)}}
\bigg]
\end{align}
and one that explicitly involves $V_{4}$:
\begin{align}
-\frac{1}{n_{\ell}n_{\ell-1}}\bigg[\langle\sigma^{(\ell)}_{\alpha_{1}}\sigma^{(\ell)}_{\alpha_{2}}|\sigma^{(\ell)}_{\alpha_{3}}\sigma^{(\ell)}_{\alpha_{5}}\rangle_{V^{(\ell)}}\langle \sigma^{(\ell)}_{\alpha_{4}}\sigma^{(\ell)}_{\alpha_{6}}\rangle_{K^{(\ell)}} + \langle\sigma^{(\ell)}_{\alpha_{1}}\sigma^{(\ell)}_{\alpha_{2}}|\sigma^{(\ell)}_{\alpha_{4}}\sigma^{(\ell)}_{\alpha_{6}}\rangle_{V^{(\ell)}}\langle \sigma^{(\ell)}_{\alpha_{3}}\sigma^{(\ell)}_{\alpha_{5}}\rangle_{K^{(\ell)}}\nonumber\\
 + \langle\sigma^{(\ell)}_{\alpha_{3}}\sigma^{(\ell)}_{\alpha_{5}}|\sigma^{(\ell)}_{\alpha_{4}}\sigma^{(\ell)}_{\alpha_{6}}\rangle_{V^{(\ell)}}\langle \sigma^{(\ell)}_{\alpha_{1}}\sigma^{(\ell)}_{\alpha_{2}}\rangle_{K^{(\ell)}}
+ \langle \sigma^{(\ell)}_{\alpha_{1}}\sigma^{(\ell)}_{\alpha_{2}}|\sigma^{(\ell)}_{\alpha_{3}}\sigma^{(\ell)}_{\alpha_{6}}\rangle_{V^{(\ell)}}\langle \sigma^{(\ell)}_{\alpha_{4}}\sigma^{(\ell)}_{\alpha_{5}}\rangle_{K^{(\ell)}}\nonumber\\
+ \langle \sigma^{(\ell)}_{\alpha_{1}}\sigma^{(\ell)}_{\alpha_{2}}|\sigma^{(\ell)}_{\alpha_{4}}\sigma^{(\ell)}_{\alpha_{5}}\rangle_{V^{(\ell)}}\langle \sigma^{(\ell)}_{\alpha_{3}}\sigma^{(\ell)}_{\alpha_{6}}\rangle_{K^{(\ell)}}
+ \langle \sigma^{(\ell)}_{\alpha_{4}}\sigma^{(\ell)}_{\alpha_{5}}|\sigma^{(\ell)}_{\alpha_{3}}\sigma^{(\ell)}_{\alpha_{6}}\rangle_{V^{(\ell)}}\langle \sigma^{(\ell)}_{\alpha_{1}}\sigma^{(\ell)}_{\alpha_{2}}\rangle_{K^{(\ell)}}\nonumber\\
+ \langle \sigma^{(\ell)}_{\alpha_{1}}\sigma^{(\ell)}_{\alpha_{3}}|\sigma^{(\ell)}_{\alpha_{2}}\sigma^{(\ell)}_{\alpha_{4}}\rangle_{V^{(\ell)}}\langle \sigma^{(\ell)}_{\alpha_{5}}\sigma^{(\ell)}_{\alpha_{6}}\rangle_{K^{(\ell)}} + \langle \sigma^{(\ell)}_{\alpha_{1}}\sigma^{(\ell)}_{\alpha_{3}}|\sigma^{(\ell)}_{\alpha_{5}}\sigma^{(\ell)}_{\alpha_{6}}\rangle_{V^{(\ell)}}\langle \sigma^{(\ell)}_{\alpha_{2}}\sigma^{(\ell)}_{\alpha_{4}}\rangle_{K^{(\ell)}}\nonumber\\
 + \langle \sigma^{(\ell)}_{\alpha_{2}}\sigma^{(\ell)}_{\alpha_{4}}|\sigma^{(\ell)}_{\alpha_{5}}\sigma^{(\ell)}_{\alpha_{6}}\rangle_{V^{(\ell)}}\langle \sigma^{(\ell)}_{\alpha_{1}}\sigma^{(\ell)}_{\alpha_{3}}\rangle_{K^{(\ell)}} 
+ \langle \sigma^{(\ell)}_{\alpha_{1}}\sigma^{(\ell)}_{\alpha_{4}}|\sigma^{(\ell)}_{\alpha_{2}}\sigma^{(\ell)}_{\alpha_{3}}\rangle_{V^{(\ell)}}\langle \sigma^{(\ell)}_{\alpha_{5}}\sigma^{(\ell)}_{\alpha_{6}}\rangle_{K^{(\ell)}}\nonumber\\
 + \langle \sigma^{(\ell)}_{\alpha_{1}}\sigma^{(\ell)}_{\alpha_{4}}|\sigma^{(\ell)}_{\alpha_{5}}\sigma^{(\ell)}_{\alpha_{6}}\rangle_{V^{(\ell)}}\langle \sigma^{(\ell)}_{\alpha_{2}}\sigma^{(\ell)}_{\alpha_{3}}\rangle_{K^{(\ell)}} + \langle \sigma^{(\ell)}_{\alpha_{2}}\sigma^{(\ell)}_{\alpha_{3}}|\sigma^{(\ell)}_{\alpha_{5}}\sigma^{(\ell)}_{\alpha_{6}}\rangle_{V^{(\ell)}}\langle \sigma^{(\ell)}_{\alpha_{1}}\sigma^{(\ell)}_{\alpha_{4}}\rangle_{K^{(\ell)}}\nonumber\\
+ \langle \sigma^{(\ell)}_{\alpha_{1}}\sigma^{(\ell)}_{\alpha_{5}}|\sigma^{(\ell)}_{\alpha_{2}}\sigma^{(\ell)}_{\alpha_{6}}\rangle_{V^{(\ell)}}\langle \sigma^{(\ell)}_{\alpha_{3}}\sigma^{(\ell)}_{\alpha_{4}}\rangle_{K^{(\ell)}} + \langle \sigma^{(\ell)}_{\alpha_{1}}\sigma^{(\ell)}_{\alpha_{5}}|\sigma^{(\ell)}_{\alpha_{3}}\sigma^{(\ell)}_{\alpha_{4}}\rangle_{V^{(\ell)}}\langle \sigma^{(\ell)}_{\alpha_{2}}\sigma^{(\ell)}_{\alpha_{6}}\rangle_{K^{(\ell)}}\nonumber\\
 + \langle \sigma^{(\ell)}_{\alpha_{2}}\sigma^{(\ell)}_{\alpha_{6}}|\sigma^{(\ell)}_{\alpha_{3}}\sigma^{(\ell)}_{\alpha_{4}}\rangle_{V^{(\ell)}}\langle \sigma^{(\ell)}_{\alpha_{1}}\sigma^{(\ell)}_{\alpha_{5}}\rangle_{K^{(\ell)}} + \langle \sigma^{(\ell)}_{\alpha_{1}}\sigma^{(\ell)}_{\alpha_{6}}|\sigma^{(\ell)}_{\alpha_{2}}\sigma^{(\ell)}_{\alpha_{5}}\rangle_{V^{(\ell)}}\langle \sigma^{(\ell)}_{\alpha_{3}}\sigma^{(\ell)}_{\alpha_{4}}\rangle_{K^{(\ell)}} \nonumber\\
 + \langle \sigma^{(\ell)}_{\alpha_{1}}\sigma^{(\ell)}_{\alpha_{6}}|\sigma^{(\ell)}_{\alpha_{3}}\sigma^{(\ell)}_{\alpha_{4}}\rangle_{V^{(\ell)}}\langle \sigma^{(\ell)}_{\alpha_{2}}\sigma^{(\ell)}_{\alpha_{5}}\rangle_{K^{(\ell)}} + \langle \sigma^{(\ell)}_{\alpha_{2}}\sigma^{(\ell)}_{\alpha_{5}}|\sigma^{(\ell)}_{\alpha_{3}}\sigma^{(\ell)}_{\alpha_{4}}\rangle_{V^{(\ell)}}\langle \sigma^{(\ell)}_{\alpha_{1}}\sigma^{(\ell)}_{\alpha_{6}}\rangle_{K^{(\ell)}}
\bigg]
\end{align}
where we have introduced the notation $\langle\sigma_{\alpha}\sigma_{\beta}|\sigma_{\delta}\sigma_{\epsilon}\rangle_{V}$ as shorthand for
\begin{align}
\langle\sigma^{(\ell)}_{\alpha}\sigma^{(\ell)}_{\beta}|\sigma^{(\ell)}_{\delta}\sigma^{(\ell)}_{\epsilon}\rangle_{V^{(\ell)}} &= \frac{1}{4}\sum_{\beta_{i} \in \{\alpha,\beta,\delta,\epsilon\}}\Biggl\langle \frac{d^2 (\sigma^{(\ell)}_{\alpha}\sigma^{(\ell)}_{\beta})}{d z^{(\ell)}_{\beta_{1}} d z^{(\ell)}_{\beta_{2}}}\Biggr\rangle_{K^{(\ell)}} V^{(\ell)}_{\beta_{1}\beta_{2}\beta_{3}\beta_{4}} \Biggl\langle \frac{d^2 (\sigma^{(\ell)}_{\delta}\sigma^{(\ell)}_{\epsilon})}{d z^{(\ell)}_{\beta_{3}} d z^{(\ell)}_{\beta_{4}}}\Biggr\rangle_{K^{(\ell)}}
\end{align}
\end{itemize}

\item In the $\frac{1}{n}$-expansion of $\mathcal{W}[3]$, only the term of order $\frac{1}{n^{5}}$ contributes. This occurs when all neural indices are distinct:
\begin{align}
\frac{2}{n_{\ell}^{2}}\bigg[\langle\sigma^{(\ell)}_{\alpha_{1}}\sigma^{(\ell)}_{\alpha_{3}}\rangle_{K^{(\ell)}}\langle\sigma^{(\ell)}_{\alpha_{2}}\sigma^{(\ell)}_{\alpha_{5}}\rangle_{K^{(\ell)}}\langle\sigma^{(\ell)}_{\alpha_{4}}\sigma^{(\ell)}_{\alpha_{6}}\rangle_{K^{(\ell)}} + \langle\sigma^{(\ell)}_{\alpha_{1}}\sigma^{(\ell)}_{\alpha_{3}}\rangle_{K^{(\ell)}}\langle\sigma^{(\ell)}_{\alpha_{2}}\sigma^{(\ell)}_{\alpha_{6}}\rangle_{K^{(\ell)}}\langle\sigma^{(\ell)}_{\alpha_{4}}\sigma^{(\ell)}_{\alpha_{5}}\rangle_{K^{(\ell)}}\nonumber\\
+ \langle\sigma^{(\ell)}_{\alpha_{1}}\sigma^{(\ell)}_{\alpha_{4}}\rangle_{K^{(\ell)}}\langle\sigma^{(\ell)}_{\alpha_{2}}\sigma^{(\ell)}_{\alpha_{5}}\rangle_{K^{(\ell)}}\langle\sigma^{(\ell)}_{\alpha_{3}}\sigma^{(\ell)}_{\alpha_{6}}\rangle_{K^{(\ell)}} + \langle\sigma^{(\ell)}_{\alpha_{1}}\sigma^{(\ell)}_{\alpha_{4}}\rangle_{K^{(\ell)}}\langle\sigma^{(\ell)}_{\alpha_{2}}\sigma^{(\ell)}_{\alpha_{6}}\rangle_{K^{(\ell)}}\langle\sigma^{(\ell)}_{\alpha_{3}}\sigma^{(\ell)}_{\alpha_{5}}\rangle_{K^{(\ell)}}\nonumber\\
+ \langle\sigma^{(\ell)}_{\alpha_{1}}\sigma^{(\ell)}_{\alpha_{5}}\rangle_{K^{(\ell)}}\langle\sigma^{(\ell)}_{\alpha_{2}}\sigma^{(\ell)}_{\alpha_{3}}\rangle_{K^{(\ell)}}\langle\sigma^{(\ell)}_{\alpha_{4}}\sigma^{(\ell)}_{\alpha_{6}}\rangle_{K^{(\ell)}} + \langle\sigma^{(\ell)}_{\alpha_{1}}\sigma^{(\ell)}_{\alpha_{5}}\rangle_{K^{(\ell)}}\langle\sigma^{(\ell)}_{\alpha_{2}}\sigma^{(\ell)}_{\alpha_{4}}\rangle_{K^{(\ell)}}\langle\sigma^{(\ell)}_{\alpha_{3}}\sigma^{(\ell)}_{\alpha_{6}}\rangle_{K^{(\ell)}}\nonumber\\
+ \langle\sigma^{(\ell)}_{\alpha_{1}}\sigma^{(\ell)}_{\alpha_{6}}\rangle_{K^{(\ell)}}\langle\sigma^{(\ell)}_{\alpha_{2}}\sigma^{(\ell)}_{\alpha_{3}}\rangle_{K^{(\ell)}}\langle\sigma^{(\ell)}_{\alpha_{4}}\sigma^{(\ell)}_{\alpha_{5}}\rangle_{K^{(\ell)}} + \langle\sigma^{(\ell)}_{\alpha_{1}}\sigma^{(\ell)}_{\alpha_{6}}\rangle_{K^{(\ell)}}\langle\sigma^{(\ell)}_{\alpha_{2}}\sigma^{(\ell)}_{\alpha_{4}}\rangle_{K^{(\ell)}}\langle\sigma^{(\ell)}_{\alpha_{3}}\sigma^{(\ell)}_{\alpha_{5}}\rangle_{K^{(\ell)}}
\bigg]
\end{align}
\item In the $\frac{1}{n}$-expansion of $\mathcal{W}[2]$, only the terms of order $\frac{1}{n^{4}}$ and $\frac{1}{n^{3}}$ contribute nontrivially. The $\frac{1}{n^{4}}$ term arises when all neural indices are distinct
\begin{align}
-\frac{1}{n_{\ell}^{2}}\bigg[\langle \sigma^{(\ell)}_{\alpha_{1}}\sigma^{(\ell)}_{\alpha_{3}}\rangle_{K^{(\ell)}}\langle\sigma^{(\ell)}_{\alpha_{2}}\sigma^{(\ell)}_{\alpha_{4}}\rangle_{K^{(\ell)}}\langle \sigma^{(\ell)}_{\alpha_{5}}\sigma^{(\ell)}_{\alpha_{6}}\rangle_{K^{(\ell)}} + \langle \sigma^{(\ell)}_{\alpha_{1}}\sigma^{(\ell)}_{\alpha_{4}}\rangle_{K^{(\ell)}}\langle\sigma^{(\ell)}_{\alpha_{2}}\sigma^{(\ell)}_{\alpha_{3}}\rangle_{K^{(\ell)}}\langle \sigma^{(\ell)}_{\alpha_{5}}\sigma^{(\ell)}_{\alpha_{6}}\rangle_{K^{(\ell)}} \nonumber\\
+ \langle \sigma^{(\ell)}_{\alpha_{3}}\sigma^{(\ell)}_{\alpha_{5}}\rangle_{K^{(\ell)}}\langle\sigma^{(\ell)}_{\alpha_{4}}\sigma^{(\ell)}_{\alpha_{6}}\rangle_{K^{(\ell)}}\langle \sigma^{(\ell)}_{\alpha_{1}}\sigma^{(\ell)}_{\alpha_{2}}\rangle_{K^{(\ell)}} + \langle \sigma^{(\ell)}_{\alpha_{3}}\sigma^{(\ell)}_{\alpha_{6}}\rangle_{K^{(\ell)}}\langle\sigma^{(\ell)}_{\alpha_{4}}\sigma^{(\ell)}_{\alpha_{5}}\rangle_{K^{(\ell)}}\langle \sigma^{(\ell)}_{\alpha_{1}}\sigma^{(\ell)}_{\alpha_{2}}\rangle_{K^{(\ell)}} \nonumber\\
+ \langle \sigma^{(\ell)}_{\alpha_{1}}\sigma^{(\ell)}_{\alpha_{5}}\rangle_{K^{(\ell)}}\langle\sigma^{(\ell)}_{\alpha_{2}}\sigma^{(\ell)}_{\alpha_{6}}\rangle_{K^{(\ell)}}\langle \sigma^{(\ell)}_{\alpha_{3}}\sigma^{(\ell)}_{\alpha_{4}}\rangle_{K^{(\ell)}} + \langle \sigma^{(\ell)}_{\alpha_{1}}\sigma^{(\ell)}_{\alpha_{6}}\rangle_{K^{(\ell)}}\langle\sigma^{(\ell)}_{\alpha_{2}}\sigma^{(\ell)}_{\alpha_{5}}\rangle_{K^{(\ell)}}\langle \sigma^{(\ell)}_{\alpha_{3}}\sigma^{(\ell)}_{\alpha_{4}}\rangle_{K^{(\ell)}}\bigg]  
\end{align}
The $\frac{1}{n^{3}}$ contribution, in turn, splits into two cases:
\begin{itemize}
\item one in which all neural indices are distinct:
\begin{align}
-\frac{1}{n_{\ell}^{2}}\bigg[\langle \sigma^{(\ell)}_{\alpha_{1}}\sigma^{(\ell)}_{\alpha_{2}}\rangle_{K^{(\ell)}}\langle \sigma^{(\ell)}_{\alpha_{3}}\sigma^{(\ell)}_{\alpha_{5}}\rangle_{K^{(\ell)}}\langle \sigma^{(\ell)}_{\alpha_{4}}\sigma^{(\ell)}_{\alpha_{6}}\rangle_{K^{(\ell)}} + \langle \sigma^{(\ell)}_{\alpha_{1}}\sigma^{(\ell)}_{\alpha_{2}}\rangle_{K^{(\ell)}}\langle \sigma^{(\ell)}_{\alpha_{3}}\sigma^{(\ell)}_{\alpha_{6}}\rangle_{K^{(\ell)}}\langle \sigma^{(\ell)}_{\alpha_{4}}\sigma^{(\ell)}_{\alpha_{5}}\rangle_{K^{(\ell)}}\nonumber\\
+ \langle \sigma^{(\ell)}_{\alpha_{1}}\sigma^{(\ell)}_{\alpha_{3}}\rangle_{K^{(\ell)}}\langle \sigma^{(\ell)}_{\alpha_{2}}\sigma^{(\ell)}_{\alpha_{4}}\rangle_{K^{(\ell)}}\langle \sigma^{(\ell)}_{\alpha_{5}}\sigma^{(\ell)}_{\alpha_{6}}\rangle_{K^{(\ell)}} + \langle \sigma^{(\ell)}_{\alpha_{1}}\sigma^{(\ell)}_{\alpha_{4}}\rangle_{K^{(\ell)}}\langle \sigma^{(\ell)}_{\alpha_{2}}\sigma^{(\ell)}_{\alpha_{3}}\rangle_{K^{(\ell)}}\langle \sigma^{(\ell)}_{\alpha_{5}}\sigma^{(\ell)}_{\alpha_{6}}\rangle_{K^{(\ell)}}\nonumber\\
+ \langle \sigma^{(\ell)}_{\alpha_{1}}\sigma^{(\ell)}_{\alpha_{5}}\rangle_{K^{(\ell)}}\langle \sigma^{(\ell)}_{\alpha_{2}}\sigma^{(\ell)}_{\alpha_{6}}\rangle_{K^{(\ell)}}\langle \sigma^{(\ell)}_{\alpha_{3}}\sigma^{(\ell)}_{\alpha_{4}}\rangle_{K^{(\ell)}} + \langle \sigma^{(\ell)}_{\alpha_{1}}\sigma^{(\ell)}_{\alpha_{6}}\rangle_{K^{(\ell)}}\langle \sigma^{(\ell)}_{\alpha_{2}}\sigma^{(\ell)}_{\alpha_{5}}\rangle_{K^{(\ell)}}\langle \sigma^{(\ell)}_{\alpha_{3}}\sigma^{(\ell)}_{\alpha_{4}}\rangle_{K^{(\ell)}}\bigg]
\end{align}
\item and one in which exactly two neural indices coincide. In the latter case, we again obtain two types of contributions: one that does not involve the four-point cumulant,
\begin{align}
\frac{1}{n_{\ell}^{2}}\bigg[\langle \sigma^{(\ell)}_{\alpha_{1}}\sigma^{(\ell)}_{\alpha_{3}}\sigma^{(\ell)}_{\alpha_{2}}\sigma^{(\ell)}_{\alpha_{4}}\rangle_{K^{(\ell)}}\langle \sigma^{(\ell)}_{\alpha_{5}}\sigma^{(\ell)}_{\alpha_{6}}\rangle_{K^{(\ell)}} + \langle \sigma^{(\ell)}_{\alpha_{1}}\sigma^{(\ell)}_{\alpha_{4}}\sigma^{(\ell)}_{\alpha_{2}}\sigma^{(\ell)}_{\alpha_{3}}\rangle_{K^{(\ell)}}\langle \sigma^{(\ell)}_{\alpha_{5}}\sigma^{(\ell)}_{\alpha_{6}}\rangle_{K^{(\ell)}} \nonumber\\
+ \langle \sigma^{(\ell)}_{\alpha_{3}}\sigma^{(\ell)}_{\alpha_{5}}\sigma^{(\ell)}_{\alpha_{4}}\sigma^{(\ell)}_{\alpha_{6}}\rangle_{K^{(\ell)}}\langle \sigma^{(\ell)}_{\alpha_{1}}\sigma^{(\ell)}_{\alpha_{2}}\rangle_{K^{(\ell)}} + \langle \sigma^{(\ell)}_{\alpha_{3}}\sigma^{(\ell)}_{\alpha_{6}}\sigma^{(\ell)}_{\alpha_{4}}\sigma^{(\ell)}_{\alpha_{5}}\rangle_{K^{(\ell)}}\langle \sigma^{(\ell)}_{\alpha_{1}}\sigma^{(\ell)}_{\alpha_{2}}\rangle_{K^{(\ell)}} \nonumber\\
+ \langle \sigma^{(\ell)}_{\alpha_{1}}\sigma^{(\ell)}_{\alpha_{5}}\sigma^{(\ell)}_{\alpha_{2}}\sigma^{(\ell)}_{\alpha_{6}}\rangle_{K^{(\ell)}}\langle \sigma^{(\ell)}_{\alpha_{3}}\sigma^{(\ell)}_{\alpha_{4}}\rangle_{K^{(\ell)}} + \langle \sigma^{(\ell)}_{\alpha_{1}}\sigma^{(\ell)}_{\alpha_{6}}\sigma^{(\ell)}_{\alpha_{2}}\sigma^{(\ell)}_{\alpha_{5}}\rangle_{K^{(\ell)}}\langle \sigma^{(\ell)}_{\alpha_{3}}\sigma^{(\ell)}_{\alpha_{4}}\rangle_{K^{(\ell)}}\bigg]  
\end{align}
and one that explicitly involves $V_{4}$
\begin{align}
\frac{1}{n_{\ell}n_{\ell-1}}\bigg[\langle \sigma^{(\ell)}_{\alpha_{1}}\sigma^{(\ell)}_{\alpha_{3}}|\sigma^{(\ell)}_{\alpha_{2}}\sigma^{(\ell)}_{\alpha_{4}}\rangle_{V^{(\ell)}}\langle \sigma^{(\ell)}_{\alpha_{5}}\sigma^{(\ell)}_{\alpha_{6}}\rangle_{K^{(\ell)}} + \langle \sigma^{(\ell)}_{\alpha_{1}}\sigma^{(\ell)}_{\alpha_{4}}|\sigma^{(\ell)}_{\alpha_{2}}\sigma^{(\ell)}_{\alpha_{3}}\rangle_{V^{(\ell)}}\langle \sigma^{(\ell)}_{\alpha_{5}}\sigma^{(\ell)}_{\alpha_{6}}\rangle_{K^{(\ell)}} \nonumber\\
+ \langle \sigma^{(\ell)}_{\alpha_{3}}\sigma^{(\ell)}_{\alpha_{5}}|\sigma^{(\ell)}_{\alpha_{4}}\sigma^{(\ell)}_{\alpha_{6}}\rangle_{V^{(\ell)}}\langle \sigma^{(\ell)}_{\alpha_{1}}\sigma^{(\ell)}_{\alpha_{2}}\rangle_{K^{(\ell)}} + \langle \sigma^{(\ell)}_{\alpha_{3}}\sigma^{(\ell)}_{\alpha_{6}}|\sigma^{(\ell)}_{\alpha_{4}}\sigma^{(\ell)}_{\alpha_{5}}\rangle_{V^{(\ell)}}\langle \sigma^{(\ell)}_{\alpha_{1}}\sigma^{(\ell)}_{\alpha_{2}}\rangle_{K^{(\ell)}} \nonumber\\
+ \langle \sigma^{(\ell)}_{\alpha_{1}}\sigma^{(\ell)}_{\alpha_{5}}|\sigma^{(\ell)}_{\alpha_{2}}\sigma^{(\ell)}_{\alpha_{6}}\rangle_{V^{(\ell)}}\langle \sigma^{(\ell)}_{\alpha_{3}}\sigma^{(\ell)}_{\alpha_{4}}\rangle_{K^{(\ell)}} + \langle \sigma^{(\ell)}_{\alpha_{1}}\sigma^{(\ell)}_{\alpha_{6}}|\sigma^{(\ell)}_{\alpha_{2}}\sigma^{(\ell)}_{\alpha_{5}}\rangle_{V^{(\ell)}}\langle \sigma^{(\ell)}_{\alpha_{3}}\sigma^{(\ell)}_{\alpha_{4}}\rangle_{K^{(\ell)}}\bigg]  
\end{align}
\end{itemize}
\end{enumerate}
\allowdisplaybreaks
Combining all contributions, we obtain
\begin{align}
\frac{1}{n_{\ell}^2}\Delta V^{(\ell+1)}_{\alpha_{1}\ldots \alpha_{6}} &= -\frac{(C_{W}^{(\ell+1)})^{3}}{n_{\ell}^{2}}\bigg[\langle \sigma^{(\ell)}_{\alpha_{1}}\sigma^{(\ell)}_{\alpha_{2}}\sigma^{(\ell)}_{\alpha_{3}}\sigma^{(\ell)}_{\alpha_{5}}\rangle_{K^{(\ell)}}\langle \sigma^{(\ell)}_{\alpha_{4}}\sigma^{(\ell)}_{\alpha_{6}}\rangle_{K^{(\ell)}} + \langle \sigma^{(\ell)}_{\alpha_{1}}\sigma^{(\ell)}_{\alpha_{2}}\sigma^{(\ell)}_{\alpha_{4}}\sigma^{(\ell)}_{\alpha_{6}}\rangle_{K^{(\ell)}}\langle \sigma^{(\ell)}_{\alpha_{3}}\sigma^{(\ell)}_{\alpha_{5}}\rangle_{K^{(\ell)}}\nonumber\\
&\hspace{-30pt} + \langle \sigma^{(\ell)}_{\alpha_{1}}\sigma^{(\ell)}_{\alpha_{2}}\sigma^{(\ell)}_{\alpha_{3}}\sigma^{(\ell)}_{\alpha_{6}}\rangle_{K^{(\ell)}}\langle \sigma^{(\ell)}_{\alpha_{4}}\sigma^{(\ell)}_{\alpha_{5}}\rangle_{K^{(\ell)}}
+ \langle \sigma^{(\ell)}_{\alpha_{1}}\sigma^{(\ell)}_{\alpha_{2}}\sigma^{(\ell)}_{\alpha_{4}}\sigma^{(\ell)}_{\alpha_{5}}\rangle_{K^{(\ell)}}\langle \sigma^{(\ell)}_{\alpha_{3}}\sigma^{(\ell)}_{\alpha_{6}}\rangle_{K^{(\ell)}}\nonumber\\
&\hspace{-30pt} + \langle \sigma^{(\ell)}_{\alpha_{1}}\sigma^{(\ell)}_{\alpha_{3}}\sigma^{(\ell)}_{\alpha_{5}}\sigma^{(\ell)}_{\alpha_{6}}\rangle_{K^{(\ell)}}\langle \sigma^{(\ell)}_{\alpha_{2}}\sigma^{(\ell)}_{\alpha_{4}}\rangle_{K^{(\ell)}} + \langle \sigma^{(\ell)}_{\alpha_{2}}\sigma^{(\ell)}_{\alpha_{4}}\sigma^{(\ell)}_{\alpha_{5}}\sigma^{(\ell)}_{\alpha_{6}}\rangle_{K^{(\ell)}}\langle \sigma^{(\ell)}_{\alpha_{1}}\sigma^{(\ell)}_{\alpha_{3}}\rangle_{K^{(\ell)}} \nonumber\\
&\hspace{-30pt} + \langle \sigma^{(\ell)}_{\alpha_{1}}\sigma^{(\ell)}_{\alpha_{4}}\sigma^{(\ell)}_{\alpha_{5}}\sigma^{(\ell)}_{\alpha_{6}}\rangle_{K^{(\ell)}}\langle \sigma^{(\ell)}_{\alpha_{2}}\sigma^{(\ell)}_{\alpha_{3}}\rangle_{K^{(\ell)}} + \langle \sigma^{(\ell)}_{\alpha_{2}}\sigma^{(\ell)}_{\alpha_{3}}\sigma^{(\ell)}_{\alpha_{5}}\sigma^{(\ell)}_{\alpha_{6}}\rangle_{K^{(\ell)}}\langle \sigma^{(\ell)}_{\alpha_{1}}\sigma^{(\ell)}_{\alpha_{4}}\rangle_{K^{(\ell)}}\nonumber\\
&\hspace{-30pt} + \langle \sigma^{(\ell)}_{\alpha_{1}}\sigma^{(\ell)}_{\alpha_{5}}\sigma^{(\ell)}_{\alpha_{3}}\sigma^{(\ell)}_{\alpha_{4}}\rangle_{K^{(\ell)}}\langle \sigma^{(\ell)}_{\alpha_{2}}\sigma^{(\ell)}_{\alpha_{6}}\rangle_{K^{(\ell)}} + \langle \sigma^{(\ell)}_{\alpha_{2}}\sigma^{(\ell)}_{\alpha_{6}}\sigma^{(\ell)}_{\alpha_{3}}\sigma^{(\ell)}_{\alpha_{4}}\rangle_{K^{(\ell)}}\langle \sigma^{(\ell)}_{\alpha_{1}}\sigma^{(\ell)}_{\alpha_{5}}\rangle_{K^{(\ell)}}
\nonumber\\
&\hspace{-30pt} + \langle \sigma^{(\ell)}_{\alpha_{1}}\sigma^{(\ell)}_{\alpha_{6}}\sigma^{(\ell)}_{\alpha_{3}}\sigma^{(\ell)}_{\alpha_{4}}\rangle_{K^{(\ell)}}\langle \sigma^{(\ell)}_{\alpha_{2}}\sigma^{(\ell)}_{\alpha_{5}}\rangle_{K^{(\ell)}} + \langle \sigma^{(\ell)}_{\alpha_{2}}\sigma^{(\ell)}_{\alpha_{5}}\sigma^{(\ell)}_{\alpha_{3}}\sigma^{(\ell)}_{\alpha_{4}}\rangle_{K^{(\ell)}}\langle \sigma^{(\ell)}_{\alpha_{1}}\sigma^{(\ell)}_{\alpha_{6}}\rangle_{K^{(\ell)}}
\bigg]\nonumber\\
&\hspace{-30pt} -\frac{(C_{W}^{(\ell+1)})^{3}}{n_{\ell}n_{\ell-1}}\bigg[\langle\sigma^{(\ell)}_{\alpha_{1}}\sigma^{(\ell)}_{\alpha_{2}}|\sigma^{(\ell)}_{\alpha_{3}}\sigma^{(\ell)}_{\alpha_{5}}\rangle_{V^{(\ell)}}\langle \sigma^{(\ell)}_{\alpha_{4}}\sigma^{(\ell)}_{\alpha_{6}}\rangle_{K^{(\ell)}} + \langle\sigma^{(\ell)}_{\alpha_{1}}\sigma^{(\ell)}_{\alpha_{2}}|\sigma^{(\ell)}_{\alpha_{4}}\sigma^{(\ell)}_{\alpha_{6}}\rangle_{V^{(\ell)}}\langle \sigma^{(\ell)}_{\alpha_{3}}\sigma^{(\ell)}_{\alpha_{5}}\rangle_{K^{(\ell)}} \nonumber\\
&\hspace{-30pt} + \langle \sigma^{(\ell)}_{\alpha_{1}}\sigma^{(\ell)}_{\alpha_{2}}|\sigma^{(\ell)}_{\alpha_{3}}\sigma^{(\ell)}_{\alpha_{6}}\rangle_{V^{(\ell)}}\langle \sigma^{(\ell)}_{\alpha_{4}}\sigma^{(\ell)}_{\alpha_{5}}\rangle_{K^{(\ell)}}
+ \langle \sigma^{(\ell)}_{\alpha_{1}}\sigma^{(\ell)}_{\alpha_{2}}|\sigma^{(\ell)}_{\alpha_{4}}\sigma^{(\ell)}_{\alpha_{5}}\rangle_{V^{(\ell)}}\langle \sigma^{(\ell)}_{\alpha_{3}}\sigma^{(\ell)}_{\alpha_{6}}\rangle_{K^{(\ell)}}
\nonumber\\
&\hspace{-30pt} + \langle \sigma^{(\ell)}_{\alpha_{1}}\sigma^{(\ell)}_{\alpha_{3}}|\sigma^{(\ell)}_{\alpha_{5}}\sigma^{(\ell)}_{\alpha_{6}}\rangle_{V^{(\ell)}}\langle \sigma^{(\ell)}_{\alpha_{2}}\sigma^{(\ell)}_{\alpha_{4}}\rangle_{K^{(\ell)}} + \langle \sigma^{(\ell)}_{\alpha_{2}}\sigma^{(\ell)}_{\alpha_{4}}|\sigma^{(\ell)}_{\alpha_{5}}\sigma^{(\ell)}_{\alpha_{6}}\rangle_{V^{(\ell)}}\langle \sigma^{(\ell)}_{\alpha_{1}}\sigma^{(\ell)}_{\alpha_{3}}\rangle_{K^{(\ell)}} \nonumber\\
&\hspace{-30pt} + \langle \sigma^{(\ell)}_{\alpha_{1}}\sigma^{(\ell)}_{\alpha_{4}}|\sigma^{(\ell)}_{\alpha_{5}}\sigma^{(\ell)}_{\alpha_{6}}\rangle_{V^{(\ell)}}\langle \sigma^{(\ell)}_{\alpha_{2}}\sigma^{(\ell)}_{\alpha_{3}}\rangle_{K^{(\ell)}} + \langle \sigma^{(\ell)}_{\alpha_{2}}\sigma^{(\ell)}_{\alpha_{3}}|\sigma^{(\ell)}_{\alpha_{5}}\sigma^{(\ell)}_{\alpha_{6}}\rangle_{V^{(\ell)}}\langle \sigma^{(\ell)}_{\alpha_{1}}\sigma^{(\ell)}_{\alpha_{4}}\rangle_{K^{(\ell)}}\nonumber\\
&\hspace{-30pt} + \langle \sigma^{(\ell)}_{\alpha_{1}}\sigma^{(\ell)}_{\alpha_{5}}|\sigma^{(\ell)}_{\alpha_{3}}\sigma^{(\ell)}_{\alpha_{4}}\rangle_{V^{(\ell)}}\langle \sigma^{(\ell)}_{\alpha_{2}}\sigma^{(\ell)}_{\alpha_{6}}\rangle_{K^{(\ell)}} + \langle \sigma^{(\ell)}_{\alpha_{2}}\sigma^{(\ell)}_{\alpha_{6}}|\sigma^{(\ell)}_{\alpha_{3}}\sigma^{(\ell)}_{\alpha_{4}}\rangle_{V^{(\ell)}}\langle \sigma^{(\ell)}_{\alpha_{1}}\sigma^{(\ell)}_{\alpha_{5}}\rangle_{K^{(\ell)}}
\nonumber\\
&\hspace{-30pt} + \langle \sigma^{(\ell)}_{\alpha_{1}}\sigma^{(\ell)}_{\alpha_{6}}|\sigma^{(\ell)}_{\alpha_{3}}\sigma^{(\ell)}_{\alpha_{4}}\rangle_{V^{(\ell)}}\langle \sigma^{(\ell)}_{\alpha_{2}}\sigma^{(\ell)}_{\alpha_{5}}\rangle_{K^{(\ell)}} + \langle \sigma^{(\ell)}_{\alpha_{2}}\sigma^{(\ell)}_{\alpha_{5}}|\sigma^{(\ell)}_{\alpha_{3}}\sigma^{(\ell)}_{\alpha_{4}}\rangle_{V^{(\ell)}}\langle \sigma^{(\ell)}_{\alpha_{1}}\sigma^{(\ell)}_{\alpha_{6}}\rangle_{K^{(\ell)}}
\bigg]\nonumber\\
&\hspace{-30pt} + \frac{2(C_{W}^{(\ell+1)})^{3}}{n_{\ell}^{2}}\bigg[\langle\sigma^{(\ell)}_{\alpha_{1}}\sigma^{(\ell)}_{\alpha_{3}}\rangle_{K^{(\ell)}}\langle\sigma^{(\ell)}_{\alpha_{2}}\sigma^{(\ell)}_{\alpha_{5}}\rangle_{K^{(\ell)}}\langle\sigma^{(\ell)}_{\alpha_{4}}\sigma^{(\ell)}_{\alpha_{6}}\rangle_{K^{(\ell)}}  + \langle\sigma^{(\ell)}_{\alpha_{1}}\sigma^{(\ell)}_{\alpha_{3}}\rangle_{K^{(\ell)}}\langle\sigma^{(\ell)}_{\alpha_{2}}\sigma^{(\ell)}_{\alpha_{6}}\rangle_{K^{(\ell)}}\langle\sigma^{(\ell)}_{\alpha_{4}}\sigma^{(\ell)}_{\alpha_{5}}\rangle_{K^{(\ell)}}\nonumber\\
&\hspace{-30pt} + \langle\sigma^{(\ell)}_{\alpha_{1}}\sigma^{(\ell)}_{\alpha_{4}}\rangle_{K^{(\ell)}}\langle\sigma^{(\ell)}_{\alpha_{2}}\sigma^{(\ell)}_{\alpha_{5}}\rangle_{K^{(\ell)}}\langle\sigma^{(\ell)}_{\alpha_{3}}\sigma^{(\ell)}_{\alpha_{6}}\rangle_{K^{(\ell)}}  + \langle\sigma^{(\ell)}_{\alpha_{1}}\sigma^{(\ell)}_{\alpha_{4}}\rangle_{K^{(\ell)}}\langle\sigma^{(\ell)}_{\alpha_{2}}\sigma^{(\ell)}_{\alpha_{6}}\rangle_{K^{(\ell)}}\langle\sigma^{(\ell)}_{\alpha_{3}}\sigma^{(\ell)}_{\alpha_{5}}\rangle_{K^{(\ell)}}\nonumber\\
&\hspace{-30pt}+ \langle\sigma^{(\ell)}_{\alpha_{1}}\sigma^{(\ell)}_{\alpha_{5}}\rangle_{K^{(\ell)}}\langle\sigma^{(\ell)}_{\alpha_{2}}\sigma^{(\ell)}_{\alpha_{3}}\rangle_{K^{(\ell)}}\langle\sigma^{(\ell)}_{\alpha_{4}}\sigma^{(\ell)}_{\alpha_{6}}\rangle_{K^{(\ell)}} + \langle\sigma^{(\ell)}_{\alpha_{1}}\sigma^{(\ell)}_{\alpha_{5}}\rangle_{K^{(\ell)}}\langle\sigma^{(\ell)}_{\alpha_{2}}\sigma^{(\ell)}_{\alpha_{4}}\rangle_{K^{(\ell)}}\langle\sigma^{(\ell)}_{\alpha_{3}}\sigma^{(\ell)}_{\alpha_{6}}\rangle_{K^{(\ell)}}\nonumber\\
&\hspace{-30pt} + \langle\sigma^{(\ell)}_{\alpha_{1}}\sigma^{(\ell)}_{\alpha_{6}}\rangle_{K^{(\ell)}}\langle\sigma^{(\ell)}_{\alpha_{2}}\sigma^{(\ell)}_{\alpha_{3}}\rangle_{K^{(\ell)}}\langle\sigma^{(\ell)}_{\alpha_{4}}\sigma^{(\ell)}_{\alpha_{5}}\rangle_{K^{(\ell)}} + \langle\sigma^{(\ell)}_{\alpha_{1}}\sigma^{(\ell)}_{\alpha_{6}}\rangle_{K^{(\ell)}}\langle\sigma^{(\ell)}_{\alpha_{2}}\sigma^{(\ell)}_{\alpha_{4}}\rangle_{K^{(\ell)}}\langle\sigma^{(\ell)}_{\alpha_{3}}\sigma^{(\ell)}_{\alpha_{5}}\rangle_{K^{(\ell)}}\bigg]\nonumber\\
&\hspace{-30pt}+ \frac{2(C_{W}^{(\ell+1)})^{3}}{n_{\ell}^{2}}\bigg[\langle \sigma^{(\ell)}_{\alpha_{1}}\sigma^{(\ell)}_{\alpha_{2}}\rangle_{K^{(\ell)}}\langle \sigma^{(\ell)}_{\alpha_{3}}\sigma^{(\ell)}_{\alpha_{5}}\rangle_{K^{(\ell)}}\langle \sigma^{(\ell)}_{\alpha_{4}}\sigma^{(\ell)}_{\alpha_{6}}\rangle_{K^{(\ell)}} + \langle \sigma^{(\ell)}_{\alpha_{1}}\sigma^{(\ell)}_{\alpha_{2}}\rangle_{K^{(\ell)}}\langle \sigma^{(\ell)}_{\alpha_{3}}\sigma^{(\ell)}_{\alpha_{6}}\rangle_{K^{(\ell)}}\langle \sigma^{(\ell)}_{\alpha_{4}}\sigma^{(\ell)}_{\alpha_{5}}\rangle_{K^{(\ell)}}\nonumber\\
&\hspace{-30pt} + \langle \sigma^{(\ell)}_{\alpha_{1}}\sigma^{(\ell)}_{\alpha_{3}}\rangle_{K^{(\ell)}}\langle \sigma^{(\ell)}_{\alpha_{2}}\sigma^{(\ell)}_{\alpha_{4}}\rangle_{K^{(\ell)}}\langle \sigma^{(\ell)}_{\alpha_{5}}\sigma^{(\ell)}_{\alpha_{6}}\rangle_{K^{(\ell)}} + \langle \sigma^{(\ell)}_{\alpha_{1}}\sigma^{(\ell)}_{\alpha_{4}}\rangle_{K^{(\ell)}}\langle \sigma^{(\ell)}_{\alpha_{2}}\sigma^{(\ell)}_{\alpha_{3}}\rangle_{K^{(\ell)}}\langle \sigma^{(\ell)}_{\alpha_{5}}\sigma^{(\ell)}_{\alpha_{6}}\rangle_{K^{(\ell)}}\nonumber\\
&\hspace{-30pt} + \langle \sigma^{(\ell)}_{\alpha_{1}}\sigma^{(\ell)}_{\alpha_{5}}\rangle_{K^{(\ell)}}\langle \sigma^{(\ell)}_{\alpha_{2}}\sigma^{(\ell)}_{\alpha_{6}}\rangle_{K^{(\ell)}}\langle \sigma^{(\ell)}_{\alpha_{3}}\sigma^{(\ell)}_{\alpha_{4}}\rangle_{K^{(\ell)}} + \langle \sigma^{(\ell)}_{\alpha_{1}}\sigma^{(\ell)}_{\alpha_{6}}\rangle_{K^{(\ell)}}\langle \sigma^{(\ell)}_{\alpha_{2}}\sigma^{(\ell)}_{\alpha_{5}}\rangle_{K^{(\ell)}}\langle \sigma^{(\ell)}_{\alpha_{3}}\sigma^{(\ell)}_{\alpha_{4}}\rangle_{K^{(\ell)}}\bigg]
\end{align}
or, more compactly
\begin{align}
\frac{1}{n_{\ell}^2}\Delta V^{(\ell+1)}_{\alpha_{1}\ldots \alpha_{6}} &= -\frac{(C_{W}^{(\ell+1)})^{3}}{n_{\ell}^{2}}\bigg[\langle \widehat{\Delta G}^{(\ell)}_{\alpha_{1}\alpha_{2}}\widehat{\Delta G}^{(\ell)}_{\alpha_{3}\alpha_{5}}\rangle_{K^{(\ell)}}\langle \sigma^{(\ell)}_{\alpha_{4}}\sigma^{(\ell)}_{\alpha_{6}}\rangle_{K^{(\ell)}} + \langle \widehat{\Delta G}^{(\ell)}_{\alpha_{1}\alpha_{2}}\widehat{\Delta G}^{(\ell)}_{\alpha_{4}\alpha_{6}}\rangle_{K^{(\ell)}}\langle \sigma^{(\ell)}_{\alpha_{3}}\sigma^{(\ell)}_{\alpha_{5}}\rangle_{K^{(\ell)}}\nonumber\\
&\hspace{-30pt} + \langle \widehat{\Delta G}^{(\ell)}_{\alpha_{1}\alpha_{2}}\widehat{\Delta G}^{(\ell)}_{\alpha_{3}\alpha_{6}}\rangle_{K^{(\ell)}}\langle \sigma^{(\ell)}_{\alpha_{4}}\sigma^{(\ell)}_{\alpha_{5}}\rangle_{K^{(\ell)}}
+ \langle \widehat{\Delta G}^{(\ell)}_{\alpha_{1}\alpha_{2}}\widehat{\Delta G}^{(\ell)}_{\alpha_{4}\alpha_{5}}\rangle_{K^{(\ell)}}\langle \sigma^{(\ell)}_{\alpha_{3}}\sigma^{(\ell)}_{\alpha_{6}}\rangle_{K^{(\ell)}}\nonumber\\
&\hspace{-30pt} + \langle \widehat{\Delta G}^{(\ell)}_{\alpha_{1}\alpha_{3}}\widehat{\Delta G}^{(\ell)}_{\alpha_{5}\alpha_{6}}\rangle_{K^{(\ell)}}\langle \sigma^{(\ell)}_{\alpha_{2}}\sigma^{(\ell)}_{\alpha_{4}}\rangle_{K^{(\ell)}} + \langle \widehat{\Delta G}^{(\ell)}_{\alpha_{2}\alpha_{4}}\widehat{\Delta G}^{(\ell)}_{\alpha_{5}\alpha_{6}}\rangle_{K^{(\ell)}}\langle \sigma^{(\ell)}_{\alpha_{1}}\sigma^{(\ell)}_{\alpha_{3}}\rangle_{K^{(\ell)}} \nonumber\\
&\hspace{-30pt} + \langle \widehat{\Delta G}^{(\ell)}_{\alpha_{1}\alpha_{4}}\widehat{\Delta G}^{(\ell)}_{\alpha_{5}\alpha_{6}}\rangle_{K^{(\ell)}}\langle \sigma^{(\ell)}_{\alpha_{2}}\sigma^{(\ell)}_{\alpha_{3}}\rangle_{K^{(\ell)}} + \langle \widehat{\Delta G}^{(\ell)}_{\alpha_{2}\alpha_{3}}\widehat{\Delta G}^{(\ell)}_{\alpha_{5}\alpha_{6}}\rangle_{K^{(\ell)}}\langle \sigma^{(\ell)}_{\alpha_{1}}\sigma^{(\ell)}_{\alpha_{4}}\rangle_{K^{(\ell)}}\nonumber\\
&\hspace{-30pt} + \langle \widehat{\Delta G}^{(\ell)}_{\alpha_{1}\alpha_{5}}\widehat{\Delta G}^{(\ell)}_{\alpha_{3}\alpha_{4}}\rangle_{K^{(\ell)}}\langle \sigma^{(\ell)}_{\alpha_{2}}\sigma^{(\ell)}_{\alpha_{6}}\rangle_{K^{(\ell)}} + \langle \widehat{\Delta G}^{(\ell)}_{\alpha_{2}\alpha_{6}}\widehat{\Delta G}^{(\ell)}_{\alpha_{3}\alpha_{4}}\rangle_{K^{(\ell)}}\langle \sigma^{(\ell)}_{\alpha_{1}}\sigma^{(\ell)}_{\alpha_{5}}\rangle_{K^{(\ell)}}
\nonumber\\
&\hspace{-30pt} + \langle \widehat{\Delta G}^{(\ell)}_{\alpha_{1}\alpha_{6}}\widehat{\Delta G}^{(\ell)}_{\alpha_{3}\alpha_{4}}\rangle_{K^{(\ell)}}\langle \sigma^{(\ell)}_{\alpha_{2}}\sigma^{(\ell)}_{\alpha_{5}}\rangle_{K^{(\ell)}} + \langle \widehat{\Delta G}^{(\ell)}_{\alpha_{2}\alpha_{5}}\widehat{\Delta G}^{(\ell)}_{\alpha_{3}\alpha_{4}}\rangle_{K^{(\ell)}}\langle \sigma^{(\ell)}_{\alpha_{1}}\sigma^{(\ell)}_{\alpha_{6}}\rangle_{K^{(\ell)}}
\bigg]\nonumber\\
&\hspace{-30pt} -\frac{(C_{W}^{(\ell+1)})^{3}}{n_{\ell}n_{\ell-1}}\bigg[\langle\widehat{\Delta G}^{(\ell)}_{\alpha_{1}\alpha_{2}}|\widehat{\Delta G}^{(\ell)}_{\alpha_{3}\alpha_{5}}\rangle_{V^{(\ell)}}\langle \sigma^{(\ell)}_{\alpha_{4}}\sigma^{(\ell)}_{\alpha_{6}}\rangle_{K^{(\ell)}} + \langle\widehat{\Delta G}^{(\ell)}_{\alpha_{1}\alpha_{2}}|\widehat{\Delta G}^{(\ell)}_{\alpha_{4}\alpha_{6}}\rangle_{V^{(\ell)}}\langle \sigma^{(\ell)}_{\alpha_{3}}\sigma^{(\ell)}_{\alpha_{5}}\rangle_{K^{(\ell)}} \nonumber\\
&\hspace{-30pt} + \langle \widehat{\Delta G}^{(\ell)}_{\alpha_{1}\alpha_{2}}|\widehat{\Delta G}^{(\ell)}_{\alpha_{3}\alpha_{6}}\rangle_{V^{(\ell)}}\langle \sigma^{(\ell)}_{\alpha_{4}}\sigma^{(\ell)}_{\alpha_{5}}\rangle_{K^{(\ell)}}
+ \langle \widehat{\Delta G}^{(\ell)}_{\alpha_{1}\alpha_{2}}|\widehat{\Delta G}^{(\ell)}_{\alpha_{4}\alpha_{5}}\rangle_{V^{(\ell)}}\langle \sigma^{(\ell)}_{\alpha_{3}}\sigma^{(\ell)}_{\alpha_{6}}\rangle_{K^{(\ell)}}
\nonumber\\
&\hspace{-30pt} + \langle \widehat{\Delta G}^{(\ell)}_{\alpha_{1}\alpha_{3}}|\widehat{\Delta G}^{(\ell)}_{\alpha_{5}\alpha_{6}}\rangle_{V^{(\ell)}}\langle \sigma^{(\ell)}_{\alpha_{2}}\sigma^{(\ell)}_{\alpha_{4}}\rangle_{K^{(\ell)}} + \langle \widehat{\Delta G}^{(\ell)}_{\alpha_{2}\alpha_{4}}|\widehat{\Delta G}^{(\ell)}_{\alpha_{5}\alpha_{6}}\rangle_{V^{(\ell)}}\langle \sigma^{(\ell)}_{\alpha_{1}}\sigma^{(\ell)}_{\alpha_{3}}\rangle_{K^{(\ell)}} \nonumber\\
&\hspace{-30pt} + \langle \widehat{\Delta G}^{(\ell)}_{\alpha_{1}\alpha_{4}}|\widehat{\Delta G}^{(\ell)}_{\alpha_{5}\alpha_{6}}\rangle_{V^{(\ell)}}\langle \sigma^{(\ell)}_{\alpha_{2}}\sigma^{(\ell)}_{\alpha_{3}}\rangle_{K^{(\ell)}} + \langle \widehat{\Delta G}^{(\ell)}_{\alpha_{2}\alpha_{3}}|\widehat{\Delta G}^{(\ell)}_{\alpha_{5}\alpha_{6}}\rangle_{V^{(\ell)}}\langle \sigma^{(\ell)}_{\alpha_{1}}\sigma^{(\ell)}_{\alpha_{4}}\rangle_{K^{(\ell)}}\nonumber\\
&\hspace{-30pt} + \langle \widehat{\Delta G}^{(\ell)}_{\alpha_{1}\alpha_{5}}|\widehat{\Delta G}^{(\ell)}_{\alpha_{3}\alpha_{4}}\rangle_{V^{(\ell)}}\langle \sigma^{(\ell)}_{\alpha_{2}}\sigma^{(\ell)}_{\alpha_{6}}\rangle_{K^{(\ell)}} + \langle \widehat{\Delta G}^{(\ell)}_{\alpha_{2}\alpha_{6}}|\widehat{\Delta G}^{(\ell)}_{\alpha_{3}\alpha_{4}}\rangle_{V^{(\ell)}}\langle \sigma^{(\ell)}_{\alpha_{1}}\sigma^{(\ell)}_{\alpha_{5}}\rangle_{K^{(\ell)}}
\nonumber\\
&\hspace{-30pt} + \langle \widehat{\Delta G}^{(\ell)}_{\alpha_{1}\alpha_{6}}|\widehat{\Delta G}^{(\ell)}_{\alpha_{3}\alpha_{4}}\rangle_{V^{(\ell)}}\langle \sigma^{(\ell)}_{\alpha_{2}}\sigma^{(\ell)}_{\alpha_{5}}\rangle_{K^{(\ell)}} + \langle \widehat{\Delta G}^{(\ell)}_{\alpha_{2}\alpha_{5}}|\widehat{\Delta G}^{(\ell)}_{\alpha_{3}\alpha_{4}}\rangle_{V^{(\ell)}}\langle \sigma^{(\ell)}_{\alpha_{1}}\sigma^{(\ell)}_{\alpha_{6}}\rangle_{K^{(\ell)}}
\bigg]\nonumber\\
&\hspace{-30pt} + \frac{2(C_{W}^{(\ell+1)})^{3}}{n_{\ell}^{2}}\bigg[\langle\sigma^{(\ell)}_{\alpha_{1}}\sigma^{(\ell)}_{\alpha_{3}}\rangle_{K^{(\ell)}}\langle\sigma^{(\ell)}_{\alpha_{2}}\sigma^{(\ell)}_{\alpha_{5}}\rangle_{K^{(\ell)}}\langle\sigma^{(\ell)}_{\alpha_{4}}\sigma^{(\ell)}_{\alpha_{6}}\rangle_{K^{(\ell)}}  + \langle\sigma^{(\ell)}_{\alpha_{1}}\sigma^{(\ell)}_{\alpha_{3}}\rangle_{K^{(\ell)}}\langle\sigma^{(\ell)}_{\alpha_{2}}\sigma^{(\ell)}_{\alpha_{6}}\rangle_{K^{(\ell)}}\langle\sigma^{(\ell)}_{\alpha_{4}}\sigma^{(\ell)}_{\alpha_{5}}\rangle_{K^{(\ell)}}\nonumber\\
&\hspace{-30pt} + \langle\sigma^{(\ell)}_{\alpha_{1}}\sigma^{(\ell)}_{\alpha_{4}}\rangle_{K^{(\ell)}}\langle\sigma^{(\ell)}_{\alpha_{2}}\sigma^{(\ell)}_{\alpha_{5}}\rangle_{K^{(\ell)}}\langle\sigma^{(\ell)}_{\alpha_{3}}\sigma^{(\ell)}_{\alpha_{6}}\rangle_{K^{(\ell)}}  + \langle\sigma^{(\ell)}_{\alpha_{1}}\sigma^{(\ell)}_{\alpha_{4}}\rangle_{K^{(\ell)}}\langle\sigma^{(\ell)}_{\alpha_{2}}\sigma^{(\ell)}_{\alpha_{6}}\rangle_{K^{(\ell)}}\langle\sigma^{(\ell)}_{\alpha_{3}}\sigma^{(\ell)}_{\alpha_{5}}\rangle_{K^{(\ell)}}\nonumber\\
&\hspace{-30pt}+ \langle\sigma^{(\ell)}_{\alpha_{1}}\sigma^{(\ell)}_{\alpha_{5}}\rangle_{K^{(\ell)}}\langle\sigma^{(\ell)}_{\alpha_{2}}\sigma^{(\ell)}_{\alpha_{3}}\rangle_{K^{(\ell)}}\langle\sigma^{(\ell)}_{\alpha_{4}}\sigma^{(\ell)}_{\alpha_{6}}\rangle_{K^{(\ell)}} + \langle\sigma^{(\ell)}_{\alpha_{1}}\sigma^{(\ell)}_{\alpha_{5}}\rangle_{K^{(\ell)}}\langle\sigma^{(\ell)}_{\alpha_{2}}\sigma^{(\ell)}_{\alpha_{4}}\rangle_{K^{(\ell)}}\langle\sigma^{(\ell)}_{\alpha_{3}}\sigma^{(\ell)}_{\alpha_{6}}\rangle_{K^{(\ell)}}\nonumber\\
&\hspace{-30pt} + \langle\sigma^{(\ell)}_{\alpha_{1}}\sigma^{(\ell)}_{\alpha_{6}}\rangle_{K^{(\ell)}}\langle\sigma^{(\ell)}_{\alpha_{2}}\sigma^{(\ell)}_{\alpha_{3}}\rangle_{K^{(\ell)}}\langle\sigma^{(\ell)}_{\alpha_{4}}\sigma^{(\ell)}_{\alpha_{5}}\rangle_{K^{(\ell)}} + \langle\sigma^{(\ell)}_{\alpha_{1}}\sigma^{(\ell)}_{\alpha_{6}}\rangle_{K^{(\ell)}}\langle\sigma^{(\ell)}_{\alpha_{2}}\sigma^{(\ell)}_{\alpha_{4}}\rangle_{K^{(\ell)}}\langle\sigma^{(\ell)}_{\alpha_{3}}\sigma^{(\ell)}_{\alpha_{5}}\rangle_{K^{(\ell)}}\bigg]\label{simplifiedv6alg}
\end{align}

\subsection{Diagrammatric derivation}\label{v6-diagrammatic-derivation}
We now derive the recursion relation for the sextic vertex~\eqref{simplifiedv6alg} using the Feynman rules of Appendix~\ref{app:general_feynman_rules}. Taking the reference pairing to be $(12)(34)(56)$, the application of rule (3) generates the 1-class diagram
\begin{align}
\, ,
\end{align}
as well as the corresponding 2-class diagrams
\begin{align}
%
 \, .
\end{align}
The application of rule (4) introduces additional diagrams, each accompanied by a Weingarten function determined by the cycle structure resulting from the composition of the corresponding pairing with the reference pairing. As discussed in the previous subsection, for 1-class diagrams there are three possible cycle types, namely $(1,1,1)$, $(2,1)$, and $(3)$. The fifteen 3-class diagrams can thus be organized as follows: one diagram of type $(1,1,1)$
\begin{align}
\mathcal{W}[1,1,1]%
\, ,
\end{align}
six diagrams of type $(2,1)$
\begin{align}
\mathcal{W}[2,1]%
 \quad ,
\end{align}
and eight diagrams of type $(3)$
\begin{align}
\mathcal{W}[3]%
 \quad .
\end{align}
Similarly, the nine 2-class diagrams generated by rule (4) can be classified according to the associated $k=2$ Weingarten factor. Explicitly, three diagrams are weighted by $\mathcal{W}[1,1]$
\begin{align}
\mathcal{W}[1,1]\cdot \mathcal{W}[1]%
\, ,
\end{align}
and six diagrams are weighted by $\mathcal{W}[2]$
\begin{align}
\mathcal{W}[2]\cdot \mathcal{W}[1]%
\quad .
\end{align}
The single 3-class diagram contributes as
\begin{align}
(\mathcal{W}[1])^{3}
%
\end{align}
M\"obius factors are incorporated through rule (5): 1-class diagrams are multiplied by $1$, 2-class diagrams by $-1$, and 3-class diagrams by $2$.

After summing over all neural indices and all diagrams described above, one recovers~\eqref{v6algebraic}, with $V^{\mathcal{G}}$ given by~\eqref{vg6} and $\Delta V$ defined in~\eqref{deltav6}.

We now apply the Feynman rules (6)–(9) to the square propagator. In particular, we focus on the layer evolution of the sextic vertex at order \(1/n^2\). To this end, the Weingarten functions are expanded as in~\eqref{w11expanded}, \eqref{w2expanded}, and~\eqref{k3wfunctions}. The bare propagators, together with the corresponding quartic and sextic vertices, are then introduced in accordance with the selection rules (a)–(f). After a careful, albeit straightforward, implementation of these rules, we arrive at the following diagrammatic decomposition of the tensor $V_{6}^{(\ell+1)}$:
\begingroup
\allowdisplaybreaks
\begin{align}
& \frac{1}{n_{\ell}^{2}}V^{(\ell + 1)}_{123456} =   \frac{1}{n_{\ell}^{3}}\sum_j

\right)\label{v6indiagrams}
\end{align}
\endgroup
Remarkably, these diagrams exactly reproduce the algebraic expression derived in the previous subsection at order \(1/n^{2}\), see~\eqref{v6ordern2}.

Moreover, \eqref{v6indiagrams} provides a direct validation of the leading-order (1/n) Feynman rules presented in Appendix~\ref{app:simplified_feynman_rules}. Specifically, the first three lines of~\eqref{v6indiagrams} follow from the Gaussian-like Feynman rules (1)-(4). The remaining contributions, namely the two groups of terms in parentheses multiplied by \(-1/n^{4}\) and \(2/n^{5}\), arise from rule (5): the former corresponds to $\lambda=(2,1)$ with $\beta_{(2,1)}=-1$, while the latter corresponds to $\lambda=(3)$ with $\beta_{(3)}=2$. In both cases, the cubic vertices defined in~\eqref{orthogonalityvertexsimp} appear as disconnected components, as prescribed by rule (5).

\subsection{Single-input case}
In the single-input setting, the recursion relation for $V^{(\ell)}_{6}$ takes a simple and compact form. Explicitly, $V^{\mathcal{G}}_{6}$ is readily obtained as
\begin{align}
 \frac{1}{n^{2}_{\ell}}V_{6}^{\mathcal{G}(\ell+1)} &= \frac{(C^{(\ell+1)}_{W})^{3}}{n^{2}_{\ell}}\left(\langle\sigma^{6}\rangle_{K} - 3\langle\sigma^{4}\rangle_{K}\langle\sigma^{2}\rangle_{K} + 2(\langle \sigma^{2}\rangle_{K})^{3} \right)\nonumber\\
 & + 6\frac{(C_{W}^{(\ell+1)})^{2}}{n_{\ell}n_{\ell-1}}V_{4}^{(\ell)}\chi^{(\ell)}_{||}\bigg[3\langle \sigma^{2}(\sigma')^{2}\rangle_{K} + \langle \sigma^{3}\sigma''\rangle_{K}- \langle(\sigma')^{2}\rangle_{K}\langle \sigma^{2}\rangle_{K} - \langle\sigma^{2}\rangle_{K}\langle \sigma\sigma''\rangle_{K}\bigg] \nonumber\\
 & + \frac{3C_{W}^{(\ell+1)}}{2n^{2}_{\ell-1}}(\chi_{||}^{(\ell)})^{2}(V_{4}^{(\ell)})^{2}\bigg[3\langle (\sigma'')^{2}\rangle_{K}+ 4\langle \sigma'\sigma'''\rangle_{K} + \langle\sigma \sigma^{''''}\rangle_{K}\bigg] + \frac{1}{n_{\ell-1}^{2}}(\chi^{(\ell)}_{||})^{3}V_{6}^{(\ell)}
\end{align}
whereas $\Delta V_{6}$ is given by
\begin{align}
\frac{1}{n_{\ell}^{2}}\Delta V_{6}^{(\ell+1)} &= \frac{(C_{W}^{(\ell+1)})^{3}}{n_{\ell}^{2}}\bigg[-12\langle\sigma^{4}\rangle_{K}\langle\sigma^{2}\rangle_{K} + 28(\langle \sigma^{2}\rangle_{K})^{3} - 12\left(\frac{n_{\ell}}{n_{\ell-1}}\right)\langle \sigma^{2}|\sigma^{2}\rangle_{V^{(\ell)}}\langle \sigma^{2}\rangle_{K}\bigg]\nonumber\\
& = \frac{(C_{W}^{(\ell+1)})^{3}}{n_{\ell}^{2}}\bigg[-12\langle\sigma^{4}\rangle_{K}\langle \sigma^{2}\rangle_{K} + 28(\langle \sigma^{2}\rangle_{K})^{3} - \frac{12}{(C_{W}^{(\ell+1)})^{2}}\left(\frac{n_{\ell}}{n_{\ell-1}}\right)V^{(\ell)}(\chi^{(\ell)}_{||})^{2}\langle \sigma^{2}\rangle_{K}\bigg]
\end{align}
Therefore, the orthogonal six-point cumulant reads
\begin{align}\label{v6-recursion-relation-si}
 \frac{1}{n^{2}_{\ell}}V_{6}^{(\ell+1)} &= \frac{(C^{(\ell+1)}_{W})^{3}}{n^{2}_{\ell}}\left(\langle\sigma^{6}\rangle_{K} - 15\langle\sigma^{4}\rangle_{K}\langle\sigma^{2}\rangle_{K} + 30(\langle \sigma^{2}\rangle_{K})^{3} \right)\nonumber\\
 & + 6\frac{(C_{W}^{(\ell+1)})^{2}}{n_{\ell}n_{\ell-1}}V_{4}^{(\ell)}\chi^{(\ell)}_{||}\bigg[3\langle \sigma^{2}(\sigma')^{2}\rangle_{K} + \langle \sigma^{3}\sigma''\rangle_{K}- 3\langle(\sigma')^{2}\rangle_{K}\langle \sigma^{2}\rangle_{K} - 3\langle\sigma^{2}\rangle_{K}\langle \sigma\sigma''\rangle_{K}\bigg]\nonumber\\
 & + \frac{3C_{W}^{(\ell+1)}}{2n^{2}_{\ell-1}}(\chi_{||}^{(\ell)})^{2}(V_{4}^{(\ell)})^{2}\bigg[3\langle (\sigma'')^{2}\rangle_{K}+ 4\langle \sigma'\sigma'''\rangle_{K} + \langle\sigma \sigma^{''''}\rangle_{K}\bigg] + \frac{1}{n_{\ell-1}^{2}}(\chi^{(\ell)}_{||})^{3}V_{6}^{(\ell)}
\end{align}
We use \textit{Mathematica} to solve~\eqref{v6-recursion-relation-si} for a rectangular $\tanh$ MLP of width $n=50$, depth $L=30$ and input vector as in~\eqref{n50-inputvector}, both exactly and asymptotically; see Appendix~\ref{app:large_l_expansion} for details. The results for the normalized tensor $\tilde{V}_{6}^{(\ell)} = V_{6}^{(\ell)}/(K^{(\ell)})^{3}$ are shown in Figure~\ref{fig:V6_tensor_plot}.
\begin{figure*}[tb]
  \centering
  \includegraphics[width=0.5\textwidth,height=0.5\textheight,keepaspectratio]{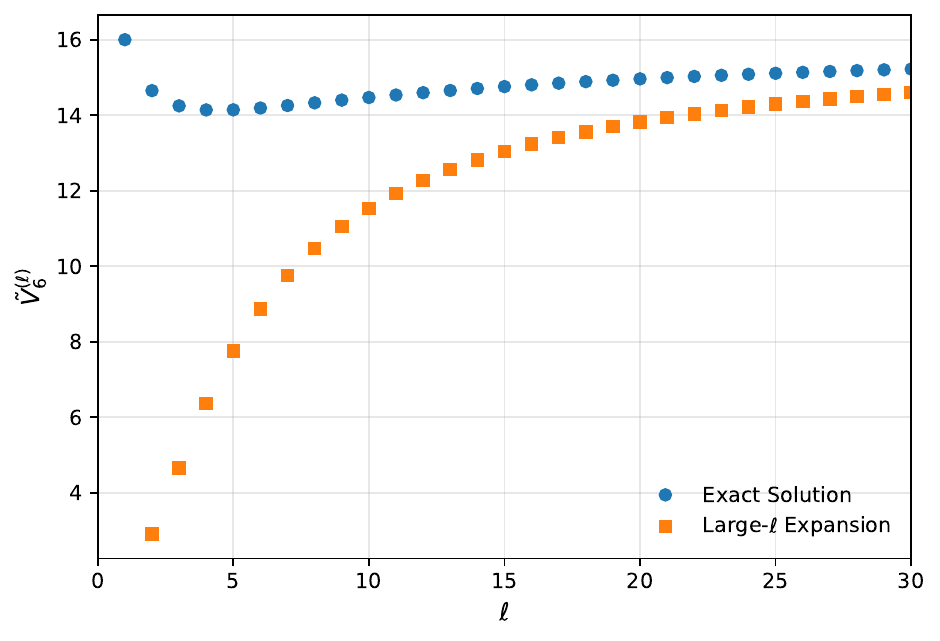}
  \vspace{-0.1cm}
  \caption{\emph{The sextic vertex.} Solution to the single-input recursion relation~\eqref{v6-recursion-relation-si} at order $1/n^{2}$. We consider a $\tanh$ network with width $n=50$ and depth $L=30$, with inputs drawn from $(0,1)$. Blue points show the exact solution of~\eqref{v6-recursion-relation-si}, while orange boxes represent the large-$\ell$ expansion~\eqref{V6largel}, evaluated at integer $\ell$. The tensor magnitude remains below its Gaussian counterpart and exhibits early-layer saturation, consistent with the behavior of lower-rank orthogonal tensors.}
  \label{fig:V6_tensor_plot}
\end{figure*}

\section{Large-$\ell$ expansion}
\label{app:large_l_expansion}
In this appendix, we solve the single-input recursion relations~\eqref{dsingleinput}-\eqref{usingleinput} and~\eqref{v6-recursion-relation-si} for $\tanh$ networks in the large-$\ell$ regime. We adopt a standard renormalization-inspired ansatz and expand a generic tensor $\mathcal{O}^{(\ell)}$ as
\begin{align}
\mathcal{O}^{(\ell)} = \ell^{-p_{\mathcal{O}}}\sum_{i,j=0}^{k}c^{\mathcal{O}}_{i,j}\frac{\log^{j}(\ell)}{\ell^{i}}\label{largelexpansionapp}
\end{align}
where $p_{\mathcal{O}}$ is the critical exponent and $k$ controls the truncation order. Throughout, we set $k=5$.

The coefficients $c^{\mathcal{O}}_{i,j}$ in~\eqref{largelexpansionapp} depend on both the input data and the network parameters through the initial conditions
\begin{align}
K^{(1)} &= \frac{C_{W}}{n}\sum_{i=1}^{n} x_{i}\cdot x_{i}, 
\quad 
\Theta^{(1)} = \lambda^{(1)}_{b} + \frac{\lambda_{W}^{(1)}}{n}\sum_{i=1}^{n} x_{i}\cdot x_{i}, 
\quad 
V^{(1)} = -2\big(K^{(1)}\big)^2, \nonumber\\
D^{(1)} &= 0, \quad
F^{(1)} = 0, \quad
A^{(1)} = 0, \quad
B^{(1)} = 0, \quad
P^{(1)} = 0, \nonumber\\
Q^{(1)} &= 0, \quad
R^{(1)} = 0, \quad
S^{(1)} = 0, \quad
T^{(1)} = 0, \quad
U^{(1)} = 0.\label{initialconditions}
\end{align}
In addition, they depend parametrically on the choice of activation function. We compute these coefficients using \textit{Mathematica}, combining symbolic manipulation with numerical evaluation. We report below the critical exponents $p_{\mathcal{O}}$ and the corresponding coefficients for all tensors considered in this work.

We consider a square neural network with activation function $\tanh$, and fix $C_{W}=1$, thereby focusing on the critical regime~\cite{roberts2021a}. As discussed in~\cite{roberts2021a}, for the $\tanh$ nonlinearity the training hyperparameters can be chosen as
\begin{align}
\lambda_{b}^{(\ell)} = \frac{1}{\ell}   \quad,\quad  \lambda_{W}^{(\ell)} = 1   
\end{align}
Since the recursion relations~\eqref{dsingleinput}-\eqref{usingleinput} connect quantities at consecutive layers, i.e. at depths $\ell$ and $\ell+1$, the following large-$\ell$ asymptotic expansions are useful:
\begin{align}
\frac{1}{(1+\ell)^{a}} &= \frac{1}{\ell^{a}}\sum_{i=0}^{\infty}\binom{-a}{i}\left(\frac{1}{\ell}\right)^{i}\label{largelidentitiesone}\\
\log(1+\ell) &= \log(\ell) - \sum_{i=0}^{\infty}\frac{1}{i}\left(-\frac{1}{\ell}\right)^{i}\label{largelidentitiestwo}
\end{align}

The Gaussian expectations appearing in~\eqref{dsingleinput}-\eqref{usingleinput} are evaluated by expanding the activation function as
\begin{align}\label{taylorexpansionsigmam}
\sigma(z) = \sum_{m=0}^{\infty} \sigma_{m}\frac{z^{m}}{m!}
\end{align}
where the coefficients $\sigma_{m} = \sigma^{(m)}(0)$ are determined by derivatives of $\sigma$ at the origin. For $\sigma=\tanh$, these coefficients are known in closed form.

For the trucantion order $k=5$ in~\eqref{largelexpansionapp}, retaining a finite number of terms in the expansion~\eqref{taylorexpansionsigmam} is sufficient to obtain a consistent large-$\ell$ expansion. In practice, we include terms up to $m=20$, which ensures stability of all coefficients at the desired order.

Concretely, the non-vanishing coefficients are
\begin{align}\label{tanhcoeff}
\sigma_{1} = 1 \quad , \quad  \sigma_{3} = -2 \quad , \quad \sigma_{5} = 16 \quad &, \quad \sigma_{7} = -272 \quad , \quad \sigma_{9} = 7936 \quad , \quad \sigma_{11} = -353792 \nonumber\\
 \sigma_{13} = 22368256 \quad &, \quad \sigma_{15} =-1903757312 \nonumber\\
\sigma_{17} = 209865342976 \quad &, \quad \sigma_{19} = -29088885112832
\end{align}
with $\sigma_{m}=0$ for $m$ even.

\subsection{The NNGP}
To illustrate the procedure, we derive the expansion of the NNGP kernel $K^{(\ell)}$. Substituting the ansatz~\eqref{largelexpansionapp} into the recursion~\eqref{nngprecursion} yields
\begin{align}
K^{(\ell)} &=\frac{1}{2\ell}
+ \frac{1}{\ell^2} \left( \frac{5 \log \ell}{24} - \frac{5 \log \ell_0}{24} \right) \nonumber\\
&+ \frac{1}{\ell^3} \left(
\frac{25 (\log \ell)^2}{288}
+ \frac{5 \log \ell \left(-5 - 10 \log \ell_0 \right)}{288}
+ \frac{53 + 25 \log \ell_0 + 25 (\log \ell_0)^2}{288}
\right) \nonumber\\
&+ \frac{1}{\ell^4} \left(
\frac{125 (\log \ell)^3}{3456}
+ \frac{25 (\log \ell)^2 \left(-25 - 30 \log \ell_0 \right)}{6912}
+ \frac{5 \log \ell \left(23 + \frac{125}{8} \log \ell_0 + \frac{75}{8} (\log \ell_0)^2 \right)}{432}
\right. \nonumber\\
&\qquad \left.
+ \frac{-8597 - 9200 \log \ell_0 - 3125 (\log \ell_0)^2 - 1250 (\log \ell_0)^3}{34560}
\right) \nonumber\\
&+ \frac{1}{\ell^5} \left(
\frac{625 (\log \ell)^4}{41472}
+ \frac{125 (\log \ell)^3 \left(-65 - 60 \log \ell_0 \right)}{124416}
+ \frac{25 (\log \ell)^2 \left(287 + \frac{650}{3} \log \ell_0 + 100 (\log \ell_0)^2 \right)}{27648}
\right. \nonumber\\
&\qquad
+ \frac{\log \ell \left(-10897 - \frac{21525}{2} \log \ell_0 - \frac{8125}{2} (\log \ell_0)^2 - 1250 (\log \ell_0)^3 \right)}{20736} \nonumber\\
&\qquad \left.
+ \frac{-2479663 + 653820 \log \ell_0 + 322875 (\log \ell_0)^2 + 81250 (\log \ell_0)^3 + 18750 (\log \ell_0)^4}{1244160}\label{nngplexpansion}
\right).
\end{align}
where we defined $c_{1,0}^{K} = -\frac{5}{24}\log(\ell_{0})$. Consistency of the expansion fixes the critical exponent to $p_{K}=1$, ensuring that the recursion~\eqref{nngprecursion} admits a solution within the ansatz. The first two lines of~\eqref{nngplexpansion} reproduce the known large-$\ell$ expansion of the midpoint kernel up to order $\ell^{-3}$~\cite{roberts2021a}. The scale $\ell_{0}$ is fixed by the initial condition~\eqref{initialconditions} after replacing $\ell = 1$ in \eqref{nngplexpansion}.

\subsection{The quartic vertex}
Having obtained the large-$\ell$ expansion of $K^{(\ell)}$, we now compute the corresponding expansion of the quartic vertex $V_{4}^{(\ell)}$ from~\eqref{fourpointrecursion}. We again adopt the ansatz~\eqref{largelexpansionapp} together with the initial conditions~\eqref{initialconditions}. Using the large-$\ell$ identities~\eqref{largelidentitiesone}, \eqref{largelidentitiestwo}, the $\tanh$ coefficients~\eqref{tanhcoeff}, and fixing the critical exponent to $p_{V}=2$, we obtain
\begin{align}\label{Vlargelgeneral}
V_{4}^{(\ell)} &= 
-\frac{1}{2\ell^2} + \frac{1}{\ell^3} \left(
- \frac{5}{12} \log \ell
+ \frac{1}{12} \left( 8 + 5 \log \ell_0 \right)
\right) \nonumber\\
&+ \frac{1}{\ell^4} \left(
- \frac{25}{96} (\log \ell)^2
+ \frac{1}{144} \log \ell \left( -167 + 75 \log \ell_0 \right)
+ c^{V}_{2,0}
\right).
\end{align}
Although the ansatz formally includes coefficients $c_{i,j}^{V}$ for
$i,j=0,\ldots ,5$, the recursion~\eqref{fourpointrecursion} only constrains a subset of these coefficients within our truncation scheme. In particular, the Gaussian expectation values in the first two terms on the left-hand side of~\eqref{fourpointrecursion}, which depend on $K^{(\ell)}$, contribute at most at order $\ell^{-5}$. Since $p_{V}=2$, this implies that only coefficients with $i=0,1,2$ can be determined. The coefficient $c^{V}_{2,0}$ is fixed by the initial conditions~\eqref{initialconditions} after replacing $\ell = 1$ in~\eqref{Vlargelgeneral}.

\subsection{The NTK}
We next determine the large-$\ell$ behavior of the frozen NTK $\Theta^{(\ell)}$ from~\eqref{ntkrecursion}. Substituting the expansion of $K^{(\ell)}$ from~\eqref{nngplexpansion} and proceeding as in the previous cases, we find that the critical exponent is $p_{\Theta}=0$, yielding
\begin{align}\label{Thetalargelgeneral}
\Theta^{(\ell)} &= \frac{3}{2}
+ \frac{1}{\ell}\Bigl[
- \frac{5}{24} (\log \ell)^2
+ \frac{1}{24} \log \ell (27 + 10 \log \ell_0)
+ c^{\Theta}_{1,0}
\Bigr] \nonumber\\
&+ \frac{1}{\ell^2}\Bigl[
- \frac{25}{288} (\log \ell)^3
+ \frac{5}{288} (\log \ell)^2 (31 + 15 \log \ell_0)
+ \frac{447 + 205 \log \ell_0}{288}
\nonumber\\
&\qquad
+ \log \ell \Bigl(
\frac{-313 - 175 \log \ell_0 - 50 (\log \ell_0)^2}{288}
\Bigr)
+ \log \ell \Bigl(\frac{5}{12} c^{\Theta}_{1,0}\Bigr)
\nonumber\\
&\qquad
- \frac{1}{12} (4 + 5 \log \ell_0) c^{\Theta}_{1,0}
\Bigr] \nonumber\\
&+ \frac{1}{\ell^3}\Bigl[
- \frac{125}{3456} (\log \ell)^4
+ \frac{125}{864} (\log \ell)^3 (2 + \log \ell_0)
\nonumber\\
&\qquad
+ \frac{-183829 - 116850 \log \ell_0 - 26500 (\log \ell_0)^2}{69120}
\nonumber\\
&\qquad
+ (\log \ell)^2 \Bigl(
- \frac{25 (143 + 93 \log \ell_0 + 25 (\log \ell_0)^2)}{3456}
\Bigr)
+ (\log \ell)^2 \Bigl(\frac{25}{144} c^{\Theta}_{1,0}\Bigr)
\nonumber\\
&\qquad
+ \frac{99 + 65 \log \ell_0 + 25 (\log \ell_0)^2}{144}\, c^{\Theta}_{1,0}
\nonumber\\
&\qquad
+ \log \ell \Bigl(
\frac{17031 + 10790 \log \ell_0 + 2650 (\log \ell_0)^2 + 500 (\log \ell_0)^3}{6912}
\Bigr)
\nonumber\\
&\qquad
+ \log \ell \Bigl(
- \frac{5}{144} (13 + 10 \log \ell_0)\, c^{\Theta}_{1,0}
\Bigr)
\Bigr] \nonumber\\
&+ \frac{1}{\ell^4}\Bigl[
- \frac{625}{41472} (\log \ell)^5
+ \frac{125 (\log \ell)^4 (103 + 50 \log \ell_0)}{82944}
\nonumber\\
&\qquad
+ \frac{287109682 + 282873525 \log \ell_0 + 85689000 (\log \ell_0)^2}{67184640}
\nonumber\\
&\qquad
+ \frac{12240000 (\log \ell_0)^3}{67184640}
\nonumber\\
&\qquad
+ (\log \ell)^3 \Bigl(
- \frac{25 (23863 + 16110 \log \ell_0 + 4050 (\log \ell_0)^2)}{746496}
\Bigr)
\nonumber\\
&\qquad
+ (\log \ell)^3 \Bigl(\frac{125}{1728} c^{\Theta}_{1,0}\Bigr)
\nonumber\\
&\qquad
- \frac{21949 + 18100 \log \ell_0 + 6125 (\log \ell_0)^2}{17280}\, c^{\Theta}_{1,0}
\nonumber\\
&\qquad
- \frac{1250 (\log \ell_0)^3}{17280}\, c^{\Theta}_{1,0}
\nonumber\\
&\qquad
+ (\log \ell)^2 \Bigl(
\frac{2029301 + 1492050 \log \ell_0 + 457875 (\log \ell_0)^2}{746496}
\Bigr)
\nonumber\\
&\qquad
+ (\log \ell)^2 \Bigl(
\frac{78750 (\log \ell_0)^3}{746496}
- \frac{25 (49 + 30 \log \ell_0)}{3456} c^{\Theta}_{1,0}
\Bigr)
\nonumber\\
&\qquad
+ \log \ell \Bigl(
\frac{-126292817 - 95368260 \log \ell_0 - 30944250 (\log \ell_0)^2}{22394880}
\Bigr)
\nonumber\\
&\qquad
+ \log \ell \Bigl(
\frac{-5130000 (\log \ell_0)^3 - 675000 (\log \ell_0)^4}{22394880}
\Bigr)
\nonumber\\
&\qquad
+ \log \ell \Bigl(
\frac{5 (362 + 245 \log \ell_0 + 75 (\log \ell_0)^2)}{1728} c^{\Theta}_{1,0}
\Bigr)
\Bigr].
\end{align}
As in the previous cases, the truncation order $k$ together with the scaling $p_{\Theta}=0$ restricts the set of coefficients $c_{i,j}^{\Theta}$ that can be determined within this approximation. The coefficient $c^{\Theta}_{1,0}$ is determined by the initial conditions~\eqref{initialconditions} after replacing $\ell=1$ in~\eqref{Thetalargelgeneral}.

\subsection{The NTK-preactivation mixed cumulant}
We next determine the preactivation-NTK mixed cumulants $D^{(\ell)}$ and $F^{(\ell)}$. Applying the same procedure as above yields the critical exponents $p_{D}=2$ and $p_{F}=1$, together with the large-$\ell$ expansions
\begin{align}
D^{(\ell)} &= -\frac{4}{3\ell^2} + \frac{1}{\ell^3} \Bigg[
\frac{5}{54} (\log \ell)^3
+ \frac{1}{36} (\log \ell)^2 \left( 51 - 10 \log \ell_0 \right)
+ c^{D}_{1,0} \nonumber\\
&\qquad\qquad
+ \frac{1}{144} \log \ell \left(
-274 - 290 \log \ell_0 - 75 (\log \ell_0)^2
- 192 c^{\Theta}_{1,0} - 288 c^{V}_{2,0}
\right)
\Bigg],\label{Dlargelgeneral}\\
F^{(\ell)} &= -\frac{1}{2\ell} + \frac{1}{\ell^2} \Bigg[
\frac{5}{48} (\log \ell)^2
+ \frac{1}{24} \log \ell \left( -21 - 5 \log \ell_0 \right)
+ \frac{1}{16} \left( 27 + 5 \log \ell_0 - 8 c^{\Theta}_{1,0} \right)
\Bigg] \nonumber\\
&+ \frac{1}{\ell^3} \Bigg[
\frac{25}{288} (\log \ell)^3
- \frac{25}{288} (\log \ell)^2 \left( 10 + 3 \log \ell_0 \right) \nonumber\\
&\qquad\qquad
+ \frac{1}{288} \log \ell \left(
1061 + 365 \log \ell_0 + 50 (\log \ell_0)^2 - 120 c^{\Theta}_{1,0}
\right) \nonumber\\
&\qquad\qquad
+ \frac{1}{288} \left(
-3321 - 575 \log \ell_0 - 25 (\log \ell_0)^2
+ 432 c^{\Theta}_{1,0} + 120 \log \ell_0\, c^{\Theta}_{1,0}
\right)
\Bigg] \nonumber\\
&+ \frac{1}{\ell^4} \Bigg[
\frac{625}{13824} (\log \ell)^4
- \frac{5}{3456} (\log \ell)^3 \left( 284 + 125 \log \ell_0 \right)
+ c^{F}_{3,0} \nonumber\\
&\qquad\qquad
+ \frac{(\log \ell)^2}{13824} \left(
23833 + 12990 \log \ell_0 + 3000 (\log \ell_0)^2 - 3600 c^{\Theta}_{1,0}
\right) \nonumber\\
&\qquad\qquad
+ \frac{\log \ell}{69120} \left(
1107057 - 450550 \log \ell_0 - 84000 (\log \ell_0)^2
- 5000 (\log \ell_0)^3 \right. \nonumber\\
&\qquad\qquad\qquad\qquad\left.
- 188640 c^{\Theta}_{1,0} + 36000 \log \ell_0\, c^{\Theta}_{1,0} \label{Flargelgeneral}
\right)
\Bigg].
\end{align}
The coefficients $c^{D}_{1,0}$ and $c^{F}_{3,0}$ are fixed by the initial conditions~\eqref{initialconditions} together with the previously determined parameters $\log \ell_{0}$, $c^{V}_{2,0}$, and $c^{\Theta}_{1,0}$.

\subsection{The NTK-variance}
We now turn to the NTK variance tensors $A^{(\ell)}$ and $B^{(\ell)}$. Using the expansion of $D^{(\ell)}$ in~\eqref{Dlargelgeneral}, we find the critical exponents $p_{A}=1$ and $p_{B}=1$ together with
\begin{align}
A^{(\ell)} &=  \frac{16}{3\ell} + \frac{1}{\ell^2} \Bigg[
- \frac{5}{54} (\log \ell)^4
+ \frac{1}{27} (\log \ell)^3 \left( -61 + 10 \log \ell_0 \right)
+ c^{A}_{1,0} \nonumber\\
&\qquad\qquad
+ \frac{4}{9} \log \ell \left(
-15 + 10 \log \ell_0 - 9 dd_{1,0} + 12 c^{\Theta}_{1,0}
\right) \nonumber\\
&\qquad\qquad
+ \frac{1}{72} (\log \ell)^2 \left(
330 + 370 \log \ell_0 + 75 (\log \ell_0)^2
+ 192 c^{\Theta}_{1,0} + 288 c^{V}_{2,0}
\right)
\Bigg],\label{Alargelgeneral}\\
B^{(\ell)} &= \frac{9}{2\ell} + \frac{1}{\ell^2} \Bigg[
- \frac{5}{12} (\log \ell)^3
+ \frac{1}{8} (\log \ell)^2 \left( 27 + 10 \log \ell_0 \right)
+ c^{B}_{1,0} \nonumber\\
&\qquad\qquad
+ \frac{3}{4} \log \ell \left( -21 + 8 c^{\Theta}_{1,0} \right)
\Bigg] \nonumber\\
&+ \frac{1}{\ell^3} \Bigg[
- \frac{125}{288} (\log \ell)^4
+ \frac{5}{144} (\log \ell)^3 \left( 121 + 50 \log \ell_0 \right) \nonumber\\
&\qquad
- \frac{1}{6} (4 + 5 \log \ell_0)\, c^{B}_{1,0} \nonumber\\
&\qquad
+ \frac{1}{72} (\log \ell)^2 \left(
-1683 - 435 \log \ell_0 - 100 (\log \ell_0)^2 + 420 c^{\Theta}_{1,0}
\right) \nonumber\\
&\qquad
+ \frac{1}{288} \left(
-18516 - 690 \log \ell_0 - 125 (\log \ell_0)^2
+ 5664 c^{\Theta}_{1,0} + 240 \log \ell_0\, c^{\Theta}_{1,0}
- 576 (c^{\Theta}_{1,0})^{2}
\right) \nonumber\\
&\qquad
+ \log \ell \Bigg(
\frac{5}{6} c^{B}_{1,0}
+ \frac{5043 + 3330 \log \ell_0 + 50 (\log \ell_0)^2
- 1344 c^{\Theta}_{1,0} - 960 \log \ell_0\, c^{\Theta}_{1,0}}{144}
\Bigg)
\Bigg] \nonumber\\
&+ \frac{1}{\ell^4} \Bigg[
- \frac{125}{384} (\log \ell)^5
+ \frac{125 (\log \ell)^4 (211 + 90 \log \ell_0)}{6912} \nonumber\\
&\qquad
+ \frac{214 + 170 \log \ell_0 + 75 (\log \ell_0)^2}{144}\, c^{B}_{1,0} \nonumber\\
&\qquad
- \frac{5}{3456} (\log \ell)^3 \left(
17317 + 7310 \log \ell_0 + 1650 (\log \ell_0)^2
- 2880 c^{\Theta}_{1,0}
\right) \nonumber\\
&\qquad
+ \frac{15676647 + 5933100 \log \ell_0 + 221500 (\log \ell_0)^2 + 22500 (\log \ell_0)^3}{69120} \nonumber\\
&\qquad
- \frac{4405920 + 1816800 \log \ell_0 + 60000 (\log \ell_0)^2}{69120}\, c^{\Theta}_{1,0} \nonumber\\
&\qquad
+ \frac{339840 + 172800 \log \ell_0}{69120}\, (c^{\Theta}_{1,0})^{2} \nonumber\\
&\qquad
+ (\log \ell)^2 \Bigg(
\frac{25}{48} c^{B}_{1,0}
+ \frac{174843 + 113770 \log \ell_0 + 16900 (\log \ell_0)^2 + 2500 (\log \ell_0)^3}{2304} \nonumber\\
&\qquad\qquad
- \frac{36000 + 21600 \log \ell_0}{2304}\, c^{\Theta}_{1,0}
\Bigg) \nonumber\\
&\qquad
+ \log \ell \Bigg(
- \frac{5}{72} (17 + 15 \log \ell_0)\, c^{B}_{1,0} \nonumber\\
&\qquad\qquad
+ \frac{-625380 - 280395 \log \ell_0 - 72950 (\log \ell_0)^2 - 1250 (\log \ell_0)^3}{3456} \nonumber\\
&\qquad\qquad
+ \frac{159888 + 64080 \log \ell_0 + 18000 (\log \ell_0)^2}{3456}\, c^{\Theta}_{1,0}
- \frac{8640}{3456}\, (c^{\Theta}_{1,0})^{2}
\Bigg)
\Bigg]. \label{Blargelgeneral}
\end{align}
The parameters $c^{A}_{1,0}$ and $c^{B}_{1,0}$ are again fixed by the initial conditions~\eqref{initialconditions}.

\subsection{The dNTK-preactivation mixed cumulant}
The asymptotic behavior of the preactivation-dNTK cumulant tensors $P^{(\ell)}$ and $Q^{(\ell)}$ follows from the same procedure. This yields the critical exponents $p_{P} =0$ and $p_{Q} = 0$, together with the large-$\ell$ expansions
\allowdisplaybreaks
\begin{align}
P^{(\ell)}={}&-\frac34
+\frac1\ell\Bigl[
\frac5{16}(\log\ell)^2
+\frac18\log\ell(-16-5\log\ell_0)
+\frac1{16}(31+5\log\ell_0-24c^{\Theta}_{1,0})
\Bigr]
\nonumber\\
&+\frac1{\ell^2}\Bigl[
-\frac{25}{576}(\log\ell)^4
+\frac5{576}(\log\ell)^3(137+20\log\ell_0)
\nonumber\\
&\qquad
+\frac{(\log\ell)^2}{576}
\bigl(-5748-1785\log\ell_0-100(\log\ell_0)^2+240c^{\Theta}_{1,0}\bigr)
\nonumber\\
&\qquad
+\frac{\log\ell}{576}
\bigl(26199+5205\log\ell_0+350(\log\ell_0)^2\bigr)
\nonumber\\
&\qquad
+\frac{\log\ell}{576}
\bigl(-5592c^{\Theta}_{1,0}-480\log\ell_0\,c^{\Theta}_{1,0}\bigr)
\nonumber\\
&\qquad
+\frac1{576}
\bigl(-40257-6300\log\ell_0-125(\log\ell_0)^2-576c^{B}_{1,0}\bigr)
\nonumber\\
&\qquad
+\frac1{576}
\bigl(9624c^{\Theta}_{1,0}+840\log\ell_0\,c^{\Theta}_{1,0}-576(c^{\Theta}_{1,0})^{2}\bigr)
\Bigr]
\nonumber\\
&+\frac1{\ell^3}\Bigl[
-\frac{55}{3456}(\log\ell)^5
+\frac5{13824}(\log\ell)^4(2009+220\log\ell_0)
\nonumber\\
&\qquad
+\frac{(\log\ell)^3}{20736}
\bigl(-128583-62430\log\ell_0-3700(\log\ell_0)^2-1920c^{\Theta}_{1,0}\bigr)
\nonumber\\
&\qquad
+\frac{(\log\ell)^2}{13824}
\bigl(396027+182700\log\ell_0+50800(\log\ell_0)^2\bigr)
\nonumber\\
&\qquad
+\frac{(\log\ell)^2}{13824}
\bigl(3000(\log\ell_0)^3-66192c^{\Theta}_{1,0}+3840\log\ell_0\,c^{\Theta}_{1,0}\bigr)
\nonumber\\
&\qquad
+\frac{\log\ell}{23040}
\bigl(1355239-1599000\log\ell_0-260500(\log\ell_0)^2\bigr)
\nonumber\\
&\qquad
+\frac{\log\ell}{23040}
\bigl(-10000(\log\ell_0)^3-19200c^{B}_{1,0}-516720c^{\Theta}_{1,0}\bigr)
\nonumber\\
&\qquad
+\frac{\log\ell}{23040}
\bigl(269600\log\ell_0\,c^{\Theta}_{1,0}
+24000(\log\ell_0)^2c^{\Theta}_{1,0}
+72960(c^{\Theta}_{1,0})^{2}\bigr)
\nonumber\\
&\qquad
+c^{P}_{3,0}
\Bigr].\label{Plargelgeneral}\\
Q^{(\ell)}={}&-\frac{17}{12}
+\frac1\ell\Bigl[
\frac58(\log\ell)^2
+\frac1{24}\log\ell(-101-30\log\ell_0)
+\frac1{12}(93+10\log\ell_0-36c^{\Theta}_{1,0})
\Bigr]
\nonumber\\
&+\frac1{\ell^2}\Bigl[
-\frac{25}{288}(\log\ell)^4
+\frac5{576}(\log\ell)^3(183+40\log\ell_0)
\nonumber\\
&\qquad
+\frac{(\log\ell)^2}{576}
\bigl(-6013-2205\log\ell_0-200(\log\ell_0)^2+480c^{\Theta}_{1,0}\bigr)
\nonumber\\
&\qquad
+\frac{\log\ell}{576}
\bigl(20897+7085\log\ell_0+750(\log\ell_0)^2\bigr)
\nonumber\\
&\qquad
+\frac{\log\ell}{576}
\bigl(-4392c^{\Theta}_{1,0}-960\log\ell_0\,c^{\Theta}_{1,0}\bigr)
\nonumber\\
&\qquad
+\frac1{576}
\bigl(-53509-8045\log\ell_0-150(\log\ell_0)^2\bigr)
\nonumber\\
&\qquad
+\frac1{576}
\bigl(11424c^{\Theta}_{1,0}+1800\log\ell_0\,c^{\Theta}_{1,0}-1152(c^{\Theta}_{1,0})^{2}\bigr)
\Bigr]
\nonumber\\
&+\frac1{\ell^3}\Bigl[
-\frac{305}{3456}(\log\ell)^5
+\frac5{13824}(\log\ell)^4(4273+1220\log\ell_0)
\nonumber\\
&\qquad
+\frac{(\log\ell)^3}{20736}
\bigl(-209617-115500\log\ell_0-15950(\log\ell_0)^2+11280c^{\Theta}_{1,0}\bigr)
\nonumber\\
&\qquad
+\frac{(\log\ell)^2}{17280}
\bigl(1546071+268150\log\ell_0+81125(\log\ell_0)^2\bigr)
\nonumber\\
&\qquad
+\frac{(\log\ell)^2}{17280}
\bigl(9375(\log\ell_0)^3-227460c^{\Theta}_{1,0}-28200\log\ell_0\,c^{\Theta}_{1,0}\bigr)
\nonumber\\
&\qquad
+\frac{\log\ell}{6912}
\bigl(-697585-864585\log\ell_0-137400(\log\ell_0)^2\bigr)
\nonumber\\
&\qquad
+\frac{\log\ell}{6912}
\bigl(-6250(\log\ell_0)^3+27648c^{F}_{3,0}-62712c^{\Theta}_{1,0}\bigr)
\nonumber\\
&\qquad
+\frac{\log\ell}{6912}
\bigl(129840\log\ell_0\,c^{\Theta}_{1,0}
+22800(\log\ell_0)^2c^{\Theta}_{1,0}
+10368(c^{\Theta}_{1,0})^{2}\bigr)
\nonumber\\
&\qquad
+c^{Q}_{3,0}
\Bigr].\label{Qlargelgeneral}
\end{align}
The coefficients $c^{P}_{3,0}$ and $c^{Q}_{3,0}$ follow from the initial conditions~\eqref{initialconditions}.
\subsection{The dd$_{\text{I}}$NTK mean}
Applying this framework to the d$_{\text{I}}$NTK mean tensor $R^{(\ell)}$ fixes the critical exponent to $p_{R}=-1$, yielding
\begin{align}\label{Rlargelgeneral}
R^{(\ell)} ={}& -0.43821 - 1.75\,\ell
+ 0.9375 (\log \ell)^2
+ \log \ell\,(-6.67172 - 2.5 \log \ell_0)
\nonumber\\
&+ 0.520833 \log \ell_0
- 6 c^{\Theta}_{1,0}
\nonumber\\
&+ \frac{1}{\ell}\Bigl[
-30.8931
- 0.173611 (\log \ell)^4
- 39.3424 \log \ell_0
- 2.08333 (\log \ell_0)^2
\nonumber\\
&\qquad
+ (\log \ell)^3 (4.17447 + 1.30208 \log \ell_0)
- 3 c^{B}_{1,0}
+ 49.7536 c^{\Theta}_{1,0}
+ 7.5 \log \ell_0\, c^{\Theta}_{1,0}
- 9 (c^{\Theta}_{1,0})^{2}
\nonumber\\
&\qquad
+ (\log \ell)^2 \bigl(
-27.3051
- 16.1153 \log \ell_0
- 1.5625 (\log \ell_0)^2
+ 3.125 c^{\Theta}_{1,0}
\bigr)
\nonumber\\
&\qquad
+ \log \ell \bigl(
75.8169
+ 33.0744 \log \ell_0
+ 3.125 (\log \ell_0)^2
\bigr)
\nonumber\\
&\qquad
+ \log \ell \bigl(
- 43.9268 c^{\Theta}_{1,0}
- 7.5 \log \ell_0\, c^{\Theta}_{1,0}
\bigr)
\Bigr]
\nonumber\\
&+ \frac{1}{\ell^2}\Bigl[
(\log \ell)^5 \bigl(
0.256462
+ 0.336123 \log \ell_0
+ 0.0434028 (\log \ell_0)^2
- 0.0347222 c^{\Theta}_{1,0}
\bigr)
\nonumber\\
&\qquad
+ (\log \ell)^4 \bigl(
-0.898367
- 1.47402 \log \ell_0
- 0.580512 (\log \ell_0)^2
- 0.036169 (\log \ell_0)^3
\nonumber\\
&\qquad\qquad
+ 1.24491 c^{\Theta}_{1,0}
+ 0.260417 \log \ell_0\, c^{\Theta}_{1,0}
\bigr)
\nonumber\\
&\qquad
+ (\log \ell)^3 \bigl(
-7.80677
- 6.49926 \log \ell_0
- 1.31955 (\log \ell_0)^2
+ 0.138889 c^{B}_{1,0}
\nonumber\\
&\qquad\qquad
- 7.24624 c^{\Theta}_{1,0}
- 4.46181 \log \ell_0\, c^{\Theta}_{1,0}
- 0.347222 (\log \ell_0)^2 c^{\Theta}_{1,0}
+ 0.416667 (c^{\Theta}_{1,0})^{2}
\bigr)
\nonumber\\
&\qquad
+ (\log \ell)^2 \bigl(
-111.626
+ 21.3942 \log \ell_0
+ 13.3665 (\log \ell_0)^2
+ 1.30208 (\log \ell_0)^3
\nonumber\\
&\qquad\qquad
- 1.125 c^{B}_{1,0}
- 0.416667 \log \ell_0\, c^{B}_{1,0}
- 3.46714 c^{\Theta}_{1,0}
- 1.02641 \log \ell_0\, c^{\Theta}_{1,0}
\nonumber\\
&\qquad\qquad
- 9.375 (c^{\Theta}_{1,0})^{2}
- 1.25 \log \ell_0\, (c^{\Theta}_{1,0})^{2}
\bigr)
\nonumber\\
&\qquad
+ \log \ell \bigl(
131.914
- 39.0856 \log \ell_0
- 32.2968 (\log \ell_0)^2
- 2.17014 (\log \ell_0)^3
\bigr)
\nonumber\\
&\qquad
+ \log \ell \bigl(
+ c^{B}_{1,0}
- 6.96453 c^{\Theta}_{1,0}
+ 32.6793 \log \ell_0\, c^{\Theta}_{1,0}
+ 6.25 (\log \ell_0)^2 c^{\Theta}_{1,0}
\bigr)
\nonumber\\
&\qquad
+ \log \ell \bigl(
- 2 c^{B}_{1,0} c^{\Theta}_{1,0}
+ 11 (c^{\Theta}_{1,0})^{2}
- 2 (c^{\Theta}_{1,0})^{3}
- 2 c^{P}_{3,0}
\bigr)
\nonumber\\
&\qquad
+ c^{R}_{3,0}
\Bigr].
\end{align}
The parameter $c^{R}_{3,0}$ is determined analogously from the initial conditions~\eqref{initialconditions}.
\subsection{The dd$_{\text{II}}$NTK mean}
Applying the same analysis to the d$_{\text{II}}$NTK mean tensors $S^{(\ell)}$, $T^{(\ell)}$ and $U^{(\ell)}$, yields the critical exponents $p_{S}=-1$, $p_{T}=-1$, and $p_{U} = 0$, together with
\begin{align}
S^{(\ell)} ={}& \frac{3}{4}\,\ell
- \frac{5}{16} (\log \ell)^2
+ \frac{1}{16} \log \ell \,(27 + 10 \log \ell_0)
+ \frac{3}{32} (45 + 16 c^{\Theta}_{1,0})
\nonumber\\
&+ \frac{1}{\ell}\Bigl[
\frac{25}{576} (\log \ell)^4
- \frac{5}{288} (\log \ell)^3 (61 + 10 \log \ell_0)
\nonumber\\
&\qquad
+ \frac{1}{576} (\log \ell)^2 \bigl(
2763 + 1560 \log \ell_0 + 100 (\log \ell_0)^2 - 240 c^{\Theta}_{1,0}
\bigr)
\nonumber\\
&\qquad
+ \frac{1}{576} \log \ell \bigl(
-7137 + 360 \log \ell_0 - 200 (\log \ell_0)^2
+ 5232 c^{\Theta}_{1,0} + 480 \log \ell_0\, c^{\Theta}_{1,0}
\bigr)
\nonumber\\
&\qquad
+ \frac{1}{576} \bigl(
6 + 765 \log \ell_0 + 50 (\log \ell_0)^2
+ 576 c^{B}_{1,0} + 2400 c^{\Theta}_{1,0}
\nonumber\\
&\qquad\qquad
- 480 \log \ell_0\, c^{\Theta}_{1,0}
+ 576 (c^{\Theta}_{1,0})^{2}
\bigr)
\Bigr]
\nonumber\\
&+ \frac{1}{\ell^2}\Bigl[
\frac{335}{3456} (\log \ell)^5
- \frac{5}{3456} (\log \ell)^4 (1328 + 335 \log \ell_0)
\nonumber\\
&\qquad
+ \frac{5}{3456} (\log \ell)^3 \bigl(
8685 + 4286 \log \ell_0 + 480 (\log \ell_0)^2 - 912 c^{\Theta}_{1,0}
\bigr)
\nonumber\\
&\qquad
+ \frac{1}{6912} (\log \ell)^2 \bigl(
-412047 - 75465 \log \ell_0 - 14625 (\log \ell_0)^2
- 1000 (\log \ell_0)^3
\nonumber\\
&\qquad\qquad
+ 164520 c^{\Theta}_{1,0}
+ 27360 \log \ell_0\, c^{\Theta}_{1,0}
\bigr)
\nonumber\\
&\qquad
+ \frac{1}{13824} \log \ell \bigl(
1165869 + 298080 \log \ell_0 + 1200 (\log \ell_0)^2
+ 2500 (\log \ell_0)^3
\nonumber\\
&\qquad\qquad
+ 52992 c^{B}_{1,0}
- 458208 c^{\Theta}_{1,0}
- 96480 \log \ell_0\, c^{\Theta}_{1,0}
- 9600 (\log \ell_0)^2 c^{\Theta}_{1,0}
\nonumber\\
&\qquad\qquad
+ 108288 (c^{\Theta}_{1,0})^{2}
\bigr)
+ c^{S}_{3,0}
\Bigr],\label{Slargelgeneral}\\
T^{(\ell)} ={}& -18.981 + 1.88889\,\ell
- 1.84028 (\log \ell)^2
+ 0.196759 \log \ell_0
\nonumber\\
&+ \log \ell \bigl(5.80387 + 1.18056 \log \ell_0\bigr)
+ 2.83333\, c^{\Theta}_{1,0}
\nonumber\\
&+ \frac{1}{\ell}\Bigl[
440.657
+ 0.868056 (\log \ell)^4
+ (\log \ell)^3 \bigl(-8.44101 - 1.04167 \log \ell_0\bigr)
\nonumber\\
&\qquad
+ 8.02013 \log \ell_0
- 0.491898 (\log \ell_0)^2
+ (\log \ell)^2 \bigl(
46.0699 + 3.38594 \log \ell_0 - 2.5\, c^{\Theta}_{1,0}
\bigr)
\nonumber\\
&\qquad
- 32.6732\, c^{\Theta}_{1,0}
+ \log \ell \bigl(
-158.347 - 13.6717 \log \ell_0 + 8.12626\, c^{\Theta}_{1,0}
\bigr)
\Bigr]
\nonumber\\
&+ \frac{1}{\ell^2}\Bigl[
(\log \ell)^5 \bigl(
-0.1989 - 0.0241419 \log \ell_0 + 0.0694444\, c^{\Theta}_{1,0}
\bigr)
\nonumber\\
&\qquad
+ (\log \ell)^4 \bigl(
-2.50973 - 1.10439 \log \ell_0 - 0.108507 (\log \ell_0)^2
- 0.376245\, c^{\Theta}_{1,0}
\bigr)
\nonumber\\
&\qquad
+ (\log \ell)^3 \bigl(
-75.8647 + 8.00378 \log \ell_0 + 0.962168 (\log \ell_0)^2
+ 1.63315\, c^{\Theta}_{1,0}
\nonumber\\
&\qquad\qquad
- 0.347222 \log \ell_0\, c^{\Theta}_{1,0}
\bigr)
\nonumber\\
&\qquad
+ (\log \ell)^2 \bigl(
-256.807 - 13.8107 \log \ell_0 - 2.88587 (\log \ell_0)^2
- 10.1288\, c^{\Theta}_{1,0}
\nonumber\\
&\qquad\qquad
+ 3.15131 \log \ell_0\, c^{\Theta}_{1,0}
\bigr)
\nonumber\\
&\qquad
+ \log \ell \bigl(
-297.85 + 308.089 \log \ell_0 + 8.69265 (\log \ell_0)^2
- 0.204958 (\log \ell_0)^3
\nonumber\\
&\qquad\qquad
+ 229.472\, c^{\Theta}_{1,0}
- 10.9002 \log \ell_0\, c^{\Theta}_{1,0}
+ (c^{\Theta}_{1,0})^{2}
- 4\, c^{Q}_{3,0}
\bigr)
\nonumber\\
&\qquad
+ c^{T}_{3,0}
\Bigr],\label{Tlargelgeneral}\\
U^{(\ell)} ={}& \frac{27}{8}
+ \frac{1}{\ell}\Bigl[
-\frac{45}{16}(\log \ell)^2
+ \frac{9}{16}\log \ell\,(37+10\log \ell_0)
\nonumber\\
&\qquad
- \frac{9}{8}\bigl(47+5\log \ell_0-12c^{\Theta}_{1,0}\bigr)
\Bigr]
\nonumber\\
&+ \frac{1}{\ell^2}\Bigl[
\frac{5}{64}(\log \ell)^5
- \frac{5}{128}(\log \ell)^4(27+10\log \ell_0)
\nonumber\\
&\qquad
+ \frac{1}{192}(\log \ell)^3\bigl(
1119+540\log \ell_0+100(\log \ell_0)^2-240c^{\Theta}_{1,0}
\bigr)
\nonumber\\
&\qquad
+ \frac{3}{128}(\log \ell)^2\bigl(
-837-260\log \ell_0+432c^{\Theta}_{1,0}
+160\log \ell_0\,c^{\Theta}_{1,0}
\bigr)
\nonumber\\
&\qquad
- \frac{3}{64}\log \ell\bigl(
-2733+405\log \ell_0+100(\log \ell_0)^2
+1104c^{\Theta}_{1,0}-192(c^{\Theta}_{1,0})^{2}
\bigr)
\nonumber\\
&\qquad
+ c^{U}_{2,0}
\Bigr]
\nonumber\\
&+ \frac{1}{\ell^3}\Bigl[
-\frac{5}{2304}(\log \ell)^5(819+230\log \ell_0)
\nonumber\\
&\qquad
+ \frac{5}{1152}(\log \ell)^4\bigl(
1593+1135\log \ell_0+225(\log \ell_0)^2-300c^{\Theta}_{1,0}
\bigr)
\nonumber\\
&\qquad
+ \frac{1}{3456}(\log \ell)^3\bigl(
-187047-81075\log \ell_0-14850(\log \ell_0)^2
-2000(\log \ell_0)^3
\nonumber\\
&\qquad\qquad
+43560c^{\Theta}_{1,0}
+18000\log \ell_0\,c^{\Theta}_{1,0}
\bigr)
\nonumber\\
&\qquad
+ \frac{1}{2304}(\log \ell)^2\bigl(
460782+110370\log \ell_0+21525(\log \ell_0)^2
+500(\log \ell_0)^3
\nonumber\\
&\qquad\qquad
-173088c^{\Theta}_{1,0}
-41760\log \ell_0\,c^{\Theta}_{1,0}
-9600(\log \ell_0)^2c^{\Theta}_{1,0}
+20160(c^{\Theta}_{1,0})^{2}
\bigr)
\nonumber\\
&\qquad
+ \frac{1}{46080}\bigl(
30188151-884850\log \ell_0+309000(\log \ell_0)^2
+65000(\log \ell_0)^3
\nonumber\\
&\qquad\qquad
-12139200c^{\Theta}_{1,0}
-302400\log \ell_0\,c^{\Theta}_{1,0}
-60000(\log \ell_0)^2c^{\Theta}_{1,0}
+1497600(c^{\Theta}_{1,0})^{2}
\nonumber\\
&\qquad\qquad
+57600\log \ell_0\,(c^{\Theta}_{1,0})^{2}
-92160(c^{\Theta}_{1,0})^{3}
\bigr)
\nonumber\\
&\qquad
+ \log \ell \Bigl(
\frac{-1946727-1035240\log \ell_0+48450(\log \ell_0)^2
+15500(\log \ell_0)^3}{4608}
\nonumber\\
&\qquad\qquad
+ \frac{526176+348480\log \ell_0+4800(\log \ell_0)^2}{4608}\,c^{\Theta}_{1,0}
- \frac{64512+46080\log \ell_0}{4608}\,(c^{\Theta}_{1,0})^{2}
\nonumber\\
&\qquad\qquad
+ \frac{5}{6}c^{U}_{2,0}\Bigr) - \frac{1}{6}(4+5\log \ell_0)\,c^{U}_{2,0}
\Bigr].\label{Ulargelgeneral}
\end{align}
The coefficients $c^{S}_{3,0}$, $c^{T}_{3,0}$, and $c^{U}_{2,0}$ are fixed by the initial conditions~\eqref{initialconditions}.
\subsection{The sextic vertex}
Finally, for the sextic vertex $V^{(\ell)}_{6}$, we obtain the critical exponent $p_{V_{6}} = 3$, and the large-$\ell$ expansion
\begin{align}
V_{6}^{(\ell)} &= \frac{2}{\ell^3}
 + \frac{1}{\ell^4}
   \left(
     -0.573233 + 2.5 \log \ell
     - 1.25 \log \ell_0
   \right).
\end{align}
Within our approximation scheme, $V_{6}^{(\ell)}$ is entirely determined by the lower-rank tensors $V_{4}^{(\ell)}$ and $K^{(\ell)}$.
\subsection{Application}
We apply the above results to a square neural network of width $n=50$, using the same input vector as in Section~\ref{app:analytical_recursion_relations}. This choice fixes the initial conditions
\begin{align}\label{initialconditionsnumerical}
K^{(1)} = 0.238565 \quad , \quad \Theta^{(1)} = 1.23857 \quad , \quad V_{4}^{(1)} = -0.113827
\end{align}
These determine the scale $\ell_{0}$ via $\log(\ell_{0}) = -2.74141$, as well as the coefficients $c^{\mathcal{O}}$ in~\eqref{nngplexpansion}-\eqref{Ulargelgeneral}:
\allowdisplaybreaks
\begin{align}
c^{V}_{2,0} &= 0.861763, 
& c^{\Theta}_{1,0} &= 0.829493, 
& c^{D}_{1,0} &= \frac{4}{3}, & c^{F}_{3,0} &= 6.49751, \nonumber\\ 
 c^{A}_{1,0} &= -\frac{16}{3}, 
& c^{B}_{1,0} &= 7.19961, 
& c^{P}_{3,0} &= 39.7952, & c^{Q}_{3,0} &= 47.0357,\nonumber\\ 
 c^{R}_{3,0} &= -49.1347, 
& c^{S}_{3,0} &= -16.4737, 
& c^{T}_{3,0} &= -372.59,
& c^{U}_{2,0} &= -196.103
\end{align}
Substituting these values into~\eqref{nngplexpansion}–\eqref{Ulargelgeneral} yields the corresponding large-$\ell$ expansions. We now present these expressions explicitly. The NNGP takes the form
\begin{align}\label{Klargel}
K^{(\ell)}
&= \frac{1}{2\ell}
 + \frac{1}{\ell^2}
   \left(
     0.571128 + \frac{5}{24}\log \ell
   \right)
\nonumber\\
&\quad
 + \frac{1}{\ell^3}
   \left(
     0.598432 + 0.389134 \log \ell
     + \frac{25}{288} (\log \ell)^2
   \right)
\nonumber\\
&\quad
 + \frac{1}{\ell^4}
   \left(
     0.54664 + 0.585901 \log \ell
   \right.
\nonumber\\
&\qquad\qquad\left.
     + 0.20704 (\log \ell)^2
     + \frac{125}{3456} (\log \ell)^3
   \right)
\nonumber\\
&\quad
 + \frac{1}{\ell^5}
   \left(
     -1.97764 + 0.666942 \log \ell
   \right.
\nonumber\\
&\qquad\qquad\left.
     + 0.401984 (\log \ell)^2
     + 0.0999518 (\log \ell)^3
     + \frac{625}{41472} (\log \ell)^4
   \right).
\end{align}
For completeness, the parameters $p_{K}$ and $c^{K}_{i,j}$ are summarized in Table~\ref{tableforK}.
\begin{table}[t]
\centering
\renewcommand{\arraystretch}{1.2}
\begin{tabular}{|c|c|c|c|c|c|c|c|}
\hline
 $p_{K}$ & \multicolumn{7}{c|}{$c^{K}_{i,j}$} \\
\hline
\multirow{6}{*}{$1$}
& $i \backslash j$ & $0$ & $1$ & $2$ & $3$ & $4$ & $5$ \\
\cline{2-8}
 & $0$ & $\frac{1}{2}$ & $0$ & $0$ & $0$ & $0$ & $0$ \\
\cline{2-8}
 & $1$ & $0.571128$ & $\frac{5}{24}$ & $0$ & $0$ & $0$ & $0$ \\
\cline{2-8}
 & $2$ & $0.598432$ & $0.389134$ & $\frac{25}{288}$ & $0$ & $0$ & $0$ \\
\cline{2-8}
 & $3$ & $0.54664$ & $0.585901$ & $0.20704$ & $\frac{125}{3456}$ & $0$ & $0$ \\
\cline{2-8}
 & $4$ & $-1.97764$ & $0.666942$ & $0.401984$ & $0.0999518$ & $\frac{625}{41472}$ & $0$ \\
\hline
\end{tabular}
\caption{Coefficients of the large-$\ell$ expansion for the tensor $K^{(\ell)}$.}\label{tableforK}
\end{table} 
The quartic vertex reads
\begin{align}\label{Vlargel}
V_{4}^{(\ell)} =& -\frac{1}{2\ell^2} + \frac{1}{\ell^3} \left( -0.475589 - \frac{5}{12}\log \ell \right) \nonumber\\
&+ \frac{1}{\ell^4} \left(
0.861763 - 2.58754 \log \ell - \frac{25}{96} (\log \ell)^2
\right),
\end{align}
with coefficients given in Table~\ref{tableforV}.
\begin{table}[t]
\centering
\renewcommand{\arraystretch}{1.2}
\begin{tabular}{|c|c|c|c|c|c|c|c|}
\hline
$p_{V}$ & \multicolumn{7}{c|}{$c^{V}_{i,j}$} \\
\hline
\multirow{4}{*}{$2$}
& $i \backslash j$ & $0$ & $1$ & $2$ & $3$ & $4$ & $5$ \\
\cline{2-8}
& $0$ & $-\frac{1}{2}$ & $0$ & $0$ & $0$ & $0$ & $0$ \\
\cline{2-8}
& $1$ & $-0.475589$ & $-\frac{5}{12}$ & $0$ & $0$ & $0$ & $0$ \\
\cline{2-8}
& $2$ & $0.861763$ & $-2.58754$ & $-\frac{25}{96}$ & $0$ & $0$ & $0$ \\
\hline
\end{tabular}
\caption{Coefficients of the large-$\ell$ expansion for the tensor $V^{(\ell)}$.}\label{tableforV}
\end{table}
The NTK takes the form
\begin{align}\label{Thetalargel}
\Theta^{(\ell)}
&= \frac{3}{2}
 + \frac{1}{\ell}
   \left(
     0.829493 - 0.0172557 \log \ell
     - \frac{5}{24} (\log \ell)^2
   \right)
\nonumber\\
&\quad
 + \frac{1}{\ell^2}
   \left(
     0.271725 - 0.380142 \log \ell
     - 0.175715 (\log \ell)^2
     - \frac{25}{288} (\log \ell)^3
   \right)
\nonumber\\
&\quad
 + \frac{1}{\ell^3}
   \left(
     -0.280312 - 0.00940075 \log \ell
   \right.
\nonumber\\
&\qquad\qquad\left.
     - 0.405269 (\log \ell)^2
     - 0.107265 (\log \ell)^3
     - \frac{125}{3456} (\log \ell)^4
   \right)
\nonumber\\
&\quad
 + \frac{1}{\ell^4}
   \left(
     -1.08234 - 0.722724 \log \ell
   \right.
\nonumber\\
&\qquad\qquad\left.
     - 0.125245 (\log \ell)^2
     - 0.279448 (\log \ell)^3
     - 0.0513459 (\log \ell)^4
     - \frac{625}{41472} (\log \ell)^5
   \right),
\end{align}   
with coefficients summarized in Table~\ref{tableforTheta}.
\begin{table}[t]
\centering
\renewcommand{\arraystretch}{1.2}
\begin{tabular}{|c|c|c|c|c|c|c|c|}
\hline
$p_{\Theta}$ & \multicolumn{7}{c|}{$c^{\Theta}_{i,j}$} \\
\hline
\multirow{6}{*}{$0$}
& $i \backslash j$ & $0$ & $1$ & $2$ & $3$ & $4$ & $5$ \\
\cline{2-8}
& $0$ & $\frac{3}{2}$ & $0$ & $0$ & $0$ & $0$ & $0$ \\
\cline{2-8}
& $1$ & $0.829493$ & $-0.0172557$ & $-\frac{5}{24}$ & $0$ & $0$ & $0$ \\
\cline{2-8}
& $2$ & $0.271725$ & $-0.380142$ & $-0.175715$ & $-\frac{25}{288}$ & $0$ & $0$ \\
\cline{2-8}
& $3$ & $-0.280312$ & $-0.00940075$ & $-0.405269$ & $-0.107265$ & $-\frac{125}{3456}$ & $0$ \\
\cline{2-8}
& $4$ & $-1.08234$ & $-0.722724$ & $-0.125245$ & $-0.279448$ & $-0.0513459$ & $-\frac{625}{41472}$ \\
\hline
\end{tabular}
\caption{Coefficients of the large-$\ell$ expansion for the tensor $\Theta^{(\ell)}$.}
\label{tableforTheta}
\end{table}

The NTK-preactivation mixed cumulants are
\begin{align}
D^{(\ell)}
&= -\frac{4}{3\ell^2} + \frac{1}{\ell^3}
   \left(
     \frac{4}{3}
     - 3.12564 \log \ell
     + 2.17817 (\log \ell)^2
     + \frac{5}{54} (\log \ell)^3
   \right),\label{Dlargel}\\
F^{(\ell)}
&= -\frac{1}{2\ell} + \frac{1}{\ell^2}
   \left(
     0.416062 - 0.303872 \log \ell
     + \frac{5}{48} (\log \ell)^2
   \right)
\nonumber\\
&\quad
 + \frac{1}{\ell^3}
   \left(
     -6.41357 + 1.16879 \log \ell
   \right.
\nonumber\\
&\qquad\qquad\left.
     - 0.154146 (\log \ell)^2
     + \frac{25}{288} (\log \ell)^3
   \right)
\nonumber\\
&\quad
 + \frac{1}{\ell^4}
   \left(
     6.49751 + 22.7949 \log \ell
   \right.
\nonumber\\
&\qquad\qquad\left.
     + 0.562927 (\log \ell)^2
     + 0.0848911 (\log \ell)^3
     + \frac{625}{13824} (\log \ell)^4
   \right),\label{Flargel}
\end{align}
with coefficients listed in Tables~\ref{tableforD} and~\ref{tableforF}.
\begin{table}[t]
\centering
\renewcommand{\arraystretch}{1.2}
\begin{tabular}{|c|c|c|c|c|c|c|c|}
\hline
$p_{D}$ & \multicolumn{7}{c|}{$c^{D}_{i,j}$} \\
\hline
\multirow{3}{*}{$2$}
& $i \backslash j$ & $0$ & $1$ & $2$ & $3$ & $4$ & $5$ \\
\cline{2-8}
& $0$ & $-\frac{4}{3}$ & $0$ & $0$ & $0$ & $0$ & $0$ \\
\cline{2-8}
& $1$ & $\frac{4}{3}$ & $-3.12564$ & $2.17817$ & $\frac{5}{54}$ & $0$ & $0$ \\
\hline
\end{tabular}
\caption{Coefficients of the large-$\ell$ expansion for the tensor $D^{(\ell)}$.}
\label{tableforD}
\end{table}
\begin{table}[t]
\centering
\renewcommand{\arraystretch}{1.2}
\begin{tabular}{|c|c|c|c|c|c|c|c|}
\hline
$p_{F}$ & \multicolumn{7}{c|}{$c^{F}_{i,j}$} \\
\hline
\multirow{5}{*}{$1$}
& $i \backslash j$ & $0$ & $1$ & $2$ & $3$ & $4$ & $5$ \\
\cline{2-8}
& $0$ & $-\frac{1}{2}$ & $0$ & $0$ & $0$ & $0$ & $0$ \\
\cline{2-8}
& $1$ & $0.416062$ & $-0.303872$ & $\frac{5}{48}$ & $0$ & $0$ & $0$ \\
\cline{2-8}
& $2$ & $-6.41357$ & $1.16879$ & $-0.154146$ & $\frac{25}{288}$ & $0$ & $0$ \\
\cline{2-8}
& $3$ & $6.49751$ & $22.7949$ & $0.562927$ & $0.0848911$ & $\frac{625}{13824}$ & $0$ \\
\hline
\end{tabular}
\caption{Coefficients of the large-$\ell$ expansion for the tensor $F^{(\ell)}$.}
\label{tableforF}
\end{table}
The NTK variance tensors are
\begin{align}
A^{(\ell)}
&= \frac{16}{3\ell} + \frac{1}{\ell^2}
   \left(
     -\frac{16}{3} - 19.7601 \log \ell + 3.98303 (\log \ell)^2
     - 3.2746 (\log \ell)^3
     - \frac{5}{54} (\log \ell)^4
   \right),\label{Alargel}\\  
B^{(\ell)}
&= \frac{9}{2\ell} + \frac{1}{\ell^2}
   \left(
     7.19961 - 10.773 \log \ell
     - 0.0517672 (\log \ell)^2
     - \frac{5}{12} (\log \ell)^3
   \right)
\nonumber\\
&\quad
 + \frac{1}{\ell^3}
   \left(
     -36.2955 - 12.3472 \log \ell
   \right.
\nonumber\\
&\qquad\qquad\left.
     - 12.4116 (\log \ell)^2
     - 0.55801 (\log \ell)^3
     - \frac{125}{288} (\log \ell)^4
   \right)
\nonumber\\
&\quad
 + \frac{1}{\ell^4}
   \left(
     24.5959 - 70.6998 \log \ell
   \right.
\nonumber\\
&\qquad\qquad\left.
     - 14.6047 (\log \ell)^2
     - 10.5449 (\log \ell)^3
     - 0.646109 (\log \ell)^4
     - \frac{125}{384} (\log \ell)^5
   \right) ,\label{Blargel}
\end{align}
with coefficients reported in Tables~\ref{tableforA} and~\ref{tableforB}.
\begin{table}[t]
\centering
\renewcommand{\arraystretch}{1.2}
\begin{tabular}{|c|c|c|c|c|c|c|c|}
\hline
$p_{A}$ & \multicolumn{7}{c|}{$c^{A}_{i,j}$} \\
\hline
\multirow{3}{*}{$1$}
& $i \backslash j$ & $0$ & $1$ & $2$ & $3$ & $4$ & $5$ \\
\cline{2-8}
& $0$ & $\frac{16}{3}$ & $0$ & $0$ & $0$ & $0$ & $0$ \\
\cline{2-8}
& $1$ & $-\frac{16}{3}$ & $-19.7601$ & $3.98303$ & $-3.2746$ & $-\frac{5}{54}$ & $0$ \\
\hline
\end{tabular}
\caption{Coefficients of the large-$\ell$ expansion for the tensor $A^{(\ell)}$.}
\label{tableforA}
\end{table}
\begin{table}[t]
\centering
\renewcommand{\arraystretch}{1.2}
\begin{tabular}{|c|c|c|c|c|c|c|c|}
\hline
$p_{B}$ & \multicolumn{7}{c|}{$c^{B}_{i,j}$} \\
\hline
\multirow{5}{*}{$1$}
& $i \backslash j$ & $0$ & $1$ & $2$ & $3$ & $4$ & $5$ \\
\cline{2-8}
& $0$ & $\frac{9}{2}$ & $0$ & $0$ & $0$ & $0$ & $0$ \\
\cline{2-8}
& $1$ & $7.19961$ & $-10.773$ & $-0.0517672$ & $-\frac{5}{12}$ & $0$ & $0$ \\
\cline{2-8}
& $2$ & $-36.2955$ & $-12.3472$ & $-12.4116$ & $-0.55801$ & $-\frac{125}{288}$ & $0$ \\
\cline{2-8}
& $3$ & $24.5959$ & $-70.6998$ & $-14.6047$ & $-10.5449$ & $-0.646109$ & $-\frac{125}{384}$ \\
\hline
\end{tabular}
\caption{Coefficients of the large-$\ell$ expansion for the tensor $B^{(\ell)}$.}
\label{tableforB}
\end{table}
The dNTK tensors read
\begin{align}
P^{(\ell)}
&= -\frac{3}{4} + \frac{1}{\ell}
   \left(
     -0.163431 - 0.286616 \log \ell
     + \frac{5}{16} (\log \ell)^2
   \right)
\nonumber\\
&\quad
 + \frac{1}{\ell^2}
   \left(
     -38.8818 + 19.1203 \log \ell
   \right.
\nonumber\\
&\qquad\qquad\left.
     - 2.44277 (\log \ell)^2
     + 0.713296 (\log \ell)^3
     - \frac{25}{576} (\log \ell)^4
   \right)
\nonumber\\
&\quad
 + \frac{1}{\ell^3}
   \left(
     39.7952 + 130.51 \log \ell
   \right.
\nonumber\\
&\qquad\qquad\left.
     + 10.9595 (\log \ell)^2
     + 0.63484 (\log \ell)^3
     + 0.508496 (\log \ell)^4
     - \frac{55}{3456} (\log \ell)^5
   \right),\label{Plargel}\\
Q^{(\ell)}
&= -\frac{17}{12} + \frac{1}{\ell}
   \left(
     2.97701 - 0.781566 \log \ell
     + \frac{5}{8} (\log \ell)^2
   \right)
\nonumber\\
&\quad
 + \frac{1}{\ell^2}
   \left(
     -48.596 + 9.80987 \log \ell
   \right.
\nonumber\\
&\qquad\qquad\left.
     - 1.86301 (\log \ell)^2
     + 0.636662 (\log \ell)^3
     - \frac{25}{288} (\log \ell)^4
   \right)
\nonumber\\
&\quad
 + \frac{1}{\ell^3}
   \left(
     47.9186 + 108.564 \log \ell
   \right.
\nonumber\\
&\qquad\qquad\left.
     + 63.8277 (\log \ell)^2
     - 0.168637 (\log \ell)^3
     + 0.33582 (\log \ell)^4
     - \frac{305}{3456} (\log \ell)^5
   \right) ,\label{Qlargel}
\end{align}
with coefficients given in Tables~\ref{tableforP} and~\ref{tableforQ}.
\begin{table}[t]
\centering
\renewcommand{\arraystretch}{1.2}
\begin{tabular}{|c|c|c|c|c|c|c|c|}
\hline
$p_{P}$ & \multicolumn{7}{c|}{$c^{P}_{i,j}$} \\
\hline
\multirow{4}{*}{$0$}
& $i \backslash j$ & $0$ & $1$ & $2$ & $3$ & $4$ & $5$ \\
\cline{2-8}
& $0$ & $-\frac{3}{4}$ & $0$ & $0$ & $0$ & $0$ & $0$ \\
\cline{2-8}
& $1$ & $-0.163431$ & $-0.286616$ & $\frac{5}{16}$ & $0$ & $0$ & $0$ \\
\cline{2-8}
& $2$ & $-38.8818$ & $19.1203$ & $-2.44277$ & $0.713296$ & $-\frac{25}{576}$ & $0$ \\
\cline{2-8}
& $3$ & $39.7952$ & $130.51$ & $10.9595$ & $0.63484$ & $0.508496$ & $-\frac{55}{3456}$ \\
\hline
\end{tabular}
\caption{Coefficients of the large-$\ell$ expansion for the tensor $P^{(\ell)}$.}
\label{tableforP}
\end{table}
\begin{table}[t]
\centering
\renewcommand{\arraystretch}{1.2}
\begin{tabular}{|c|c|c|c|c|c|c|c|}
\hline
$p_{Q}$ & \multicolumn{7}{c|}{$c^{Q}_{i,j}$} \\
\hline
\multirow{4}{*}{$0$}
& $i \backslash j$ & $0$ & $1$ & $2$ & $3$ & $4$ & $5$ \\
\cline{2-8}
& $0$ & $-\frac{17}{12}$ & $0$ & $0$ & $0$ & $0$ & $0$ \\
\cline{2-8}
& $1$ & $2.97701$ & $-0.781566$ & $\frac{5}{8}$ & $0$ & $0$ & $0$ \\
\cline{2-8}
& $2$ & $-48.596$ & $9.80987$ & $-1.86301$ & $0.636662$ & $-\frac{25}{288}$ & $0$ \\
\cline{2-8}
& $3$ & $47.9186$ & $108.564$ & $63.8277$ & $-0.168637$ & $0.33582$ & $-\frac{305}{3456}$ \\
\hline
\end{tabular}
\caption{Coefficients of the large-$\ell$ expansion for the tensor $Q^{(\ell)}$.}
\label{tableforQ}
\end{table}
The dd$_{\text{I}}$NTK tensor is
\begin{align}
R^{(\ell)}
&= -6.84299 - 1.75\,\ell
 + 0.181818 \log \ell
 + 0.9375 (\log \ell)^2
\nonumber\\
&\quad
 + \frac{1}{\ell}
   \left(
     57.7277 - 10.7504 \log \ell
   \right.
\nonumber\\
&\qquad\qquad\left.
     + 7.72305 (\log \ell)^2
     + 0.604921 (\log \ell)^3
     - 0.173611 (\log \ell)^4
   \right)
\nonumber\\
&\quad
 + \frac{1}{\ell^2}
   \left(
     -49.1347 - 77.9822 \log \ell
   \right.
\nonumber\\
&\qquad\qquad\left.
     - 101.16 (\log \ell)^2
     + 3.35098 (\log \ell)^3
     - 0.034582 (\log \ell)^4
     - 0.367605 (\log \ell)^5
   \right), \label{Rlargel}
\end{align}
with coefficients listed in Table~\ref{tableforR}.
\begin{table}[t]
\centering
\renewcommand{\arraystretch}{1.2}
\begin{tabular}{|c|c|c|c|c|c|c|c|}
\hline
$p_{R}$ & \multicolumn{7}{c|}{$c^{R}_{i,j}$} \\
\hline
\multirow{4}{*}{$-1$}
& $i \backslash j$ & $0$ & $1$ & $2$ & $3$ & $4$ & $5$ \\
\cline{2-8}
& $0$ & $-\frac{7}{4}$ & $0$ & $0$ & $0$ & $0$ & $0$ \\
\cline{2-8}
& $1$ & $-6.84299$ & $0.181818$ & $0.9375$ & $0$ & $0$ & $0$ \\
\cline{2-8}
& $2$ & $57.7277$ & $-10.7504$ & $7.72305$ & $0.604921$ & $-0.173611$ & $0$ \\
\cline{2-8}
& $3$ & $-49.1347$ & $-77.9822$ & $-101.16$ & $3.35098$ & $-0.034582$ & $-0.367605$ \\
\hline
\end{tabular}
\caption{Coefficients of the large-$\ell$ expansion for the tensor $R^{(\ell)}$.}
\label{tableforR}
\end{table}
The dd$_{\text{II}}$NTK tensors are
\begin{align}
S^{(\ell)}
&= 5.46299 + \frac{3}{4}\,\ell
 - 0.0258836 \log \ell
 - \frac{5}{16} (\log \ell)^2
\nonumber\\
&\quad
 + \frac{1}{\ell}
   \left(
     10.2607 - 11.0739 \log \ell
   \right.
\nonumber\\
&\qquad\qquad\left.
     - 1.66866 (\log \ell)^2
     - 0.583088 (\log \ell)^3
     + \frac{25}{576} (\log \ell)^4
   \right)
\nonumber\\
&\quad
 + \frac{1}{\ell^2}
   \left(
     -16.4737 + 39.1868 \log \ell
   \right.
\nonumber\\
&\qquad\qquad\left.
     - 31.861 (\log \ell)^2
     - 0.309359 (\log \ell)^3
     - 0.592631 (\log \ell)^4
     + \frac{335}{3456} (\log \ell)^5
   \right),\label{Slargel}\\
T^{(\ell)}
&= -17.1701 + 1.88889\,\ell
 + 2.56748 \log \ell
 - 1.84028 (\log \ell)^2
\nonumber\\
&\quad
 + \frac{1}{\ell}
   \left(
     387.871 - 114.126 \log \ell
   \right.
\nonumber\\
&\qquad\qquad\left.
     + 34.7139 (\log \ell)^2
     - 5.58537 (\log \ell)^3
     + 0.868056 (\log \ell)^4
   \right)
\nonumber\\
&\quad
 + \frac{1}{\ell^2}
   \left(
     -372.59 - 1045.22 \log \ell
   \right.
\nonumber\\
&\qquad\qquad\left.
     - 256.202 (\log \ell)^2
     - 88.4311 (\log \ell)^3
     - 0.609713 (\log \ell)^4
     - 0.0751136 (\log \ell)^5
     \right),\label{Tlargel}\\
U^{(\ell)}
&= \frac{27}{8} + \frac{1}{\ell}
   \left(
     -26.2564 + 5.39205 \log \ell
     - \frac{45}{16} (\log \ell)^2
   \right)
\nonumber\\
&\quad
 + \frac{1}{\ell^2}
   \left(
     -196.103 + 108.191 \log \ell
   \right.
\nonumber\\
&\qquad\qquad\left.
     - 3.04052 (\log \ell)^2
     + 0.995277 (\log \ell)^3
     + 0.0161773 (\log \ell)^4
     + \frac{5}{64} (\log \ell)^5
   \right)
\nonumber\\
&\quad
 + \frac{1}{\ell^3}
   \left(
     218.985 - 21.8079 \log \ell
   \right.
\nonumber\\
&\qquad\qquad\left.
     + 93.3553 (\log \ell)^2
     - 11.5693 (\log \ell)^3
     - 0.331592 (\log \ell)^4
     - 0.409017 (\log \ell)^5
   \right),\label{Ulargel}
\end{align}
with coefficients given in Tables~\ref{tableforS}-\ref{tableforU}.
\begin{table}[t]
\centering
\renewcommand{\arraystretch}{1.2}
\begin{tabular}{|c|c|c|c|c|c|c|c|}
\hline
$p_{S}$ & \multicolumn{7}{c|}{$c^{S}_{i,j}$} \\
\hline
\multirow{4}{*}{$-1$}
& $i \backslash j$ & $0$ & $1$ & $2$ & $3$ & $4$ & $5$ \\
\cline{2-8}
& $0$ & $\frac{3}{4}$ & $0$ & $0$ & $0$ & $0$ & $0$ \\
\cline{2-8}
& $1$ & $5.46299$ & $-0.0258836$ & $-\frac{5}{16}$ & $0$ & $0$ & $0$ \\
\cline{2-8}
& $2$ & $10.2607$ & $-11.0739$ & $-1.66866$ & $-0.583088$ & $\frac{25}{576}$ & $0$ \\
\cline{2-8}
& $3$ & $-16.4737$ & $39.1868$ & $-31.861$ & $-0.309359$ & $-0.592631$ & $\frac{335}{3456}$ \\
\hline
\end{tabular}
\caption{Coefficients of the large-$\ell$ expansion for the tensor $S^{(\ell)}$.}
\label{tableforS}
\end{table}
\begin{table}[t]
\centering
\renewcommand{\arraystretch}{1.2}
\begin{tabular}{|c|c|c|c|c|c|c|c|}
\hline
$p_{T}$ & \multicolumn{7}{c|}{$c^{T}_{i,j}$} \\
\hline
\multirow{4}{*}{$-1$}
& $i \backslash j$ & $0$ & $1$ & $2$ & $3$ & $4$ & $5$ \\
\cline{2-8}
& $0$ & $1.88889$ & $0$ & $0$ & $0$ & $0$ & $0$ \\
\cline{2-8}
& $1$ & $-17.1701$ & $2.56748$ & $-1.84028$ & $0$ & $0$ & $0$ \\
\cline{2-8}
& $2$ & $387.871$ & $-114.126$ & $34.7139$ & $-5.58537$ & $0.868056$ & $0$ \\
\cline{2-8}
& $3$ & $-372.59$ & $-1045.22$ & $-256.202$ & $-88.4311$ & $-0.609713$ & $-0.0751136$ \\
\hline
\end{tabular}
\caption{Coefficients of the large-$\ell$ expansion for the tensor $T^{(\ell)}$.}
\label{tableforT}
\end{table}
\begin{table}[t]
\centering
\renewcommand{\arraystretch}{1.2}
\begin{tabular}{|c|c|c|c|c|c|c|c|}
\hline
$p_{U}$ & \multicolumn{7}{c|}{$c^{U}_{i,j}$} \\
\hline
\multirow{4}{*}{$0$}
& $i \backslash j$ & $0$ & $1$ & $2$ & $3$ & $4$ & $5$ \\
\cline{2-8}
& $0$ & $\frac{27}{8}$ & $0$ & $0$ & $0$ & $0$ & $0$ \\
\cline{2-8}
& $1$ & $-26.2564$ & $5.39205$ & $-\frac{45}{16}$ & $0$ & $0$ & $0$ \\
\cline{2-8}
& $2$ & $-196.103$ & $108.191$ & $-3.04052$ & $0.995277$ & $0.0161773$ & $\frac{5}{64}$ \\
\cline{2-8}
& $3$ & $218.985$ & $-21.8079$ & $93.3553$ & $-11.5693$ & $-0.331592$ & $-0.409017$ \\
\hline
\end{tabular}
\caption{Coefficients of the large-$\ell$ expansion for the tensor $U^{(\ell)}$.}
\label{tableforU}
\end{table}
Finally, the sextic vertex reads
\begin{align}\label{V6largel}
V_{6}^{(\ell)} &= \frac{2}{\ell^3}
 + \frac{1}{\ell^4}
   \left(
     2.85353 + 2.5 \log \ell
   \right).
\end{align}
with coefficients reported in Table~\ref{tableforV6}.
\begin{table}[t]
\centering
\renewcommand{\arraystretch}{1.2}
\begin{tabular}{|c|c|c|c|c|c|c|c|}
\hline
$p_{V_6}$ & \multicolumn{7}{c|}{$c^{V_6}_{i,j}$} \\
\hline
\multirow{3}{*}{$3$}
& $i \backslash j$ & $0$ & $1$ & $2$ & $3$ & $4$ & $5$ \\
\cline{2-8}
& $0$ & $2$ & $0$ & $0$ & $0$ & $0$ & $0$ \\
\cline{2-8}
& $1$ & $2.85353$ & $2.5$ & $0$ & $0$ & $0$ & $0$ \\
\hline
\end{tabular}
\caption{Coefficients of the large-$\ell$ expansion for the tensor $V_6^{(\ell)}$.}
\label{tableforV6}
\end{table}

\section{Experimental setup}
\label{app:experimental_setup}


\subsection{Stability analysis}

In this section, we provide details of the experiments in Section~\ref{sec:applications-orthogonal} and present additional results validating the stability at criticality of the NNGP $K$, the four-point vertex $V_{4}$, and the NTK tensors $D$, $F$, $A$, and $B$ as the network depth $\ell$ increases.

We estimate the preactivation and NTK tensors via Monte Carlo sampling. For this purpose, individual networks are initialized in JAX using \texttt{neural-tangents} \texttt{stax} layers. Computing NTK-related quantities requires the Jacobian $\frac{\partial z_{i}^{(\ell)}(x_{\alpha})}{\partial\theta_{\mu}}$, which we obtain via JAX's automatic differentiation. Since forming the full Jacobian is computationally and memory intensive for wide networks, we instead compute layer-wise statistics sequentially and propagate them to deeper layers using the chain rule. Sampling across initializations is parallelized with \texttt{vmap}. Orthogonal weight initialization is implemented using JAX's \texttt{orthogonal} initializer. As tanh is not directly available in this setup, we approximate it using \texttt{ElementwiseNumerical} with degree 80.

All experiments are conducted on tanh MLPs without biases. Inputs are drawn with i.i.d. standard normal components, resulting in
\begin{align}
x_0 &=
\begin{pmatrix}\label{eq:x0_input}
0.934738 \\
0.26696 \\
0.784097 \\
0.656448 \\
0.305308 \\
0.401958 \\
0.894594 \\
0.0559893 \\
0.000643274 \\
0.0274513
\end{pmatrix}
\oplus
\begin{pmatrix}
0.377754 \\
0.127474 \\
0.879907 \\
0.710555 \\
0.509949 \\
0.312682 \\
0.0854376 \\
0.869372 \\
0.114232 \\
0.0851646
\end{pmatrix}
\oplus
\begin{pmatrix}
0.254697 \\
0.560475 \\
0.508664 \\
0.0271565 \\
0.426426 \\
0.457646 \\
0.913778 \\
0.40436 \\
0.407187 \\
0.0644401
\end{pmatrix}
\oplus
\begin{pmatrix}
0.256718 \\
0.869761 \\
0.0406222 \\
0.431362 \\
0.906228 \\
0.55979 \\
0.275852 \\
0.553722 \\
0.235762 \\
0.751627
\end{pmatrix}
\oplus
\begin{pmatrix}
0.178558 \\
0.411167 \\
0.100846 \\
0.220264 \\
0.215917 \\
0.490943 \\
0.596323 \\
0.0799147 \\
0.205998 \\
0.0372218
\end{pmatrix}\,,
\end{align}

\begin{align}\label{eq:x1_input}
x_1 &=
\begin{pmatrix}
0.986304 \\
0.396331 \\
0.829442 \\
0.163461 \\
0.0583318 \\
0.0385021 \\
0.885504 \\
0.160447 \\
0.583711 \\
0.389812
\end{pmatrix}
\oplus
\begin{pmatrix}
0.193679 \\
0.632048 \\
0.2954 \\
0.12661 \\
0.440952 \\
0.949114 \\
0.824362 \\
0.373311 \\
0.386399 \\
0.101208
\end{pmatrix}
\oplus
\begin{pmatrix}
0.815355 \\
0.475999 \\
0.493653 \\
0.267819 \\
0.0133166 \\
0.814708 \\
0.315126 \\
0.47199 \\
0.992467 \\
0.570161
\end{pmatrix}
\oplus
\begin{pmatrix}
0.23285 \\
0.183045 \\
0.92565 \\
0.199642 \\
0.38384 \\
0.184987 \\
0.518954 \\
0.078869 \\
0.456603 \\
0.333712
\end{pmatrix}
\oplus
\begin{pmatrix}
0.752504 \\
0.0212271 \\
0.805301 \\
0.696779 \\
0.0208987 \\
0.71065 \\
0.335234 \\
0.908037 \\
0.94978 \\
0.318431
\end{pmatrix}\,,
\end{align}

\begin{align}\label{eq:x2_input}
x_2 &=
\begin{pmatrix}
0.268954 \\
0.486888 \\
0.22212 \\
0.915653 \\
0.795563 \\
0.797374 \\
0.826671 \\
0.11106 \\
0.853911 \\
0.298177
\end{pmatrix}
\oplus
\begin{pmatrix}
0.0628893 \\
0.679924 \\
0.792366 \\
0.986086 \\
0.936489 \\
0.273049 \\
0.360604 \\
0.921974 \\
0.820319 \\
0.53683
\end{pmatrix}
\oplus
\begin{pmatrix}
0.631468 \\
0.779267 \\
0.0763499 \\
0.669761 \\
0.155254 \\
0.343043 \\
0.907627 \\
0.0647726 \\
0.0988953 \\
0.761094
\end{pmatrix}
\oplus
\begin{pmatrix}
0.901974 \\
0.196523 \\
0.939642 \\
0.794076 \\
0.225599 \\
0.66191 \\
0.912069 \\
0.0218007 \\
0.303759 \\
0.0650933
\end{pmatrix}
\oplus
\begin{pmatrix}
0.345452 \\
0.888144 \\
0.360429 \\
0.207685 \\
0.600649 \\
0.377979 \\
0.827548 \\
0.871934 \\
0.364661 \\
0.384977
\end{pmatrix}\,,
\end{align}

\begin{align}\label{eq:x3_input}
x_3 &=
\begin{pmatrix}
0.0682295 \\
0.140096 \\
0.340508 \\
0.359852 \\
0.613812 \\
0.19288 \\
0.134531 \\
0.49299 \\
0.591431 \\
0.893789
\end{pmatrix}
\oplus
\begin{pmatrix}
0.773795 \\
0.604593 \\
0.842464 \\
0.882483 \\
0.93913 \\
0.872499 \\
0.317102 \\
0.289739 \\
0.0178811 \\
0.160385
\end{pmatrix}
\oplus
\begin{pmatrix}
0.814678 \\
0.0561101 \\
0.193425 \\
0.413272 \\
0.622813 \\
0.668977 \\
0.771141 \\
0.884678 \\
0.516841 \\
0.54484
\end{pmatrix}
\oplus
\begin{pmatrix}
0.54987 \\
0.101687 \\
0.224264 \\
0.370917 \\
0.731496 \\
0.234507 \\
0.906523 \\
0.431445 \\
0.115241 \\
0.275449
\end{pmatrix}
\oplus
\begin{pmatrix}
0.810621 \\
0.0049892 \\
0.133457 \\
0.778419 \\
0.169455 \\
0.0719057 \\
0.1728 \\
0.906299 \\
0.322143 \\
0.847235
\end{pmatrix}\,.
\end{align}
For the two-dimensional tensors $K$ and $\Theta$, computations are restricted to the input pair $(x_0, x_1)$.

\textbf{Monte Carlo estimation of the kernels.}
We consider an ensemble of \(N_{\text{net}}\) neural networks with orthogonal weight initialization. The empirical NNGP and NTK are estimated as
\begin{align}
  \overline{K}^{(\ell)}_{\alpha \beta} &= \frac{1}{N_\text{net}} \sum_{I=1}^{N_\text{net}} z^{(\ell)}_{I;i,\alpha} z^{(\ell)}_{I;i,\beta} \,, \label{eq:nngp_mc}\\
\overline{\Theta}^{(\ell)}_{\alpha \beta} &= \frac{1}{N_\text{net}} \sum_{I=1}^{N_\text{net}} \left( \sum_{\mu}\frac{\partial z^{(\ell)}_{I;i,\alpha}}{\partial\theta_{\mu}}\frac{\partial z^{(\ell)}_{I; i, \beta}}{\partial\theta_{\mu}}\right) \,,
\end{align}
where \(i\) denotes a fixed channel and \(I\) indexes network initializations. Error bars indicate standard errors, computed as the sample standard deviation divided by \(\sqrt{N_{\text{net}}}\).


\begin{figure}
    \begin{center}
        \includegraphics[width=0.99\textwidth]{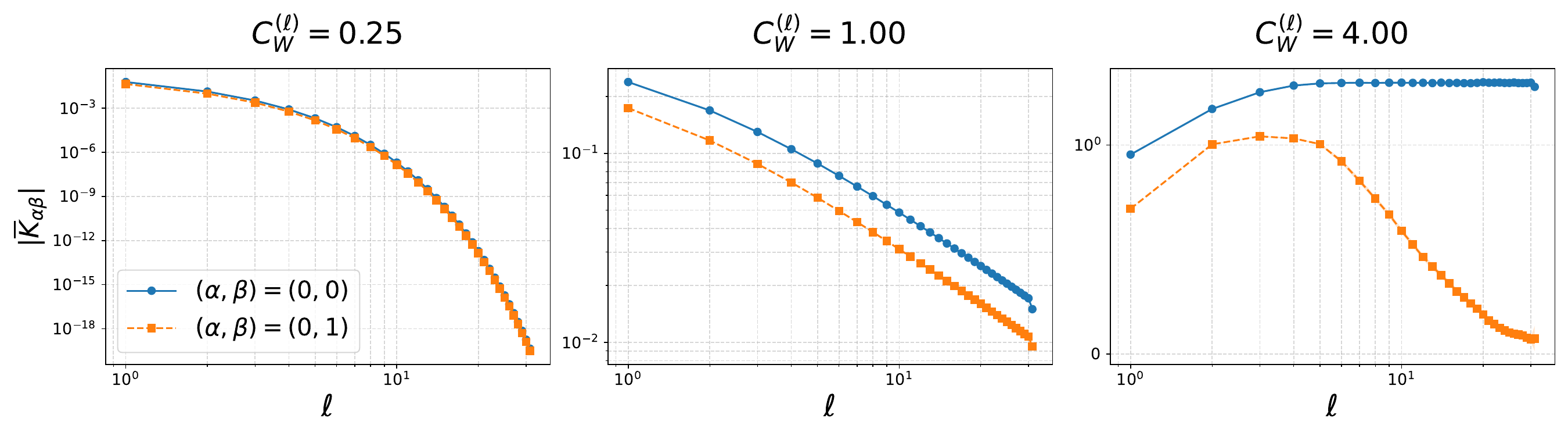}
    \end{center}
    \caption{\emph{Variance stability.} Components of the Monte Carlo estimate of the NNGP \(\overline{K}_{\alpha \beta}\) for a \(\tanh\) MLP, shown as a function of layer depth \(\ell\) for identical and distinct inputs, across three values of \(C_W^{(\ell)}\). Hidden layers have width \num{50}. Means are computed over \num{600} initializations for both non-critical (left, right) and critical (middle) cases. Error bars are shown in all panels (see text).}\label{fig:nngp_comparison}
\end{figure}
\textbf{Variance stability.} Figure~\ref{fig:nngp_comparison} demonstrates variance stability under the same setup used for the gradient analysis in Figure~\ref{fig:gradient_stability_plot}. At the critical value $C_{W}=1$, both the orthogonal NNGP and NTK remain stable. While variance stability has been studied numerically for Gaussian initializations in~\cite{banta2024}, our results show that the corresponding criticality conditions transfer directly to the orthogonal setting.

Furthermore, at criticality, the empirical NNGP remains stable beyond the perturbative regime $\ell<n$. In Figure~\ref{fig:nngp_V_100_layers} (a), we compare its diagonal component with both the single-input exact solution (Section~\ref{app:solut-single-input}) and the large-$\ell$ expansion (Section~\ref{app:large_l_expansion}). The Monte Carlo estimate agrees quantitatively with the exact solution across all depths, and matches the large-$\ell$ expansion at large $\ell$. This confirms the stability of the orthogonal NNGP in the regime $\ell>n$, consistent with the observations of~\cite{day2023}.

\begin{figure}
    \begin{center}
        \includegraphics[width=0.99\textwidth]{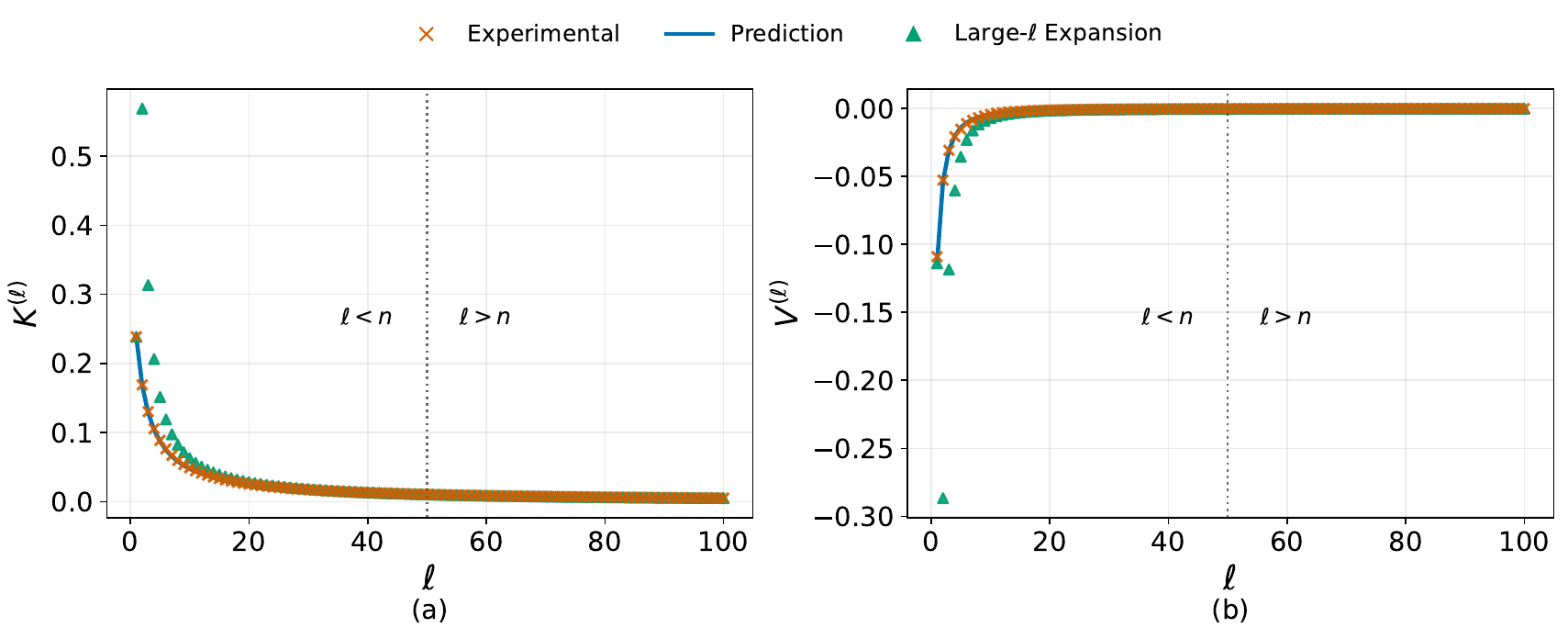}
    \end{center}
\caption{\emph{Stability beyond the perturbative regime.}
Comparison of the diagonal components of the Monte Carlo estimate, single-input exact solution, and large-$\ell$ expansion for the NNGP $K$ and quartic vertex $V$ in a tanh MLP with orthogonal initialization. Hidden layers have width \num{50}; means are computed over \num{600} initializations. (a) The NNGP estimates are in quantitative agreement with the exact solution at both small and large depths. The large-$\ell$ expansion is inaccurate at small $\ell$, as expected, but becomes accurate after a few layers. Stability persists up to $\ell=100$, well beyond the perturbative regime $\ell<n$. (b) The quartic vertex $V$ exhibits analogous scaling behavior.
\label{fig:nngp_V_100_layers}}
\end{figure}
\textbf{Monte--Carlo estimation of the four-point cumulant $V_{4}$.} We estimate the quartic vertex \(V_{4}\) to leading order as
\begin{align}
  \overline{V}^{(\ell)}_{\alpha \beta \gamma \delta} &= \left(\frac{1}{N_\text{net}} \sum_{I=1}^{N_\text{net}} \frac{n_{\ell -1}}{n_{\ell} (n_{\ell}-1)} \sum_{\substack{i, j = 1 \\i \ne j}}^{n_\ell} z^{(\ell)}_{I;i,\alpha} z^{(\ell)}_{I;i,\beta} z^{(\ell)}_{I;j,\gamma} z^{(\ell)}_{I;j,\delta}\right)  - n_{\ell-1}\overline{K}^{(\ell)}_{\alpha \beta} \overline{K}^{(\ell)}_{\gamma \delta}+ \mathcal{O}\left(\frac{1}{n}\right) \,. 
\end{align}
where \(\overline{K}\) is defined in~\eqref{eq:nngp_mc}. Here we exploit channel symmetry and estimate
\begin{equation}
\overline{K}^{(\ell)}_{\alpha \beta}
= \frac{1}{N_{\text{net}}} \sum_{I=1}^{N_{\text{net}}}
\frac{1}{n_\ell} \sum_{i=1}^{n_\ell}
z^{(\ell)}_{I;i,\alpha} z^{(\ell)}_{I;i,\beta}.
\label{eq:nngp_trace_avg}
\end{equation}
To estimate uncertainties, we repeat the Monte Carlo procedure \(N_{\text{stats}}\) times and report the mean and standard deviation. Results up to \(\ell=30\) are shown in Figures~\ref{fig:v4_comparison_critical} (critical) and~\ref{fig:v4_comparison_non_critical} (non-critical). In the critical case, fitted asymptotic power laws confirm the stability of \(V_{4}\).

Moreover, at criticality, the quartic vertex $V$ remains well-behaved beyond the perturbative regime $\ell<n$. In Figure~\ref{fig:nngp_V_100_layers} (b), we compare its diagonal component with the single-input exact solution (Section~\ref{app:solut-single-input}) and the large-$\ell$ expansion (Section~\ref{app:large_l_expansion}). The Monte Carlo estimates closely track the exact solution across all depths and align with the asymptotic expansion at large $\ell$. This provides further evidence for the stability of orthogonal networks in the regime $\ell>n$, consistent with~\cite{day2023}.

\textbf{Monte--Carlo estimation of the tensors $D, F, A$ and $B$.} Similarly to $V$, these tensors are estimated to leading order as
\begin{align}
    \overline{D}^{(\ell)}_{\alpha\beta\gamma\delta} &= \frac{1}{N_\text{net}}\sum_{I=1}^{N_\text{net}}\frac{n_{\ell-1}}{n_{\ell}^2}\sum_{i,j=1}^{n_{\ell}}z_{I;i,\alpha}^{(\ell)}z_{I;i,\beta}^{(\ell)}\widehat{\Delta \Theta}_{I;jj,\gamma\delta}^{(\ell)} + \mathcal{O}\left(\frac{1}{n}\right) \\
    \overline{F}^{(\ell)}_{\alpha\gamma\beta\delta} &= \frac{1}{N_\text{net}}\sum_{I=1}^{N_\text{net}}\frac{n_{\ell-1}}{n_{\ell}^2}\sum_{i,j=1}^{n_{\ell}}z_{I;i,\alpha}^{(\ell)}z_{I;j,\beta}^{(\ell)}\widehat{\Delta \Theta}_{I;ij,\gamma\delta}^{(\ell)} + \mathcal{O}\left(\frac{1}{n}\right) \\
    \overline{A}^{(\ell)}_{\alpha\beta\gamma\delta} &= \frac{1}{N_\text{net}}\sum_{I=1}^{N_\text{net}}\frac{n_{\ell-1}}{n_{\ell}^2}\sum_{i,j=1}^{n_{\ell}}\widehat{\Delta\Theta}^{(\ell)}_{I;ii,\alpha\beta}\widehat{\Delta \Theta}_{I;jj,\gamma\delta}^{(\ell)} + \mathcal{O}\left(\frac{1}{n}\right) \\
    \overline{B}^{(\ell)}_{\alpha\gamma\beta\delta} &= \frac{1}{N_\text{net}}\sum_{I=1}^{N_\text{net}}\frac{n_{\ell-1}}{n_{\ell}^2}\sum_{i,j=1}^{n_{\ell}}\widehat{\Delta\Theta}^{(\ell)}_{I;ij,\alpha\beta}\widehat{\Delta \Theta}_{I;ij,\gamma\delta}^{(\ell)} + \mathcal{O}\left(\frac{1}{n}\right) \, ,
\end{align}
where the NTK fluctuation $\widehat{\Delta \Theta}^{(\ell)}_{ij,\alpha \beta}$ was introduced in Section~\ref{subsec:ntk-finite-width}. The computation is repeated $N_\text{stats}$ times to estimate the mean and standard deviation. For each tensor, both critical and non-critical cases are evaluated using the same input configurations as for $V$, up to $\ell=30$. Results for $D$, $F$, $A$, and $B$ are shown in Figures~\ref{fig:D_comparison_critical}--\ref{fig:B_comparison_non_critical}. In the critical case, regressions of the asymptotic power laws (shown in orange) confirm the expected power-law behavior. This provides consistent evidence that the NTK tensors $D$, $F$, $A$, and $B$ are stabilized under the same criticality conditions as in the infinite-width limit.

Similarly, at criticality, the NTK tensors remain well-behaved beyond the perturbative regime $\ell<n$. In Figure~\ref{fig:D_F_100_layers}, we compare the diagonal components of the tensors D and F with their single-input exact solutions (Section~\ref{app:solut-single-input}) and large-$\ell$ expansions (Section~\ref{app:large_l_expansion}). The Monte Carlo estimates closely track the exact solutions across all depths and align with the asymptotic expansions at large $\ell$. This provides further evidence for the stability of orthogonal networks in the regime $\ell>n$, consistent with~\cite{day2023}.
\begin{figure}
    \begin{center}
        \includegraphics[width=0.99\textwidth]{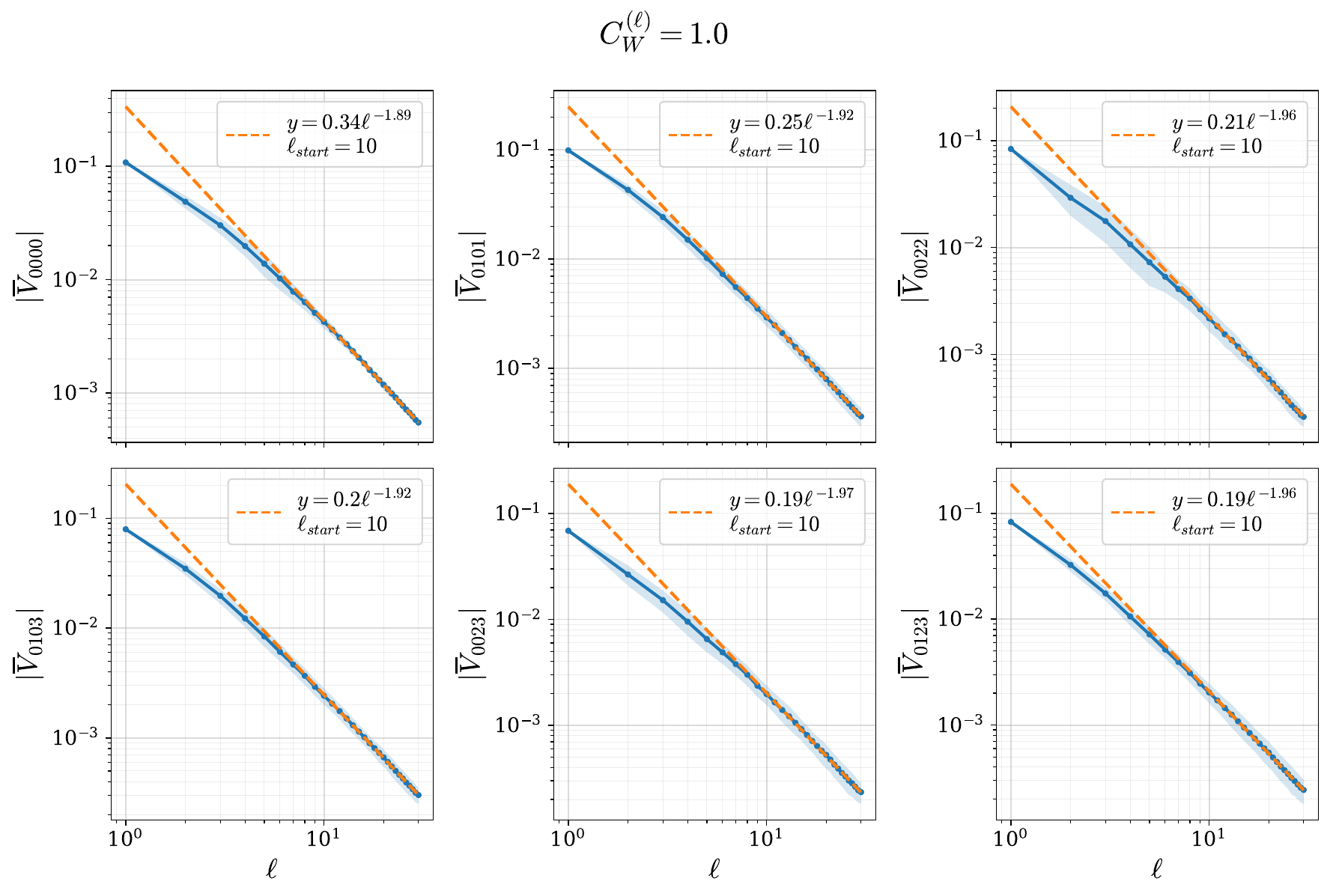}
    \end{center}
\caption{\emph{Stability of the four-point cumulant \(V_4\) at criticality.}
Selected components of the Monte Carlo estimate \(\overline{V}_4\) for a \(\tanh\) MLP are shown as a function of layer depth \(\ell\) at the critical value \(C_W^{(\ell)} = 1\). Hidden layers have width \num{50}. An asymptotic power-law fit is shown in orange, with the fit starting at \(\ell_{\text{start}}\). Estimates are obtained from \(N_\text{net}=\num{600}\) initializations, with means and error bars computed over \(N_\text{stats}=\num{10}\) repetitions (see text).}
\label{fig:v4_comparison_critical}
\end{figure}

\begin{figure}
  \begin{center}
    \includegraphics[width=0.95\textwidth]{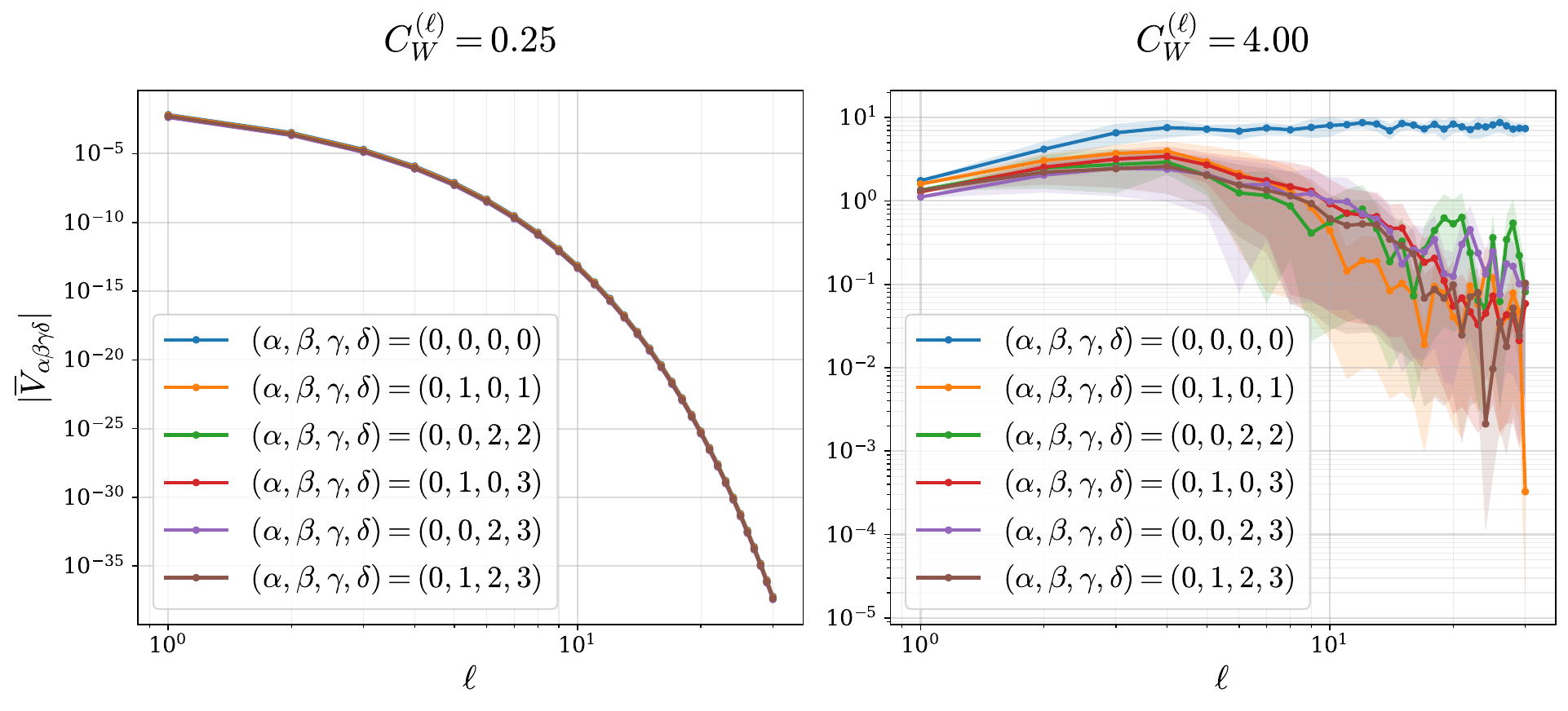}
  \end{center}
\caption{\emph{Instability of the four-point cumulant \(V_4\) away from criticality.}Selected components of the Monte Carlo estimate \(\overline{V}_4\) are shown for \(C_W^{(\ell)} < 1\) (left) and \(C_W^{(\ell)} > 1\) (right). Estimates are computed from \(N_\text{net}=\num{600}\) initializations, with means and error bars obtained from \(N_\text{stats}=\num{10}\) repetitions (see text).}
\label{fig:v4_comparison_non_critical}
\end{figure}

\begin{figure}
    \begin{center}
        \includegraphics[width=0.99\textwidth]{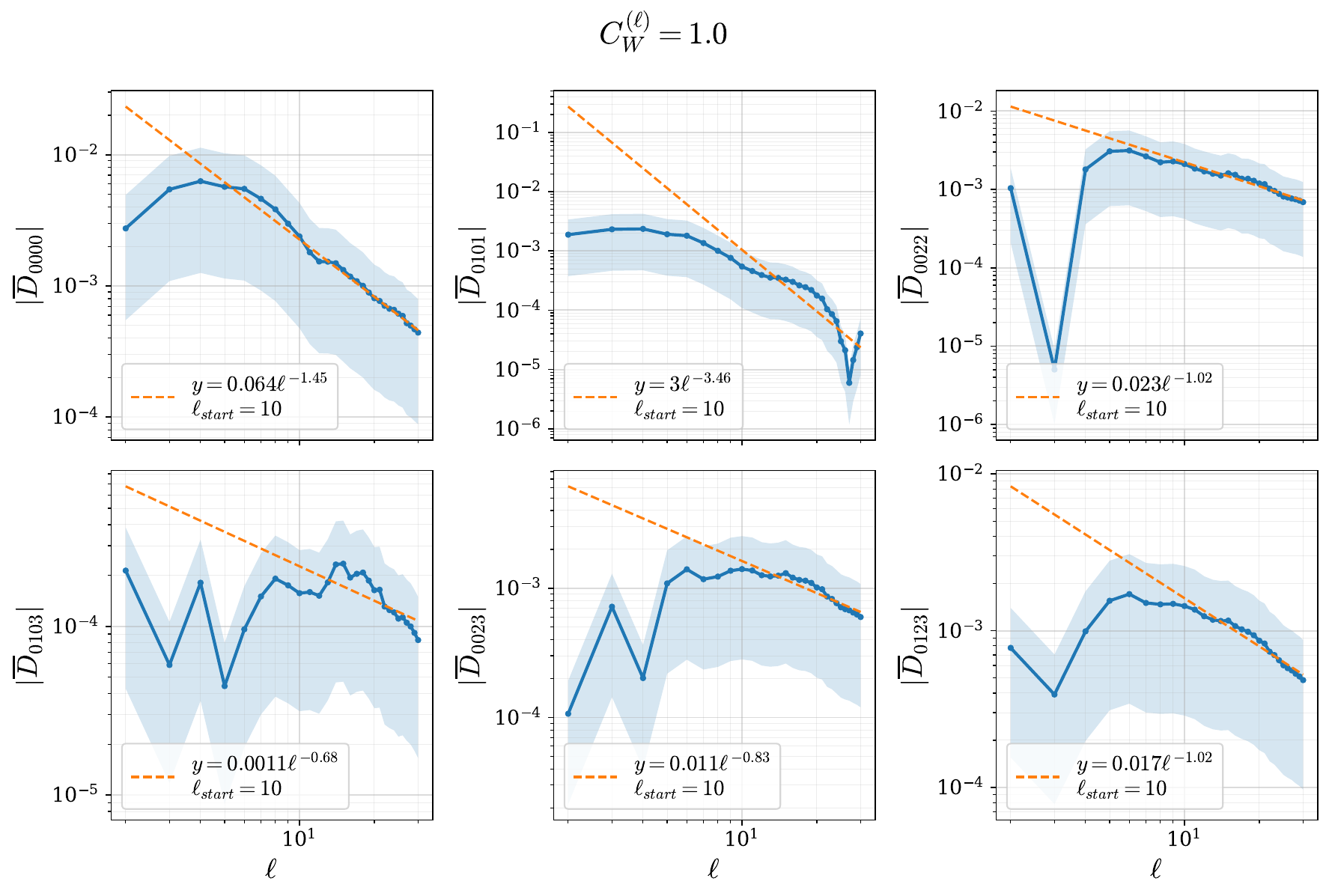}
    \end{center}
    \caption{\emph{Stability of the $D$ tensor at criticality.} Selected components of the Monte Carlo estimate \(\overline{D}\) for a \(\tanh\) MLP are shown as a function of layer depth \(\ell\) at the critical value \(C_W^{(\ell)} = 1\). Hidden layers have width \num{50}. An asymptotic power-law fit is shown in orange, with the fit starting at \(\ell_{\text{start}}\). Estimates are obtained from \(N_\text{net}=\num{600}\) initializations, with means and error bars computed over \(N_\text{stats}=\num{10}\) repetitions (see text).}\label{fig:D_comparison_critical}
\end{figure}

\begin{figure}
  \begin{center}
    \includegraphics[width=0.95\textwidth]{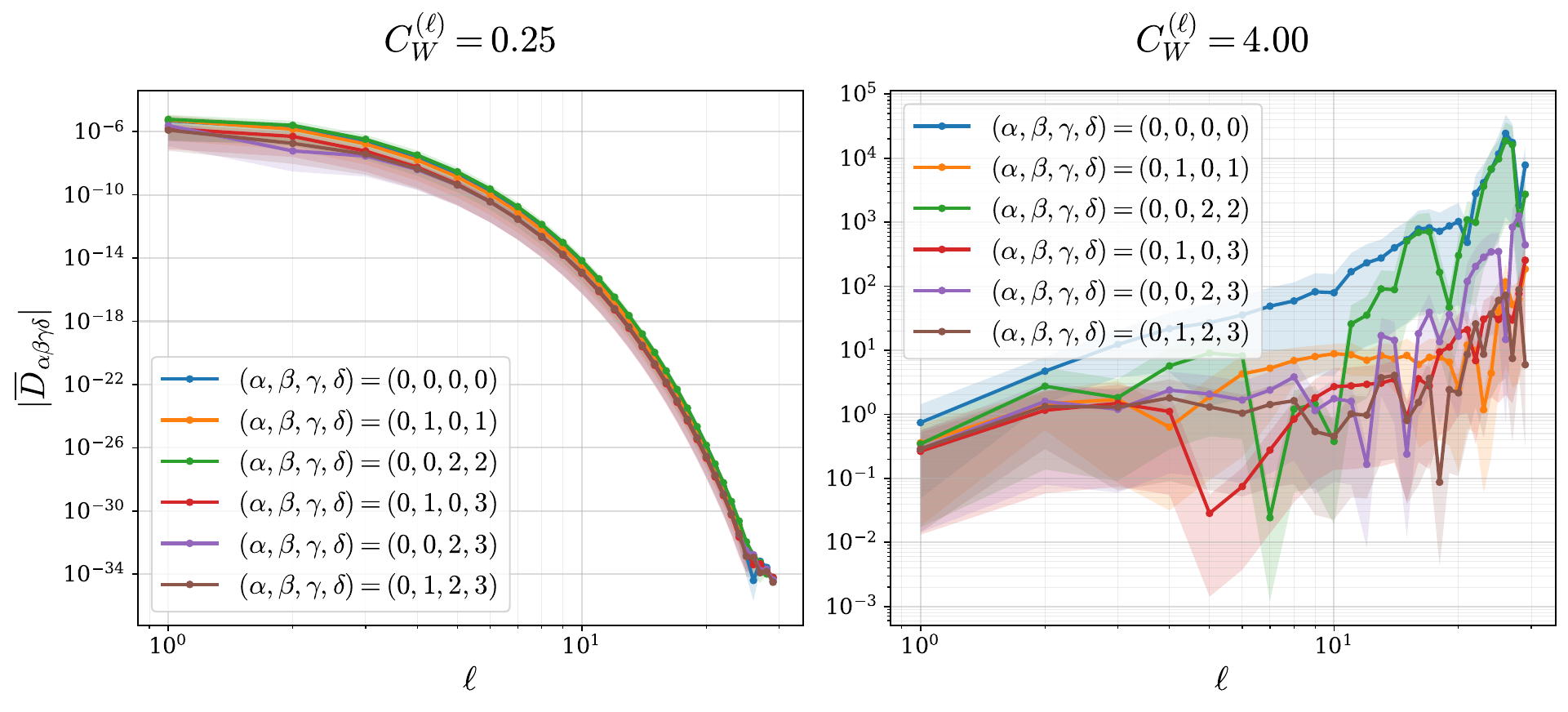}
  \end{center}
  \caption{\emph{Instability of the $D$ tensor away from criticality.} Selected components of the Monte Carlo estimate \(\overline{D}\) are shown for \(C_W^{(\ell)} < 1\) (left) and \(C_W^{(\ell)} > 1\) (right). Estimates are computed from \(N_\text{net}=\num{600}\) initializations, with means and error bars obtained from \(N_\text{stats}=\num{10}\) repetitions (see text).}\label{fig:D_comparison_non_critical}
\end{figure}

\begin{figure}
    \begin{center}
        \includegraphics[width=0.99\textwidth]{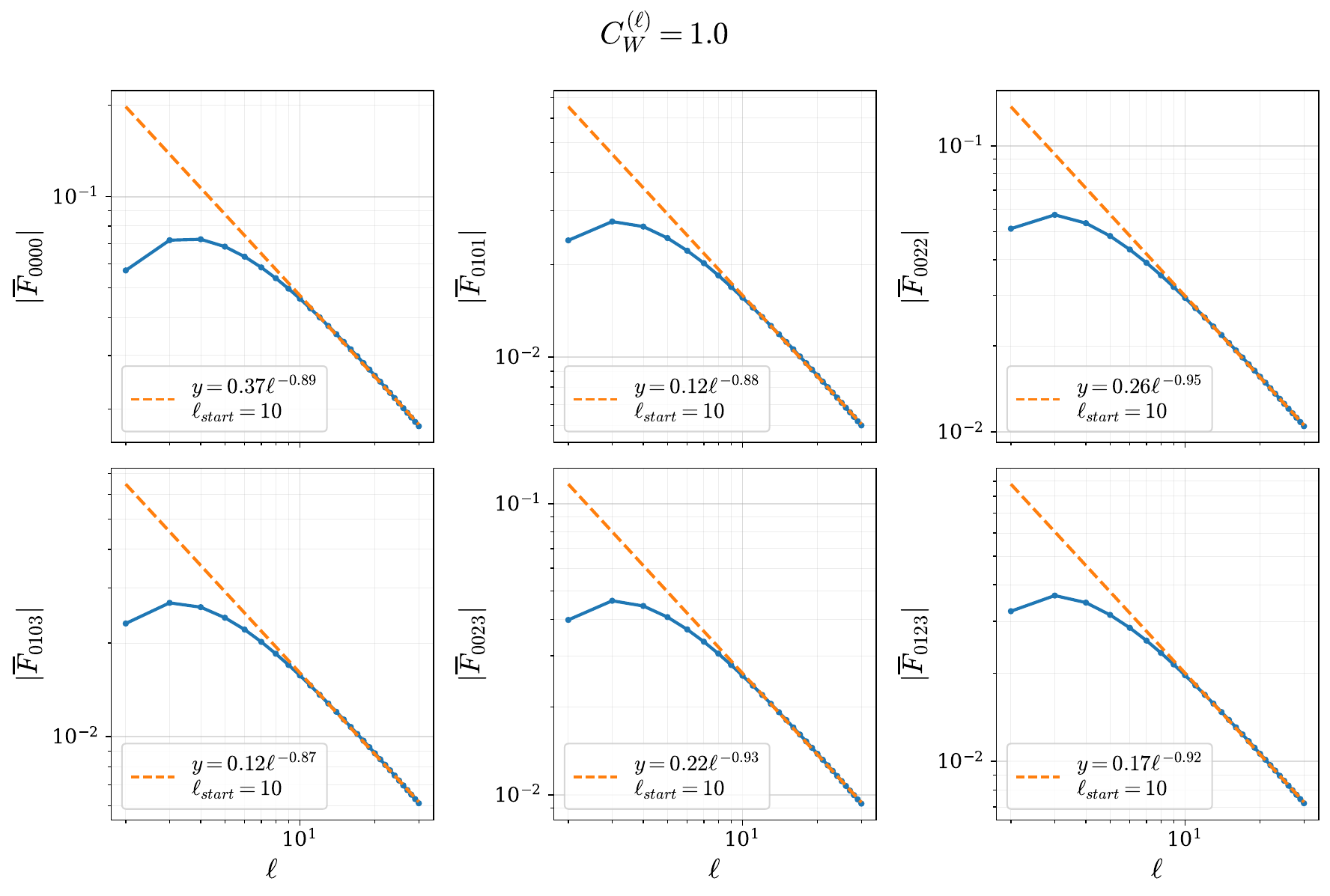}
    \end{center}
    \caption{\emph{Stability of the $F$ tensor at criticality.} Selected components of the Monte Carlo estimate \(\overline{F}\) for a \(\tanh\) MLP are shown as a function of layer depth \(\ell\) at the critical value \(C_W^{(\ell)} = 1\). Hidden layers have width \num{50}. An asymptotic power-law fit is shown in orange, with the fit starting at \(\ell_{\text{start}}\). Estimates are obtained from \(N_\text{net}=\num{600}\) initializations, with means and error bars computed over \(N_\text{stats}=\num{10}\) repetitions (see text).}\label{fig:F_comparison_critical}
\end{figure}

\begin{figure}
  \begin{center}
    \includegraphics[width=0.95\textwidth]{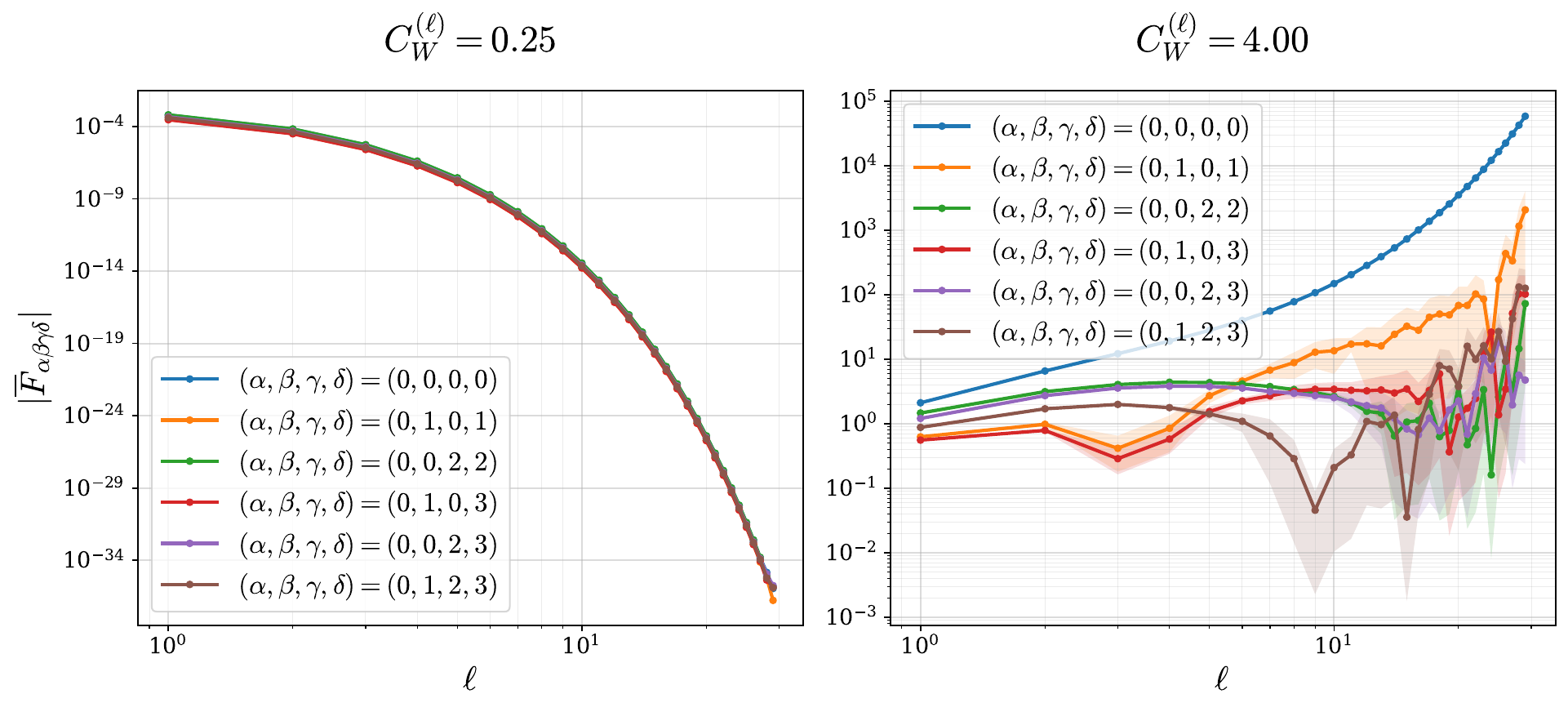}
  \end{center}
  \caption{\emph{Instability of the $F$ tensor away from criticality.} Selected components of the Monte Carlo estimate \(\overline{F}\) are shown for \(C_W^{(\ell)} < 1\) (left) and \(C_W^{(\ell)} > 1\) (right). Estimates are computed from \(N_\text{net}=\num{600}\) initializations, with means and error bars obtained from \(N_\text{stats}=\num{10}\) repetitions (see text).}\label{fig:F_comparison_non_critical}
\end{figure}

\begin{figure}
    \begin{center}
        \includegraphics[width=0.99\textwidth]{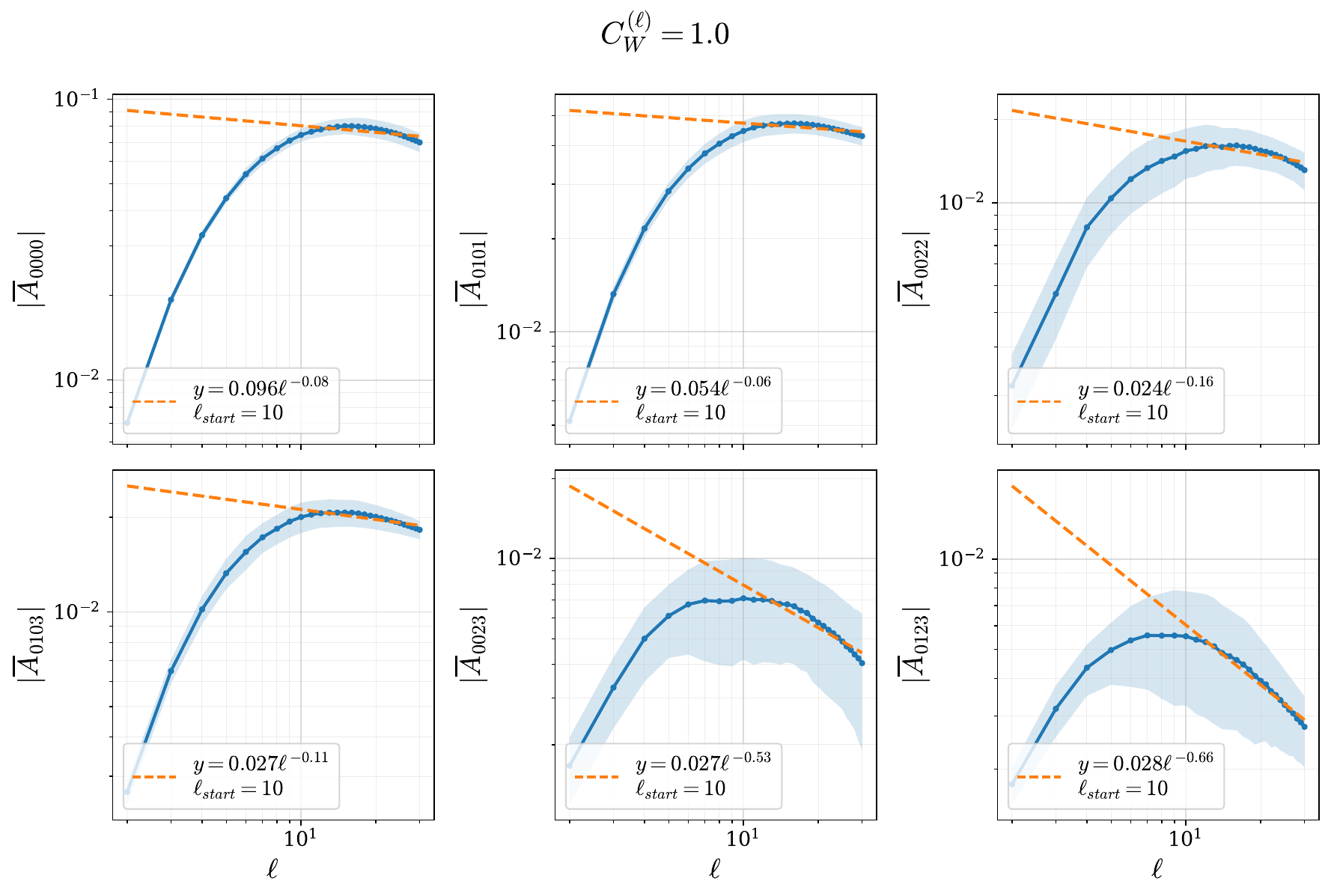}
    \end{center}
    \caption{\emph{Stability of the $A$ tensor at criticality.} Selected components of the Monte Carlo estimate \(\overline{A}\) for a \(\tanh\) MLP are shown as a function of layer depth \(\ell\) at the critical value \(C_W^{(\ell)} = 1\). Hidden layers have width \num{50}. An asymptotic power-law fit is shown in orange, with the fit starting at \(\ell_{\text{start}}\). Estimates are obtained from \(N_\text{net}=\num{600}\) initializations, with means and error bars computed over \(N_\text{stats}=\num{10}\) repetitions (see text).}\label{fig:A_comparison_critical}
\end{figure}

\begin{figure}
  \begin{center}
    \includegraphics[width=0.95\textwidth]{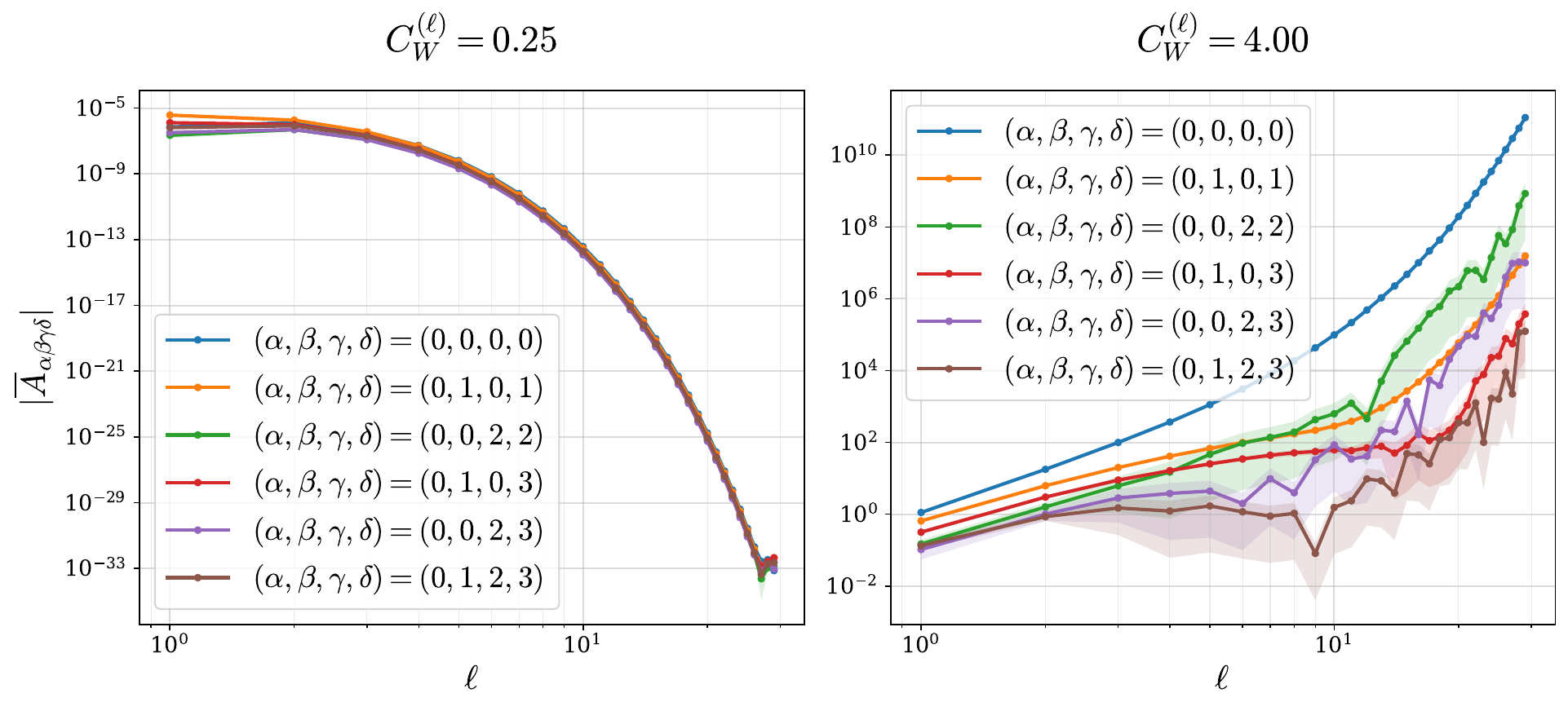}
  \end{center}
  \caption{\emph{Instability of the $A$ tensor away from criticality.} Selected components of the Monte Carlo estimate \(\overline{A}\) are shown for \(C_W^{(\ell)} < 1\) (left) and \(C_W^{(\ell)} > 1\) (right). Estimates are computed from \(N_\text{net}=\num{600}\) initializations, with means and error bars obtained from \(N_\text{stats}=\num{10}\) repetitions (see text).}\label{fig:A_comparison_non_critical}
\end{figure}

\begin{figure}
    \begin{center}
        \includegraphics[width=0.99\textwidth]{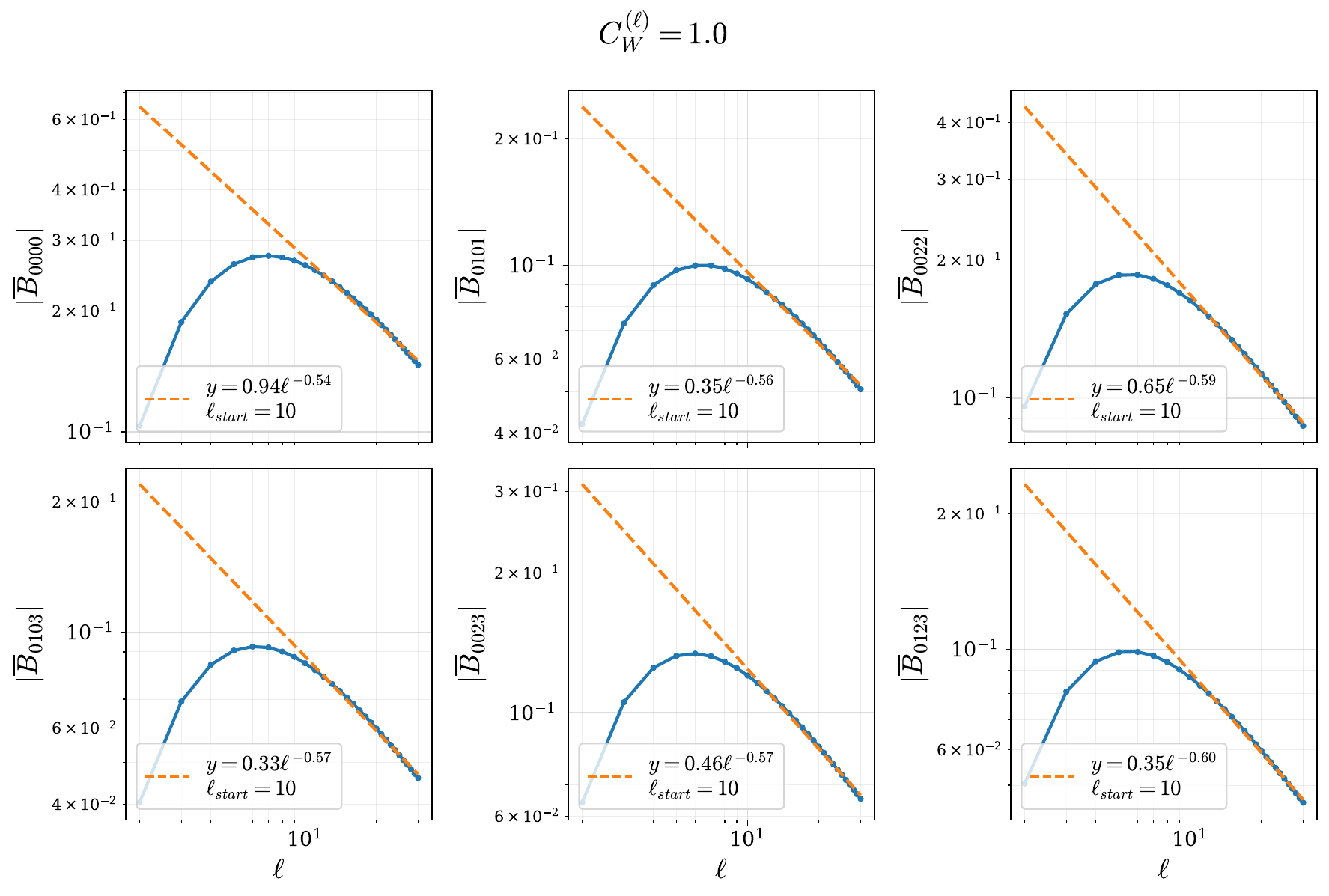}
    \end{center}
    \caption{\emph{Stability of the $B$ tensor at criticality.} Selected components of the Monte Carlo estimate \(\overline{B}\) for a \(\tanh\) MLP are shown as a function of layer depth \(\ell\) at the critical value \(C_W^{(\ell)} = 1\). Hidden layers have width \num{50}. An asymptotic power-law fit is shown in orange, with the fit starting at \(\ell_{\text{start}}\). Estimates are obtained from \(N_\text{net}=\num{600}\) initializations, with means and error bars computed over \(N_\text{stats}=\num{10}\) repetitions (see text).}\label{fig:B_comparison_critical}
\end{figure}

\begin{figure}
  \begin{center}
    \includegraphics[width=0.95\textwidth]{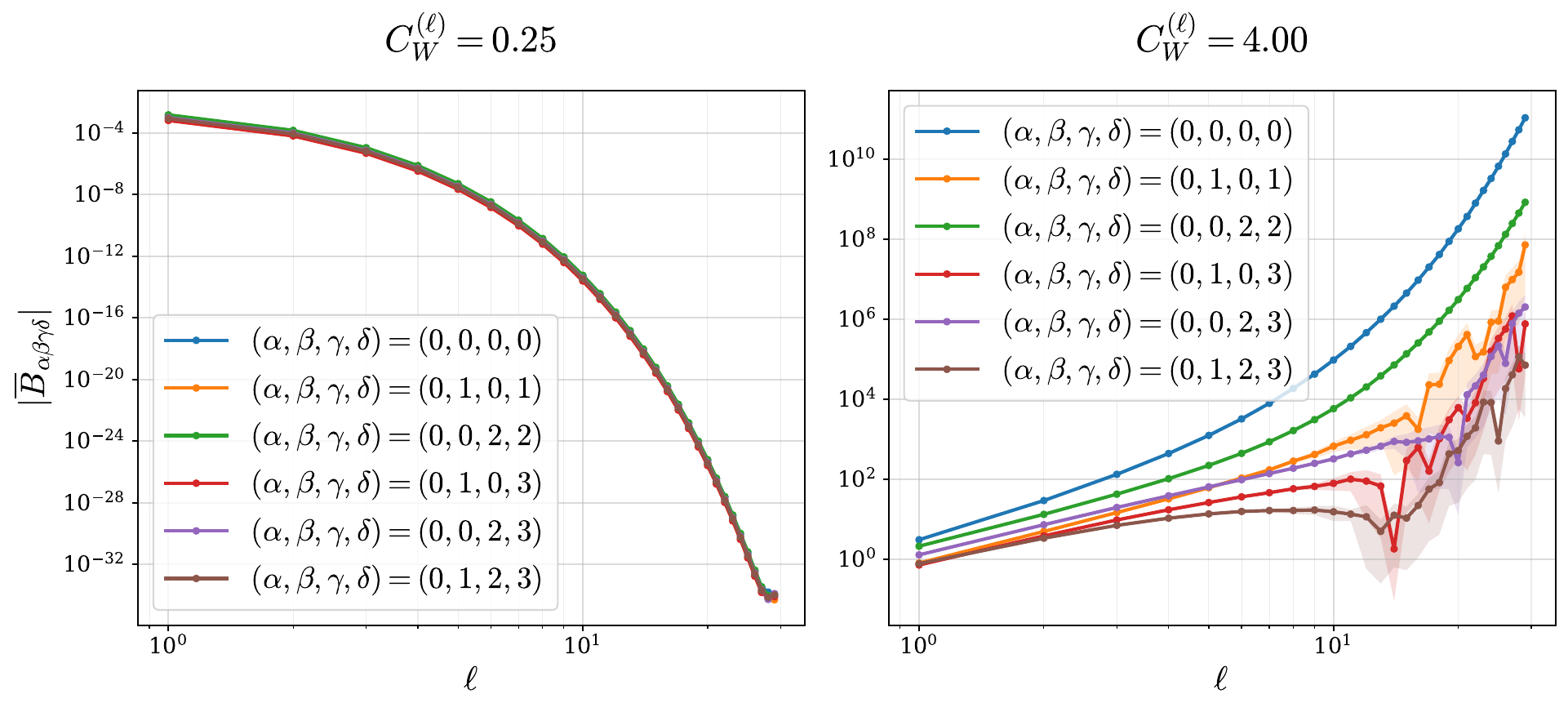}
  \end{center}
  \caption{\emph{Instability of the $B$ tensor away from criticality.} Selected components of the Monte Carlo estimate \(\overline{B}\) are shown for \(C_W^{(\ell)} < 1\) (left) and \(C_W^{(\ell)} > 1\) (right). Estimates are computed from \(N_\text{net}=\num{600}\) initializations, with means and error bars obtained from \(N_\text{stats}=\num{10}\) repetitions (see text).}\label{fig:B_comparison_non_critical}
\end{figure}

\begin{figure}
    \begin{center}
        \includegraphics[width=0.99\textwidth]{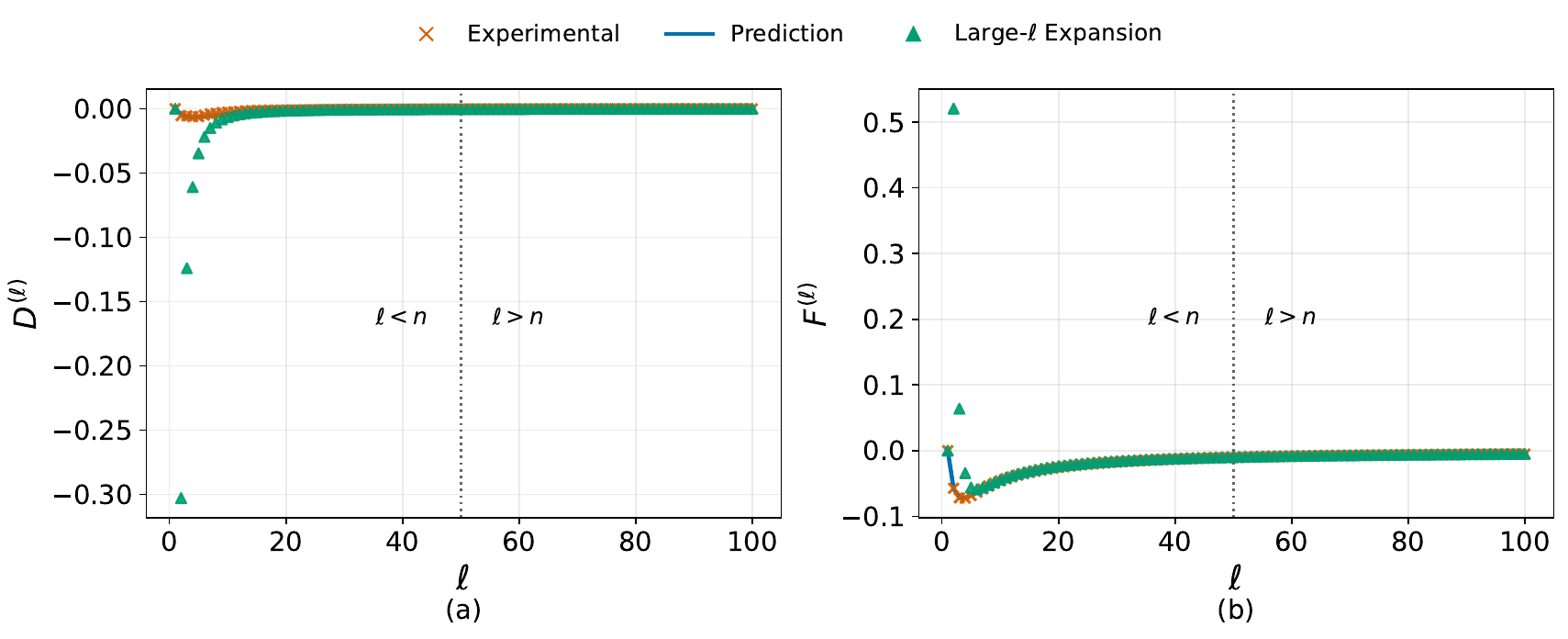}
    \end{center}
\caption{\emph{Stability beyond the perturbative regime.}
Comparison of the diagonal components of the Monte Carlo estimate, single-input exact solution, and large-$\ell$ expansion for the NTK tensors $D$ and $F$ in a tanh MLP with orthogonal initialization. Hidden layers have width \num{50}; means are computed over \num{600} initializations. (a) The tensor $D$ estimates are in quantitative agreement with the exact solution at both small and large depths. The large-$\ell$ expansion is inaccurate at small $\ell$, as expected, but becomes accurate after a few layers. Stability persists up to $\ell=100$, well beyond the perturbative regime $\ell<n$. (b) The tensor $F$ exhibits analogous scaling behavior.
\label{fig:D_F_100_layers}}
\end{figure}

\subsection{Compute resources}
\label{sec:compute-resources}

All experiments were conducted on a laptop equipped with an AMD Radeon Pro 5500M (4\,GB VRAM). Computations were performed on a single GPU. The stability analysis for the NNGP and $V_4$ was completed within several hours for all three values of $C_W$. Gradient-dependent tensors are more computationally demanding; accordingly, the layer width was reduced to \num{50}, and each tensor (for three values of $C_W$) required up to $\mathcal{O}(10)$ hours to compute.













 }



\end{document}